\pdfoutput=1
\documentclass[10pt]{article}

\usepackage{etoolbox}
\newtoggle{coltformat}
\togglefalse{coltformat} 
\newcommand{\colt}[1]{\iftoggle{coltformat}{#1}{}}
\newcommand{\arxiv}[1]{\iftoggle{coltformat}{}{#1}}

\arxiv{

\usepackage{tgpagella}
\usepackage{parskip}

\usepackage[letterpaper, left=1in, right=1in, top=1in, bottom=1in]{geometry}

\usepackage[dvipsnames]{xcolor}
\usepackage[colorlinks=true, linkcolor=blue!70!black, citecolor=blue!70!black]{hyperref}
\usepackage{microtype}

\usepackage[numbers]{natbib}

\usepackage{amsthm}
\usepackage{mathtools}
\usepackage{amsmath}
\usepackage{bm}
\usepackage{bbm}
\usepackage{amsfonts}
\usepackage{amssymb}

\usepackage{xpatch}
\usepackage{pifont}
\usepackage{array}
\usepackage{booktabs}
\usepackage{floatrow}
\newfloatcommand{capbtabbox}{table}[][\FBwidth]
\usepackage{blindtext}
\usepackage{caption}
\usepackage{subcaption}

\usepackage{algorithm}
\usepackage{algorithmicx}
\usepackage[noend]{algpseudocode}

}

\arxiv{

\usepackage{bbm}

\newcommand\numberthis{\addtocounter{equation}{1}\tag{\theequation}} 
\allowdisplaybreaks

\usepackage{mathtools}

\DeclarePairedDelimiter{\abs}{\lvert}{\rvert} 
\DeclarePairedDelimiter{\brk}{[}{]}
\DeclarePairedDelimiter{\crl}{\{}{\}}
\DeclarePairedDelimiter{\prn}{(}{)}
\DeclarePairedDelimiter{\nrm}{\|}{\|}
\DeclarePairedDelimiter{\tri}{\langle}{\rangle}

\let\Pr\undefined
 
\DeclareMathOperator{\En}{\mathbb{E}}

\DeclareMathOperator{\Pr}{\bbP}

\DeclareMathOperator*{\argmin}{argmin} 
\DeclareMathOperator*{\argmax}{argmax}

\newcommand{\ind}[1]{\mathbbm{1}\crl*{#1}}    
\newcommand{\indd}[1]{\mathbbm{1}\crl{#1}}    
\newcommand{\eps}{\varepsilon}

\newcommand{\wt}[1]{\widetilde{#1}}
\newcommand{\wh}[1]{\widehat{#1}}

\def\ddefloop#1{\ifx\ddefloop#1\else\ddef{#1}\expandafter\ddefloop\fi}
\def\ddef#1{\expandafter\def\csname bb#1\endcsname{\ensuremath{\mathbb{#1}}}}
\ddefloop ABCDEFGHIJKLMNOPQRSTUVWXYZ\ddefloop
\def\ddefloop#1{\ifx\ddefloop#1\else\ddef{#1}\expandafter\ddefloop\fi}
\def\ddef#1{\expandafter\def\csname b#1\endcsname{\ensuremath{\mathbf{#1}}}}
\ddefloop ABCDEFGHIJKLMNOPQRSTUVWXYZ\ddefloop
\def\ddef#1{\expandafter\def\csname c#1\endcsname{\ensuremath{\mathcal{#1}}}}
\ddefloop ABCDEFGHIJKLMNOPQRSTUVWXYZ\ddefloop
\def\ddef#1{\expandafter\def\csname h#1\endcsname{\ensuremath{\widehat{#1}}}}
\ddefloop ABCDEFGHIJKLMNOPQRSTUVWXYZabcdefghijklmnopqrsuvwxyz\ddefloop    
\def\ddef#1{\expandafter\def\csname hc#1\endcsname{\ensuremath{\widehat{\mathcal{#1}}}}}
\ddefloop ABCDEFGHIJKLMNOPQRSTUVWXYZ\ddefloop
\def\ddef#1{\expandafter\def\csname t#1\endcsname{\ensuremath{\widetilde{#1}}}}
\ddefloop ABCDEFGHIJKLMNOPQRSTUVWXYZ\ddefloop
\def\ddef#1{\expandafter\def\csname tc#1\endcsname{\ensuremath{\widetilde{\mathcal{#1}}}}}
\ddefloop ABCDEFGHIJKLMNOPQRSTUVWXYZ\ddefloop

\newcommand{\sign}{\mathrm{sign}}

\usepackage{tikz}

\makeatletter
\newtheorem*{rep@theorem}{\rep@title}
\newcommand{\newreptheorem}[2]{%
\newenvironment{rep#1}[1]{%
 \def\rep@title{#2 \ref*{##1}}%
 \begin{rep@theorem}}%
 {\end{rep@theorem}}}
\makeatother

\newtheorem{claim}{Claim}
\newtheorem{assumption}{Assumption}

\newtheorem{corollary}{Corollary}
\newtheorem{proposition}{Proposition}

{\newtheorem{lemma}{Lemma}}

{\newtheorem{theorem}{Theorem}}
\newreptheorem{theorem}{Theorem}
\newtheorem{theorem*}{Theorem}
\newtheorem{definition}{Definition}

\usepackage{prettyref}
\newcommand{\pref}[1]{\prettyref{#1}}

\newcommand{\savehyperref}[2]{\texorpdfstring{\hyperref[#1]{#2}}{#2}}
\newrefformat{eq}{\savehyperref{#1}{\textup{(\ref*{#1})}}}
\newrefformat{eqn}{\savehyperref{#1}{Equation~\ref*{#1}}}
\newrefformat{fact}{\savehyperref{#1}{Fact~\ref*{#1}}}
\newrefformat{con}{\savehyperref{#1}{Conjecture~\ref*{#1}}}
\newrefformat{lem}{\savehyperref{#1}{Lemma~\ref*{#1}}}
\newrefformat{def}{\savehyperref{#1}{Definition~\ref*{#1}}}
\newrefformat{line}{\savehyperref{#1}{line~\ref*{#1}}}
\newrefformat{thm}{\savehyperref{#1}{Theorem~\ref*{#1}}}
\newrefformat{corr}{\savehyperref{#1}{Corollary~\ref*{#1}}}
\newrefformat{sec}{\savehyperref{#1}{Section~\ref*{#1}}}
\newrefformat{app}{\savehyperref{#1}{Appendix~\ref*{#1}}}
\newrefformat{ass}{\savehyperref{#1}{Assumption~\ref*{#1}}}
\newrefformat{ex}{\savehyperref{#1}{Example~\ref*{#1}}}
\newrefformat{fig}{\savehyperref{#1}{Figure~\ref*{#1}}}
\newrefformat{alg}{\savehyperref{#1}{Algorithm~\ref*{#1}}}
\newrefformat{rem}{\savehyperref{#1}{Remark~\ref*{#1}}}
\newrefformat{conj}{\savehyperref{#1}{Conjecture~\ref*{#1}}}
\newrefformat{prop}{\savehyperref{#1}{Proposition~\ref*{#1}}}
\newrefformat{proto}{\savehyperref{#1}{Protocol~\ref*{#1}}}
\newrefformat{prob}{\savehyperref{#1}{Problem~\ref*{#1}}}
\newrefformat{claim}{\savehyperref{#1}{Claim~\ref*{#1}}}
\newrefformat{tab}{\savehyperref{#1}{Table~\ref*{#1}}}

\newcommand{\algcomment}[1]{\textcolor{blue!70!white}{\footnotesize{\texttt{\textbf{//
          #1}}}}}

\usepackage[shortlabels]{enumitem} 
\usepackage{makecell}
\usepackage{comment}
\usepackage{pifont}
\usepackage{caption}
\usepackage[dvipsnames]{xcolor}
\newcommand{\cmark}{\ding{51}}%
\newcommand{\xmark}{\ding{55}}%
\usepackage{nicefrac} 

\usepackage{tikz} 
\usetikzlibrary{arrows.meta, positioning}

}

\newcommand{\ccov}{C_\mathsf{cov}}
\newcommand{\cconc}{C_\mathsf{conc}}

\newcommand{\spancap}{C_\mathsf{span}}
\newcommand{\poly}{\mathrm{poly}}

\newcommand{\unif}{\mathrm{Unif}}
\newcommand{\ber}{\mathrm{Ber}}

\newcommand{\cpush}{C_\mathsf{push}}
\newcommand{\cpushcov}{C_\mathsf{push\_cov}}

\newcommand{\estpi}{{\wh{\pi}}}

\newcommand{\optpi}{\pi^\star}

\newcommand{\optlatp}{P_\mathsf{lat}}
\newcommand{\optlatr}{R_\mathsf{lat}}
\newcommand{\estlatp}{\wh{P}_\mathsf{lat}}
\newcommand{\estlatr}{\wh{R}_\mathsf{lat}}
\newcommand{\samplelatp}{\wt{P}_\mathsf{lat}}

\colt{
    \newcommand{\detalg}{{\small \textsf{PLHR.D}}}
    \newcommand{\detdecoder}{{\small \textsf{Decoder.D}}}
    \newcommand{\detrefit}{{\small \textsf{Refit.D}}}
}
\arxiv{
    \newcommand{\detalg}{{\textsf{PLHR.D}}}
    \newcommand{\detdecoder}{{\textsf{Decoder.D}}}
    \newcommand{\detrefit}{{\textsf{Refit.D}}}
}

\newcommand{\estlatentmdp}{\wh{M}_\mathsf{lat}}

\newcommand{\Pitest}{\Pi^{\mathsf{test}}}

\newcommand{\vestarg}[2]{V_\mathsf{mc}(#1 \mid #2)}

\newcommand{\qestarg}[2]{Q_\mathsf{mc}(#1 \mid #2)}

\newcommand{\taudec}{\epsilon_\mathsf{dec}}
\newcommand{\tauref}{\epsilon_\mathsf{tol}}

\colt{
    \newcommand{\stochalg}{{\small \textsf{PLHR}}}
    \newcommand{\stochdecoder}{{\small \textsf{Decoder}}}
    \newcommand{\stochrefit}{{\small \textsf{Refit}}}
}
\arxiv{
    \newcommand{\stochalg}{{\textsf{PLHR}}}
    \newcommand{\stochdecoder}{{\textsf{Decoder}}}
    \newcommand{\stochrefit}{{\textsf{Refit}}}
    
}

\newcommand{\estmdp}{\wh{M}}
\newcommand{\estp}{\wh{P}}
\newcommand{\empobs}{x} 
\newcommand{\push}{_\sharp}

\newcommand{\estmdpobsspace}[1]{\statesp_{#1}[\estmdp]}
\newcommand{\pitest}[2]{\pi_{#1, #2}}

\newcommand{\nreset}{n_\mathsf{reset}}
\newcommand{\ndec}{n_\mathsf{dec}}

\newcommand{\nmc}{n_\mathsf{mc}}

\newcommand{\eventemulator}{\cE^\mathsf{init}}
\newcommand{\eventdec}{\cE^\mathsf{D}}
\newcommand{\eventref}{\cE^\mathsf{R}}

\newcommand{\epsPRS}[1]{\cS_{#1}^{\eps\textsf{-push}}}

\newcommand{\cXL}{\cX^\mathsf{L}}
\newcommand{\cXR}{\cX^\mathsf{R}}
\newcommand{\gobs}{\cG_\mathsf{obs}}
\newcommand{\cc}{\bbC}
\newcommand{\ccL}{\bbC^\mathsf{L}}
\newcommand{\ccR}{\bbC^\mathsf{R}}
\newcommand{\optdec}{\phi}
\newcommand{\projm}[1]{\mathsf{Proj}_{#1}}
\newcommand{\projR}{\projm{\cXR}}
\newcommand{\nrch}{n_\mathsf{reach}}
\newcommand{\nurch}{n_\mathsf{unreach}}

\newcommand{\resetmodel}{Hybrid Resets}

\colt{
    \newcommand{\psdp}{{\small \textsf{PSDP}}}
    \newcommand{\cpi}{{\small \textsf{CPI}}}
}
\arxiv{
    \newcommand{\psdp}{{\textsf{PSDP}}}
    \newcommand{\cpi}{{\textsf{CPI}}} 
}

\newcommand{\statesp}{\cX}
\newcommand{\actionsp}{\cA}
\newcommand{\latentsp}{\cS}
\newcommand{\emission}{\psi}

\newcommand{\supp}{\mathrm{supp}}

\newcommand{\mcest}{\mathsf{MC}}
\newcommand{\violations}{\mathsf{Violations}}

\newcommand{\alg}{\mathsf{Alg}}

\newcommand{\sgood}{\mathsf{g}}
\newcommand{\sbad}{\mathsf{b}}
\newcommand{\sdis}{\mathsf{d}}

\newcommand{\dtv}{D_{\mathsf{TV}}}

\newcommand{\rootlayer}{\mathsf{RootLayer}}
\newcommand{\eventfresh}{\cE_F}
\newcommand{\eventnew}{\cE_N}

\newcommand{\cfnor}{\underline{\cF_{t,H}}}

\colt{
    \newcommand{\epspc}{\eps_{\scriptscriptstyle \mathsf{PC}}}
}
\arxiv{
    \newcommand{\epspc}{\eps_{\mathsf{PC}}}
}

\newcommand{\epsstat}{\eps_\mathsf{stat}}
\newcommand{\Piopen}{\Pi_\mathsf{OL}}
\usepackage[suppress]{color-edits} %
\addauthor{gl}{red}
\addauthor{ak}{orange}
\addauthor{as}{purple} 
\usepackage{enumitem} 
\algrenewcommand\algorithmicrequire{\textbf{Input:}}

\arxiv{
    \title{The Role of Environment Access in\\Agnostic Reinforcement Learning\thanks{Authors are listed in alphabetical order of their last names.}}
    \author{%
      Akshay Krishnamurthy\\
      Microsoft Research
      \and
      Gene Li\\
      TTIC
      \and
      Ayush Sekhari\\
      MIT
    }
    \date{}
}

\begin{document}

\maketitle

\begin{abstract}
    \noindent We study Reinforcement Learning (RL) in environments with large state spaces, where function approximation is required for sample-efficient learning. Departing from a long history of prior work, we consider the weakest possible form of function approximation, called agnostic policy learning, where the learner seeks to find the best policy in a given class $\Pi$, with no guarantee that $\Pi$ contains an optimal policy for the underlying task.  
    Although it is known that sample-efficient agnostic policy learning is not possible in the standard online RL setting without further assumptions, we investigate the extent to which this can be overcome with stronger forms of access to the environment. Specifically, we show that:
    \begin{enumerate}%
    \item Agnostic policy learning remains statistically intractable when given access to a local simulator, from which one can reset to any previously seen state. This result holds even when the policy class is realizable, and stands in contrast to a positive result of~\citep{mhammedi2024power} showing that value-based learning under realizability is tractable with local simulator access. 
    \item Agnostic policy learning remains statistically intractable when given online access to a reset distribution with good coverage properties over the state space (the so-called $\mu$-reset setting).  We also study stronger forms of function approximation for policy learning, showing that PSDP~\citep{bagnell2003policy} and CPI~\citep{kakade2002approximately} provably fail in the absence of policy completeness. 
    \item On a positive note, agnostic policy learning is statistically tractable for Block MDPs with access to both of the above reset models. We establish this via a new algorithm that 
    carefully constructs a \emph{policy emulator}: a tabular MDP with a small state space that approximates the value functions of all policies $\pi \in \Pi$. These values are approximated \emph{without} any explicit value function class. 
    \end{enumerate}
 Taken together, our results contribute to a deeper understanding of the interplay between function approximation and environment access in RL. 
\end{abstract}

\arxiv{
    \clearpage
    \renewcommand{\contentsname}{Contents}
    \tableofcontents 
    \addtocontents{toc}{\protect\setcounter{tocdepth}{2}} 
    \clearpage 
}

\section{Introduction}

Reinforcement Learning (RL) is a widely studied framework for sequential decision-making, in which an agent interacts with an environment, and seeks to learn how to maximize a notion of long-term or cumulative reward \citep{sutton2018reinforcement}. However, due to the interactive and sequential nature of the problem, RL presents two significant challenges to learning agents: \emph{exploration}---the agent must deliberately explore the environment to gather information---and \emph{error amplification}---the agent must account for potential future errors when making decisions in the present. All RL algorithms must address these two challenges in some manner, and, in theory, almost all prior works do so by imposing stringent representational conditions on the function classes used by the learning  algorithm.  Accordingly, it is an important open question to understand the extent to which these representational conditions are necessary for sample-efficient learning. 

Consider, for example, the class of algorithms based on \emph{value function approximation} \cite{wang2020provably, jin2021bellman, xie2022role, foster2024model}. These methods typically address exploration via uncertainty quantification, exploration bonuses, and the optimism principle, and they address error amplification by optimizing surrogate objectives based on Bellman errors rather than directly optimizing policy performance. Unfortunately, obtaining guarantees for such reinforcement learning algorithms typically requires the function class to satisfy a representational condition called \emph{Bellman completeness}, which is much more stringent than what is required for supervised learning. Beyond this, \emph{all} methods based on value function approximation require a minimal assumption of value-function realizability---that the function class contains the optimal value function---which is already stronger than assumption-free/agnostic guarantees one can obtain in supervised learning \citep{dong2020expressivity}.

In this paper, we contribute to a growing body of work on understanding the role of representational conditions in RL \cite{chen2019information, xie2021batch, amortila2024harnessing, foster2021offline, jia2024offline, mhammedi2024power}. We focus on the setting of \emph{agnostic policy learning}, the most basic/fundamental setting in RL in which the learner is given a policy class $\Pi$ and is asked to find a policy $\wh{\pi}$ which performs nearly as well as the best policy in the class $\Pi$ \cite{kakade2003sample}. Policy learning methods are often viewed as more flexible than value- or model-based counterparts because they only model the main object of interest; however, these methods can be provably sample-inefficient because there are no algorithmic mechanisms to address exploration and error amplification \cite{agarwal2020pc, agarwal2021theory}. Accordingly, prior works on policy learning have imposed additional representational conditions to enable sample efficiency \cite{bagnell2003policy, kakade2002approximately, sekhari2021agnostic, jia2023agnostic}.

Rather than imposing representational conditions, we instead investigate whether stronger forms of environment access (beyond standard online RL), can circumvent the above algorithmic limitations and enable sample-efficient agnostic policy learning. 
This line of inquiry is motivated by practical applications where stronger forms of access to the environment are available---such as robotic control tasks with a simulator or game playing---as well as recent theoretical developments showing that value-based methods can benefit from such access~\citep{mhammedi2024power}. We consider several forms of environment access: \emph{generative model} (the learner can query the reward and next state on any state-action tuple), \emph{local simulator} (such queries can only be made on a previously observed state), \emph{\(\mu\)-resets} (the learner can rollout from a given exploratory distribution), and \emph{hybrid resets} (combining both local simulator access and $\mu$-resets); see \pref{sec:interaction_models} for details on these interaction models. We shed light on whether they can be leveraged to address the challenges of exploration and error amplification. Our key contributions, summarized in \pref{tab:results}, are:

    \vspace{-0.5pt}
    \begin{enumerate}
    \item{} Regarding the exploration challenge, we show that even with a strong function approximation assumption called \emph{policy completeness}, and \emph{generative access}---perhaps the strongest possible access to the MDP---policy learning methods cannot achieve sample complexity guarantees that scale with the intrinsic complexity of exploration, as measured via the \emph{coverability coefficient}~\cite{xie2022role} of the MDP---see~\pref{thm:lower-bound-coverability}. This resolves an open problem posed by \cite{jia2023agnostic} and shows, in a strong, information-theoretic sense, that policy learning methods cannot explore. %
    \item We next consider the \emph{error amplification} challenge. We study the $\mu$-reset setting, where the learner can rollout from an exploratory reset distribution $\mu$, and investigate whether error amplification can be controlled without policy completeness. Here, we show that agnostic policy learning is information-theoretically impossible---see~\pref{thm:lower-bound-policy-completeness}. We also show algorithm-specific lower bounds for \psdp{}~\citep{bagnell2003policy} and \cpi{}~\citep{kakade2002approximately}---algorithms that address error amplification under $\mu$-resets and policy completeness---when only realizability of the policy class is satisfied.
        \item In light of these lower bounds, we introduce a new model of access called \emph{hybrid resets}, which subsumes both local simulators (which is weaker than generative access) and $\mu$-resets. We show that under hybrid resets, and when the reset distribution satisfies \emph{pushforward concentrability} \cite{xie2021batch}, sample-efficient policy learning is possible in Block MDPs \cite{jiang2017contextual, du2019provably} via a new algorithm \stochalg{} (Policy Learning for Hybrid Resets)---see~\pref{thm:block-mdp-result}.  Since all of our lower bound constructions are Block MDPs, this indicates the significant power of hybrid reset access in agnostic policy learning. 

        On a technical level, we introduce a new algorithmic tool called \emph{policy emulator} that allows us to efficiently evaluate various policies within a  large class $\Pi$ (\pref{def:policy-emulator}). Informally speaking, a policy emulator is the ``minimal object'' useful for solving policy learning. Instead of learning the Block MDP in a traditional model-based sense (which would require samples scaling with the observation space size), \stochalg{} instead leverages hybrid resets to construct a policy emulator in a statistically efficient manner.
        
    \end{enumerate}
    Taken together, our results reveal intriguing interplays between function approximation and environment access in RL. Specifically, RL can remain tractable with extremely weak assumptions on the function approximation class, provided one has stronger environment access. We believe further investigation in this direction has potential to yield new algorithmic insights for complex RL settings.

\arxiv{
\paragraph{Paper Outline.} \pref{sec:preliminaries} introduces the problem setting and provides background on interaction models, coverage conditions, and \psdp{}. \pref{sec:technical-overview} gives a technical overview of our main results.  
\pref{sec:lower-bounds} gives intuition for the lower bounds. \pref{sec:warmup} presents a simplified algorithm for an easier setting, and \pref{sec:main-upper-bound} presents our main upper bound.  We close with discussion and open problems in \pref{sec:discussion}.
}

\begin{figure}
\begin{minipage}[b]{.64\linewidth}
    \centering
    \resizebox{1\textwidth}{!}{%
    \begin{tabular}{ p{0.33\linewidth} |c c c} 
             & Gen/Local Sim. & $\mu$-Resets & \resetmodel \\ 
          \hline
        Policy Completeness (\pref{def:policy-completeness}) & \makecell{\color{red}{\color{red}{\xmark}} \\ Thm.~\ref{thm:lower-bound-coverability}} & \makecell{\color{Green}{\cmark} \\ \psdp{} } & \makecell{\color{Green}{\cmark} \\ \psdp{} } \\[1em]

        Policy Realizability ($\optpi \in \Pi$) & \makecell{\color{red}{\xmark} \\ Thm.~\ref{thm:lower-bound-coverability}} & \textcolor{Dandelion}{\large \textbf{?}}$^\star$ & \makecell{\color{Green}{\cmark} \\ Thm.~\ref{thm:block-mdp-result} \\ (for BMDP)} \\[1em]

         Agnostic ($\pi^\star \notin \Pi)$ & \makecell{\color{red}{\xmark} \\ Thm.~\ref{thm:lower-bound-coverability}} & \makecell{\color{red}{\xmark} \\ Thm.~\ref{thm:lower-bound-policy-completeness}} & \makecell{\color{Green}{\cmark} \\ Thm.~\ref{thm:block-mdp-result} \\ (for BMDP) }
    \end{tabular}}
\end{minipage}
\begin{minipage}[b]{.35\linewidth}
\centering
\resizebox{1.05\textwidth}{!}{%
\raisebox{-4em}{
\includegraphics[scale=0.15, trim={0cm 18cm 40cm 0cm}, clip]{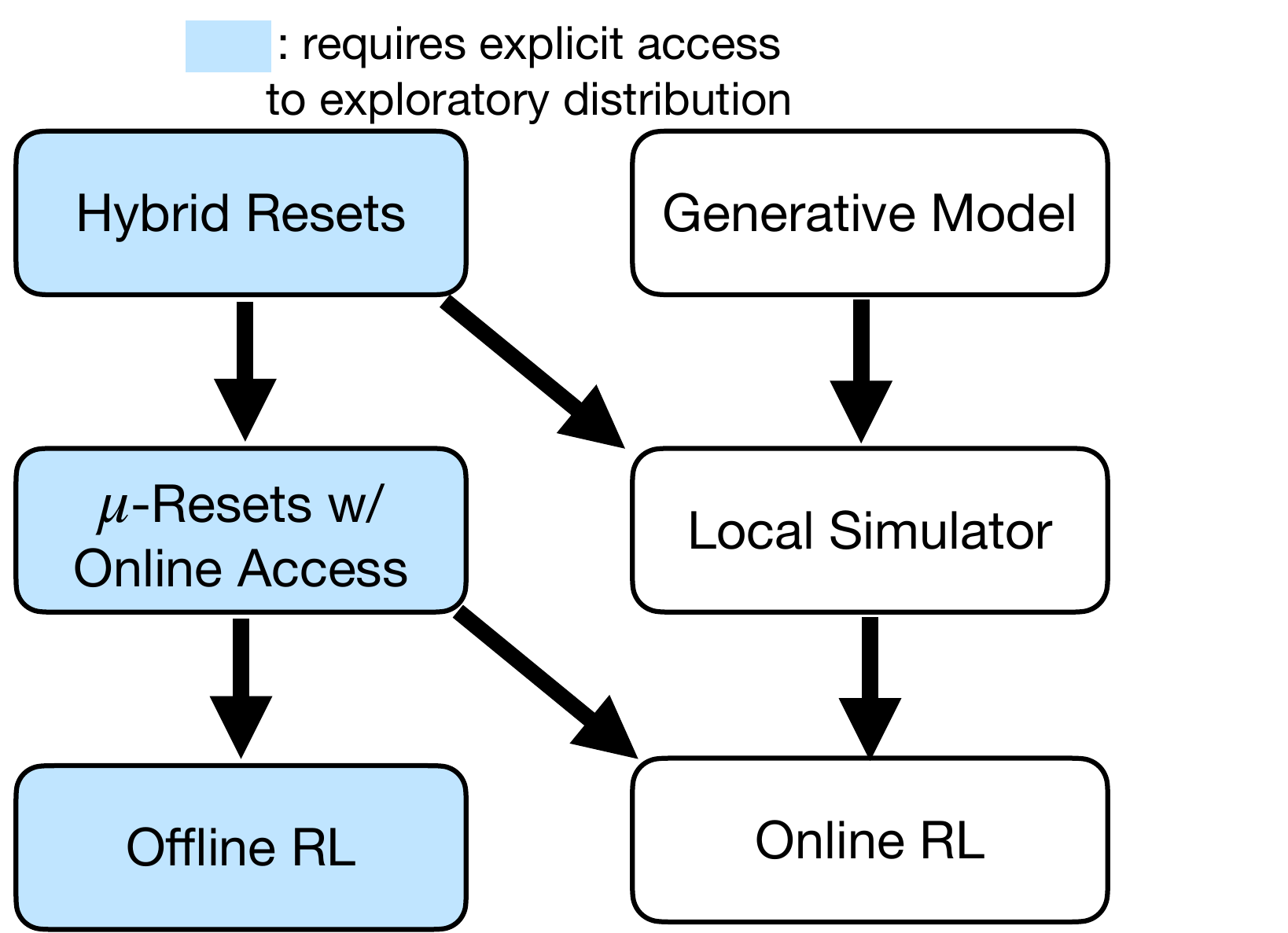} 
} }

\end{minipage}
  \captionof{table}{\textbf{Left.} Summary of results for policy learning under various forms of access to the MDP. A \textcolor{Green}{\cmark{}} indicates there exists an algorithm that adapts to coverage conditions, while \textcolor{red}{\xmark} indicates a lower bound showing impossibility. Remarks: For realizability + $\mu$-resets (\textcolor{Dandelion}{\textbf{?}}), we establish sample-inefficiency for \psdp{} and \cpi{} (\pref{sec:policy-completeness}), but impossibility remains open. Two settings are omitted: in online RL, adapting to coverability is impossible (implied by \pref{thm:lower-bound-coverability}); in offline RL, adapting to concentrability of the offline distribution is impossible \cite[Appendix G of][]{jia2024offline}. \textbf{Right.} Relationships between interaction models. An arrow $A \boldsymbol{\rightarrow} B$ implies that interaction model $B$ can be simulated using interaction model $A$.}
  \label{tab:results}
\end{figure}

\section{Preliminaries}\label{sec:preliminaries}

\arxiv{\subsection{Markov Decision Process}}

\paragraph{Markov Decision Process.} We study reinforcement learning (RL) in a finite horizon Markov Decision Process (MDP). We denote the MDP by the tuple $M = \prn{\statesp, \actionsp, P, R, H, d_1}$, which consists of a state space $\statesp$, action space $\cA$ with cardinality $A$, probability transition function $P: \statesp \times \actionsp \to \Delta(\statesp)$, reward function $R: \statesp \times \actionsp \to \Delta([0,1])$, horizon $H \in \bbN$, and initial state distribution $d_1 \in \Delta(\statesp)$. For simplicity we assume that the state space $\statesp$ is layered across time, i.e., $\statesp= \statesp_1 \cup \cdots \cup \statesp_H$ where $\statesp_i \cap \statesp_j = \emptyset$ for all $i \ne j$. Thus, given a state $x \in \statesp$ it can be inferred which layer $x$ belongs to, which we will overload as the function $h: \statesp \to [H]$. Beginning with $x_1 \sim d_1$, an episode proceeds in $H$ steps, where at each time step $h \in [H]$, the learner plays an action $a_h$, the reward is sampled as $r_h \sim R(x_h, a_h)$, and the next state is sampled as $x_{h+1} \sim  P(\cdot \mid x_h, a_h)$. We assume that the rewards are normalized so that $\sum_{h=1}^H r_h \in [0,1]$ a.s.

\paragraph{Policy-Based Reinforcement Learning.} A \emph{policy} is a function $\pi: \statesp \to \Delta(\actionsp)$. For any policy, $\pi(\cdot  \mid  x_h)$ denotes the distribution over actions that the policy takes when presented with state $x_h$. We denote $\En^\pi[\cdot]$ and $\Pr^\pi[\cdot]$ to denote the expectation and probability under the process of running $\pi$ in the MDP $M$. The value function and the $Q$ function for a given $\pi$ are defined such that for any \(x\) and \(a\), 
\begin{align*}
    V^\pi_h(x) = \En^\pi \brk*{ \sum_{h' = h}^H r_h  \mid  x_h = x}, \quad \text{and} \quad Q^\pi_h(x, a) = \En^\pi \brk*{ \sum_{h' = h}^H r_h  \mid  x_h = x, a_h = a }.
\end{align*}
We let $\pi^\star$ denote an optimal (deterministic) policy which maximizes $Q^\pi(x,a)$ for every $(x,a) \in \statesp \times \actionsp$ simulatenously. Furthermore when clear from the context we denote $V^\pi \coloneqq \En_{x_1 \sim d_1} V^\pi(x_1)$. We also define the occupancy measures $d^\pi_h(x,a) \coloneqq \Pr^\pi \brk*{x_h = x, a_h = a}$ and $d^\pi_h(x) \coloneqq \Pr^\pi \brk*{x_h = x}$.

We assume the learner is given a policy class $\Pi \subseteq \Delta(\actionsp)^\statesp$. For any $h \in [H]$ we let $\Pi_h \subseteq \Delta(\actionsp)^{\statesp_h}$ be the restriction of the policy class to the states in layer $h$. We define a \emph{partial policy} to be one that is defined over a contiguous subset of layers $[l, \cdots, r] \subseteq [H]$, and use $\Pi_{l:r}$ to denote the set of partial policies defined by $\Pi$. We say the policy class $\Pi$ satisfies \emph{realizability} if $\pi^\star \in \Pi$. Otherwise, we say we are in the \emph{agnostic} \asedit{RL} setting.

\paragraph{Block MDPs.} Block MDPs \cite{jiang2017contextual, du2019provably} are a prototypical setting for RL with large state spaces but low intrinsic complexity. Formally, a Block MDP is given by the tuple $M = (\statesp, \latentsp, \actionsp, H, \optlatp, \optlatr, \psi)$. Compared to the definition of the MDP, we additionally specify a \emph{latent state space} $\latentsp$ and an \emph{emission function} $\psi: \latentsp \to \Delta(\statesp)$. To avoid confusion we refer to observed states $x \in \statesp$ as \emph{observations}. Typically, we assume the latent state space $\latentsp$ is finite, while the observation space $\statesp$ can be arbitrarily large or infinite. Without loss of generality, we will assume that the initial latent state $s_1$ is fixed and known to the learner. %

The dynamics of the Block MDP take the following form: Starting from an initial latent state $s_1$, an emission $x_1 \sim \emission(s_1)$ is generated. For every layer $h \in [H]$, the latent state evolves according to $s_{h+1} \sim \optlatp(\cdot  \mid  s_h, a_h)$ and the reward is sampled as $r_h \sim \optlatr(s_h, a_h)$. The latent state $s_h$ is never observed by the learner, and instead the learner only receives the observation $x_h \sim \emission(s_h)$.

The emission function $\psi$ satisfies the property of \emph{decodability}, which asserts that for every pair $s \ne s'$, we have $\supp(\emission(s)) \cap \supp(\emission(s')) = \emptyset$. Therefore, we can define the ground-truth decoder function $\optdec: \statesp \to \latentsp$ which maps every observation $x$ to the corresponding latent $s$ from which it was been emitted. Under decodability, the observation-level transition function (resp.~reward function) can be written as $P(\cdot  \mid  x_h, a_h) = \emission \circ \optlatp(\cdot \mid  \optdec(x_h), a_h)$ (resp.~$R(x_h, a_h) = \optlatr(\optdec(x_h), a_h)$). A priori, both the emission $\psi$ and the decoder $\optdec$ are unknown to the learner and, in a departure from prior work on Block MDPs~\citep[e.g.,][]{misra2020kinematic}, in policy learning the learner does not have access to a decoder class $\Phi$ containing the true decoder $\optdec$, or an emission class \(\Psi\) containing \(\psi\).  

\subsection{Interaction Models and Sample Complexity} \label{sec:interaction_models}

\paragraph{Interaction Models.} We consider various models for a learner to access the unknown MDP $M$. First, we recall the standard online reinforcement learning framework, where the learner accesses $M$ through the following protocol: in every episode, it can submit any policy $\pi$ and receive a trajectory sampled by running $\pi$ from the initial state distribution $x_1 \sim d_1$. We consider stronger models of interaction which augment the standard online RL framework.

\begin{itemize}
    \item \textbf{Generative Model.} Also known as a \emph{global simulator}. The learner can query any tuple $(x,a)$ and receive a sample $(x', r)$ where $x' \sim P(\cdot \mid x,a)$ and $r \sim R(x,a)$.
    \item \textbf{Local Simulator.} In addition to starting from a random initial state $x_1 \sim d_1$, the learner can choose to reset the MDP to any state $x_h$ which has been previously encountered and then generate a (partial) trajectory starting from this state.
    \item \textbf{$\mu$-Resets.} The learner has access to an exploratory reset distribution $\mu = \crl{\mu_h}_{h=1}^H$ with \(\mu_h \in \Delta(\cX_h)\), and can choose to either receive trajectories sampled by running policies from the initial state distribution $d_1$ or any of the exploratory distributions $\mu_h$.\footnote{A related, weaker setting is \emph{offline RL} \cite{levine2020offline, chen2019information}, where instead of on-demand sampling access to $M$, the learner receives a dataset $\cD = \crl{\cD_h}_{h=1}^H$ where each $\cD_h$ is comprised of tuples $(x_h,a_h,x'_{h+1},r_h)$ where $(x_h,a_h)$ are i.i.d.~from a distribution $\mu_h \in \Delta(\statesp_h \times \actionsp)$ and $(x', r)$ are sampled as $x'_{h+1} \sim P(\cdot \mid x_h,a_h)$ and $r_h \sim R(x_h,a_h)$.}
    \item \textbf{\resetmodel.} The learner has access to the exploratory reset distribution $\mu = \crl{\mu_h}_{h=1}^H$ and local simulator access. This is the strongest form of access, subsuming both the local simulator and $\mu$-resets. To the best of our knowledge, this setting has not been considered in prior work.
\end{itemize}

To summarize the connections between different problem settings, we refer the reader to the figure on the right side of \pref{tab:results} as well as \pref{app:related-works} for further discussion. 

\paragraph{Objective: Sample-Efficient PAC Learning.} A \emph{sample} from any of these interaction models is a single episode of interaction with $M$, i.e., a partial trajectory $\tau_{h:H} = (x_h, a_h, r_h, \cdots, x_{H}, a_{H}, r_{H})$ that is obtained by running some policy in $M$. Up to a factor of $H$, this is equivalent to other notions of sample complexity studied in the literature. We study the standard agnostic PAC learning objective: 
\textit{How many samples are  needed to learn a policy 
$\estpi$ such that with probability at least $1-\delta$, $\estpi$ competes with the best policy in the class $\Pi$:} $$V^\estpi \ge \max_{\pi \in \Pi} V^\pi - \eps ?$$

\vspace{-0.1cm}

\subsection{Policy Search By Dynamic Programming}

Policy Search By Dynamic Programming (\psdp{}) is a widely studied policy learning algorithm \cite{bagnell2003policy} that relies on $\mu$-reset access.  \psdp{} constructs partial policies $\estpi_{h:H} \in \Pi_{h:H}$, starting from layer $H$, and returns the estimated policy $\estpi_{1:H}$. We provide pseudocode and analysis of \psdp{} in \pref{app:psdp}. The classic analysis of \psdp{} requires two key assumptions: (1) an exploration condition called \emph{concentrability}; (2) a representation condition called \emph{policy completeness}.

\paragraph{Concentrability.}One can measure the quality of the reset distribution $\mu$ by how well it covers the state space. These so-called \emph{coverage conditions} are well-studied in RL (see \pref{app:related-works}). Roughly speaking, coverage conditions are intrinsic properties of the underlying MDP which measure the expansiveness of the set of state-occupancy measures for policies in a given class $\Pi$. We state a classical notion called concentrability, which depends on the reset distribution, MDP, and policy class. Here and throughout, we use $\nrm{p/q}_\infty$ to denote $\sup_{x \in \statesp} p(x)/q(x)$ for distributions $p,q \in \Delta(\statesp)$.

\begin{definition}[Concentrability] The concentrability coefficient for a distribution $\mu = \crl{\mu_h}_{h=1}^H$ with respect to class $\Pi$ and MDP $M$  is defined as
\arxiv{\begin{align*}
    \cconc(\mu; \Pi, M) \coloneqq \sup_{\pi \in \Pi, h \in [H]} \nrm*{\frac{d^\pi_h}{\mu_h}}_\infty. 
\end{align*}}
When clear from the context we denote the concentrability coefficient by just $\cconc$.  
\end{definition} 

\paragraph{Policy Completeness.} Completeness assumptions on the function approximator class are often assumed in the study of RL algorithms (c.f.~\pref{app:related-works}). \psdp{} requires a notion called \emph{policy completeness}, which ensures that the policy class is closed under the policy improvement operator \cite{dann2018oracle, misra2020kinematic}. %

\begin{definition}[Policy Completeness]\label{def:policy-completeness}
A policy class $\Pi$ satisfies policy completeness if for every $\pi \in \Pi$ and $h \in [H]$, there exists a policy $\wt{\pi} \in \Pi$ such that:
\arxiv{\begin{align*}
    \text{for all $x \in \statesp_h$}: \quad \wt{\pi}_h(x) = \argmax_{a \in \actionsp} Q^\pi(x,a).
\end{align*}}
\end{definition}
Here, we state a worst-case variant of policy completeness, but the analysis of \psdp{} only requires a weaker $\ell_1$ variant of policy completeness, see \pref{app:psdp} for more details. Policy realizability (which asserts that such a $\wt{\pi}$ exists for $\optpi_{h+1:H}$ at every $h \in [H]$) is implied by policy completeness. 

\paragraph{Sample Complexity Guarantee for \psdp.} As a prototypical classical result on policy learning, we now state the guarantee for \psdp{}.

\begin{theorem}\label{thm:psdp-ub}
Suppose the policy class $\Pi$ satisfies policy completeness (\pref{def:policy-completeness}), and the reset distribution $\mu$ satisfies concentrability with parameter $\cconc$. With  probability \(1 - \delta\), \psdp{} finds an $\eps$-optimal policy using $\poly( \cconc, A, H, \eps^{-1}, \log \abs{\Pi}, \log \delta^{-1})$ samples from the reset distribution.
\end{theorem} 

\section{Technical Overview of Results}\label{sec:technical-overview}

Our paper studies whether the classical results on policy learning (e.g., \pref{thm:psdp-ub}) can be improved: can we avoid requiring access to a reset distribution or the stringent policy completeness assumption? %

\subsection{Question 1: Do we need a reset distribution?} 

First, we study if sample-efficient learning is possible without requiring explicit access to the exploratory distribution $\mu$. In this setting, a popular notion is \emph{coverability}, which posits merely the existence of a good reset distribution, and thus lower bounds concentrability coefficient for any distribution $\mu$. 

\begin{definition}[Coverability \cite{xie2022role}]\label{def:ccov}
The coverability coefficient for a policy class $\Pi$ and MDP $M$ is defined as %
\arxiv{
    \begin{align*}
        \ccov(\Pi, M)\coloneqq 
        \max_{h \in [H]} \inf_{\mu_h \in \Delta(\statesp)}\sup_{\pi \in \Pi} ~ \nrm*{\frac{d^\pi_h}{\mu_h}}_\infty.
    \end{align*} 
}When clear from the context we denote the coverability coefficient as $\ccov$.
\end{definition}

Coverability is an intrinsic property that depends on the underlying MDP and the policy class. Recent work also defined  \emph{spanning capacity} which is the worst case (over all MDPs defined over fixed state/action spaces and horizon) value of coverability and is solely a structural property of the policy class $\Pi$ itself.
\begin{definition}[Spanning Capacity \cite{jia2023agnostic}]
The spanning capacity of a policy class $\Pi$ is defined as
\arxiv{\begin{align*}
    \spancap(\Pi) \coloneq \sup_{M} \ccov(\Pi, M).
\end{align*}}
\end{definition}
We ask whether a sample complexity that scales polynomially with the problem parameters can be achieved without access to the reset distribution. Prior work provides some partial answers: 
\begin{itemize}
    \item In online RL, \cite{jia2023agnostic} show that polynomial sample complexity in terms of the spanning capacity is not possible in general. Since spanning capacity upper bounds coverability for any MDP, their lower bound also rules out a sample complexity upper bound in terms of coverability.\footnote{\cite{jia2023agnostic} show that when $\Pi$ additionally satisfies the sunflower property, it is possible to achieve a bound which depends polynomially on coverability and the parameters of the sunflower property. However, it is not known if the sunflower property is a fundamental structural property required for agnostic policy learning in online RL.} 
    \item With local simulator access, \cite{jia2023agnostic} show that the minimax (i.e., worst case over all MDPs) sample complexity for any class  $\Pi$ is $\Theta(\spancap(\Pi))$ (ignoring dependence on other parameters). Unfortunately, spanning capacity is exponentially large for many policy classes of interest (such as linear policies) and can be arbitrarily larger than coverability. \cite{jia2023agnostic} leave it as an open question whether there exists an instance-dependent algorithm that adapts to coverability, finding a near-optimal policy using sample complexity scaling with $\ccov(\Pi, M)$ instead $\spancap(\Pi)$. %
\end{itemize}

\subsection*{\raisebox{0.25ex}{$\blacktriangleright~$} Result 1: Impossibility of Adapting to Coverability} We resolve the question raised by \cite{jia2023agnostic}, showing that it is not possible to adapt to coverability, even with generative access.

\begin{theorem}\label{thm:lower-bound-coverability} For any $H \in \bbN$, there exists a policy class $\Pi$ of size $2^H$ and a family of MDPs $\cM$ over a  state space of size $2^{O(H)}$, binary action space, and horizon $H$ such that every $M \in \cM$ satisfies (A) $\ccov(\Pi, M) = 2$ and  (B) $\Pi$ is policy complete for $M$, so that any proper deterministic algorithm that returns a $1/8$-optimal policy must use at least $2^{\Omega(H)}$ generative access samples for some MDP in $\cM$. 
\end{theorem}

Key ideas behind the lower bound construction are found in \pref{sec:lower-bounds}, and the proof is given in \pref{app:lower-bound-generative-access}. 

\pref{thm:lower-bound-coverability} shows that even under the strongest model of interaction to $M$ and the strongest representational condition on $\Pi$, the mere existence of a good exploratory distribution $\mu$ is insufficient for policy learning. In other words, it formalizes the folklore intuition that ``policy learning methods cannot explore''. Prior work \cite[Proposition 4.1 of][]{agarwal2020pc} suggests that policy gradient methods may fail to explore due to vanishing gradients; \pref{thm:lower-bound-coverability} shows that this is not an algorithmic limitation of policy gradient methods but an information theoretic barrier. Furthermore, there is no contradiction between \pref{thm:lower-bound-coverability} and works which imbue policy gradient methods with exploration capabilities \cite{agarwal2020pc, zanette2021cautiously, liu2024optimistic}, since the latter impose stronger dynamics and/or function approximation assumptions.

Additionally, \pref{thm:lower-bound-coverability} reveals a strict separation between policy-based RL and value-based RL with a local simulator. Under the stronger assumption that the learner has access to a $Q$-function class $\cF \in [0,1]^{\statesp \times \actionsp}$ satisfying value function realizability ($Q^\star \in \cF$), \cite[Theorem 3.1 of ][]{mhammedi2024power} gives an algorithm that achieves sample complexity $\poly(\ccov, H, \log \abs{\cF}, 1/\eps, \log 1/\delta)$. Again, this is not in contradiction with our result because in \pref{thm:lower-bound-coverability}, the implicitly defined value function class $\cF$ has cardinality which is double-exponential in $H$.

\subsection{Question 2: Do we need policy completeness?}\label{sec:policy-completeness}
Granting the learner access to an exploratory reset distribution via $\mu$-resets---as is done in \psdp{}---is a natural way to overcome the lower bound in~\pref{thm:lower-bound-coverability}.
Next, we investigate if the policy completeness assumption can be removed if the learner has access to $\mu$-resets. 

\subsection*{\raisebox{0.25ex}{$\blacktriangleright~$} Result 2: Impossibility of Agnostic Policy Learning for $\mu$-Resets}
We show that the policy completeness assumption cannot be removed in general. Specifically, one cannot achieve sample-efficient agnostic policy learning under $\mu$-reset access. 

\begin{theorem}\label{thm:lower-bound-policy-completeness}
For any $H \in \bbN$, there exists a policy class $\Pi$ of size $2^H$, a family of MDPs $\cM$ over a state space of size $2^{O(H)}$, binary action space, horizon $H$, and a reset distribution $\mu$  satisfying $\cconc(\mu; \Pi, M) = 6$ for all $M \in \cM$, so that any proper deterministic algorithm that returns a $1/16$-optimal policy must use at least $2^{\Omega(H)}$ samples from $\mu$-reset access for some MDP in $\cM$.
\end{theorem}

Key ideas of the construction are found in \pref{sec:lower-bounds}, and the proof is given in \pref{app:lower-bound-policy-completeness}. 

It is interesting to ask what happens if the policy class satisfies realizability, which lies between policy completeness and the agnostic setting. The construction of \pref{thm:lower-bound-policy-completeness} critically relies on the fact that the policy class is \emph{not realizable}, and we do not have an information-theoretic lower bound with a realizable policy class. However, it is easy to see that policy realizability is insufficient for \psdp{} even for horizon $H=2$, as shown in \pref{fig:psdp-lower-bound-simple}. Similar to lower bounds for offline RL \cite{foster2021offline}, the construction relies on \emph{overcoverage}, as $\mu$ has nonzero mass on a nonreachable state $\bar{s}_{1}$, which is somewhat unnatural. Therefore, in \pref{app:psdp} we study \psdp{} when the exploratory distribution is \emph{admissible} (can be realized as a mixture of policies in $\Pi$ \cite{jia2024offline}). Here, we tightly characterize the worst case sample complexity of \psdp{} as $(\cconc)^{O(H)}$ by giving (1) a substantially more involved lower bound construction with compounding errors, and (2) a new analysis for \psdp{} which accounts for the recursive structure of policy completeness errors when $\mu$ is admissible. 

Lastly, we remark that our lower bound constructions against \psdp{} also apply to similar algorithms based on policy iteration, e.g., the classic Conservative Policy Iteration (\cpi{}) \cite{kakade2002approximately}, which also requires policy completeness for global optimality.%

\begin{figure}[!t]
    \centering
\includegraphics[scale=0.30, trim={0cm 24cm 21cm 0cm}, clip]{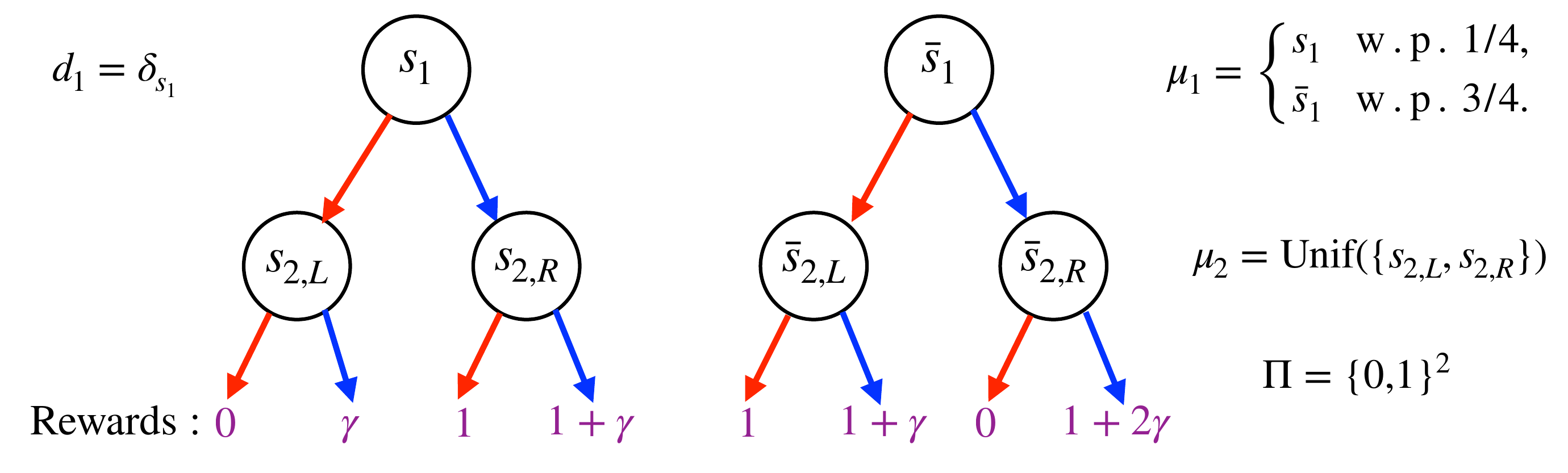}
    \caption{Lower bound for \psdp{} without policy completeness. \textcolor{red}{Red} arrows represent action $0$ and \textcolor{blue}{blue} arrows represent action $1$. In \textcolor{Purple}{purple} we denote the expectation of the stochastic reward. Let $\gamma > 0$ be an arbitrarily small constant. At layer $h=2$, with constant probability, \psdp{} selects $\textcolor{red}{\estpi^{(2)} \gets 0}$ since $\En_{x\sim \mu_2} V^{\pi_0}(x) = 1/2$ and $\En_{x\sim \mu_2} V^{\pi_1}(x) = 1/2 + \gamma$. Conditioned on $\estpi^{(2)} = 0$, we have $\En_{x \sim \mu_1} V^{\pi_0 \circ \estpi^{(2)}}(x) = 3/4$ while $\En_{x \sim \mu_1} V^{\pi_1 \circ \estpi^{(2)}}(x) = 1/4$, so therefore \psdp{} selects $\textcolor{red}{\estpi^{(1)} \gets 0}$. The returned policy $\estpi^{(1)} \circ \estpi^{(2)}$ is $(1+\gamma)$-suboptimal on $d_1$. Note that $\mu = \crl{\mu_1, \mu_2}$ satisfies $\cconc = 4$, and that $\Pi$ satisfies realizability.} 
    \label{fig:psdp-lower-bound-simple}
\end{figure}   

\subsection*{\raisebox{0.25ex}{$\blacktriangleright~$}   Result 3: Positive Result under Hybrid Resets}

The previous negative results motivate us to consider hybrid reset access, where we handle the \emph{exploration challenge} via exploratory resets, and the \emph{error amplification challenge} via local simulator access. For value-based learning, \cite{mhammedi2024power} show that local simulator access can overcome the notorious \emph{double sampling problem}, which leads to error amplification. %
Furthermore, local simulator access circumvents the lower bound construction used to prove~\pref{thm:lower-bound-policy-completeness}. Given this, it is conceivable that local simulators might provide significant power in agnostic policy learning. 

Our main positive result formalizes this intuition, where we provide a  new algorithm that leverages hybrid resets for sample-efficient learning in Block MDPs. %
 Block MDPs are perhaps the simplest setting with large state spaces for developing RL algorithms, as well as a stepping stone to more challenging settings such as low-rank MDPs or coverable MDPs (the \psdp{}/\cpi{} setting). Since our lower bound constructions are all Block MDPs, a positive result here already indicates the significant power of hybrid resets. 
As a caveat, we require the exploratory distribution $\mu$ to satisfy \emph{pushforward concentrability}, a strengthened version of concentrability introduced by \cite{xie2021batch}.

\begin{definition}[Pushforward Concentrability]\label{def:exploratory-pushforward-distribution}
The pushforward concentrability coefficient for a distribution $\mu = \crl{\mu_h}_{h\in[H]}$ with respect to MDP $M$ is \arxiv{
\begin{align*}
    \cpush(\mu; M) \coloneqq \max_{h \in [H]} \sup_{(x,a,x') \in \statesp_{h-1} \times \actionsp \times \statesp_{h}} \frac{P(x' \mid x,a)}{\mu_h(x')}.
\end{align*}
When clear from the context we denote the pushforward concentrability coefficient as $\cpush$.
}\end{definition}
Note that unlike concentrability, pushforward concentrability only depends on the distribution $\mu$ and the MDP $M$, and does not depend on the policy class $\Pi$. It is known that the pushforward concentrability coefficient is always an upper bound on the concentrability coefficient for any distribution $\mu$, but concentrability can be arbitrarily smaller \cite{xie2021batch}. However, it can be checked that in the lower bounds in this paper (namely \pref{thm:lower-bound-policy-completeness} and \pref{thm:psdp-lower-bound}), the constructed resets $\mu$ indeed satisfy bounded pushforward concentrability.

\begin{theorem}\label{thm:block-mdp-result}
    Let $M$ be a Block MDP of horizon $H$ with $S$ states and $A$ actions. Let $\Pi$ be any policy class. Suppose we are given an exploratory reset distribution $\mu = \crl{\mu_h}_{h=1}^H$ which satisfies pushforward concentrability with parameter $\cpush$ and can be factorized as $\mu_h = \emission \circ \nu_h$ for some $\nu_h \in \Delta(\latentsp_h)$ for all $h\in[H]$.\footnote{The factorization assumption is made for technical convenience, and can be removed (see \pref{app:upper-bound-preliminaries}).} With probability at least $1-\delta$, \stochalg{} (\pref{alg:stochastic-bmdp-solver-v2}) returns an $\eps$-optimal policy using 
    \arxiv{
    \begin{align*}
        \poly\prn*{ \cpush, S, A, H, \frac{1}{\eps}, \log \abs{\Pi}, \log \frac{1}{\delta}} \quad \text{samples from hybrid resets.}
    \end{align*}    
    }
\end{theorem}

To support the presentation of our main result, we first present a simplified algorithm called \detalg{} for an easier setting in \pref{sec:warmup}, then present \stochalg{} in \pref{sec:main-upper-bound}.  We now discuss several implications of \pref{thm:block-mdp-result}.

\begin{itemize} 
    \item Hybrid resets enables new statistical guarantees which are impossible with just local simulator access (cf.~\pref{thm:lower-bound-coverability}) and $\mu$-resets (cf.~\pref{thm:lower-bound-policy-completeness}).
    \item As previously discussed, \psdp{} provably fails in the absence of policy completeness, and even policy realizability does not help. In contrast, \pref{thm:block-mdp-result} achieves sample-efficient learning in the agnostic setting. Therefore, at least in Block MDPs, policy completeness is not an information theoretic barrier, only an algorithmic barrier. 
    \item Departing from prior work on Block MDPs, we do not require decoder realizability, namely that the learner is given a decoder class $\Phi \subseteq \latentsp^\statesp$ which satisfies $\optdec \in \Phi$. With decoder realizability, sample-efficient learning is possible with standard online RL access. Since an (approximately) realizable policy class of size $\log \abs{\Pi} \le \poly(S,A, \log \abs{\Phi}, 1/\eps)$ can be constructed from a decoder class by a standard covering argument, \pref{thm:block-mdp-result} provides substantially stronger guarantees than previously known (albeit under the stronger hybrid reset access). 
\end{itemize}

\paragraph{Key Technical Insights for the Upper Bound.} The fundamental challenge in agnostic policy learning is to simultaneously estimate the values of all policies $\crl{V^\pi}_{\pi \in \Pi}$ in a statistically efficient manner. In the absence of any structure, this can require $\Omega(\min \crl{A^H, \abs{\statesp} A, \abs{\Pi} })$ samples \cite{krishnamurthy2016pac, jia2023agnostic}. This bound is attained by adopting the best of: (a) rolling out with uniformly random actions and utilizing importance sampling, (b) learning via tabular methods, %
or (c) individually evaluating each policy using Monte Carlo methods. Unfortunately, this sample complexity is too large for most practical scenarios.

To improve upon this result, prior works in agnostic policy learning have identified \emph{additional structure} which facilitates the simultaneous estimation of $\crl{V^\pi}_{\pi \in \Pi}$. For example, \cite{sekhari2021agnostic} utilize autoregressive extrapolation when the MDP is low-rank, and \cite{jia2023agnostic} construct policy-specific Markov Reward Processes to take advantage of a so-called 
sunflower property of $\Pi$.

Our paper adds a new technical tool called the \emph{policy emulator} to this burgeoning toolbox (see \pref{def:policy-emulator}). A policy emulator, denoted $\wh{M}$, is a carefully constructed tabular MDP which for an $\eps > 0$ satisfies 
\begin{align*}
\text{for all \( \pi \in \Pi \):} \quad \abs{ V^\pi - \widehat{V}^\pi } \leq \epsilon. \numberthis\label{eq:informal-emulator}
\end{align*}
Here, $V^\pi$ denotes the value of $\pi$ in the underlying MDP, while $\wh{V}^\pi$ denotes the value of $\pi$ in the policy emulator $\wh{M}$. Once the policy emulator has been constructed, returning an $O(\eps)$-optimal policy is straightforward by simply returning $\argmax_{\pi \in \Pi} \wh{V}^\pi$. In this sense, the policy emulator is a ``minimal object'' for agnostic policy learning. In fact, we show in \pref{app:beyond-bmdp} that every pushforward-coverable MDP admits a policy emulator of bounded size. The remaining question is: how can we construct this policy emulator using few samples?

Our key contribution is to devise a statistically efficient method for constructing this policy emulator in a bottom-up manner, leveraging the power of hybrid resets. As a warmup, we first explore a simpler scenario in \pref{sec:warmup} where the latent dynamics of the Block MDP are deterministic and the learner has the capability to draw samples from the emission function $\emission(\cdot)$. Here, the emulator can directly be constructed over the latent state space $\latentsp$ in a model-based fashion. We then study the fully general setting in \pref{sec:main-upper-bound}. Here, we construct the emulator directly over $\poly(\cpush, S, A, H, \eps^{-1}, \log \abs{\Pi}, \log \delta^{-1})$ random observations sampled from the reset distributions $\mu_1, \cdots \mu_H$. We will prove that the transitions/rewards of this policy emulator can be accurately estimated so that the guarantee in \eqref{eq:informal-emulator} holds.

\arxiv{
    \section{Main Ideas for Lower Bounds}\label{sec:lower-bounds}

We now explain the main ideas for both of our information-theoretic lower bounds which show that sample-efficient learning is impossible with policy completeness + generative access (\pref{thm:lower-bound-coverability}), and with an agnostic policy class + $\mu$-resets (\pref{thm:lower-bound-policy-completeness}). The proofs are deferred to \pref{app:lower-bounds}.

\paragraph{Rich Observation Combination Lock.} Our lower bounds take the form of \emph{rich observation combination locks}, which are Block MDP variants of the classic combination lock construction \cite{sutton2018reinforcement}. At a high level, the latent transitions of these instances are given by a combination lock parameterized by an unknown open-loop policy $\optpi \in \Pi_\mathsf{open}$; taking the optimal policy $\optpi$ gives the learner reward of 1, while deviation from $\optpi$ at any layer gives the learner reward of zero. Also, the emission function $\emission$ for each state is supported on an exponentially large set which is a-priori unknown to the learner (hence the name ``rich observations''). Such constructions have appeared in previous lower bounds for online RL \cite{sekhari2021agnostic, jia2023agnostic}. The classic combination lock can easily be solved in $\poly(H)$ samples using tabular RL approaches which use the principle of \emph{optimism in the face of uncertainty}---when the learner sees a previously observed state $x_h$, they explore by trying out a new action $a_h$ since it could potentially lead to higher reward. However, the addition of rich observations makes the problem statistically intractable, since it is likely that the learner always sees new observations, so they cannot identify what latent state they are in or when they have deviated from $\optpi$ in a given episode.

Since the rich observation combination lock is a Block MDP, it naturally satisfies small coverability, and furthermore, exploratory distributions $\mu$ can be constructed which satisfy small concentrability. Therefore, it is a natural starting point for proving lower bounds in our setting. Our main technical contribution is to adapt the basic construction to prove information-theoretic lower bounds for the \emph{stronger forms of access} considered in this paper. Our proofs depart from prior results which relied on a complicated stopping time argument \cite{domingues2021episodic, sekhari2021agnostic, jia2023agnostic}; we instead leverage recently developed techniques for proving lower bounds in interactive learning \cite{chen2024assouad}. In particular, we use an interactive variant of Le Cam's Convex Hull Method (\pref{thm:interactive-lecam-cvx}), which follows as a corollary of \cite[Thm.~2 of][]{chen2024assouad}.

\begin{figure}[h]
    \centering
\includegraphics[scale=0.32, trim={0cm 19.5cm 32cm 0cm}, clip]{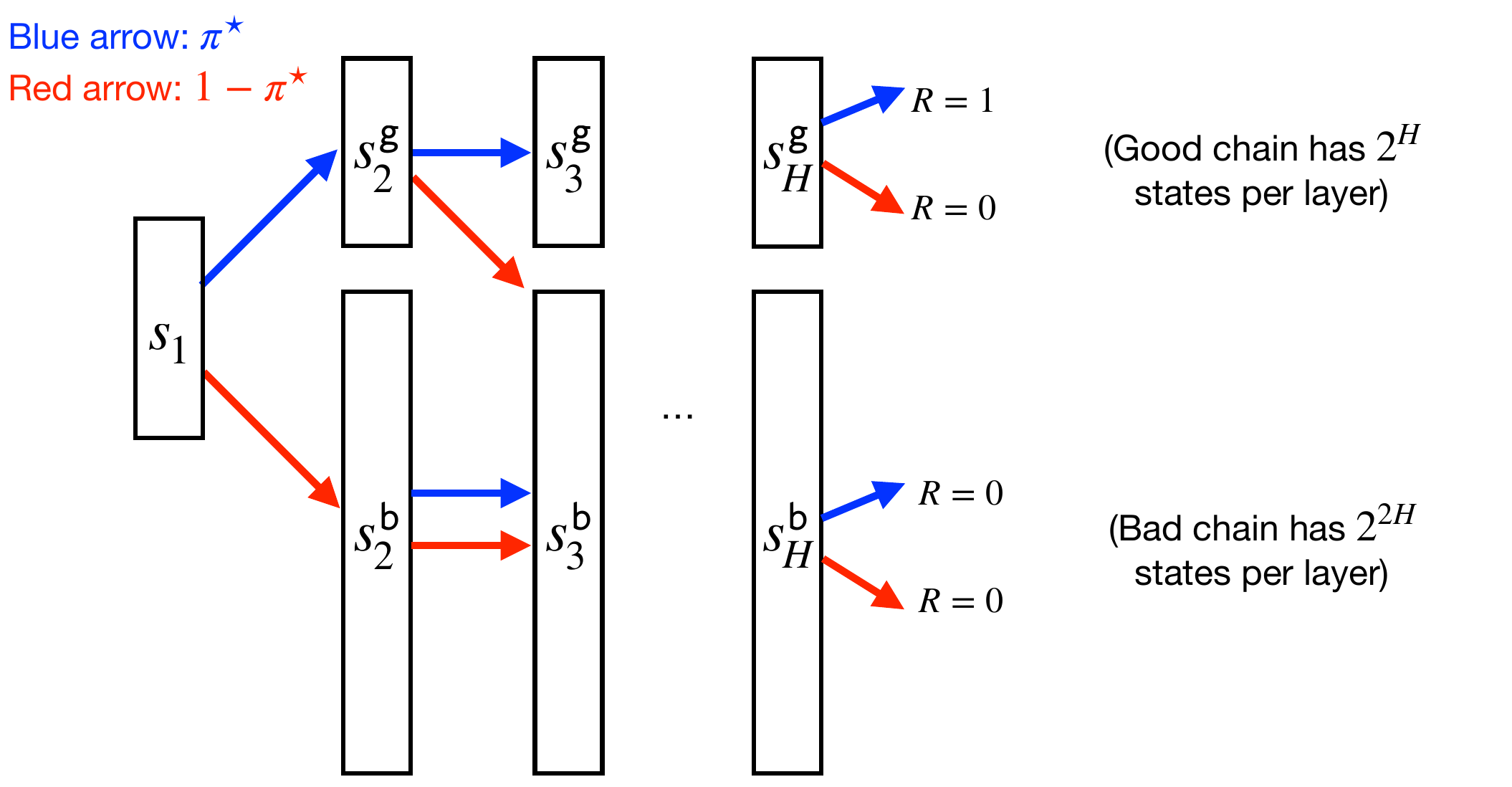}
    \caption{Construction used for proof of \pref{thm:lower-bound-coverability}.} 
    \label{fig:lb1}
\end{figure}

\paragraph{Construction for \pref{thm:lower-bound-coverability}.} An example can be found in \pref{fig:lb1}. In order to prevent the learner from using the more powerful generative model, the lower bound construction has unbalanced emission supports: namely for all $h \ge 2$, the support of $\emission(s_h^\mathsf{g})$ is of size $2^H$, while the support of $\emission(s_h^\mathsf{b})$ is of size $2^{2H}$. Intuitively, the learner receives little information unless they can sample from $(s_H^\mathsf{g}, \optpi_H)$ and receive reward of 1. Since the emission support for $s_h^\mathsf{g}$ is exponentially smaller than that of $s_h^\mathsf{b}$, unless the learner guesses $\exp(H)$ times with the generative model, it is likely that they only receive observations sampled from $s_h^\mathsf{b}$. Stated in a different way, it is not possible for the learner to construct an exploratory distribution $\mu$ which has $\cconc = \poly(H)$, even using $\poly(H)$ adaptive queries to the generative model. Thus, the generative model provides no real additional power over the online RL setting, for which we know $2^{\Omega(H)}$ lower bounds \cite{sekhari2021agnostic}.

\begin{figure}[h]
    \centering
\includegraphics[scale=0.32, trim={0cm 15cm 30cm 0cm}, clip]{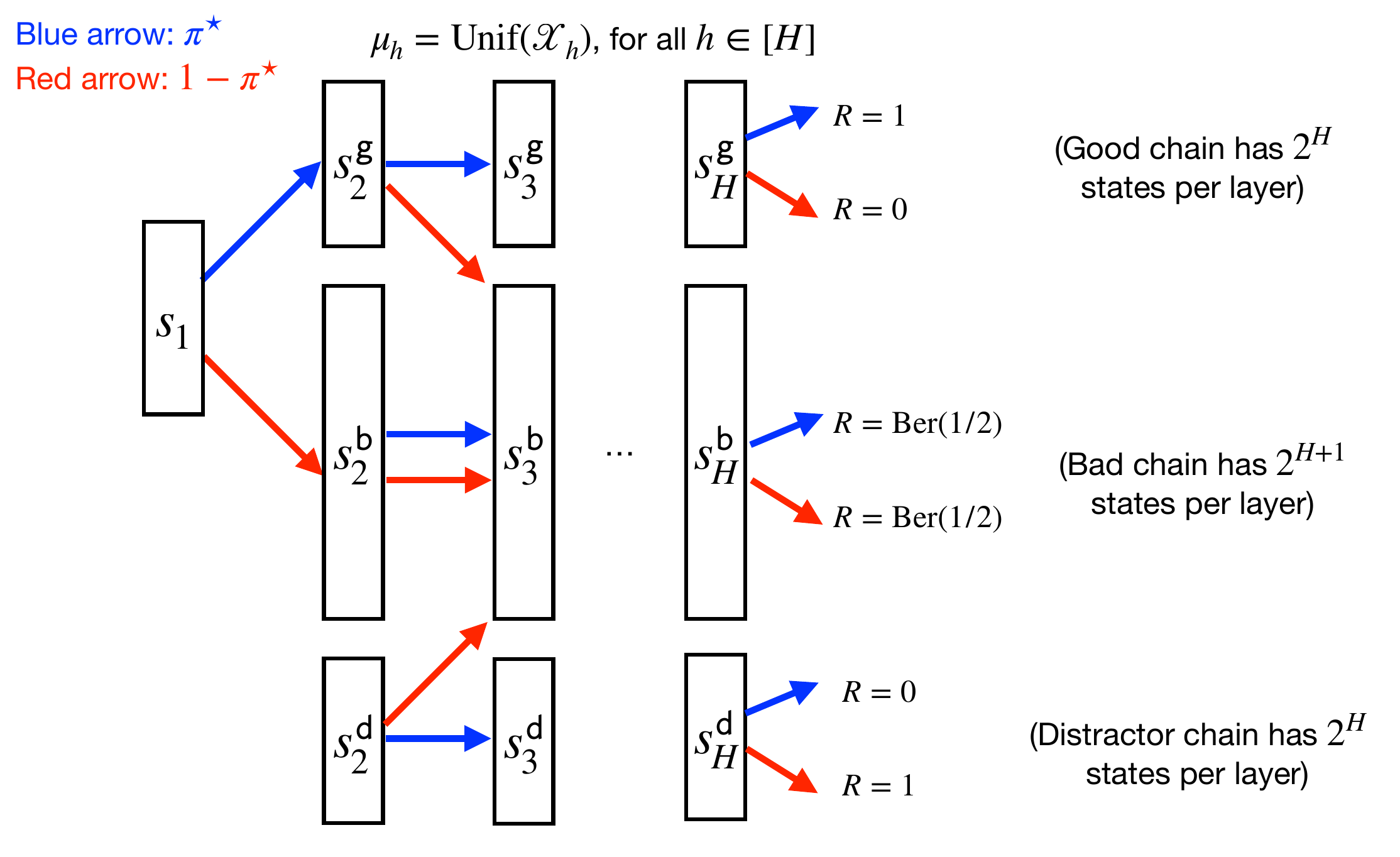}
    \caption{Construction used for \pref{thm:lower-bound-policy-completeness}.} 
    \label{fig:lb2}
\end{figure}

\paragraph{Construction for \pref{thm:lower-bound-policy-completeness}.} An example can be found in \pref{fig:lb2}. We introduce a set of \emph{distractor} latent states $\crl{s_h^\mathsf{d}}_{h \ge 2}$, which are not reachable from the initial distribution $d_1$, and we set $\mu_h$ to be the uniform distribution over all observations in layer $h$. Thus, the exploratory distribution $\mu$ has \emph{overcoverage} over these unreachable states. The distractor states have the same latent transitions as the good states, and the only difference is that the reward at $s_H^\mathsf{d}$ is flipped compared to the reward at $s_H^\mathsf{g}$. This causes rollouts from $\mu_h$ to be noninformative. As for some rough intuition, observe that the distribution of rewards for executing \emph{any} open-loop policy $\pi_{h:H}$ from $\mu_h$ with $h\ge 2$ is $\ber(1/2)$. This is shown by the following casework:
\begin{itemize}
    \item If $\pi_{h:H} = \pi^\star_{h:H}$, then we get a reward of 1 by either sampling $x \sim \emission(s_h^\mathsf{g})$ with probability $1/4$ and getting reward 1 at $s_H^\mathsf{g}$ or sampling $x \sim \emission(s_h^\mathsf{b})$ with probability $1/2$ and getting reward $\ber\prn*{\frac12}$ at $s_h^\mathsf{b}$. Thus the distribution is $\ber\prn*{\frac12}$. 
    \item Similar reasoning holds if $\pi_{h:H} = \pi^\star_{h:H-1} \circ (1-\pi^\star_H)$, but with the reward of 1 coming from sampling the states $x \sim \emission(s_h^\mathsf{d})$.
    \item If $\pi_{h:H}$ is any other policy, then it always reaches $s_H^\mathsf{b}$ and it gets reward $\ber\prn*{\frac12}$.
\end{itemize}
Therefore, observing the reward distribution obtained by executing open-loop policies reveals \emph{no information} about $\optpi$; due to the rich observations, executing non-open loop policies does not really help, and the learner cannot really learn any information about the transition dynamics from the reset $\mu$. Again, the best the learner can do is online RL which requires $2^{\Omega(H)}$ samples. 

We remark that if the learner had local simulator access, then it could easily decode states starting from layer $H$, going backwards, since the reward distributions for a particular $(x_H,a_H)$ pair are different depending on the latent state $\optdec(x_H)$. This idea is precisely the intuition that motivates our main algorithm \stochalg{}.

}

\section{\detalg{}: Algorithm and Results for Warmup Setting}\label{sec:warmup}

We first study an easier setting 
and provide a simplified algorithm that illustrates the main approach that we will take in the general setting (in \pref{sec:main-upper-bound}).

\subsection{Warmup Setting: Deterministic Dynamics and Sampling Access to Emissions}
We make the following simplications: 

\begin{assumption}\label{ass:det-transitions} Assume that: %
\arxiv{
\begin{enumerate}[(1)]
    \item $M$ has deterministic latent transitions $\optlatp$ and (possibly) stochastic rewards $\optlatr$. 
    \item The learner is given both local simulator access and sampling access to the emission function $\emission$. 
\end{enumerate}}
\end{assumption}
Intuitively, \pref{ass:det-transitions} simplifies the problem considerably. Sampling access to the emission enables us to directly estimate the latent reward function $\optlatr$. Furthermore, we can associate a single observation $x \sim \psi(s)$ with each state allowing us to query for $x' \sim P(\cdot \mid s,a)$. 
However, the fundamental challenge of identifying the \emph{latent transition} $\optdec(x')$ remains, which is the main focus of \detalg{}. A few remarks are in order: 
\begin{itemize}
    \item Without loss of generality, we can restrict ourselves to the open-loop policy class $\Piopen = \crl{\pi: \forall x \in \statesp_h, \pi_h(x) \equiv a_h, (a_1, \cdots, a_H) \in \actionsp^H}$. The reasoning is as follows. The optimal policy $\optpi$ for $M$ is constant over $\supp(\emission(s))$ for every $s \in \latentsp$. Due to deterministic latents, there exists some $\wt{\pi} \in \Piopen$ which experiences the same (latent) trajectory $(s_1^\star, a_H^\star, \cdots, s_H^\star, a_H^\star)$ that $\optpi$ experiences. Such a policy $\wt{\pi}$ achieves the optimal value from the fixed starting latent state $s_1$, even though it may not be the optimal policy $\optpi$ that achieves the optimal value from \emph{every} state. 
    \item We implicitly require knowledge of the latent state space $\latentsp = \latentsp_1 \cup \cdots \cup \latentsp_H$ in order to sample from $\psi$. The main algorithm, \stochalg{}, will only require knowledge of a bound $\abs{\latentsp} \le S$.
    \item Sampling access to the emission is more powerful than $\mu$-reset access, since a reset distribution with $\cpush = S$ can be simulated for any $h \in [H]$ by first $s \sim \unif(\cS_h)$ then sampling $x \sim \emission(s)$.
\end{itemize}

\paragraph{Additional Notation: Monte Carlo Rollouts.} Our algorithms (both \detalg{} and \stochalg{}) interact with the environment primarily by collecting Monte Carlo rollouts from states (or distributions over states). For a partial policy $\pi_{h:h'}$, starting state $x \in \statesp_h$, and sample size $n \in \bbN$, we denote the algorithmic primitive $\mcest(x, \pi_{h:h'}, n)$ that:  
\begin{enumerate}[label=\((\arabic*)\)]   
    \item Collects $n$ rollouts $\crl{(x_h^{(t)}, a_h^{(t)}, r_h^{(t)}, \cdots, x_{h'}^{(t)}, a_{h'}^{(t)}, r_{h'}^{(t)})}_{t\in[n]}$ by running $\pi_{h:h'}$ starting from state $x$, \item Returns the estimate $\frac{1}{n}\sum_{t=1}^n \sum_{h\le k \le h'} r_k^{(t)}$. 
\end{enumerate}
 We overload the notation and use $\mcest(d, \pi_{h:h'}, n)$ for $d \in \Delta(\statesp_h)$ to denote a Monte Carlo estimate which first samples $x_h^{(t)} \sim d$ then rolls out with $\pi_{h:h'}$. 

\subsection{The \detalg{} Algorithm and Analysis Sketch}\label{sec:warmup-algorithm-analysis-sketch}
\begin{algorithm}[t]
    \caption{\detalg{} (Policy Learning for Hybrid Resets, Deterministic Version)} \label{alg:det-bmdp-solver}
        \begin{algorithmic}[1]
            \Require Sampling access to emission $\emission(\cdot)$, policy class $\Pi = \Piopen$, parameter $\eps > 0$.             \State Initialize $\estlatentmdp = \varnothing$, test policies $\crl{\Pitest_h}_{h\in[H]} = \crl{\varnothing}_{h \in [H]}$, and confidence sets $\cP = \crl{\latentsp}_{(s,a) \in \latentsp\times \actionsp}$. 

            \For {all $(s,a) \in \latentsp \times \actionsp$} \hfill \algcomment{Estimate all rewards.}
            \State Estimate $\estlatr(s, a)$ via Monte Carlo to precision $\eps/H^2$.\label{line:reward-est} 
            \EndFor
            \State Initialize current layer index $\ell \gets H$.
            \While {$\ell \ne 0$}
            \State \textbf{If} $\ell = H$ \textbf{then} go to line \ref*{line:refit}. 
            
            \For {all $(s_\ell, a_\ell) \in \latentsp_\ell \times \actionsp$} \hfill \algcomment{Construct transitions at layer $\ell$.}
            
            \State Set $\cP(s_\ell, a_\ell) \gets \detdecoder(s_\ell, a_\ell, \estlatentmdp, \cP, \Pitest_{\ell+1})$.\hfill \algcomment{Algorithm \ref{alg:decoder}}
            \State Set $\estlatp(s_\ell, a_\ell) \in \cP(s_\ell, a_\ell)$ arbitrarily. \label{line:est-transition}
            
            \hspace{-0.5em}\algcomment{Construct test policies and refit transitions.}
            \EndFor
            \State Set $(\ell_\mathsf{next}, \estlatentmdp, \cP, \crl{\Pitest_h}_{h\in[H]}) \gets \detrefit(\ell, \estlatentmdp, \cP, \crl{\Pitest_h}_{h\in[H]}, \eps)$. \label{line:refit}\hfill \algcomment{Algorithm \ref{alg:refit}} 
            \State Update current layer index $\ell \gets \ell_\mathsf{next}$.
            \EndWhile 
            \State \textbf{Return} $\estpi \gets \argmax_{\pi \in \Pi} \wh{V}^\pi(s_1)$.
        \end{algorithmic}
\end{algorithm}

\begin{algorithm}[h]
    \caption{\detdecoder{} (Decoder, Deterministic Version)}\label{alg:decoder} 
    \begin{algorithmic}[1]
        \Require Tuple $(s_{h}, a_{h})$, estimated MDP $\estlatentmdp$, confidence sets $\cP$, $\tauref$-valid test policies $\Pitest_{h+1}$.
        \State Sample an observation $x_{h+1} \sim P(\cdot  \mid s_{h}, a_{h})$. 
        \For {$(s,s') \in \cS_{h+1} \times \cS_{h+1}$} 
        \State Estimate $\vestarg{x_{h+1}}{\pi_{s, s'}} \gets \mcest(x_{h+1}, \pi_{s, s'}, \wt{O}(1/\tauref^2))$ to precision $\tauref/2$.  
        \EndFor  
        \State \textbf{Return} $\cP_\mathsf{out} \gets \cP(s_{h}, a_{h}) \cap \crl{s \in \cS_{h+1}: ~\forall s' \ne s,~\abs{\vestarg{x_{h+1}} {\pi_{s, s'}} - \wh{V}^{\pi_{s, s'}}(s)} \le 2\tauref}$.\label{line:decoderd-ret} 
    \end{algorithmic}
\end{algorithm}

\begin{algorithm}[ht] 
\caption{\detrefit{} (Refit, Deterministic Version)}\label{alg:refit}
        \addtolength{\abovedisplayskip}{-5pt}
        \addtolength{\belowdisplayskip}{-5pt}
    \begin{algorithmic}[1]
        \Require Layer $h$, estimated MDP $\estlatentmdp$, confidence sets $\cP$, test policies $\crl{\Pitest_h}_{h\in[H]}$, parameter $\eps > 0$. 
        \State Set tolerance $\tauref \coloneqq 2^5 \cdot \eps/H$.
        \For {\((s , s') \in \cS_h \times \cS_h\) } \hfill \algcomment{Compute candidate test policies at layer $h$}
        \State Let $\pi_{s,s'} \gets \argmax_{\pi \in \Pi} \abs{ \wh{V}^\pi(s) - \wh{V}^\pi(s') }$.\label{line:test-policy} 
        \State Estimate to precision $\eps/H$:\label{line:eval-policy}
        \begin{align*}
            \vestarg{s}{\pi_{s,s'}} \gets \mcest(\emission(s), \pi_{s,s'}, \wt{O}(H^2/\eps^2)) \quad \text{and} \quad  \vestarg{s'}{\pi_{s,s'}} \gets \mcest(\emission(s'), \pi_{s,s'}, \wt{O}(H^2/\eps^2)). 
        \end{align*} 
        \EndFor  
        \State Set \(\violations \gets \crl{(s,\pi) \text{~estimated in \pref{line:eval-policy} s.t.~} \abs{\vestarg{s}{\pi} - \wh{V}^\pi(s)} \ge \tauref -\eps/H}\). 
        \If { \(\violations = \varnothing\) } \hfill \algcomment{No violations found, so return test policies.}  
        \State Set $\Pitest_{h} = \cup_{s,s' \in \latentsp_h} \crl{\pi_{s,s'}}$, and  \textbf{return} $(h-1,\estlatentmdp, \cP, \crl{\Pitest_h}_{h\in[H]} )$. \label{line:great-success} 
        \Else  \hfill \algcomment{Refit transitions to handle violations}  
        \For { $(s, \pi) \in \violations$ } \label{line:violation-statement}
        \State Let $\tau = (\bar{s}_h = s, \cdots \bar{s}_H)$ be the sequence of states obtained by executing \(\pi\) from \(s\) in  $\estlatentmdp$. 
        \For {each \(\bar s \in \tau\)}
        \State Estimate $\vestarg{\bar{s}}{\pi} \gets \mcest(\bar{s}, \pi, \wt{O}(H^4/\eps^2))$ to precision $\eps/H^2$.\label{line:mc-additional}
        \EndFor
        \State \textbf{for} each \( \bar s \in \tau \) such that $\abs{ \vestarg{\bar{s}}{\pi} - \estlatr(\bar{s}, \pi) -  \vestarg{\estlatp(\bar{s}, \pi) } {\pi} } \ge 4\eps/H^2$:  
        \State \qquad Update $\cP(\bar{s}, \pi) \gets \cP(\bar{s}, \pi) ~\backslash~\estlatp(\bar{s}, \pi)$.\label{line:bad-state}           
        \EndFor 
        \State Reset $\estlatp(s, a) \in \cP(s, a)$ arbitrarily for all \((s, a)\) updated in \pref{line:bad-state}. \label{line:bad-state-2}
        \State \textbf{Return} $( \ell,\estlatentmdp, \cP,  \crl{\Pitest_h}_{h\in[H]} )$ where \( \ell \) is the max layer for which transitions were updated in \pref{line:bad-state-2}.\label{line:great-success-2}  
        \EndIf
    \end{algorithmic}
\end{algorithm}

Now, we present an algorithm \detalg{} (\pref{alg:det-bmdp-solver}), which achieves the following guarantee.

\begin{theorem}\label{thm:det-bmdp-solver-guarantee}
    Let $\eps, \delta \in (0,1)$ be given and suppose that \pref{ass:det-transitions} holds. Then, with probability at least $1-\delta$, \detalg{} (\pref{alg:det-bmdp-solver}) finds an $\eps$-optimal policy using 
    \begin{align*}
        \wt{O}\prn*{\frac{S^5A^2H^5}{\eps^2} \cdot \log\frac{1}{\delta}} \quad\text{samples.}
    \end{align*}
\end{theorem}

The proof of \pref{thm:det-bmdp-solver-guarantee} is found in \pref{app:proof-warmup}. In the rest of this section, we will explain \detalg{} and illustrate the main ideas.

\detalg{} is an inductive algorithm that works from layer $H$ down to layer $1$. It maintains an estimated latent MDP $\estlatentmdp$, which approximates the ground truth latent transitions and rewards, as well as two other objects: transition confidence sets $\cP$, which assigns a set of plausible next states to each state-action pair, and a set of $S^2$ many test policies $\Pitest$, which it uses to estimate the latent transitions. 
In the pseudocode and analysis, we use $\wh{V}^\pi(\cdot)$ and $\wh{Q}^\pi(\cdot, \cdot)$ to denote the value function and $Q$-function on the estimated $\estlatentmdp$. Furthermore, we let  $\optlatp(s,a)$ (resp.~$\estlatp$) denote the latent state which $(s,a)$ transitions to in $M$ (resp.~$\estlatentmdp$).

At every layer $h \in [H]$, \detalg{} tries to enforce three invariant properties: 
\begin{enumerate}[(A)]
    \item \emph{Policy Evaluation Accuracy.} For all pairs $(s,a) \in \latentsp_h \times \actionsp$ and $\pi \in \Piopen$: $\abs{Q^\pi(s,a) - \wh{Q}^\pi(s,a)} \le \Gamma_h$, where the error bound $\Gamma_h$ grows linearly with $H-h$.
    \item \emph{Confidence Set Validity.} For all pairs $(s,a)\in \latentsp_h \times \actionsp$, we have $\optlatp(s,a) \in \cP(s,a)$. %
    \item \emph{Test Policy Validity.} The $S^2$ many test policies for layer $h$, i.e. $\Pitest_h := \{\pi_{s,s'}\}_{s,s' \in \cS_h} \subseteq \Piopen$, are defined for pairs of states $s, s' \in \latentsp_h$ and are \emph{valid} (maximally distinguishing and accurate):
    \begin{align}
        \pi_{s,s'} = \argmax_{\pi \in \Pi_{h:H}}~ \abs{\wh{V}^\pi(s) - \wh{V}^\pi(s')}, \quad \text{and} \quad \max_{\bar{s} \in \{s,s'\}}~ |V^{\pi_{s,s'}}(\bar{s}) - \wh{V}^{\pi_{s,s'}}(\bar{s})| \leq \tauref. \label{eq:validity-main}
    \end{align}
    Crucially, the accuracy level $\tauref$ \emph{does not grow} with $H-h$.
\end{enumerate}

\paragraph{Error Decomposition.} To motivate these three properties, we first state a standard error decomposition for $Q$-functions, and then show how \detalg{} controls each of terms separately. In what follows, fix some tuple $(s,a)$. We denote $\optlatr = \optlatr(s,a)$ and $\optlatp = \optlatp(s,a)$, as well as the estimated counterparts $\estlatr, \estlatp$ similarly. The Bellman error for $(s,a)$ can be decomposed as follows:
\begin{align*}
    \abs*{Q^\pi(s,a) - \wh{Q}^\pi(s,a)} \le \underbrace{\abs*{\optlatr - \estlatr}}_{\text{reward error}} + \underbrace{\abs*{\wh{V}^\pi(\optlatp) - \wh{V}^\pi(\estlatp)}}_{\text{transition error}} + \underbrace{\abs*{V^\pi(\optlatp) - \wh{V}^\pi(\optlatp)}}_{\text{policy eval.~error at next layer}}. \numberthis \label{eq:bellman-error-main} 
\end{align*}
Controlling the reward error is easy: we can simply collect i.i.d.~samples using sampling access to $\emission$ to estimate $\estlatr$ up to $\eps$ accuracy (see \pref{line:reward-est}). Furthermore, if (A) holds at layer $h+1$, then we can bound the last term of Eq.~\eqref{eq:bellman-error-main} by $\Gamma_{h+1}$. Controlling the transition error requires more work, since the learner only gets to see observations $x_{\mathsf{new}}\sim P(\cdot\mid{}s,a)$, but not the latent state $\phi(x_{\mathsf{new}})$.

\paragraph{Decoding via Test Policies.} Our main insight is to estimate the latent state $\optdec(x_\mathsf{new})$ by using rollouts from $x_\mathsf{new}$ to compare value functions with other latent states. Denoting $\vestarg{x_\mathsf{new}}{\pi}$ to be a Monte-Carlo estimate of $V^\pi(x_\mathsf{new})$, if we find some $s' \in \latentsp_{h+1}$ such that 
\begin{align*}
    \vestarg{x_\mathsf{new}}{\pi} \approx \wh{V}^\pi(s'), \quad \text{for all}~\pi \in \Piopen, \numberthis\label{eq:policy-estimation-all}
\end{align*}
then we declare the latent state of $x_\mathsf{new}$ to be $s'$. 
This allows us to bypass the statistical hardness of learning the decoder function $\optdec$ itself,
but, unfortunately, estimating $V^\pi(x_\mathsf{new})$ \emph{for all} $\pi \in \Piopen$ seems to require number of samples proportional to $\spancap(\Piopen) = A^H$ \cite{jia2023agnostic}. In other words, there is nothing better than just executing each policy one-by-one. However, in our algorithm, the test policies $\Pitest$ allow us to circumvent this. In \detdecoder{} (\pref{alg:decoder}), we use $\Pitest$ to run a ``tournament'' with only $S^2$ Monte Carlo rollouts from $x_\mathsf{new}$ to estimate the confidence set $\cP$ of plausible latent states. In \pref{line:decoderd-ret} of \detdecoder{}, the confidence set is updated to be 
\begin{align*}
    \cP(s, a) \gets \cP(s, a) \cap \Big\{s \in \cS_{h+1}: ~\forall s' \ne s,~\abs{ \vestarg{x_\mathsf{new}} {\pi_{s, s'}} - \wh{V}^{\pi_{s, s'}}(s)} \lesssim \tauref \Big\}. \numberthis\label{eq:confidence-set-warmup} 
\end{align*}
We show in \pref{lem:controlling-transition-error} that test policy validity (C) at layer $h+1$ implies that the confidence set \eqref{eq:confidence-set-warmup} is valid (B) for layer $h$ and furthermore, setting the transition to be any $\estlatp \in \cP$ allows us to extrapolate to statement \eqref{eq:policy-estimation-all}, thus giving us a bound on the transition error. As we have shown a bound for all three terms in Eq.~\eqref{eq:bellman-error-main}, we conclude that (A) also holds at layer $h$. %

\paragraph{Refitting Latent Dynamics.} \detrefit{} (\pref{alg:refit}) computes test policies for layer $h$ that satisfy (C) after we have estimated the transitions/rewards. It does so by solving the maximally distinguishing planning problem (\eqref{eq:validity-main}, left) in $\estlatentmdp$ for each $s,s'\in \cS_h$. Since (A) holds at layer $h$, these policies are guaranteed to be accurate; however, test policies are required to satisfy a higher level of accuracy $\tauref \ll \Gamma_{h}$ which \emph{does not increase with the horizon}. To provide intuition on why the higher level of accuracy is required for the test policies, we refer the reader to \pref{fig:error-amplification}.

Fortunately, since there are only $S^2$ test policies we can use Monte Carlo rollouts to check whether they are $\tauref$-accurate. If they are, we simply decrement to layer $h-1$ and continue (\pref{line:great-success}). If not, the rollouts will find a ``certificate of inaccuracy'': some tuple $(s, \pi)$ for which $\abs{\wh{V}^\pi(s) - V^\pi(s)}$ is large, which we can use to find and delete an erroneous transition in $\estlatp$ from a confidence set. Since this update can occur at some layer $\ell \gg h$, $\estlatentmdp$ may no longer satisfy the inductive hypotheses, so \detrefit{} restarts the outer loop of \detalg{} at the maximum layer $\ell$ for which some transition was updated (\pref{line:great-success-2}). Critically, we show in \pref{lem:refitting} that refitting never deletes the true $P_{\mathsf{lat}}$, so revisiting only happens $SA\cdot(S-1)$ times.

\begin{figure}[t]
    \centering
\includegraphics[scale=0.32, trim={0cm 18.5cm 32cm 0.1cm}, clip]{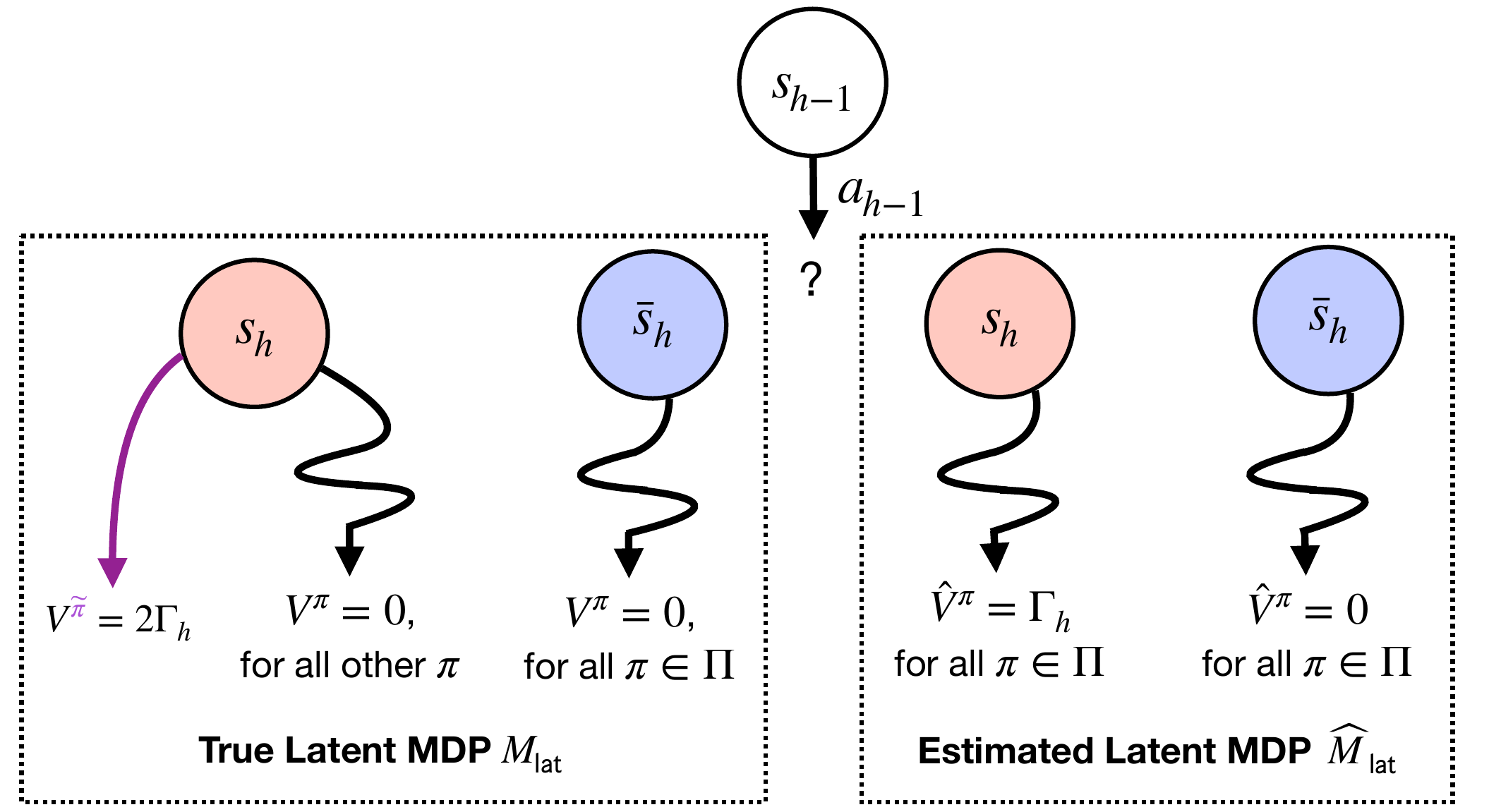}
\caption{Illustration of how certifying accuracy of test policies prevents error amplification. Suppose we want to learn the transition $\optlatp(s_{h-1}, a_{h-1}) = s_h$. In $M_\mathsf{lat}$, all policies get value 0 from both $s_h$ and $\bar{s}_h$, with the exception of a special $\color{Purple}{\wt{\pi}}$ that gets value $2\Gamma_h$ from $s_h$; in $\estlatentmdp$ all policies get value $\Gamma_h$ from $s_h$ and value 0 from $\bar{s}_h$. Thus, $\estlatentmdp$ satisfies $(A)$ but any test policy $\pi_{s_h, \bar{s}_h} \in \Pi$ will not satisfy (C). It is unlikely that $\pi_{s_h, \bar{s}_h} = \color{Purple}{\wt{\pi}}$ is selected, and if we execute any other $\pi$ from the true transition $s_h$, we will observe value $0$, and thus decode the transition to $\estlatp(s_{h-1}, a_{h-1}) = \bar{s}_h$. Therefore, $\abs{Q^{\pi}(s_{h-1}, a_{h-1}) - \wh{Q}^\pi(s_{h-1}, a_{h-1})} = 2\Gamma_{h}$, thus \emph{doubling} the policy evaluation error from layer $h$ to $h-1$. Unchecked, this could cause exponential (in $H$) error amplification. Certifying test policy accuracy prevents this, as \detrefit{} would detect the violation $\abs{V^\pi(s_{h}) -\wh{V}^\pi(s_{h})} = \Gamma_{h} \gg \tauref$ for any $\pi \in \Pi$ and refit $\estlatentmdp$ instead.}
\label{fig:error-amplification}
\end{figure}

\paragraph{Performance of Estimated Policy.} Eventually, \detalg{} will terminate at layer $h=1$. Thanks to (A), we can evaluate all $\pi \in \Piopen$ on the fully constructed $\estlatentmdp$ and return the policy $\estpi$ which achieves the highest value. The inductive argument we have outlined shows that $\estpi$ is an $\eps$-optimal policy and that \detalg{} uses $\poly(S,A,H, \eps^{-1})$ samples.

\arxiv{
    \section{\stochalg{}: Algorithm and Main Results}\arxiv{\label{sec:main-upper-bound}}%
In this section, we extend our result in \pref{thm:det-bmdp-solver-guarantee} to handle the general setting. We give our main algorithm, \stochalg{}, which takes inspiration from \detalg{}. We show how \stochalg{} leverages hybrid resets to solve agnostic policy learning, with sample complexity that scales with the pushforward concentrability $\cpush$ of the reset distribution $\mu$, a measure of the intrinsic difficulty of exploration.

First, we restate our main result of \pref{thm:block-mdp-result} with the precise dependence on the problem parameters. 

\begin{reptheorem}{thm:block-mdp-result} 
    Let $M$ be a Block MDP of horizon $H$ with $S$ states and $A$ actions, and let $\Pi$ be any policy class. Suppose we are given an exploratory reset distribution $\mu = \crl{\mu_h}_{h=1}^H$ which satisfies pushforward concentrability with parameter $\cpush$ and can be factorized as $\mu_h = \emission \circ \nu_h$ for some $\nu_h \in \Delta(\latentsp_h)$ for all $h\in[H]$. With probability at least $1-\delta$, the {\normalfont \stochalg{}} algorithm (\pref{alg:stochastic-bmdp-solver-v2}) returns an $\eps$-optimal policy using
\begin{align*}
    \frac{\cpush^4 S^{24} A^{30} H^{39} }{\eps^{18}} \cdot \mathrm{polylog} \prn*{\cpush, S, A, H, \abs{\Pi}, \eps^{-1}, \delta^{-1}} \quad \text{samples from hybrid resets.}
\end{align*}
\end{reptheorem}

The proof is deferred to \pref{app:main-upper-bound-proofs}. In the rest of this section, we discuss the main aspects of \stochalg{} and provide intuition for how it addresses new technical challenges once we relax \pref{ass:det-transitions}.

\subsection{Algorithm Overview}
We now present an overview of \stochalg{}, whose pseudocode can be found in \pref{alg:stochastic-bmdp-solver-v2}. Similar to \detalg{}, it uses two subroutines: \stochdecoder, found in \pref{alg:stochastic-decoder-v2}, and \stochrefit, found in \pref{alg:stochastic-refit-v2}. Overall, \stochalg{} has a similar structure to \detalg{}, but it requires several new ideas to address several  challenges to circumvent needing \pref{ass:det-transitions}:  

\begin{itemize}
    \item Under \pref{ass:det-transitions}, the learner had sampling access to the emission function $\emission$; as a consequence, we could construct an estimate of the latent MDP $\estlatentmdp$ which was defined over the latent state space $\latentsp$. Sampling access to $\emission$ was crucial since it allowed us to disambiguate observations. If the learner only has access to the reset distribution $\mu$, it is nontrivial even to estimate the latent reward function $\optlatr$, since we cannot access the decoder for observations $x \sim \mu$.
    \item In \detalg{}, even though we were supplied a policy class $\Pi$, we could instead use the open-loop policy class $\Pi_\mathsf{open}$ as a proxy, since we were guaranteed that $\max_{\pi \in \Pi_\mathsf{open}} V^\pi \ge \max_{\pi \in \Pi} V^\pi$. If the MDP has stochastic latent transitions, $\Pi_\mathsf{open}$ might not contain any good policy. Thus, we need to directly evaluate the given policies $\pi \in \Pi$ in order to solve the agnostic policy learning problem. 
\end{itemize}

\paragraph{Policy Emulators.} To address these challenges, we take the more straightforward approach: instead of trying to construct latent transitions/rewards, \emph{we directly construct an MDP $\estmdp$ over observations}. The MDP $\estmdp$ has a restricted state space $\estmdpobsspace{} \subseteq \cX$ but inherits the same action space $\cA$ and horizon $H$. Unlike the standard approach taken in tabular RL, we cannot hope to approximate the dynamics of the true MDP $M$ in an information theoretic sense, as the transition $P(\cdot \mid x,a)$ is an $\abs{\statesp}$-dimensional object (requiring $\Omega(\abs{\cX})$ samples to estimate). Taking a step back, all we need is that $\estmdp$ enables accurate policy evaluation, i.e., denoting $\wh{V}^\pi$ to be the value function of $\pi$ on $\estmdp$, we have $\max_{\pi \in \Pi}~\abs{V^\pi - \wh{V}^\pi} \le \eps$. In this sense $\estmdp$ is a ``minimal object'' which allows us to emulate the values of all policies $\pi \in \Pi$. This is formalized in the following definition.\footnote{Similar terminology of an \emph{emulator} is defined in \cite{golowich2024exploring}. Their definition formalizes what it means for estimated transitions to approximate certain Bellman backup operations, and is tailored to linear MDPs.} In the sequel, we denote $\estmdpobsspace{h}$ and $\estmdpobsspace{h:H}$ to be the restriction of the state space of $\estmdp$ to the given layer(s).  

\begin{definition}[Policy Emulator]\label{def:policy-emulator} Let $\Pi$ be a policy class and $M$ be an MDP. Fix any $\nu \in \Delta(\cX)$. We say $\estmdp$ is an $\eps$-accurate \emph{policy emulator} for $\nu$ if there exists $\wh{\nu} \in \Delta(\estmdpobsspace{})$ such that:
\begin{align*}
    \max_{\pi \in \Pi}~\abs*{\En_{x \sim \nu} \brk{V^\pi(x)} - \En_{x \sim \wh{\nu}} \brk{\wh{V}^\pi(x)} } \le \eps.  
\end{align*}
\end{definition}

\pref{def:policy-emulator} naturally extends the concept of \emph{uniform convergence} \cite{shalev2014understanding} to the interactive setting of policy learning. Clearly, if $\estmdp$ is an $\eps$-accurate policy emulator for the starting distribution $d_1$, we can find an $O(\eps)$-optimal policy. One inspiration for \pref{def:policy-emulator} is the Trajectory Tree algorithm \cite{kearns1999approximate}, which can be viewed as a way to use local simulator access to build a policy emulator with $\abs{\estmdpobsspace{}} = \wt{O}(H \spancap(\Pi)/\eps^2)$ states, requiring sample complexity scaling with the worst-case notion of complexity $\spancap(\Pi)$ \cite{jia2023agnostic}.

In contrast, \stochalg{} utilizes the reset distribution $\mu$ to construct a policy emulator with state space and sample complexity scaling with the instance-dependent notion of complexity $\cpush$. We do this in an inductive fashion, working back from layer $H$ to layer 1.
\begin{itemize}
    \item At every layer $h$, we sample $\poly(\cpush, S, A, H, \eps^{-1}, \log \abs{\Pi})$ states from $\mu_h$ to form the policy emulator's state space $\estmdpobsspace{h}$. The rewards of every tuple $(x_h,a_h) \in \estmdpobsspace{h} \times \actionsp$ are estimated via the local simulator.
    \item Once the transitions of $\estmdp$ has been constructed from layer $h+1$ onward, we call \stochdecoder{} on every $(x_h, a_h) \in \estmdpobsspace{h} \times \actionsp$. \stochdecoder{} first samples a dataset $\cD$ of transitions from $P(\cdot \mid x_h, a_h)$ (in \pref{line:sample-decode-dataset}) and then performs Monte Carlo rollouts over observations in $\cD$ using test policies $\Pitest_{h+1}$ (in \pref{line:monte-carlo}). In contrast with \detalg{}, since \stochalg{} directly works in observation space, the test policies are defined for pair of observations $x, x' \in \estmdpobsspace{h+1}$, not pairs of latent states. \stochdecoder{} estimates a transition function $\estp(\cdot \mid x_h, a_h) \in \Delta(\estmdpobsspace{h+1})$ as well as a confidence set $\cP(x_h, a_h) \subseteq \Delta(\estmdpobsspace{h+1})$.
    \item After transitions at layer $h$ are constructed, we call \stochrefit{} which tries to compute accurate test policies $\Pitest_h$ for layer $h$. If \stochrefit{} succeeds, then \stochalg{} continues the decoding/refitting loop at layer $h-1$. Otherwise, \stochrefit{} searches in the policy emulator $\estmdp$ for an inaccurate transition $\estp(\cdot| \bar{x}, \bar{a})$ and updates it. The layer index $\ell$ is set to the maximum layer for which an $(\bar{x}, \bar{a})$ is updated, and \stochalg{} restarts at that layer $\ell$. 
\end{itemize}
Eventually, \stochalg{} will reach layer 1, giving a fully-constructed policy emulator $\estmdp$. Returning the best policy in $\estmdp$ is guaranteed to be a near-optimal policy for the true MDP $M$.

\begin{algorithm}[t]
    \caption{\stochalg{} (Policy Learning for Hybrid Resets)}\label{alg:stochastic-bmdp-solver-v2}
        \begin{algorithmic}[1]
            \Require Reset distributions $\mu = \crl{\mu_h}_{h \in [H]}$, policy class $\Pi$, parameters $\eps >0$ and $\delta \in (0,1)$.
            \State Initalize policy emulator $\estmdp = \varnothing$, test policies $\crl{\Pitest_h}_{h\in[H]} = \crl{\varnothing}_{h\in[H]}$, transition confidence sets $\cP = \varnothing$.
            \State Set $\nreset \asymp \tfrac{\cpush SA^2}{\eps^3} \cdot \log \tfrac{SA \abs{\Pi}}{\delta} $.
            
            \For {$h = 1, \cdots, H$} \hfill \algcomment{Initialize policy emulator} \label{line:init-start} 
            \State Sample $\nreset$ observations from $\mu_h$ and add to $\estmdpobsspace{h}$. \label{line:sampling-from-reset}
            \For {every $(\empobs_h,a_h) \in \estmdpobsspace{h} \times \actionsp$}
            \State Estimate $\wh{R}(\empobs_h, a_h) \gets \mcest(x_h, a_h, \wt{O}(H^2/\eps^2))$.\label{line:reward-estimation}
            \State Initialize $\cP(\empobs_h, a_h) = \Delta(\estmdpobsspace{h+1})$.\label{line:init-end}
            \EndFor
            \EndFor
            \State Set current layer index $\ell \gets H$.
            \While {$\ell \ne 0$}
            
            \State \textbf{If} $\ell=H$: \textbf{go to line \ref*{line:refit-v2}.}

            \hspace{-0.5em}\algcomment{Construct transitions at layer $\ell$}

            \For {each $(\empobs_\ell, a_\ell) \in \estmdpobsspace{\ell} \times \cA$} 
            \State Set $\cP(\empobs_\ell, a_\ell) \gets \stochdecoder((\empobs_\ell, a_\ell), \estmdp, \cP, \Pitest_{\ell+1}, \eps, \delta)$ \hfill \algcomment{See \pref{alg:stochastic-decoder-v2}}
            \State Set $\estp(\cdot \mid \empobs_\ell, a_\ell) \in \cP(\empobs_\ell, a_\ell)$ arbitrarily.\label{line:est-transition-v2}
            \EndFor

            \hspace{-0.5em}\algcomment{Construct test policies and refit transitions.}

            \State Set $(\ell_\mathsf{next}, \estmdp, \crl{\Pitest_h}_{h\in[H]}, \cP) \gets \stochrefit(\ell, \estmdp, \cP, \crl{\Pitest_h}_{h\in[H]}, \eps, \delta)$. \label{line:refit-v2} \hfill \algcomment{See \pref{alg:stochastic-refit-v2}}
            \State Update current layer index $\ell \gets \ell_\mathsf{next}$.
            \EndWhile
            \State \textbf{Return} $\estpi \gets  \argmax_{\pi \in \Pi} \En_{x_1 \sim \unif(\estmdpobsspace{1})} \brk{\wh{V}^\pi(x_1)}$.
        \end{algorithmic}
\end{algorithm}

\subsection{\stochdecoder{} Subroutine}\label{sec:stochdecoder}

In this section, we explain \stochdecoder{}, which for a given $(x_h,a_h)$ pair computes a confidence set of transitions $\cP(x_h,a_h)$ over the policy emulator states in the next layer $\estmdpobsspace{h+1}$. The main salient difference with \detdecoder{} is that  we now adopt a more sophisticated confidence set construction to ensure that arbitrary policies $\pi \in \Pi$ can be emulated by $\estmdp$. 

\begin{algorithm}[h]
    \caption{\stochdecoder}\label{alg:stochastic-decoder-v2}
        \begin{algorithmic}[1]
            \addtolength{\abovedisplayskip}{-5pt}
            \addtolength{\belowdisplayskip}{-5pt}
            \Require Tuple~$(x_h, a_h)$, policy emulator $\estmdp$, confidence sets $\cP$, $\taudec$-valid test policies $\Pitest_{h+1}$, parameters $\eps >0$, $\delta \in (0,1)$. 
            \State Set $\ndec \asymp \tfrac{S^2 A^2}{\eps^2} \cdot \log \tfrac{\cpush SAH \abs{\Pi}}{\eps \delta}$, $\nmc \asymp \tfrac{1}{\eps^2} \cdot \log \tfrac{\cpush SAH \abs{\Pi}}{\eps \delta}$.
            \State Sample dataset of $\ndec$ observations $\cD \sim P(\cdot \mid x_h, a_h)$. \label{line:sample-decode-dataset}
            \For {every $x^{(i)} \in \cD$}
            \hfill \algcomment{Individually decode every observation}
            \For {every $(\empobs, \empobs') \in \estmdpobsspace{h+1} \times \estmdpobsspace{h+1}$}:
            \State Estimate $\vestarg{x^{(i)}}{\pitest{\empobs}{\empobs'}} \gets \mcest(x^{(i)}, \pitest{\empobs}{\empobs'}, \nmc)$. \label{line:monte-carlo} 
            \EndFor
            \State Define:
            \begin{align*}
                \cT[x^{(i)}] \gets \crl*{\empobs \in \estmdpobsspace{h+1}: ~\forall \empobs' \ne \empobs,~\abs*{\vestarg{x^{(i)}}{\pitest{\empobs}{\empobs'}} - \wh{V}^{\pitest{\empobs}{\empobs'}}(\empobs) } \le \taudec + 2\eps}. 
            \end{align*}
            \EndFor 
            \State Define $\gobs$ as the decoder graph with \hfill \algcomment{See \pref{def:decoder-graph-obs}}  
            \begin{align*} 
                \cXL \coloneqq \cD, \quad \cXR \coloneqq \estmdpobsspace{h+1}, \quad \text{and decoder function}~\cT.
            \end{align*}
            \State \textbf{Return}: $\cP$ defined using Eq.~\pref{eq:confidence-set-construction-v2}.
        \end{algorithmic}
\end{algorithm}

We first introduce an intermediate object, called the decoder graph.

\begin{definition}[Decoder Graph]\label{def:decoder-graph-obs} Let $\cXL,\cXR \subseteq \cX$, and let $\cT: \cXL \mapsto 2^{\cXR}$ be a decoder function. The \emph{decoder graph}, denoted $\gobs$, is defined as the bipartite graph with vertices $V = \cXL \cup \cXR$ and edges $E = \crl{(x_l, \empobs_r): x_l \in \cXL, \empobs_r \in \cT[x_l]}$. 
\end{definition}
In words, the decoder graph $\gobs$ draws an edge from every observation $x_l$ sampled from the transition to observations $x_r$ sampled from the reset if the value functions for all test policies are similar. Thus, the decoder graph $\gobs$ summarizes the similarity information encoded by individually decoding each observation.

The other ingredient is a notion of \emph{pushforward distribution}, which, when supplied a distribution over observations, collapses a policy $\pi$ to a distribution over actions.

\begin{definition}[Pushforward Distribution/Policy]\label{def:pushforward-policy} Let $\nu \in \Delta(\statesp)$ be a distribution over observations. For any policy $\pi: \statesp \to \Delta(\actionsp)$, define the \emph{pushforward distribution}, denoted $\pi \push \nu \in \Delta(\cA)$, as 
\begin{align*}
    \brk*{\pi \push \nu}(a) \coloneqq \En_{x \sim \nu} \brk*{\ind{\pi(x) = a} } \quad \text{for all}~a \in \cA.
\end{align*}
For any $\pi \in \Pi$, the emission $\emission: \cS \to \Delta(\cX)$ induces a pushforward distribution; we slightly abuse notation and call the function $\pi \push \emission: \cS \to \Delta(\cA)$ the \emph{pushforward policy}. 
\end{definition}

\paragraph{Confidence Set Construction.}
Now we are ready to specify the confidence set construction of \stochdecoder{}. Denote $\crl{\cc_j}_{j \ge 1}$ to be the connected components of $\gobs$. For any $\cc \in \crl{\cc_j}_{j \ge 1}$, denote $\ccL \subseteq \cXL$ and $\ccR \subseteq \cXR$ to be the left/right observation sets respectively. In what follows, we use $p(\cdot \mid \ccR)$ to denote the conditional distribution over $\ccR$, i.e., $p(\empobs \mid \ccR) = p(\empobs)/p(\ccR) \cdot \ind{\empobs \in \ccR}$. Given a decoder graph $\gobs$ and input confidence set $ \cP(x_h, a_h)$, the updated confidence set is defined for $\beta \coloneqq \wt{O}\prn{\prn{\sqrt{SA^2} + S}\eps}$ as 
\begin{align*}
    \cP \coloneqq \Big \{p \in \cP(x_h, a_h): ~ &\sum_{\cc \in \crl{\cc_j}} \abs*{p(\ccR) - \frac{\abs{\ccL}}{\abs{\cXL}} } \le 3\eps, \quad \max_{\pi \in \Pi}~\sum_{\cc \in \crl{\cc_j}} \frac{\abs{\ccL}}{\abs{\cXL}} \cdot \nrm*{\pi \push \unif(\ccL) - \pi \push p(\cdot \mid \ccR)}_1 \le \beta \Big\}. \numberthis \label{eq:confidence-set-construction-v2} 
\end{align*}
\paragraph{Intuition for \eqref{eq:confidence-set-construction-v2}.} We give some intuition for the construction in \eqref{eq:confidence-set-construction-v2}, and refer the reader to the example in \pref{fig:confidence-set}. The high level goal is to find a set of distributions $\cP(x_h, a_h)$ supported on $\estmdpobsspace{h+1}$ such that if we plug any $\estp \in \cP(x_h, a_h)$ into our policy emulator, the policy evaluation error is bounded, i.e., 
\begin{align*}
    Q^\pi(x_h, a_h) \approx \wh{R}(x_h, a_h) + \En_{x' \sim \estp} \brk{ \wh{V}^\pi(x')}, \quad \text{for all} \quad \pi \in \Pi.
\end{align*}
In particular, we need every $\estp \in \cP$ to witness accurate policy emulation for the distribution $P = P(\cdot\mid x_h, a_h)$, so we require $\estp$ to satisfy a bound on:
\begin{align*}
    \max_{\pi \in \Pi}~ \abs*{ \En_{x' \sim P} \brk{V^\pi(x)} - \En_{x' \sim \estp}\brk{\wh{V}^\pi(x')} }.  \numberthis \label{eq:policy-emulation-transition} 
\end{align*}
Now we discuss how the constraints for $\cP$ control this policy emulation error for every $\wh{P} \in \cP$. Intuitively, the connected components $\crl{\cc_j}_{j \ge 1}$ of $\gobs$ represent a ``soft'' clustering of observations, since all observations in a given connected component $\cc \in \crl{\cc_j}_{j \ge 1}$ have similar $Q$-functions for every test policy. We further prove that this implies that the $Q$-functions are similar within $\cc$ for every $\pi \in \Pi$. Now we discuss the constraints.

\begin{figure}[!t]
    \centering
\includegraphics[scale=0.24, trim={0cm 18cm 1cm 0cm}, clip]{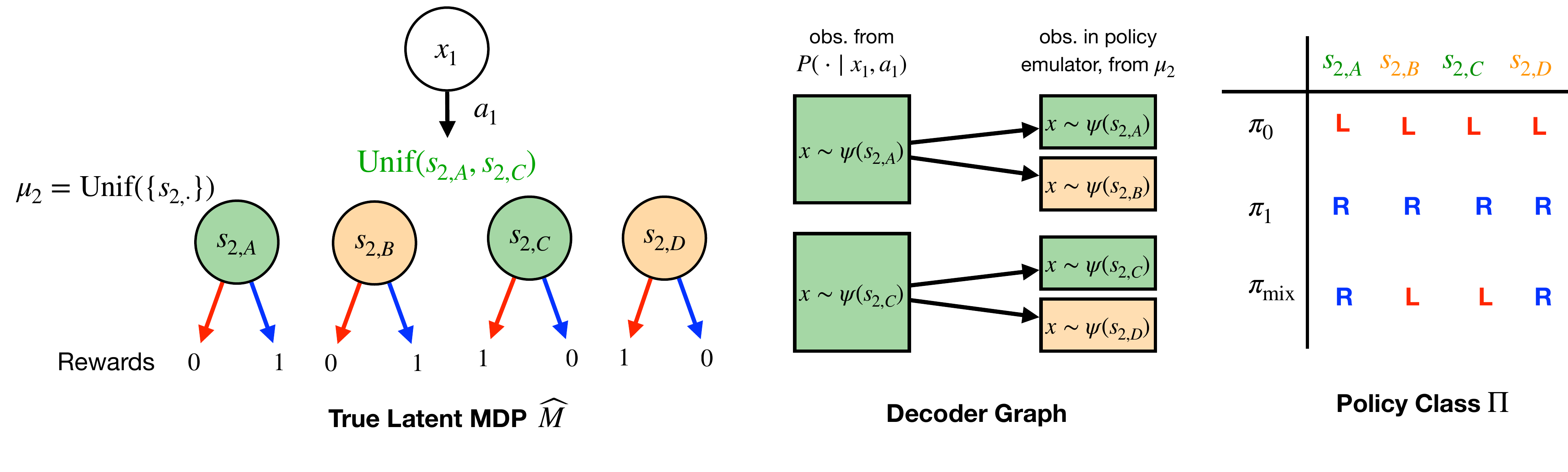} 
\caption{Confidence set construction example with $H=2$. At layer $2$, the MDP has 4 latent states, $s_{2,A}$, $s_{2,B}$, $s_{2,C}$, and $s_{2,D}$. Since $\mu_2$ has uniform mass, we sample representative observations from each latent state in our policy emulator $\estmdp$. Now consider using \stochdecoder{} to learn the transition $P(\cdot \mid x_1, a_1)$. We cannot disambiguate between observations from $s_{2,A}$ and $s_{2,B}$ via test policies (similarly for $s_{2,C}$ and $s_{2,D}$). Thus, the learned decoder graph $\gobs$ has the two connected components as shown. The \textbf{marginal constraint} enforces that every $\estp \in \cP$ must place half the mass on observations from $s_{2,A}$ and $s_{2,B}$ and the other half on observations from $s_{2,C}$ and $s_{2,D}$. This is enough to ensure that the policy evaluation error for $\pi_0$ and $\pi_1$ are controlled (cf.~Eq.~\eqref{eq:policy-emulation-transition}). However, it is not enough to ensure that policy evaluation error for $\pi_\mathrm{mix}$ is controlled, since $\pi_\mathrm{mix}$ is not constant over each connected component. As an example, consider the $\estp$ which puts uniform mass on the observations from $s_{2,B}$ and $s_{2,D}$ (the orange blocks). We have $Q^{\pi_\mathrm{mix}}(x_1, a_1) = 1$ while $\wh{Q}^{\pi_\mathrm{mix}}(x_1, a_1) = 0$. This explains why we need the \textbf{pushforward constraint}, which requires that the pushforward distribution of $\pi_\mathrm{mix}$ is matched on every connected component.} 
\label{fig:confidence-set}
\end{figure}

\begin{itemize}
    \item \textbf{Marginal Constraint}: The first condition expresses a TV distance constraint on the marginals over connected components: that is, the estimated distribution $\estp$ must place a similar amount of mass on each connected component as we observe in the samples from $P$. This ensures that for all $a \in \cA, \pi \in \Pi$:
    \begin{align*}
        \En_{x' \sim P} \brk{Q^\pi(x', a)} \approx \En_{x' \sim \estp}\brk{\wh{Q}^\pi(x', a)}.
    \end{align*}
    \item \textbf{Pushforward Constraint}: However, the marginal constraint is insufficient for accurate policy emulation because in general, policies in the given class $\Pi$ are \emph{not constant} over a given $\cc$. We give an example of this in \pref{fig:confidence-set}. To address this, we need to ensure that over each $\cc$, the pushforward distributions also match. This is precisely captured in an averaged sense by the second condition.
\end{itemize}
We show that the set of $\cP$ which satisfies both constraints yields a bound on the policy emulation error \eqref{eq:policy-emulation-transition}.

\paragraph{Technical Tool: Projected Measures.}
The technical challenge in establishing Eq.~\eqref{eq:policy-emulation-transition} is that the high-dimensional $P$ is supported on $\cX$, while we want to approximate it with $\estp$ supported on the states $\estmdpobsspace{h+1} \subseteq \cX$ of the policy emulator. To address this, we introduce a notion of \emph{projected measures} onto the state space $\estmdpobsspace{h+1}$, denoted $\projm{}: \Delta(\latentsp) \to \Delta(\estmdpobsspace{h+1})$ (see \pref{def:projected-measure} for a formal definition), which approximates $\emission \circ d$ for any distribution over latent states $d$. Using the triangle inequality on Eq.~\eqref{eq:policy-emulation-transition}, we can decompose the policy emulation error using the projected measure as an intermediary quantity: 
\begin{align*}
    \text{\eqref{eq:policy-emulation-transition}} &\le \underbrace{ \abs*{ \En_{x' \sim P} \brk{V^\pi(x')} - \En_{x' \sim \projm{}(\optlatp)} \brk{V^\pi(x')} } }_{\text{projection error}} ~+~ \underbrace{ \abs*{ \En_{x' \sim \projm{}(\optlatp)} \brk{V^\pi(x')}- \En_{x' \sim \projm{}(\optlatp) }\brk{\wh{V}^\pi(x')} } }_{\text{policy eval.~error at next layer}} \\
    &\qquad\qquad + ~ \underbrace{ \abs*{ \En_{x' \sim \projm{}(\optlatp)} \brk{\wh{V}^\pi(x')}- \En_{x' \sim \estp}\brk{\wh{V}^\pi(x')} } }_{\text{transition error}}
\end{align*}
This decomposition generalizes Eq.~\eqref{eq:confidence-set-warmup} to the stochastic BMDP setting. To obtain a bound on the projection error, we observe that pushforward concentrability implies that the observations sampled from $\mu$ are sufficiently representative of observations from the transition $P$, and therefore $\projm{}(\optlatp)$ approximates $P$ well. Similar to the analysis of \detdecoder{}, a bound on the policy evaluation error at the next layer can be shown via induction. Lastly, our analysis shows that the construction \eqref{eq:confidence-set-construction-v2} admits a bound on the transition error.

\begin{algorithm}[t]
    \caption{\stochrefit}\label{alg:stochastic-refit-v2}
        \begin{algorithmic}[1]
            \addtolength{\abovedisplayskip}{-5pt}
            \addtolength{\belowdisplayskip}{-5pt}
            \Require Layer $h$, policy emulator $\estmdp$, confidence sets $\cP$, test policies $\crl{\Pitest_h}_{h\in[H]}$, parameters $\eps > 0$ and $\delta \in (0,1)$.
            \State Set $\tauref \coloneqq 80 \cdot H \eps$, $\nmc \asymp \tfrac{1}{\eps^2} \cdot \log \tfrac{\cpush SAH \abs{\Pi}}{\eps \delta}$
            \For {every $(\empobs, \empobs') \in \estmdpobsspace{h} \times \estmdpobsspace{h}$}: \hfill \algcomment{Construct candidate test policies at layer $h$}
            \State Define $\pitest{\empobs}{\empobs'} \gets \argmax_{\pi \in \cA \circ \Pi_{h+1:H}} \abs{ \wh{V}^\pi(\empobs) - \wh{V}^\pi(\empobs') }$.
            
            \State Estimate:\hfill \algcomment{Verify accuracy of test policies}\label{line:mc1}
            \begin{align*}
                \vestarg{\empobs}{\pitest{\empobs}{\empobs'}} \gets \mcest(\empobs, \pitest{\empobs}{\empobs'}, \nmc), \quad \vestarg{\empobs'}{\pitest{\empobs}{\empobs'}} \gets \mcest(\empobs', \pitest{\empobs}{\empobs'}, \nmc)
            \end{align*}
            \EndFor
            \State Set \(\violations \gets \crl{(x,\pi) \text{~estimated in \pref{line:mc1} such that~} \abs{\vestarg{x}{\pi} - \wh{V}^\pi(x)} \ge \tauref }\). 
            \If { \(\violations = \varnothing\) } \hfill \algcomment{No violations found, so return test policies.}   
            \State Set $\Pitest_h = \cup_{x, x' \in \estmdpobsspace{h}} \crl{\pitest{\empobs}{\empobs'}}$ and \textbf{Return} $(h-1, \estmdp,\cP, \crl{\Pitest_h}_{h\in[H]})$.\label{line:great-success-stochastic}
            \Else \label{line:else-triggered} \hfill \algcomment{Refit transitions to handle violations}
            \For {every $(\empobs, \pi) \in \violations$}\label{line:for-every}
            \State \textbf{for} {each $(\bar{\empobs}, \bar{a}) \in \estmdpobsspace{h:H} \times \actionsp$}: Estimate $\qestarg{\bar{\empobs}, \bar{a}}{\pi} \gets \mcest(\bar{\empobs}, \bar{a} \circ \pi, \nmc)$. \label{line:mc2}
            \State Define for every $(\bar{\empobs}, \bar{a}) \in \estmdpobsspace{h:H} \times \actionsp$:
            \begin{align*}
                \Delta(\bar{\empobs}, \bar{a}) \coloneqq \wh{R}(\bar{\empobs},\bar{a}) + \En_{\empobs' \sim \estp(\cdot  \mid  \bar{\empobs},\bar{a})} \brk*{\qestarg{\empobs', \pi(\empobs')} {\pi} } - \qestarg{\bar{\empobs}, \bar{a}}{\pi}.
            \end{align*}
            \For {every $(\bar{\empobs}, \bar{a})$ such that $\abs{ \Delta(\bar{\empobs}, \bar{a}) } \ge \tauref/(8H)$}:\label{line:bad-obs-stochastic}

            \hspace{2.5em}\algcomment{Define loss vectors, overwriting if already defined.}

            \State Set $\ell_\mathrm{loss}(\bar{\empobs}, \bar{a}) \coloneqq \sign(\Delta(\bar{\empobs}, \bar{a})) \cdot \qestarg{\cdot , \pi(\cdot) }{\pi}\in [0,1]^{\cX_{h(\bar{\empobs})+1}[\estmdp]}$\label{line:loss-vector}
            \EndFor      
            \EndFor

            \hspace{-0.5em}\algcomment{OMD update with negative entropy Bregman Divergence on violations.}

            \State \textbf{for} every $(\bar{\empobs}, \bar{a})$ from  \pref{line:loss-vector}: Update 
            \begin{align*}
                \estp(\cdot  \mid  \bar{\empobs}, \bar{a}) \gets \argmin_{p \in \cP(\bar{\empobs}, \bar{a})}~ \tri*{p, \ell_\mathrm{loss}\prn*{\bar{\empobs}, \bar{a} } } + \frac{1}{\eps} \cdot D_\mathsf{ne}\prn*{p~\Vert~  \estp(\cdot  \mid  \bar{\empobs}, \bar{a})}
            \end{align*}\label{line:loss-vector-obs}
            \State \textbf{Return} $(\ell, \estmdp,\cP, \crl{\Pitest_h}_{h\in[H]})$ where $\ell$ is the maximum layer s.t.~$(\bar{\empobs}, \bar{a}) \in \statesp_\ell \times \cA$ was updated in \pref{line:loss-vector-obs}. \label{line:great-success-stochastic-else}
            \EndIf
        \end{algorithmic}
\end{algorithm}

\subsection{\stochrefit{} Subroutine}

Now we discuss \stochrefit{}. The skeleton is the same as in \detrefit{}: once the transition functions for $\estmdp$ have been estimated for a given layer $h$, \stochrefit{} attempts to compute a set of valid test policies $\Pitest_h$ for pairs of observations (see \pref{def:valid-test-policy}). If it cannot, this implies that at least one transition that we previously estimated in layer $h$ onward must have been incorrectly estimated, and we search for it starting in \pref{line:mc2}. In this case, we revisit the maximum layer where some transition was updated and restart the decoding procedure.

\paragraph{OMD Regret as a Potential Function.} Our main innovation to control the number of refitting iterations is to design the right potential function. In \detalg{}, since we were working with deterministic transitions, we used the size of $\cP(s,a)$ as the potential function.  Since we are now estimating $\wh{P}(\cdot \mid x,a)$ in a continuous space, this idea does not extend. 

Instead, we use the regret of online mirror descent (OMD)  against the competitor vector $\projm{}(\optlatp(\cdot \mid x,a))$ as the potential function. We show that every transition in \pref{line:loss-vector-obs} witnesses constant regret with respect to $\projm{}(P(\cdot \mid x,a))$. In our analysis, we maintain the invariant property that $\cP$ is just big enough so that $\projm{}(P(\cdot \mid x,a)) \in \cP$ throughout the execution of \stochalg{}. Therefore, the standard analysis of OMD \cite[see, e.g.,][]{bubeck2011introduction} gives us an upper bound on the cumulative regret. Letting $T_\mathrm{refit}$ denote the number of updates on a given $(x,a)$ pair, we have
\begin{align*} 
    \eps \cdot T_\mathrm{refit} \lesssim \text{Regret of OMD} \lesssim \sqrt{\log \abs{\estmdpobsspace{}} \cdot T_\mathrm{refit}}.
\end{align*}
Rearranging, we get a bound on the number of updates $T_\mathrm{refit}$ for any $(x,a)$, and since the total number of states in the policy emulator $\estmdp$ is bounded, we get a bound on the total number of updates made by \stochrefit{}.

}

\section{Discussion}\label{sec:discussion}  
Our results show interesting trade-offs between representational conditions and environment access for achieving sample-efficient policy learning. When the environment access is either the generative model or $\mu$-resets, we show lower bounds which illustrate the challenge of agnostic policy learning in MDPs with large state spaces. On the positive side, we give a new algorithm PLHR which leverages hybrid resets to efficiently learn Block MDPs; this is accomplished via a new technical tool called the policy emulator. We highlight several open problems: 

\begin{itemize}
    \item \textit{Extending the Positive Result:} Can \pref{thm:block-mdp-result} be extended to more general settings? While we establish that policy emulators of bounded size exist for pushforward coverable MDPs (\pref{app:beyond-bmdp}), we do not know how to efficiently construct them. One natural class of problems to study is the low-rank MDP, which generalizes the Block MDP and also satisfies low (pushforward) coverability. An algorithm achieving $\poly(d)$ sample complexity would showcase the power of hybrid resets, as prior work \cite{sekhari2021agnostic} shows that $\exp(d)$ sample complexity is necessary and sufficient for agnostic RL in low-rank MDPs with just online access.  
    Another direction for improving \pref{thm:block-mdp-result} is replacing the dependence on pushforward concentrability with the smaller concentrability. Unfortunately, our guarantee for \stochalg{} breaks down because it uses pushforward concentrability to enable accurate policy emulation of the transitions from every state in the emulator.
    \item \textit{Benefits of Realizability:} Is it possible to achieve positive results for the $\mu$-reset model with policy realizability (thus directly improving upon \psdp{} and contrasting with our lower bound \pref{thm:lower-bound-policy-completeness})? This question can be viewed as the policy-based analogue of the question raised by \cite{mhammedi2024power} on whether it is possible to achieve sample-efficient learning with standard online access if one assumes only coverability and value function realizability ($Q^\star \in \cF$). 
\end{itemize} 

\subsection*{Acknowledgements}
We thank Dylan Foster, Sasha Rakhlin, Zeyu Jia, Cong Ma, Nathan Srebro, and Wen Sun for helpful conversations. AS acknowledges support from ARO through award W911NF-21-1-0328, as well as Simons Foundation and the NSF through award DMS-2031883.

\bibliographystyle{alpha} 
\newcommand{\etalchar}[1]{$^{#1}$}

\clearpage
\appendix

\section{Additional Related Works}\label{app:related-works}
\paragraph{Access Models in RL.} The $\mu$-reset access setting was introduced in \cite{kakade2002approximately, kakade2003sample}, and is widely studied in the policy learning literature \cite{agarwal2021theory, brukhim2022boosting, agarwal2023variance}. We refer the reader to Appendix A of \cite{mhammedi2024power} for an exemplary survey of related works on local/global simulators, both theoretical and empirical. As a summary, in terms of theory, the study of local simulator access has mostly focused on linear function approximation settings, where it is shown that state revisiting enables one to circumvent statistical lower bounds for online RL, or enables computationally efficient approaches which are not known to exist for online RL. Generative model (or global simulator) access has mostly been studied for tabular or linear settings.

\paragraph{Algorithms for Policy Learning.} We highlight several algorithms for policy learning in large state spaces. For abstract policy classes, the predominant approaches are Policy Search by Dynamic Programming (\psdp{}) \cite{bagnell2003policy} and Conservative Policy Iteration (\cpi{}) \cite{kakade2002approximately} (see also \cite{scherrer2014approximate, scherrer2014local}). In particular, \psdp{} is a backbone of many contemporary theoretical works in RL \cite[see e.g.,][]{misra2020kinematic, uchendu2023jump, amortila2024scalable, mhammedi2023representation, mhammedi2024efficient}. Both \psdp{} and \cpi{} operate under the $\mu$-reset setting, assume policy completeness, and achieve similar guarantees (see discussion in \pref{app:psdp}). The agnostic policy learning setting (where representational conditions such as policy completeness are \emph{not} assumed) was initiated by \cite{kearns1999approximate, kakade2003sample} and has recently received more attention in the papers \cite{sekhari2021agnostic, jia2023agnostic}.

Specializing to smoothly-parameterized policy classes $\Pi = \crl{\pi_\theta}_{\theta \in \Theta}$, many works have studied policy gradient methods such as REINFORCE \cite{williams1992simple}, Policy Gradient \cite{sutton1999policy}, and Natural Policy Gradient \cite{kakade2001natural}. Empirically this has given rise to state-of-the-art algorithms for policy optimization \cite{schulman2017trustregionpolicyoptimization, schulman2017proximal}. In terms of theory, a line of work studies policy gradient methods \cite{agarwal2020pc, zanette2021cautiously, liu2024optimistic, sherman2023rate} for the restricted setting of linear MDPs \cite{jin2020provably}, designing algorithms which do not require $\mu$-reset access (note that policy completeness is naturally satisfied for linear MDPs). Going beyond linear MDPs, the papers \cite{bhandari2024global, huang2024occupancy} study policy gradient methods but require $\mu$-reset access as well some type of completeness/closure assumptions for global 
optimality guarantees.

\paragraph{Coverage Conditions.} Coverage conditions have been extensively studied in RL. In offline RL, many works study the concentrability coefficient \cite{munos2003error, munos2008finite, chen2019information, foster2021offline, jia2024offline} as well as weaker notions such as single-policy concentrability \cite{jin2021pessimism, rashidinejad2021bridging}, conditions based on value-function approximation \cite{chen2019information, xie2021bellman}, and approximate notions for continuous dynamics \cite{song2024rich}. In addition, under the $\mu$-reset model, the standard assumption made is on bounded concentrability coefficient, sometimes called the \emph{distribution mismatch coefficient} \cite{agarwal2021theory}. More recently, \cite{xie2022role} introduced the notion of \emph{coverability coefficient} and study it for standard online RL access with value function approximation. Coverability (and the related pushforward variant) is further studied in the papers \cite{amortila2024scalable, amortila2024harnessing, amortila2024reinforcement, jia2023agnostic, mhammedi2024power}.

\paragraph{Block MDPs.} Block MDPs are a canonical model for studying reinforcement learning with large state spaces but low intrinsic complexity. In particular, Block MDPs are known to satisfy low (pushforward) coverability \cite{mhammedi2024power}, implying that reset distributions exist which satisfy low (pushforward) concentrability. They have been studied in a long line of work \cite{jiang2017contextual, du2019provably, misra2020kinematic, zhang2022efficient, uehara2021representation, mhammedi2023representation}. Recently, \cite{amortila2024reinforcement} study a more general setting of RL with latent dynamics which covers the Block MDP as a special case. A common theme among these works is that standard online access to $M$ is assumed, and the assumption of \emph{decoder realizability} is made, i.e., that the learner is given access to a class $\Phi$ such that $\optdec \in \Phi$, with the achievable bounds scaling with $\log \abs{\Phi}$. Under standard online access, a minimax lower bound of $\log \abs{\Phi}$ can be obtained by reduction to supervised learning. In contrast, our work studies how to achieve sample-efficient learning without decoder realizability but with \emph{stronger forms of access} to $M$. Our bounds replace the dependence on $\log \abs{\Phi}$ (which in the worst case can scale with $\abs{\statesp}$) with dependence on $\log \abs{\Pi}$, which can be arbitrarily smaller. 

\section{Background and Additional Results for \psdp}\label{app:psdp}

In this section, we provide a description of the \psdp{} algorithm and analyze its sample complexity. We show the standard upper bound for \psdp{} which has appeared in prior works \cite[e.g.,][]{misra2020kinematic} in \pref{app:psdp-policy-completeness}. We also prove several new results about \psdp{} when only policy realizability is satisfied: namely if the reset distribution $\mu$ satisfies stronger properties beyond bounded concentrability, we show exponential in $H$ upper bounds in \pref{app:psdp-upper} as well as a matching lower bound in \pref{app:psdp-lower}. We also discuss in \pref{app:psdp-lower} how our lower bounds against \psdp{} also apply against the \cpi{} algorithm, as claimed in the main text.

\subsection{\psdp{} Guarantee Under Policy Completeness}\label{app:psdp-policy-completeness}

First, we define an averaged notion of policy completeness; compared to \pref{def:policy-completeness}, this notion is weaker since it only requires completeness to hold in an averaged sense over the reset $\mu$.

For readability, we slightly abuse notation: for $Q$-functions we denote $Q^\pi_h(x,\pi) \coloneqq Q^\pi_h(x, \pi(x))$. Similarly, we sometimes denote rewards as $R(x,\pi) \coloneqq R(x, \pi(x))$ and transitions as $P(\cdot\mid x, \pi) \coloneqq P(\cdot\mid x, \pi(x))$. 

\begin{definition}[Average Policy Completeness]\label{def:averaged-pc} Fix any policy class $\Pi$, as well as exploratory distribution $\mu = \crl{\mu_h}_{h\in[H]}$. For any layer $h \in [H]$ and policy $\estpi \coloneqq \estpi_{h+1:H} \in \Pi_{h+1:H}$ we define the (average) \emph{policy completeness error}, denoted $\epspc: \Pi_{h+1:H} \to \bbR$, as
\begin{align*}
    \epspc(\estpi) \coloneqq \min_{\pi_h \in \Pi_h} ~ \En_{x \sim \mu_h} \brk*{\max_{a \in \cA}~ Q^{\estpi}(x, a) - Q^{\estpi}(x, \pi_h)}.
\end{align*}
\end{definition}
\pref{def:averaged-pc} is similar to previously defined notions of policy completeness \cite{scherrer2014local, agarwal2023variance}. As a point of comparison, Definition 2 of \cite{agarwal2023variance} defines the average policy completeness to be the worst case over the convex hull of suffix policies $\estpi$, i.e. $\epspc \coloneqq \sup_{\estpi \in \mathsf{Conv}(\Pi)} \epspc(\estpi)$, while we define it as a function which takes as input a rollout policy $\estpi$.

We state the \psdp{} algorithm in \pref{alg:psdp} and then prove \pref{thm:psdp-ub}.

\begin{algorithm}[!htp]
\caption{\psdp{} \cite{bagnell2003policy}}\label{alg:psdp}
	\begin{algorithmic}[1]
        \Require Reset distributions $\mu = \crl{\mu_h}_{h\in [H]}$, policy class $\Pi$.
        \For {$h = H,\cdots, 1$} 
            \State Initialize dataset $\cD_h = \varnothing$.
            \For {$n$ times}: \hfill \algcomment{Collecting $(x_h, a_h, v_h)$ requires $\mu$-reset access.}
            \State Sample $(x_h, a_h)$ where $x_h \sim \mu_{h}$ and $a_h \sim \unif(\cA)$.
            \State Let $v_h \coloneqq \sum_{h'=h}^H r_{h'}$ be the value of executing $a_h \circ \estpi_{h+1:H}$ from $x_h$.
            \State Set $\cD_h \gets \cD_h \cup \crl{(x_h, a_h, v_h)}$.
        \EndFor
        \State Call CB oracle: $\estpi_h \coloneqq \argmax_{\pi \in \Pi} \frac{1}{n} \sum_{(x_h, a_h, v_h)\in \cD_h} \frac{\ind{a_h = \pi(x_h)}}{A} \cdot v_h$. \label{line:cb-oracle}

        \EndFor
        \State \textbf{Return} $\estpi_{1:H}$.
	\end{algorithmic}
\end{algorithm} 

\begin{proof}[Proof of \pref{thm:psdp-ub}]
First, we state a standard generalization bound on the contextual bandit oracle invoked in \pref{line:cb-oracle}. With probability at least $1-\delta$, for every $h \in [H]$ the returned policy $\estpi_h$ satisfies
\begin{align*}
    \En_{x \sim \mu_h} \brk*{Q^{\estpi}(x, \estpi_h)} \ge \max_{\pi_h \in \Pi_h} \En_{x \sim \mu_h}\brk*{Q^{\estpi}(x, \pi_h)} - \epsstat, \quad \text{where}\quad \epsstat \coloneqq O\prn*{\sqrt{\frac{A\log (\abs{\Pi}/\delta)}{n} }}.\numberthis\label{eq:cb-eq}
\end{align*}
For every $h \in [H]$, let us define:
\begin{align*}
    \wt{\pi}_h^\star(x) \coloneqq \argmax_{a \in \cA}~ Q^{\estpi}(x,a), \quad \text{and} \quad \wt{\pi}_h \coloneqq \argmax_{\pi_h \in \Pi_h}~ \En_{x \sim \mu_h}\brk*{Q^{\estpi}(x, \pi_h)}
\end{align*}
Then we calculate:
\begin{align*}
    V^\star - V^{\estpi} &\overset{(i)}{=} \sum_{h=1}^H \En_{x \sim d^{\optpi}_h} \brk*{Q^{\estpi}(x, \optpi) - Q^{\estpi}(x, \estpi_h) } \\
    &\overset{(ii)}{\le} \sum_{h=1}^H  \En_{x \sim d^{\optpi}_h} \brk*{Q^{\estpi}(x, \wt{\pi}_h^\star) - Q^{\estpi}(x, \estpi_h) } \\
    &\overset{(iii)}{\le} \sum_{h=1}^H \nrm*{\frac{d^{\optpi}_h}{\mu_h}}_\infty \En_{x \sim \mu_h} \brk*{ Q^{\estpi}(x, \wt{\pi}_h^\star) - Q^{\estpi}(x, \estpi_h) } \\
    &\overset{(iv)}{\le} \cconc \cdot \sum_{h=1}^H \En_{x \sim \mu_h} \brk*{ Q^{\estpi}(x, \wt{\pi}_h^\star) - Q^{\estpi}(x, \estpi_h) }  \\
    &= \cconc \cdot \sum_{h=1}^H \prn*{ \En_{x \sim \mu_h} \brk*{ Q^{\estpi}(x, \wt{\pi}_h^\star) - Q^{\estpi}(x, \wt{\pi}_h) } + \En_{x \sim \mu_h} \brk*{ Q^{\estpi}(x, \wt{\pi}_h) - Q^{\estpi}(x, \estpi_h) } }\\
    &\overset{(v)}{\le} H \cconc \epsstat + \cconc \sum_{h=1}^H  \epspc(\estpi_{h+1:H}).
\end{align*}
Here, $(i)$ follows by the Performance Difference Lemma, $(ii)$ is due to the optimality of $\wt{\pi}^\star_h$, $(iii)$ is due to nonnegativity of $Q^{\estpi}(x, \wt{\pi}_h^\star) - Q^{\estpi}(x, \estpi_h)$, $(iv)$ is due to the definition of $\cconc$, and $(v)$ follows by \pref{def:averaged-pc} and Eq.~\eqref{eq:cb-eq}. Therefore, if the policy completeness error is zero, then we have a bound which is at most $H \cconc \epsstat$, and therefore \psdp{} returns an $\eps$-optimal policy using $\poly(\cconc, A, H, \log\abs{\Pi}, \eps^{-1}, \log \delta^{-1})$ samples.
\end{proof}

\subsection{Upper Bounds for \psdp{} with Policy Realizability}\label{app:psdp-upper}
As shown by the example in \pref{fig:psdp-lower-bound-simple}, without policy completeness, \psdp{} may not even be consistent, since one can take $\gamma$ to be arbitrarily close to 0 so that with constant probability \psdp{} returns a $(1+\gamma)$-suboptimal policy. In this section, we circumvent the lower bound and show that if we make stronger assumptions on the reset distribution $\mu$, \psdp{} achieves consistency:
\begin{enumerate}
    \item If $\Pi$ is realizable and the reset $\mu$ has bounded pushforward concentrability, \pref{thm:psdp-ub-pushforward} achieves $(\cpush)^{O(H)}$ sample complexity. 
    \item If $\Pi$ is realizable and the reset $\mu$ is admissible (\pref{def:admissible}) and has bounded concentrability, \pref{thm:psdp-ub-admissible} achieves $(\cconc)^{O(H)}$ sample complexity.
\end{enumerate}
The two upper bounds are in general incomparable, as there exist settings in which one achieves a better guarantee than the other. In addition, to the best of our knowledge, neither result is implied by any known results for policy learning---note that the trivial bound of $A^H$ achieved by importance sampling \cite{kearns1999approximate, agarwal2019reinforcement} can be much larger when $\cpush \ll A$.

\subsubsection{Policy Realizability + Pushforward Concentrability}

\begin{theorem}\label{thm:psdp-ub-pushforward}
Suppose $\Pi$ is realizable, and the reset $\mu$ satisfies pushforward concentrability with parameter $\cpush$. With high probability, \psdp{} returns an $\eps$-optimal policy using \begin{align*}
    \poly((\cpush)^{H}, A, \log \abs{\Pi}, \eps^{-1}) \quad \text{samples.}
\end{align*}
\end{theorem}

The proof relies on the following lemma, which relates the policy completeness error to the pushforward concentrability coefficient of $\mu$.

\begin{lemma}\label{lem:pc-bound-pushforward}
Fix any layer $h \in [H]$. For any suffix policy $\estpi_{h+1:H}$ we have
\begin{align*}
    \epspc(\estpi_{h+1:H}) \le  \cpush \cdot \En_{x'\sim \mu_{h+1}} \brk*{ V^\star(x') - V^{\estpi}(x') }. 
\end{align*}
\end{lemma}

\begin{proof}[Proof of \pref{lem:pc-bound-pushforward}] 
We have the following computation: 
\begin{align*}
    \epspc(\estpi_{h+1:H}) &= \min_{\pi_h \in \Pi_h}~\En_{x \sim \mu_h} \brk*{\max_{a \in \cA}~ Q^{\estpi}(x,a) - Q^{\estpi}(x, \pi_h)} \\
    &\le \En_{x \sim \mu_h} \brk*{\max_{a \in \cA}~ Q^{\estpi}(x,a) - Q^{\estpi}(x, \optpi)} \\
    &\le \En_{x \sim \mu_h} \brk*{ Q^{\star}(x,\optpi) - Q^{\estpi}(x, \optpi)} \\
    &= \En_{x \sim \mu_h} \brk*{ r(x,\optpi) + \En_{x' \sim P(\cdot \mid x,\optpi)} V^\star(x')}   - \En_{x \sim \mu_h} \brk*{r(x, \optpi) + \En_{x' \sim P(\cdot \mid x,\optpi)} V^{\estpi}(x')} \\
    &= \En_{x \sim \mu_h, x' \sim P(\cdot \mid x , \optpi)} \brk*{ V^\star(x') - V^{\estpi}(x')}. \numberthis\label{eq:ub-policy-completeness-error} 
\end{align*}
The first inequality is due to the realizability $\optpi \in \Pi$, and the second one is due to the optimality of $\optpi$. 
Now we will perform a change of measure to relate the bound in Eq.~\eqref{eq:ub-policy-completeness-error} to the error of $\estpi$ on the layer $h+1$.
\begin{align*}
    \En_{x \sim \mu_h, x' \sim P(\cdot \mid x, \optpi)} \brk*{ V^\star(x') - V^{\estpi}(x')} &= \En_{x' \sim \mu_{h+1}} \brk*{ \frac{\En_{x \sim \mu_h} P(x' \mid x, \optpi)}{\mu_{h+1}(x')} \cdot \prn*{ V^\star(x') - V^{\estpi}(x')} } \\
    &\le \cpush \cdot \En_{x' \sim \mu_{h+1}} \brk*{  V^\star(x') - V^{\estpi}(x') }, 
\end{align*}
where the inequality uses the nonnegativity of $V^\star(x') - V^{\estpi}(x')$ and the definition of pushforward concentrability. Plugging this back into Eq.~\eqref{eq:ub-policy-completeness-error} proves \pref{lem:pc-bound-pushforward}.
\end{proof}

\begin{proof}[Proof of \pref{thm:psdp-ub-pushforward}]
Using Performance Difference Lemma we have for the learned policy $\estpi \in \Pi$:
\begin{align*}
    V^\star - V^{\estpi} &= \sum_{h=1}^H \En_{x \sim d^{\estpi}_h} \brk*{V^\star(x) - Q^{\star}(x, \estpi_h) } = \sum_{h=1}^H \En_{x \sim \mu_h} \brk*{\frac{d^{\estpi_h}(x)}{\mu_h(x)}\prn*{V^\star(x) - Q^{\star}(x, \estpi_h) }}\\
    &\le \cconc \cdot \sum_{h=1}^H \En_{x \sim \mu_h} \brk*{V^\star(x) - Q^\star(x, \estpi)} \le  \cconc \cdot \sum_{h=1}^H \En_{x \sim \mu_h} \brk*{V^\star(x) - V^{\estpi}(x)}.
\end{align*}
The first inequality uses the fact that $\estpi \in \Pi$ as well as $V^\star(x) \ge Q^{\star}(x, \estpi_h)$, and the second inequality uses the latter fact again. From here, we apply an inductive argument to bound the suboptimality $\En_{x \sim \mu_h} \brk{V^\star(x) - V^{\estpi}(x)}$ for all $h \in [H]$. Fix any $h \in [H]$. We have
\begin{align*}
    \En_{x \sim \mu_h} \brk*{V^\star(x) -  V^{\estpi}(x)}
    &= \En_{x \sim \mu_h} \brk*{Q^\star(x, \optpi) -  Q^{\estpi}(x,\estpi)}\\
    &\le \En_{x \sim \mu_h} \brk*{Q^\star(x, \optpi) - Q^{\estpi}(x,\optpi) + \max_{a} Q^{\estpi}(x,a) -  Q^{\estpi}(x,\estpi)} \\
    &= \En_{x \sim \mu_h, x' \sim P(\cdot \mid x,\optpi)} \brk*{V^\star(x') - V^{\estpi}(x')} + \En_{x \sim \mu_h} \brk*{\max_{a} Q^{\estpi}(x,a) - Q^{\estpi}(x,\estpi)} \\
    &\le \cpush \En_{x' \sim \mu_{h+1}} \brk*{V^\star(x') - V^{\estpi}(x')} + \epsstat + \epspc(\estpi_{h+1:H}) \\
    &\le 2C_\mathrm{push} \cdot \En_{x' \sim \mu_{h+1}} \brk*{ V^\star(x') - V^{\estpi}(x') } + \epsstat, \numberthis\label{eq:recursion-1}
\end{align*}
where the last inequality uses \pref{lem:pc-bound-pushforward}. Recursive application of \pref{eq:recursion-1} and the fact that $\En_{x \sim \mu_H} \brk{V^\star(x) - V^{\estpi}(x)} = \En_{x \sim \mu_H} \brk{r(x, \optpi) - r(x, \estpi_H)} \le \epsstat$ gives us
\begin{align*}
    \En_{x \sim \mu_h} \brk*{V^\star(x) -  V^{\estpi}(x)} \le H \cdot (2\cpush)^{H} \epsstat,
\end{align*}
so therefore the final suboptimality of \psdp{} is at most
\begin{align*}
    V^\star - V^{\estpi} \le \cconc \cdot \sum_{h=1}^H \En_{x \sim \mu_h} \brk*{V^\star(x) - V^{\estpi}(x)} \le H^2 \cdot (2 \cpush)^{H+1} \epsstat.
\end{align*}
Choosing $n = \poly((\cpush)^{H}, A, \log \abs{\Pi}, \eps^{-1})$ so that the right hand side is at most $\eps$ proves the final bound.
\end{proof}

\subsubsection{Policy Realizability + Admissibility + Concentrability}

\begin{definition}
\label{def:admissible}
We say a distribution $\mu$ is admissible if for every $h \in [H]$ there exists some $\pi_b \in \Delta(\Pi)$:
\begin{align*}
    \mu_h(x) = d^{\pi_b}_h(x) \quad \text{for all}~x \in \statesp_h.
\end{align*}

\end{definition}

\begin{theorem}\label{thm:psdp-ub-admissible}
Suppose $\Pi$ is realizable, and the reset $\mu$ (1) satisfies concentrability with parameter $\cconc$, and (2) is admissible. With high probability, \psdp{} finds an $\eps$-optimal policy using $\poly((\cconc)^{H}, A, \log \abs{\Pi}, \eps^{-1})$ samples.
\end{theorem}

To prove \pref{thm:psdp-ub-admissible}, we first establish a few helper lemmas on the errors of the learned policy $\estpi$.

\begin{lemma}\label{lem:transfer} For any layer $h \in [H]$ and admissible distribution $\nu \in \Delta(\statesp_h)$, we have 
\begin{align*}
    \max_{\pi \in \Pi_h}~\En_{x \sim \nu} \brk*{Q^{\estpi}(x, \pi) - Q^{\estpi}(x, \estpi)} \le \cconc \prn*{\epsstat + \epspc(\estpi_{h+1:H})}.
\end{align*}
\end{lemma}

\begin{proof}
We calculate that
\begin{align*}
    \hspace{2em}&\hspace{-2em} \max_{\pi \in \Pi_h}~ \En_{x \sim \nu} \brk*{Q^{\estpi}(x, \pi) - Q^{\estpi}(x, \estpi_h)} \\
    &= \max_{\pi \in \Pi_h} \En_{x \sim \nu} \brk*{Q^{\estpi}(x, \pi) - \max_a Q^{\estpi}(x,a) } + \En_{x \sim \nu} \brk*{\max_a Q^{\estpi}(x,a) - Q^{\estpi}(x, \estpi_h)} \\
    &\le \cconc \cdot \En_{x \sim \mu_h} \brk*{\max_a Q^{\estpi}(x,a) - Q^{\estpi}(x, \estpi_h)} \\
    &= \cconc \cdot \Bigg( \En_{x \sim \mu_h} \brk*{\max_a Q^{\estpi}(x,a)} - \max_{\pi \in \Pi_h} \En_{x \sim \mu_h} \brk*{ Q^{\estpi}(x, \pi)}  \\
    &\hspace{8em} + \max_{\pi \in \Pi_h} \En_{x \sim \mu_h} \brk*{Q^{\estpi}(x, \pi)} - \En_{x \sim \mu_h} \brk*{Q^{\estpi}(x, \estpi_h)} \Bigg) \\
    &\le \cconc \prn*{\epsstat + \epspc(\estpi_{h+1:H})}.
\end{align*}
In the first inequality we use the fact that $\nu$ is admissible, so we can use concentrability to relate the density ratios $\nrm{\nu/\mu}_\infty$.
\end{proof}

\emph{Additional Notation.} In the subsequent analysis, for any distribution $\nu$ we denote $\epspc(\estpi, \nu)$ to be the policy completeness error under distribution $\nu$, i.e.,
\begin{align*}
    \epspc(\estpi, \nu) \coloneqq \min_{\pi_h \in \Pi_h} ~ \En_{x \sim \nu} \brk*{\max_{a \in \cA}~ Q^{\estpi}(x, a) - Q^{\estpi}(x, \pi)}.
\end{align*}
For any partial policy $\pi_{h:t-1}$, we also denote $\nu \circ \pi_{h:t-1} \in \Delta(\statesp_{t})$ to denote the distribution over states in layer $t$ which is achieved by first sampling a state $x_h \sim \nu$ then rolling out with partial policy $\pi_{h:t-1}$.

\begin{lemma}\label{lem:pc-bound}
For any layer $h \in [H]$ and admissible distribution $\nu \in \Delta(\statesp_h)$, we have
\begin{align*}
    \epspc(\estpi_{h+1:H}, \nu) \le (H- h) \cdot \cconc \epsstat + \cconc \cdot \sum_{h' = h+1}^H \epspc(\estpi_{h'+1:H}) 
\end{align*}
\end{lemma}

\begin{proof}
Using the definition of policy completeness we have
\begin{align*}
    \epspc(\estpi_{h+1:H}, \nu) &\le \En_{x \sim \nu} \brk*{\max_a Q^{\estpi}(x,a) - Q^{\estpi}(x, \optpi)} \le \En_{x \sim \nu} \brk*{ Q^{\star}(x,\optpi) - Q^{\estpi}(x, \optpi)}. 
\end{align*}
Now, we apply a recursive argument, which gives us 
\begin{align*}
     \epspc(\estpi_{h+1:H}, \nu) &\le \En_{x \sim \nu} \brk*{ Q^{\star}(x,\optpi) - Q^{\estpi}(x, \optpi)} \\
    &= \En_{x' \sim \nu \circ \optpi} \brk*{ Q^{\star}(x',\optpi) - Q^{\estpi}(x', \estpi)} \\
    &= \En_{x' \sim  \nu \circ \optpi} \brk*{ Q^{\star}(x',\optpi) - Q^{\estpi}(x', \optpi) + Q^{\estpi}(x', \optpi) - Q^{\estpi}(x', \estpi) }
\end{align*}
Because $\nu$ is admissible, so is $\nu \circ \optpi$. Therefore, the second term in the sum is bounded using \pref{lem:transfer}:
\begin{align*}
    \En_{x' \sim  \nu \circ \optpi} \brk*{ Q^{\estpi}(x', \optpi) - Q^{\estpi}(x', \estpi) } &\le \cconc \prn*{ \epspc(\estpi_{h+2:H}) + \epsstat }.
\end{align*}
The first term in the sum can be rewritten as
\begin{align*}
    \En_{x' \sim  \nu \circ \optpi} \brk*{ Q^{\star}(x',\optpi) - Q^{\estpi}(x', \optpi)} = \En_{{x''} \sim  \nu \circ \optpi \circ \optpi} \brk*{ Q^{\star}(x'',\optpi) - Q^{\estpi}(x'', \estpi)}.
\end{align*}
Applying recursion, we get the final bound of
\begin{align*}
    \epspc(\estpi_{h+1:H}, \nu) &\le (H - h) \cdot \cconc \epsstat + \cconc \cdot \sum_{h' = h+1}^H \epspc(\estpi_{h'+1:H}).
\end{align*}
This concludes the proof of \pref{lem:pc-bound}.
\end{proof}

\begin{proof}[Proof of \pref{thm:psdp-ub-admissible}]
We compute the suboptimality as
\begin{align*}
    V^\star - V^{\estpi} &= \sum_{h=1}^H \En_{x \sim d^{\star}_h} \brk*{Q^{\estpi}(x, \optpi) - Q^{\estpi}(x, \estpi) } &\text{(Performance Difference Lemma)} \\
    &\le H \cconc \epsstat + \cconc \prn*{ \sum_{h=1}^H \epspc(\estpi_{h+1:H}) }. &\text{ (\pref{lem:transfer})}
\end{align*}
Now we apply \pref{lem:pc-bound} to show that the policy completeness error can be bounded by the downstream policy completeness errors, using the admissibility of $\mu$.
\begin{align*}
    V^\star - V^{\estpi} &\le H \cconc \epsstat +  \cconc \sum_{h=1}^H \epspc(\estpi_{h+1:H}) \\
    &= H \cconc \epsstat +  \cconc \cdot \prn*{\epspc(\estpi_{2:H}) + \sum_{h=2}^H \epspc(\estpi_{h+1:H}) } \\
    &\le H \cconc \epsstat +  \cconc \cdot \prn*{\prn{H - 1} \cconc \epsstat + (1+ \cconc) \cdot \sum_{h=2}^H \epspc(\estpi_{h+1:H}) } &\text{(\pref{lem:pc-bound})}\\
    &\lesssim H (\cconc)^2 \epsstat + \prn{1+\cconc}^2 \cdot \prn*{ \sum_{h=2}^H \epspc(\estpi_{h+1:H}) } \\
    &= H (\cconc)^2 \epsstat + \prn{1+\cconc}^2 \cdot \prn*{  \epspc(\estpi_{3:H}) + \sum_{h=3}^H \epspc(\estpi_{h+1:H}) } \\
    &\le H (\cconc)^2 \epsstat \\
    &\hspace{3em} + \prn{1+\cconc}^2 \cdot \prn*{  \prn{H - 2} \cconc \epsstat + \prn{1+\cconc} \sum_{h=3}^H \epspc(\estpi_{h+1:H}) } &\text{(\pref{lem:pc-bound})}\\
    &\lesssim H \prn*{1+\cconc}^3 \epsstat + \prn{1+\cconc}^3 \prn*{ \sum_{h=3}^H \epspc(\estpi_{h+1:H}) }.
\end{align*}
Continuing this way (and observing that $\epspc(\estpi_{H+1} = \varnothing) = 0$) we get a final bound of
\begin{align*}
    V^\star - V^{\estpi} \lesssim H \prn{1+\cconc}^H \epsstat.
\end{align*}
Setting $n = \poly((\cconc)^{H}, A, \log \abs{\Pi}, \eps^{-1})$ makes the RHS at most $\eps$, thus proving \pref{thm:psdp-ub-admissible}.
\end{proof}

\subsection{Lower Bounds for \psdp{} and \cpi{}}\label{app:psdp-lower}
Now we will show that exponential error compounding is unavoidable for \psdp{} in the absence of policy completeness. \psdp{} relies on a reduction to a contextual bandit oracle. For the lower bound statement, we will assume that $\epsstat > 0$ is a fixed constant and \psdp{} is equipped with a \emph{worst case} oracle $\mathsf{CB}_{\epsstat}$ which for every layer $h \in [H]$ always returns an \emph{arbitrary policy} $\estpi_h$ satisfying
\begin{align*}
    \En_{x \sim \mu_h} \brk*{Q^{\estpi}(x, \estpi_h)} \ge \max_{\pi \in \Pi} \En_{x \sim \mu_h}\brk*{Q^{\estpi}(x, \pi)} - \epsstat.
\end{align*}
Thus, the lower bound statement has the flavor of a statistical query lower bound \cite{kearns1998efficient}, which also assumes a worst-case response up to accuracy $\epsstat$.

\begin{theorem}\label{thm:psdp-lower-bound}
Let $H \ge 2$. Fix any $\epsstat > 0$ and parameter $\cpush \ge 5H$. There exists a tabular MDP $M$ with $S = O(H^2)$ states, $A = O(H)$ actions, and horizon $H$, realizable policy class $\Pi$ of size $2^{\wt{O}(H)}$, and exploratory distribution $\mu$ which is admissible and satisfies pushforward concentrability with parameter $\cpush$, so that \psdp{} equipped with oracle $\mathsf{CB}_{\epsstat}$ returns a policy $\estpi$:
\begin{align*}
    V^\star - V^{\estpi} \ge (\cpush)^{\Omega(H)} \epsstat.
\end{align*}
\end{theorem}

\pref{thm:psdp-lower-bound} is a converse to the positive results of \pref{thm:psdp-ub-pushforward} and \ref{thm:psdp-ub-admissible}, showing that \psdp{} can have exponential in $H$ sample complexity. The lower bound construction in \pref{thm:psdp-lower-bound} as well as the earlier one from \pref{fig:psdp-lower-bound-simple} are given by tabular MDPs, which our main upper bound in this paper (\pref{thm:block-mdp-result}) can solve with polynomial number of samples with $\mu$-reset access. (Note that when the state space $\cX$ itself is bounded in size, \stochalg{} does not require local simulator access because it can perform resets directly using rejection sampling from $\mu$.) Thus, \pref{thm:psdp-lower-bound} indicates an \emph{algorithmic limitation} of using dynamic programming to solve policy learning.

We also remark that the constructions in \pref{fig:psdp-lower-bound-simple} and \pref{thm:psdp-lower-bound} also apply to \cpi{} \cite{kakade2002approximately}; we refer the reader to \cite[Section 14 of][]{agarwal2019reinforcement} for an exposition of the \cpi{} algorithm. At a high level, the \cpi{} algorithm generates a sequence of policy iterates $\pi^{(1)}, \pi^{(2)}, \cdots$ such that each policy iterate improves upon the previous one and terminates whenever:
\begin{align*}
    \max_{\wt{\pi} \in \Pi}~\En_{x \sim d^{\pi^{(t)}}_\mu} \brk*{Q^{\pi^{(t)}}(x, \wt{\pi}) - Q^{\pi^{(t)}}(x, \pi^{(t)})} \le \eps.
\end{align*}
where $d^{\pi^{(t)}}_\mu$ is the occupancy measure obtained by running the current iterate $\pi^{(t)}$ starting from the reset $\mu$ and $\eps> 0$ is some predefined threshold which represents the accuracy to which \cpi{} solves the policy improvement problem. Thus, if it is not possible to greatly improve (by at least $\eps$) the average $Q$-function by selecting a different policy $\wt{\pi}$, then \cpi{} will terminate. In our constructions, one can check that if we initialize to the all-zeros policy $\pi^{(1)} \equiv 0$, then \cpi{} will terminate immediately even though $\pi^{(1)}$ has constant suboptimality.

In the rest of this section, we will prove \pref{thm:psdp-lower-bound}.

\subsubsection{Lower Bound Construction}
\begin{figure}[!t]
    \centering
\includegraphics[scale=0.31, trim={0.4cm 0cm 15cm 0cm}, clip]{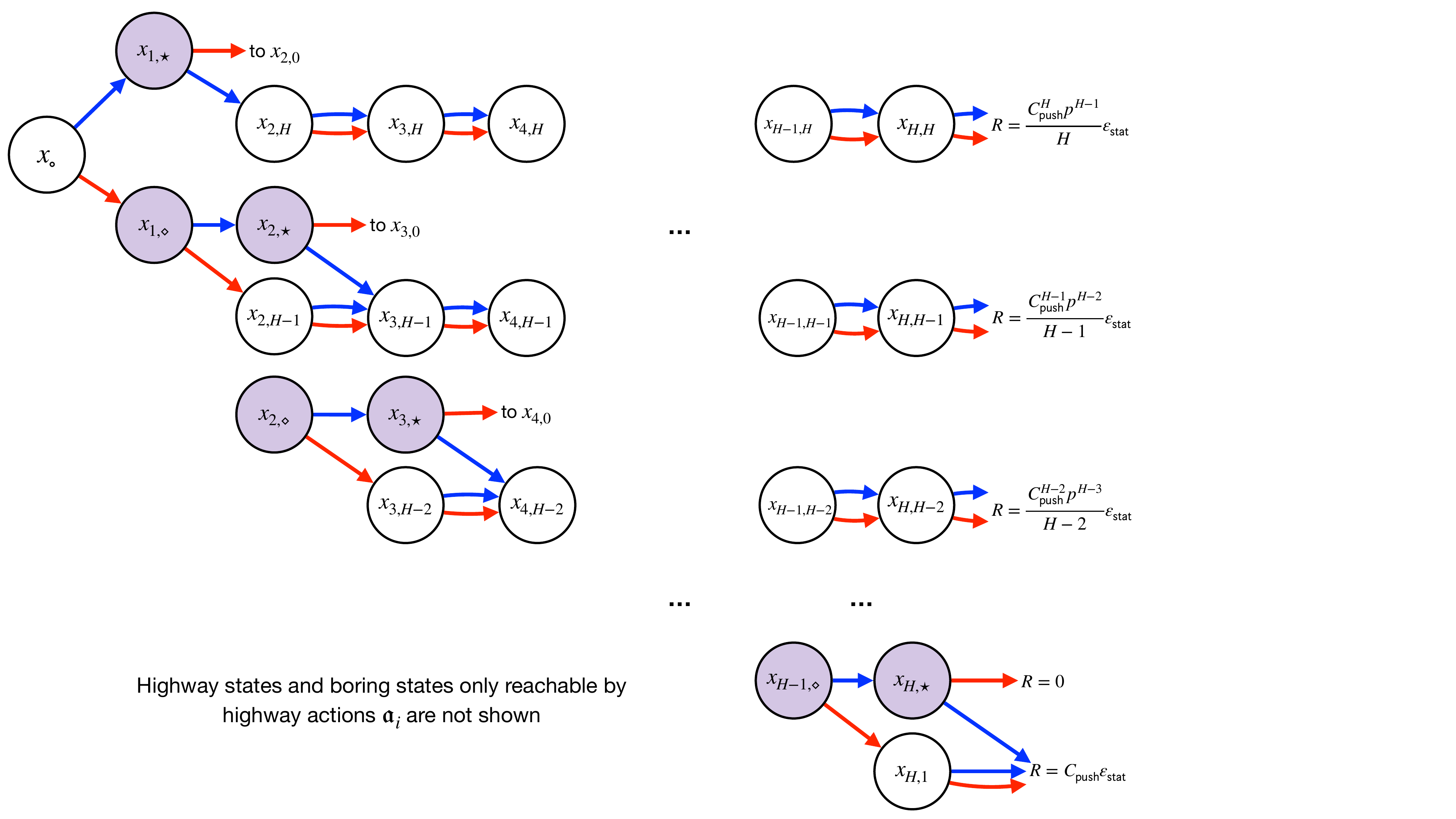}
    \caption{Lower bound construction for \pref{thm:psdp-lower-bound}. To avoid clutter, we do not illustrate the highway states as well as any boring states which are only reachable by taking highway actions $\mathfrak{a}_i$ at layer 0, since their role is only to make sure that the construction satisfies admissibility.} 
    \label{fig:psdp-lb-2}
\end{figure}

Our lower bound construction is illustrated in \pref{fig:psdp-lb-2}.

For notational convenience, we number the layers starting with $h=0$, so that there are $H+1$ layers.

\paragraph{State and Action Spaces.} At $h=0$ there is a single state $x_\circ$ and the action set is $\cA_0 = \crl{0, 1, \mathfrak{a}_1, \cdots, \mathfrak{a}_H}$. For $h \ge 1$, we have 
\begin{align*}
    \statesp_h = \underbrace{ \crl*{x_{h,0}, x_{h,1} \cdots x_{h,H}} }_{\text{$H+1$ boring states}} \cup \underbrace{\crl*{x_{h, \diamond}} \cup  \crl*{x_{h, \star}}  }_{\text{2 special states}} \cup  \underbrace{ \crl*{\bar{x}_{h \to h+1}, \bar{x}_{h\to h+2}, \cdots \bar{x}_{h \to H}} }_{\text{$H-h$ highway states}},
\end{align*}
except for $\cX_H$ which does not have the special state $x_{h, \diamond}$. The action set is $\cA_h = \crl{0,1}$.

\paragraph{Policy Class.} The policy class $\Pi$ is taken to be all open-loop policies over each layer's action space:
\begin{align*}
    \Pi \coloneqq \crl*{\pi: \forall x \in \statesp_h, \pi_h(x) \equiv a_h, (a_0, a_1, \cdots, a_H) \in \prod_{h=0}^H \cA_h}.
\end{align*}

\paragraph{Reset Distribution.} At layer $h=0$ we have $d_0 = \mu_0 = \delta_{x_\circ}$. At layer $h \ge 1$, the distribution $\mu_h$ puts $1/\cpush$ mass on each of the non-diamond states $ \crl{x_{h,0}, x_{h,1} \cdots x_{h,H}} \cup \crl{x_{h, \star}}  \cup \crl{\bar{x}_{h \to h+1}, \bar{x}_{h\to h+2}, \cdots \bar{x}_{h \to H}}$ and the rest on $x_{h, \diamond}$. Therefore $x_{h, \diamond}$ has mass \emph{at least} $p \coloneqq \prn{1 - \frac{2H+1}{\cpush}}$. We have $p > 1/2$ as long as $\cpush \ge 5H$.

\paragraph{Transitions.} At $h=0$, we have
\begin{align*}
    \Pr(\cdot \mid x_\circ, a) = \begin{cases}
        \delta_{x_{1, \diamond}} &\text{if}~a = 0\\
        \delta_{x_{1, \star}} &\text{if}~a = 1\\
        \mu_1 &\text{if}~a = \mathfrak{a}_{1}\\
        \delta_{\bar{x}_{1\to h'}} &\text{if}~a = \mathfrak{a}_{h'}, h' \ge 2.\\
    \end{cases}
\end{align*}
For $h \ge 1$, we have:
\begin{itemize}
    \item \textit{Boring States:} At the boring state $x_{h,i}$, we always transit to the corresponding boring state in the next layer $x_{h+1, i}$ regardless of the action.
    \item \textit{Highway States:} On the highway state $\bar{x}_{h \to h+1}$, we transit to $\mu_{h+1}$ regardless of the action. On highway states $\bar{x}_{h \to h'}$ for $h' > h+1$ we transit to $\bar{x}_{h+1, h'}$ regardless of the action.
    \item \textit{Special States:} We have
    \begin{align*}
        P(\cdot \mid x_{h, \diamond}, a) = \begin{cases}
            \delta_{x_{h+1, H-h}} &\text{if}~a= 0\\
            \delta_{x_{h+1, \star}} &\text{if}~a = 1,
        \end{cases} \quad \text{and} \quad P(\cdot \mid x_{h, \star}, a) = \begin{cases}
            \delta_{x_{h+1, 0}} &\text{if}~a= 0\\
            \delta_{x_{h+1, H-h+1}} &\text{if}~a = 1.
        \end{cases}
    \end{align*}
\end{itemize}

\paragraph{Rewards.} All the rewards are at layer $H$:
\begin{align*}
    R(x_{H, \star}, 0) = 0, \quad R(x_{H, \star}, 1) = \cpush \epsstat, \quad \text{and} \quad \forall~i \in \crl{0, 1, \cdots, H}:~ R(x_{H, i}, \cdot) = 
    \frac{\cpush^{i} p^{i-1}}{i} \epsstat,
\end{align*}

\paragraph{Properties of the Construction.} Now we list several properties of the construction which are more or less immediate to verify.
\begin{enumerate}[(1)]
    \item The state space is of size $O(H^2)$, the action space is of size $O(H)$, and the policy class is of size $(H+2) \cdot 2^{H}$.
    \item Due to the transitions for the highway states, the distribution $\mu$ is admissible at all layers $h \ge 0$.
    \item The minimum probability that $\mu_h$ places on any state $x \in \statesp_h$ is at least $\min\crl{1/\cpush, p} \ge 1/\cpush$, so therefore pushforward concentrability is satisfied with parameter $\cpush$.
    \item The optimal policy is the all-ones policy, $\optpi_h \equiv 1$ for all $h \ge 0$.  Therefore $\Pi$ is realizable.
\end{enumerate}

\subsubsection{Analysis of \psdp{}}
We will show inductively that \psdp{} returns the all-zeros policy $\estpi_h \equiv 0$ for all $h \in [H]$.
\begin{itemize}
    \item At layer $H$, the only state for which the value of taking $a_H=0$ and $a_H=1$ differ is on $x_{H, \star}$, which is sampled under $\mu_H$ with probability $1/\cpush$. The gap between values is $\En_{x \sim \mu_H} \brk*{r(x,1) - r(x,0)} = \epsstat$, so we set $\mathsf{CB}_{\epsstat}$ to return $\estpi_H \equiv 0$.
    \item At layer $H-1$, the two states for which there is a gap in value are the special states $x_{H-1, \diamond}$ and $x_{H-1, \star}$. We can compute that
    \begin{align*}
        \En_{x \sim \mu_{H-1}} \brk*{Q^{\estpi_H}(x, 0) - Q^{\estpi_H}(x, 1)} \ge p \cdot \cpush \epsstat - \frac{1}{\cpush} \cdot \frac{\cpush^2 p\epsstat}{2} = \frac{\cpush p \epsstat}{2} > \epsstat.
    \end{align*}
    Here we use the fact that $\mu_{H-1}(x_{H-1, \diamond}) \ge p > 1/2$, as well as the assumed lower bound on $\cpush$. Therefore, $\mathsf{CB}_{\epsstat}$ must return $\estpi_{H-1} \equiv 0$.
    \item Suppose we are at layer $h$ and for all $h' > h$ \psdp{} selects $\estpi_{h'} \equiv 0$. Then the gap in value between action 0 and action 1 is
    \begin{align*}
        \En_{x \sim \mu_{h}} \brk*{Q^{\estpi_{h+1:H}}(x, 0) - Q^{\estpi_{h+1:H}}(x, 1)} &\ge p \cdot \frac{\cpush^{H-h} p^{H-h-1} \epsstat}{H-h} - \frac{1}{\cpush} \cdot \frac{\cpush^{H-h+1} p^{H-h} \epsstat}{H-h+1} \\
        &= \frac{\cpush^{H-h} p^{H-h} \epsstat}{(H-h)(H-h+1)} > \epsstat.
    \end{align*}
    The last equality uses the fact that $\cpush p \ge 5H/2$.
    \item Continuing this way, we can see that for all $h \ge 1$, \psdp{} equipped with $\mathsf{CB}_{\epsstat}$ selects $\estpi_h \equiv 0$. We can calculate that:
    \begin{align*}
        Q^{\estpi_{1:H}}(x_\circ, a) \begin{cases}
            = 0 &\text{if}~a=1\\[1em]
            = \frac{\cpush^{H-1}p^{H-2}}{H} \epsstat &\text{if}~a=0\\[1em]
            \le \prn*{ \frac{\cpush^{H-1}p^{H-2}}{H} + \frac{\cpush^{H-1}p^{H-1}}{H-1} } \epsstat \le \frac{2\cpush^{H-1}p^{H-1}}{H-1} \epsstat.  &\text{if}~a = \mathfrak{a_i}~ \text{for any}~i \in [H]
        \end{cases}
    \end{align*}
    For the last case, we use the rough estimate that $\mu_h$ places $1/\cpush$ mass on $x_{h,H}$ and the rest elsewhere.
\end{itemize}
Plugging in the optimal value $V^\star$ we have that the suboptimality of \psdp{} is at least
\begin{align*}
    V^\star - V^{\estpi} \ge \prn*{\frac{\cpush^H p^{H-1}}{H} - \frac{2\cpush^{H-1} p^{H-1}}{H-1}} \epsstat = (\cpush)^{\Omega(H)} \epsstat.
\end{align*}
This concludes the proof of \pref{thm:psdp-lower-bound}. \qed

\clearpage
\section{Existence of Emulators Under Pushforward Coverability}\label{app:beyond-bmdp}
A natural question to ask is how to generalize \stochalg{} beyond the Block MDP setting. As a starting point for this future research direction,  we can show that every pushforward coverable MDP admits a policy emulator with a bounded state space size. We first define pushforward coverability which posits the existence of a good distribution satisfying pushforward concentrability (c.f.~\pref{def:exploratory-pushforward-distribution}).

\begin{definition}[Pushforward Coverability \cite{mhammedi2024power, amortila2024reinforcement}]\label{def:pushforward-coverability}
The pushforward coverabilty coefficient for an MDP $M$ is
\begin{align*}
    \cpushcov(M) \coloneqq \max_{h \in [H]} \inf_{\mu_h \in \Delta(\statesp_h)} \sup_{(x,a,x') \in \statesp_{h-1} \times \actionsp \times \statesp_{h}} \frac{P(x' \mid x,a)}{\mu_h(x')}.
\end{align*}
When clear from the context we denote the pushforward concentrability coefficient as $\cpushcov$.
\end{definition}

\begin{proposition}[Pushforward Coverable MDPs $\Rightarrow$ Small Policy Emulators]\label{prop:pushforward-coverable}
Let $M$ be an MDP with pushforward coverability coefficient $\cpushcov$ and $\Pi$ be any policy class. Then there exists a policy emulator $\estmdp$ with state space size
\begin{align*}
    \poly \prn*{ \cpushcov, A, H, \eps^{-1}, \log \abs{\Pi}, \log \delta^{-1} }.
\end{align*}
\end{proposition}

A few remarks:
\begin{itemize}
    \item Strictly speaking, the policy emulator we construct in \pref{prop:pushforward-coverable} is not a true MDP, since our construction requires the ``transition'' $\estp(\cdot\mid x_{h-1}, a_{h-1})$ to be an unnormalized measure over the the states in the next layer $\estmdpobsspace{h}$, which may sum to $\cpushcov \ge 1$. Thus, we slightly abuse the notation for expectation:
    \begin{align*}
        \En_{x \sim \estp(\cdot\mid x_{h-1}, a_{h-1})}\brk*{V^\pi(x)} \coloneqq \sum_{ x \in \estmdpobsspace{h}} \estp( \cdot\mid x_{h-1}, a_{h-1}) V^\pi(x).
    \end{align*}
    As discussed, the policy emulator anyways is not guaranteed to be a reasonable approximation of the underlying MDP $M$, just an object which enables uniform policy evaluation, so this issue is minor.
    \item Lemma 3.1 of \cite{amortila2024reinforcement} give a result of similar flavor, which shows that pushforward coverable MDPs are approximately \emph{low-rank}. Their proof, however, seems to be quite different. It relies on the Johnson-Lindenstrauss lemma to construct random embeddings which enable approximation of the Bellman backup operator for any arbitrary value function class $\cF$. 
    \item Unfortunately, we do not know how to leverage hybrid resets to construct such a policy emulator in a statistically efficient manner---the naive way to do so requires sample complexity scaling with $\spancap(\Pi)$ (which could be much larger than $\cpushcov$). We believe this is an interesting direction for future work.
\end{itemize}

\begin{proof}[Proof of \pref{prop:pushforward-coverable}]
We will prove this by explicitly constructing the policy emulator using the same algorithmic template as in \stochalg{}. To construct the state space of the policy emulator, we sample
$\wt{O}(\cpush/\eps^2)$ observations per layer from the distributions $\mu_1, \cdots, \mu_H$, respectively, that witnesses pushforward coverability at every $h \in [H]$. As shown before, the instantaneous rewards $\wh{R}(x,a)$ for every $x \in \estmdpobsspace{} \times \actionsp$ of the emulator can be learned via the local simulator up to $\eps$ accuracy. Now we show that it is possible to define the transition functions $\estp(\cdot \mid x,a)$ for every $x \in \estmdpobsspace{} \times \actionsp$ so that the resulting $\estmdp$ is an $O(\eps)$-accurate policy emulator for $d_1$. We do this inductively:

\begin{claim}
    Let $\Gamma_h > 0$. Suppose that at layer $h \in [H]$:
    \begin{align*}
        \forall~x \in \estmdpobsspace{h},~\forall~\pi \in \Pi:\quad \abs*{V^\pi(x) - \wh{V}^\pi(x)} \le \Gamma_h.
    \end{align*}
    Then for every $(x_{h-1},a_{h-1}) \in \estmdpobsspace{h-1} \times \actionsp$, there exists some $\estp \in \Delta(\estmdpobsspace{h})$ such that
    \begin{align*}
        \forall~\pi \in \Pi:\quad \abs*{Q^\pi(x_{h-1},a_{h-1}) - \wh{R}(x_{h-1}, a_{h-1}) - \En_{x \sim \estp} \brk{\wh{V}^\pi(x)} } \le \Gamma_h + 2 \eps.
    \end{align*}
\end{claim}
Applying this claim backwards from $h=H, \cdots, 1$ and using the fact that $\Gamma_H = \eps$ proves \pref{prop:pushforward-coverable}.

It remains to prove the claim. Let $\estp$ be an unnormalized measure over $\estmdpobsspace{}$ (to be defined later). First, we apply the decomposition
\begin{align*}
    &\abs*{Q^\pi(x_{h-1},a_{h-1}) - \wh{R}(x_{h-1}, a_{h-1}) - \En_{x \sim \estp} \brk{\wh{V}^\pi(x)} }\\
    &=~\abs*{R(x_{h-1}, a_{h-1}) - \wh{R}(x_{h-1}, a_{h-1})} + \abs*{ \En_{x \sim P} \brk{V^\pi(x)} - \En_{x \sim \estp} \brk*{V^\pi(x)} } + \abs*{ \En_{x \sim \estp} \brk{V^\pi(x)} - \En_{x \sim \estp} \brk{\wh{V}^\pi(x)} } \\[0.5em]
    &\le~ \Gamma_h + \eps + \abs*{ \En_{x \sim P} \brk{V^\pi(x)} - \En_{x \sim \estp} \brk{V^\pi(x)} }, \numberthis\label{eq:upperbound-pushforward-cov}
\end{align*}
where the last inequality uses the reward estimation accuracy and the assumption in the claim. To control the last term, we apply a change of measure:
\begin{align*}
    \En_{x \sim P} \brk*{V^\pi(x)} = \En_{x \sim \mu_h} \brk*{\frac{P(x \mid x_{h-1}, a_{h-1})}{\mu_h(x)} \cdot V^\pi(x)}.
\end{align*}
Observe that $\estmdpobsspace{h} = \crl{x_h^{(1)}, \cdots, x_h^{(n)} }$ are drawn i.i.d.~from $\mu_h$, and by pushforward coverability, the importance ratio $P(x \mid x_{h-1}, a_{h-1})/\mu_h(x) \le \cpushcov$. Via a standard uniform convergence bound, with probability at least $1-\delta$, for every $\pi \in \Pi$
\begin{align*}
    \bigg| \En_{x \sim \mu_h} \brk*{\frac{P(x \mid x_{h-1}, a_{h-1})}{\mu_h(x)} \cdot V^\pi(x)} - \sum_{i=1}^n \underbrace{ \frac{P(x_h^{(i)} \mid x_{h-1}, a_{h-1})}{n \cdot \mu_h(x_h^{(i)})} }_{\eqqcolon \estp(x_h^{(i)})} \cdot V^\pi(x_h^{(i)}) \bigg| \le \eps. 
\end{align*}
Plugging back our choice of $\estp$ into Eq.~\eqref{eq:upperbound-pushforward-cov} proves the claim.
\end{proof}

\clearpage
\section{Proof of Lower Bounds}\label{app:lower-bounds}

In this section, we prove our two main information theoretic lower bounds, \pref{thm:lower-bound-coverability} and \ref{thm:lower-bound-policy-completeness}.
\subsection{Lower Bound Preliminaries}

Our lower bounds are facilitated by recent developments that build a unified framework for \emph{interactive statistical decision making} (ISDM) \cite{chen2024assouad}. We will use an interactive version of Le Cam's convex hull method, which can be derived as a consequence of \cite[Thm.~2][]{chen2024assouad}. For completeness, we include the proof. It closely mirrors the proof of \cite[Prop.~4 of][]{chen2024assouad}, which shows how \cite[Thm.~2 of][]{chen2024assouad} recovers the noninteractive variant of Le Cam's convex hull method.

\begin{theorem}[Interactive Le Cam's Convex Hull Method]\label{thm:interactive-lecam-cvx}
For parameter space $\Theta$, let $\cM = \crl{M_{\theta} \mid \theta \in \Theta}$ be a class of models indexed by $\Theta$. Let $\cY$ be an observation space. For any fixed $\alg$ and distribution $\nu \in \Delta(\Theta)$, let $\Pr^{\nu, \alg} \in \Delta(\cY)$ be defined as the distribution over observations when (1) a parameter is drawn $\theta \sim \nu$, (2) the algorithm interacts with model $M_\theta$. Let $L:\Theta \times \cY \to \bbR_{+}$ be a loss function. Suppose that $\Theta_0 \subseteq \Theta$ and $\Theta_1 \subseteq \Theta$ are subsets that satisfy the separation condition
\begin{align*}
    L(\theta_0, y) + L(\theta_1, y) \ge 2\Delta, \quad \forall~y \in \cY, \theta_0 \in \Theta_0, \theta_1 \in \Theta_1.
\end{align*}
for some parameter $\Delta > 0$. Then it holds that for any $\alg$,
\begin{align*}
    \sup_{\theta \in \Theta} \En_{Y \sim \Pr^{M_\theta, \alg}} \brk*{L(\theta, Y)} \ge \frac{\Delta}{2} \max_{\nu_0 \in \Delta(\Theta_0), \nu_1 \in \Delta(\Theta_1)} \prn*{1 - \dtv\prn*{\Pr^{\nu_0, \alg}, \Pr^{\nu_1, \alg}}}. \numberthis\label{eq:le-cam-result}
\end{align*}
\end{theorem}

\begin{proof}
We will use \cite[Thm.~2][]{chen2024assouad} with total-variational (TV) distance $D_f \coloneqq \dtv$. Define the enlarged model class $\overline{\cM} \coloneqq \crl{M_\nu: \nu \in \Delta(\Theta)}$ as well as the loss function extension $\overline{L}: \overline{\cM} \times \cY \to \bbR_{+}$
\begin{align*}
    \overline{L}(M_\nu, y) \coloneqq \inf_{\theta \in \supp(\nu)} L(M_{\theta}, y).
\end{align*}
By the separation condition we have
\begin{align*}
    \overline{L}(M_{\nu_0}, y) + \overline{L}(M_{\nu_1}, y) \ge 2\Delta, \quad \forall~y \in \cY, \nu_0 \in \Delta(\Theta_0), \nu_1 \in \Delta(\Theta_1).
\end{align*}
We pick the prior $\mu \coloneqq \unif(\crl{M_{\nu_0}, M_{\nu_1}})$ and the reference distribution $\bbQ \coloneqq \En_{M\sim \mu} \brk{\Pr^{M, \alg}}$. Observe that
\begin{align*}
    \rho_{\Delta, \bbQ} \coloneqq \Pr_{M \sim \mu, Y \sim \bbQ} \brk*{\overline{L}(M, Y) < \Delta} \le \frac{1}{2}.
\end{align*}
Furthermore
\begin{align*}
    \En_{M \sim \mu} \brk*{\dtv\prn*{\Pr^{M, \alg} , \bbQ} } &= \frac{1}{2} \dtv\prn*{\Pr^{{\nu_0}, \alg} , \bbQ}  + \frac{1}{2} \dtv\prn*{\Pr^{{\nu_1}, \alg} , \bbQ}  \\
    &\le \frac{1}{2} \dtv\prn*{\Pr^{{\nu_0}, \alg} , \Pr^{{\nu_1}, \alg}  }.
\end{align*}
Therefore for any $\delta \in [0, \tfrac{1}{2} - \tfrac{1}{2} \dtv\prn{\Pr^{{\nu_0}, \alg} , \Pr^{{\nu_1}, \alg}  } )$ we have
\begin{align*}
    \En_{M \sim \mu} \brk*{\dtv\prn*{\Pr^{M, \alg} , \bbQ} } \le \begin{cases} 
        \dtv\prn*{ \ber(1-\delta), \ber(\rho_{\Delta, \bbQ}) } &\text{if}~\rho_{\Delta, \bbQ} \le 1-\delta \\
        0 &\text{otherwise.}
    \end{cases}
\end{align*}
Therefore \cite[Thm.~2][]{chen2024assouad} gives
\begin{align*}
    \En_{\theta \sim \frac{\nu_0 + \nu_1}{2}, Y \sim \Pr^{M_\theta, \alg}} \brk*{L(\theta, Y)} \ge \En_{M \sim \mu, Y \sim \Pr^{M, \alg}} \brk*{\overline{L}(M, Y)} \ge \frac{\Delta}{2} \cdot \prn*{ 1 - \dtv\prn*{\Pr^{{\nu_0}, \alg} , \Pr^{{\nu_1}, \alg}  } }.
\end{align*}
Taking supremum over $\nu_0$ and $\nu_1$ gives the result.
\end{proof}

In light of \pref{thm:interactive-lecam-cvx}, we need to analyze the TV distance between an algorithm $\alg$ interactions with two separate environments given by $\nu_0$ and $\nu_1$. The following chain rule lemma will be useful.

\begin{lemma}[Chain Rule for TV Distance, Exercise I.43 of \cite{polyanskiy2025information}]\label{lem:chain-rule-tv}
Let $\cZ$ be any observation space, let $\bbP^{Z_n}$ and $\bbQ^{Z_n}$ be distributions over $n$-tuples of $\cZ$. Then
\begin{align*}
    \dtv\prn*{ \bbP^{Z_n}, \bbQ^{Z_n} } \le \sum_{i=1}^n \En_{Z_{1:i-1} \sim \bbP^{Z_n}} \brk*{ \dtv \prn*{\bbP \brk*{Z_i \mid Z_{1:i-1}}, \bbQ \brk*{Z_i \mid Z_{1:i-1}}} }.
\end{align*}
\end{lemma}

\paragraph{Additional Notation.} We use $\alg$ to denote a deterministic algorithm that collects $T$ samples, i.e., full-length episodes (from either the generative model or $\mu$-reset access). For any $t \in [T]$ we define $\cF_{t-1}$ to be the sigma-field of everything observed in the first $t-1$ episodes. We further define for any $h \in [H]$ the filtration $\cF_{t,h-1}$ to be the sigma-field of everything observed in the first $t-1$ episodes as well as the first $h-1$ steps of the $t$-th sample. To handle the difference in interaction models, $\cF_{t,h-1}$ is defined slightly differently:
\begin{itemize}
    \item For the generative model, $\cF_{t,h-1} \coloneqq \sigma( \cF_{t-1}, \crl{\prn{X_{t,i}, A_{t,i}, R_{t,i}, X'_{t,i}}}_{i \le h-1} )$, where $R_{t,i}$ and $X'_{t,i}$ are the reward and transition which is returned by the environment. The tuple $(X_{t,h}, A_{t,h})$ is measurable with respect to $\cF_{t,h-1}$ (since $\alg$ is deterministic).
    \item For the $\mu$-reset model, $\cF_{t,h-1} = \sigma( \cF_{t-1}, \crl{\prn{X_{t,i}, A_{t,i}, R_{t,i}}}_{h_\bot \le i \le h-1} )$. Here, $h_\bot$ is the starting layer of episode $t$, which is measurable with respect to $\cF_{t-1}$; furthermore, the action $A_{t,h}$ is measurable with respect to $\cF_{t,h-1} \cup \crl{X_{t,h}}$ (since $\alg$ is deterministic).
\end{itemize}

Lastly, we denote partial policies $\pi_{h_\bot:h_\top} \in \Pi_{h_\bot:h_\top}$ for some $1 \le h_\bot \le h_\top \le H$. We sometimes drop the subscript $h_\bot:h_\top$ if clear from context. We may also overload equality to compare partial policies $\pi_{h_\bot: h_\top}$ with complete policies $\pi'_{1:H}$, i.e., we write $\pi = \pi'$ iff $\pi_{h_\bot: h_\top} = \pi'_{h_\bot: h_\top}$.

\subsection{Proof of \pref{thm:lower-bound-coverability}}\label{app:lower-bound-generative-access}

The construction of \pref{thm:lower-bound-coverability} is given by the rich observation combination lock, which has appeared in previous lower bounds for RL \cite{sekhari2021agnostic, jia2023agnostic}. Since the rich observation combination lock is a Block MDP with $2$ latent states per layer, it satisfies $\ccov = 2$. The key intuition is that the set of observations associated with latent states on the good chain is much smaller than the set of observations associated with latent states on the bad chain. Therefore, even though coverability is small, the learner cannot effectively use the generative model to ``guess'' observations which are emitted from states on the good chain. In other words, they cannot sample from a distribution with low concentrability, which is crucial for learning $\optpi$.

\paragraph{Lower Bound Construction.} First, we define the policy class $\Pi$ to be open loop policies:
\begin{align*}
    \Pi \coloneqq \crl{\pi: \forall x \in \statesp_h, \pi_h(x) \equiv a_h, (a_1, \cdots, a_H) \in \actionsp^H}.
\end{align*}
We define a family of Block MDPs $\cM = \crl{M_{\optpi, \optdec}}_{\optpi \in \Pi, \optdec \in \Phi}$ which are parameterized by an optimal policy $\optpi \in \Pi$ and a decoding function $\optdec \in \Phi$ (to be described). An example is illustrated in \pref{fig:lb1}.
\begin{itemize}
    \item \textbf{Latent MDP:} The latent state space $\latentsp$ is layered where each $\latentsp_h \coloneqq \crl{ s_h^\sgood, s_h^\sbad}$ is comprised of a good and bad state. We abbreviate the state as $\crl{\sgood, \sbad}$ if the layer $h$ is clear from context. The action space $\actionsp = \crl{0,1}$. The starting state is always $\sgood$. Let $\optpi \in \Pi$ be any policy, which can be represented by a vector in $(\pi^\star_1, \cdots, \pi^\star_H) \in \crl{0,1}^H$. The latent transitions/rewards of an MDP parameterized by $\optpi \in \Pi$ are given by the standard combination lock. For every $h \in [H]$: 
    \begin{align*}
        \optlatp(\cdot  \mid  s, a) = \begin{cases}
            \delta_{s_{h+1}^\sgood} & \text{if}~s = s_h^\sgood ~\text{and}~a = \pi^\star_h\\
            \delta_{s_{h+1}^\sbad}& \text{otherwise}.
        \end{cases} \quad \text{and} \quad  \optlatr(s,a) = \ind{s = s_H^\sgood, a = \pi^\star_H}.
    \end{align*} 
    \item \textbf{Rich Observations:} The observation state space $\statesp$ is layered where each $\cX_h \coloneqq \crl{x_h^{(1)}, \cdots, x_h^{(m)}}$ with $m = 2^{2H}$. The decoding function class $\Phi$ is the collection of all decoders which for every $h \ge 2$ assigns $s_h^\sgood$ to a subset of $\statesp_h$ of size $2^H$ and $s_h^\sbad$ to the rest:
    \begin{align*}
        \Phi &\coloneqq \crl*{\optdec: \statesp \mapsto \latentsp :~\forall~ x_1 \in \statesp_1,~\optdec(x_1) = \sgood,  \text{ and } \forall~h \ge 2, \abs*{ \crl*{x \in \statesp_h: \optdec(x) = \sgood}} = 2^H},\\
        &\text{so that}~\abs{\Phi} = {2^{2H} \choose 2^H}^{H-1} = 2^{2^{\wt{O}(H)}}.
    \end{align*}
    In the MDP parameterized by $\optdec \in \Phi$, the emission for every $s \in \latentsp$ is $\emission(s) = \unif\prn*{\crl*{x \in \statesp_h: \optdec(x) = s}}$.
\end{itemize}

Now we establish several facts about the lower bound construction. Fix any $M = M_{\optpi, \optdec}$.
\begin{enumerate}
    \item Since $M$ is a Block MDP with 2 latent states per layer, $\ccov(\Pi, M) = 2$.
    \item The class $\Pi$ satisfies policy completeness with respect to $M$. To see this, fix any layer $h \in [H]$ and partial policy $\pi \in \Pi_{h+1:H}$. We have:
    \begin{align*}
        \forall~(x,a) \in \statesp_h \times \actionsp:\quad Q^\pi(x,a) = \ind{\optdec(x) = s_h^\sgood, a = \pi^\star_h, \pi = \pi^\star_{h+1:H}}.
    \end{align*}
    Therefore in \pref{def:policy-completeness} we can take $\wt{\pi}_h \coloneqq \pi^\star_h$, which satisfies $\wt{\pi}_h \in \argmax_{a \in \actionsp} Q^\pi(x,a)$ for all $x \in \statesp_h$.
\end{enumerate}

\paragraph{Sample Complexity Lower Bound.} We will use \pref{thm:interactive-lecam-cvx} to prove our lower bound. First we need to instantiate the parameter space. We will let $\Theta \coloneqq \crl{(\optpi, \optdec):~\optpi \in \Pi, \optdec \in \Phi}$ so that $\cM = \crl{M_\theta}_{\theta \in \Theta} = \crl{M_{\optpi, \optdec}}_{\optpi \in \Pi, \optdec \in \Phi}$. We further denote the subsets 
\begin{align*}
    \Theta_0 &\coloneqq \crl{(\optpi, \optdec):~\optpi \in \Pi~\text{s.t.}~\optpi_H = 0, \optdec \in \Phi} \\
    \Theta_1 &\coloneqq \crl{(\optpi, \optdec):~\optpi \in \Pi~\text{s.t.}~\optpi_H = 1, \optdec \in \Phi}
\end{align*}
The observation space $\cY$ is defined as the set of observations over $T$ rounds as well as returned proper policy for an algorithm interacting with the MDP, i.e.,
\begin{align*}
    \cY \coloneqq (\statesp \times \actionsp \times [0,1])^{HT} \times \Pi.
\end{align*}
For technical convenience, we will suppose that $\alg$ sequentially queries the generative model by looping over layers, i.e., it queries $(X_1, A_1) \in \statesp_1 \times \actionsp$, then $(X_2, A_2) \in \statesp_2 \times \actionsp$, etc. This only increases the sample complexity of $\alg$ by a factor of $H$, which is negligible since we will show that $\alg$ requires $\exp(H)$ samples.

For an observation $y \in \cY$ we define the final returned policy as $y^\pi$. The loss function is given by
\begin{align*}
    L((\optpi, \phi), y) \coloneqq \ind{\optpi \ne y^\pi}. 
\end{align*}
Then we have for any $y \in \cY$, $(\optpi_0, \phi_0) \in \Theta_0$, and $(\optpi_1, \phi_1) \in \Theta_1$ that
\begin{align*}
    L((\optpi_0, \phi_0), y) +  L((\optpi_1, \phi_1), y) \ge 1 \coloneqq 2\Delta,
\end{align*}
since the last bit of $y^\pi$ can be either 0 or 1, thus only matching exactly one of $\optpi_0$ and $\optpi_1$.

Now we are ready to apply \pref{thm:interactive-lecam-cvx}. We get that for any $\alg$, we must have
\begin{align*}
    \sup_{(\optpi, \optdec) \in \Pi \times \Phi} \En_{Y \sim \Pr^{M_{\optpi, \optdec}, \alg}} \brk*{V^\star - V^{\estpi}} &= \sup_{(\optpi, \optdec) \in \Pi \times \Phi} \En_{Y \sim \Pr^{M_{\optpi, \optdec}, \alg}} \brk*{1 - \ind{\optpi = Y^\pi} } \\
    &= \sup_{(\optpi, \optdec) \in \Pi \times \Phi} \En_{Y \sim \Pr^{M_{\optpi, \optdec}, \alg}} \brk*{ L((\optpi, \phi), Y) } \\
    &\ge \frac{1}{4} \cdot \max_{\nu_0 \in \Delta(\Theta_0), \nu_1 \in \Delta(\Theta_1)} \prn*{1 - \dtv\prn*{\Pr^{\nu_0, \alg}, \Pr^{\nu_1, \alg}}} \\
    &\ge \frac{1}{4} \cdot \prn*{1 - \dtv\prn*{\Pr^{\unif(\Theta_0), \alg}, \Pr^{\unif(\Theta_1), \alg}}}. 
\end{align*}
It remains to compute an upper bound $\dtv\prn*{\Pr^{\unif(\Theta_0), \alg}, \Pr^{\unif(\Theta_1), \alg}}$ which holds for any $\alg$. This is accomplished by the following lemma.

\begin{lemma}\label{lem:tv-bound-generative} Let $T = 2^{O(H)}$. For any deterministic $\alg$ that adaptively collects $HT$ samples via generative access, we have
\begin{align*}
    \dtv\prn*{\Pr^{\unif(\Theta_0), \alg}, \Pr^{\unif(\Theta_1), \alg}} \le \frac{T^4H}{2^{H-9}}.
\end{align*}
\end{lemma}
Plugging in \pref{lem:tv-bound-generative}, we conclude that for any $\alg$ that collects $2^{cH}$ samples for sufficiently small constant $c > 0$ must be $1/8$-suboptimal in expectation. This concludes the proof of \pref{thm:lower-bound-coverability}.\qed

\subsection{Proof of \pref{lem:tv-bound-generative} (TV Distance Calculation for \pref{thm:lower-bound-coverability})}

Let us denote $\nu_0 \coloneqq \unif(\Theta_0)$ and $\nu_1 \coloneqq \unif(\Theta_1)$. By the TV distance chain rule (\pref{lem:chain-rule-tv}) we have
\begin{align*}
     \hspace{2em}&\hspace{-2em} \dtv\prn*{\Pr^{\nu_0, \alg}, \Pr^{\nu_1, \alg}} \\
     &\le \sum_{t=1}^T \sum_{h=1}^H \En^{\nu_0, \alg} \brk*{\dtv \prn*{\Pr^{\nu_0, \alg} \brk*{ X_{t,h}, A_{t,h} \mid \cF_{t,h-1} }, \Pr^{\nu_1, \alg} \brk*{ X_{t,h}, A_{t,h} \mid \cF_{t,h-1} }}} \\
     &\qquad + \En^{\nu_0, \alg} \brk*{\dtv \prn*{\Pr^{\nu_0, \alg} \brk*{ X_{t,h}', R_{t,h} \mid X_{t,h}, A_{t,h}, \cF_{t,h-1} }, \Pr^{\nu_1, \alg} \brk*{ X_{t,h}', R_{t,h} \mid X_{t,h}, A_{t,h}, \cF_{t,h-1} }}} \\
     &= \sum_{t=1}^T \sum_{h=1}^H \En^{\nu_0, \alg} \brk*{\dtv \prn*{\Pr^{\nu_0, \alg} \brk*{ X_{t,h}', R_{t,h} \mid \cF_{t,h-1} }, \Pr^{\nu_1, \alg} \brk*{ X_{t,h}', R_{t,h} \mid \cF_{t,h-1} }}} \\
     &= \underbrace{ \sum_{t=1}^T \sum_{h=1}^{H-1} \En^{\nu_0, \alg} \brk*{\dtv \prn*{\Pr^{\nu_0, \alg} \brk*{ X_{t,h}' \mid  \cF_{t,h-1} }, \Pr^{\nu_1, \alg} \brk*{ X_{t,h}' \mid  \cF_{t,h-1} }}} }_{\text{transition TV distance}} \\
     &\qquad + 
     \underbrace{ \sum_{t=1}^T \En^{\nu_0, \alg} \brk*{\dtv \prn*{\Pr^{\nu_0, \alg} \brk*{ R_{t,H} \mid  \cF_{t,H-1} }, \Pr^{\nu_1, \alg} \brk*{ R_{t,H} \mid  \cF_{t,H-1} }}} }_{\text{reward TV distance}}.
\end{align*}
The first equality follows from the fact that the TV distance for the distribution over state-action pairs $(X_{t,h}, A_{t,h})$ is zero since $(X_{t,h}, A_{t,h})$ is measurable with respect to $\cF_{t,h-1}$. The second equality follows because the rewards only come at the last layer in every MDP instance.

We now show how to bound each term separately.

\paragraph{Transition TV Distance.} For the transition TV distance, we have the following computation for all $t \in [T], h \in [H-1]$:
\begin{align*}
    \hspace{2em}&\hspace{-2em} \En^{\nu_0, \alg} \brk*{\dtv \prn*{ \Pr^{\nu_0, \alg} \brk*{ X_{t,h}' \mid  \cF_{t, h-1} }, \Pr^{\nu_1, \alg} \brk*{X_{t,h}' \mid  \cF_{t, h-1} } } } \\
    &\overset{(i)}{\le} \En^{\nu_0, \alg} \brk*{\dtv \prn*{ \Pr^{\nu_0, \alg} \brk*{ X_{t,h}' \mid  \cF_{t, h-1} }, \unif(\cX_{h+1}) } }  + \En^{\nu_0, \alg} \brk*{\dtv \prn*{ \Pr^{\nu_1, \alg} \brk*{ X_{t,h}' \mid  \cF_{t, h-1} }, \unif(\cX_{h+1}) } } \\ 
    &\overset{(ii)}{\le} \frac{t}{2^{H-3}}. \numberthis \label{eq:ub-transition-calculation}
\end{align*}
The inequality $(i)$ follows by triangle inequality and the inequality $(ii)$ uses \pref{lem:tv-distance-from-uniform-generative}.

\paragraph{Reward TV Distance.} 
We can compute that
\begin{align*}
    \hspace{2em}&\hspace{-2em} \En^{\nu_0, \alg} \brk*{\dtv \prn*{\Pr^{\nu_0, \alg} \brk*{ R_{t,H} \mid  \cF_{t,H-1} }, \Pr^{\nu_1, \alg} \brk*{ R_{t,H} \mid  \cF_{t,H-1} }}} \\
    &\overset{(i)}{\le} \En^{\nu_0, \alg} \brk*{\dtv \prn*{\Pr^{\nu_0, \alg} \brk*{ R_{t,H} \mid  \cF_{t,H-1} },  \delta_0 }} + \En^{\nu_0, \alg} \brk*{\dtv \prn*{\Pr^{\nu_1, \alg} \brk*{ R_{t,H} \mid  \cF_{t,H-1} },  \delta_0 }} \\
    &\overset{(ii)}{=} \En^{\nu_0, \alg} \brk*{\Pr^{\nu_0, \alg} \brk*{ R_{t,H} = 1 \mid  \cF_{t,H-1} } } + \En^{\nu_0, \alg} \brk*{ \Pr^{\nu_1, \alg} \brk*{ R_{t,H} = 1 \mid  \cF_{t,H-1} } } \overset{(iii)}{\le} t \cdot \frac{T^2H}{2^{H-8}}. \numberthis\label{eq:reward-tv}
\end{align*}
The inequality $(i)$ follows by triangle inequality, while $(ii)$ uses the fact that the rewards are in $\crl{0,1}$. Lastly, $(iii)$ follows by \pref{lem:reward-bound-lb1}.

\paragraph{Final Bound.} Thus, combining Eqs.~\eqref{eq:ub-transition-calculation} and \eqref{eq:reward-tv} we can conclude that:
\begin{align*}
    \dtv\prn*{\Pr^{\nu_0, \alg}, \Pr^{\nu_1, \alg}} \le \frac{T^2H}{2^{H-3}} + \frac{T^4H}{2^{H-8}} \le \frac{T^4H}{2^{H-9}}.
\end{align*}
This concludes the proof of \pref{lem:tv-bound-generative}.\qed

\begin{lemma}[Transition TV Distance for Construction in \pref{thm:lower-bound-coverability}]\label{lem:tv-distance-from-uniform-generative}
For any $t \in [T], h \in [H-1]$, we have
\begin{align*}
    \nrm*{\Pr^{\nu_0, \alg}\brk*{X_{t,h}'  \mid X_{t,h}, A_{t,h}, \cF_{t, h-1}} - \unif(\statesp_{h+1}) }_1 &\le \frac{t}{2^{H-2}}, \\
    \nrm*{\Pr^{\nu_1, \alg}\brk*{X_{t,h}'  \mid X_{t,h}, A_{t,h}, \cF_{t, h-1}} - \unif(\statesp_{h+1})}_1 &\le \frac{t}{2^{H-2}}.
\end{align*}
\end{lemma}
\begin{proof}[Proof of \pref{lem:tv-distance-from-uniform-generative}]
We prove the bound for $\nu_0$, since the proof for $\nu_1$ is identical. Denote the ``annotated'' sigma-field
\begin{align*}
    \cF_{t, h-1}' = \sigma \Bigg( &\cF_{t, h-1}, X_{t,h}, A_{t,h}, \crl*{\optdec(X): X \in \cF_{t, h-1} \cup \crl{X_{t,h}} }, \\
    &\crl*{ \ind{A = \optpi(X)}: (X,A) \in \cF_{t, h-1} \cup \crl{X_{t,h}, A_{t,h}} } \Bigg) 
\end{align*}
to be the sigma-field which also includes the latent state labels for all of the seen observations as well as whether the actions taken followed $\optpi$ or not. Let us denote $\ell = \phi(X_{t,h}') \in \crl{\sgood, \sbad}$ to be the latent state of the next observation. Observe that the label $\ell$ is measurable with respect to $\cF_{t,h-1}'$ since the filtration $\cF_{t-1}'$ includes $\optdec(X_{t,h})$ as well as $\ind{A_{t,h} = \optpi(X_{t,h})}$. Furthermore denote $\cX_\mathsf{obs}$ to denote the total number of observations that we have encountered already in layer $h+1$ and $\cX_\mathsf{obs}^\ell$ to denote the observations we have encountered whose latent state is $\ell$.

Under the uniform distribution over decoders, the assignment of the remaining observations is equally likely. Therefore we can write the distribution of $X_{t,h}'$ as: 
\begin{align*}
    \text{if}~\ell = \sgood:&\quad \Pr^{\nu_0, \alg} \brk*{X_{t,h}' = x  \mid  \cF_{t,h-1}'} = \begin{cases}
        \frac{1}{2^H} &\text{if}~x \in \cX_\mathsf{obs}^\ell \\
        0 &\text{if}~x \in \cX_\mathsf{obs} - \cX_\mathsf{obs}^\ell \\
        \frac{1}{2^H} \cdot \frac{2^H - \abs{\cX_\mathsf{obs}^\ell}}{2^{2H} - \abs{\cX_\mathsf{obs}}} &\text{if}~x \in \cX_{h+1} - \cX_\mathsf{obs}
    \end{cases}\\
    \text{if}~\ell = \sbad:&\quad \Pr^{\nu_0, \alg} \brk*{X_{t,h}' = x  \mid  \cF_{t,h-1}'} = \begin{cases}
        \frac{1}{2^{2H}-2^{H}} &\text{if}~x \in \cX_\mathsf{obs}^\ell \\
        0 &\text{if}~x \in \cX_\mathsf{obs} - \cX_\mathsf{obs}^\ell \\
        \frac{1}{2^{2H}-2^{H}} \cdot \frac{2^{2H}-2^{H} - \abs{\cX_\mathsf{obs}^\ell}}{2^{2H} - \abs{\cX_\mathsf{obs}}} &\text{if}~x \in \cX_{h+1} - \cX_\mathsf{obs}
    \end{cases}
\end{align*}
We elaborate on the calculation for the last probability in each case. Suppose $\ell = \sgood$. Then for any $x \in \cX_{h+1} - \cX_\mathsf{obs}$ which has not been observed yet we assign $\optdec(x) = \ell$ in
\begin{align*}
    &{2^{2H} - \abs{\cX_\mathsf{obs}} -1 \choose 2^{H} -\abs{\cX_\mathsf{obs}^\ell} - 1 } \text{ ways out of } {2^{2H} - \abs{\cX_\mathsf{obs}} \choose 2^{H} -\abs{\cX_\mathsf{obs}^\ell } } \text{ assignments.} \\
    &\quad \Longrightarrow \optdec(x) = \sgood \text{ with probability } \frac{2^{H} -\abs{\cX_\mathsf{obs}^\ell}}{2^{2H} - \abs{\cX_\mathsf{obs}}}.
\end{align*}
For each assignment where $\optdec(x) = \sgood$ we will select it with probability $1/2^H$ since the emission is uniform, giving us the final probability as claimed. A similar calculation can be done for the case where $\ell = \sbad$.

Therefore we can calculate the final bound that
\begin{align*}
    \nrm*{\Pr^{\nu_0, \alg}\brk*{X_{t,h}'  \mid  \cF_{t, h-1}'} - \unif(\statesp_{h+1}) }_1 &= \sum_{x \in \cX_{h+1}} \abs*{\Pr^{\nu_0, \alg}\brk*{X_{t,h}' = x \mid  \cF_{t, h-1}'} - \frac{1}{2^{2H}}} \\
    &\le \begin{cases}
        \frac{\abs{\cX_\mathsf{obs}^\sgood}}{2^H} + \frac{\abs{\cX_\mathsf{obs}^\sbad} }{2^{2H}} + \abs*{\frac{2^H - \abs{\cX_\mathsf{obs}^\sgood}}{2^{H}} - \frac{2^{2H} - \abs{\cX_\mathsf{obs}}}{2^{2H}}} &\text{if}~\ell = \sgood, \\[0.5em]
        \frac{\abs{\cX_\mathsf{obs}^\sbad}}{2^{2H}- 2^H} + \frac{\abs{\cX_\mathsf{obs}^\sgood} }{2^{2H}} + \abs*{\frac{2^{2H} - 2^H - \abs{\cX_\mathsf{obs}^\sbad}}{2^{2H} - 2^H} - \frac{2^{2H} - \abs{\cX_\mathsf{obs}}}{2^{2H}}} &\text{if}~\ell = \sbad.
    \end{cases}\\
    &\le \frac{4 \cdot \abs{\cX_\mathsf{obs}}}{2^H} \le \frac{4 t}{2^H}.
\end{align*}
Since $\Pr^{\nu_0, \alg}\brk*{X_{t,h}'  \mid  X_{t,h}, A_{t,h}, \cF_{t, h-1}} = \En^{\nu_0, \alg} \Pr^{\nu_0, \alg}\brk{X_{t,h}'  \mid  \cF_{t, h-1}'}$, we have by convexity of TV distance and Jensen's inequality,
\begin{align*}
    \nrm*{\Pr^{\nu_0, \alg}\brk*{X_{t,h}'  \mid X_{t,h}, A_{t,h}, \cF_{t, h-1}} - \unif(\statesp_{h+1}) }_1 \le \frac{4 t}{2^H},
\end{align*}
which concludes the proof of \pref{lem:tv-distance-from-uniform-generative}.
\end{proof}

\begin{lemma}[Reward Bound for Construction in \pref{thm:lower-bound-coverability}]\label{lem:reward-bound-lb1} 
Let $T \le 2^H$. For any $t \in [T]$:
    \begin{align*}
        \En^{\nu_0, \alg} \brk*{\Pr^{\nu_0, \alg} \brk*{ R_{t,H} = 1 \mid  \cF_{t,H-1} } } &\le t \cdot \frac{HT^2}{2^{H-7}}.\\
        \En^{\nu_0, \alg} \brk*{ \Pr^{\nu_1, \alg} \brk*{ R_{t,H} = 1 \mid  \cF_{t,H-1} } } &\le t \cdot \frac{HT^2}{2^{H-7}}.
    \end{align*}
\end{lemma}

\begin{proof}[Proof of \pref{lem:reward-bound-lb1}]

To show the proof, we use induction to show that the probability of see nonzero reward remains small throughout the entire execution of $\alg$.

\paragraph{Peeling Off Bad Events.} First, we will peel off a couple ``bad'' events which occur with low probability:
\begin{itemize}
    \item Let $\eventfresh$ be the event that every freshly sampled observation (i.e., querying the generative model on some observation $X_{t,h} \notin \cF_{t,h-1}$) in any layer $h \ge 2$ has a bad label:
    \begin{align*}
        \eventfresh \coloneqq \crl*{\phi(X_{t,h}) = \sbad \text{ for every } t \in [T], h \ge 2, X_{t,h} \notin \cF_{t,h-1}}.
    \end{align*}
    We will show in \pref{lem:fresh-states} that due to the unbalanced sizes of every layer and the uniform distribution over decoders, $\eventfresh$ must occur with high probability. This captures the intuition that the generative model affords no additional power over local simulation, since data generated from states with a bad label are not informative, and with high probability all fresh samples have a bad label. 
    \item Let $\eventnew$ be the event that every sampled transition is a new observation that has never been seen before:
    \begin{align*}
        \eventnew \coloneq \crl*{X_{t,h}' \notin \cF_{t,h-1} \text{ for every }t \in [T], h \in [H] }.
    \end{align*} 
    We show in \pref{lem:repeated-transitions} that due to the large state space in every layer, $\eventnew$ also occurs with high probability, therefore capturing the intuition that transitions are not informative for learning the optimal policy $\optpi$.
\end{itemize}

We can compute that:
\begin{align*}
     \En^{\nu_0, \alg} \brk*{\Pr^{\nu_0, \alg} \brk*{ R_{t,H} = 1 \mid \cF_{t, H-1} } } &\le \Pr^{\nu_0, \alg}\brk*{\eventfresh^c} + \Pr^{\nu_0, \alg}\brk*{\eventnew^c} + \En^{\nu_0, \alg} \brk*{\ind{\eventfresh \wedge \eventnew} \Pr^{\nu_0, \alg} \brk*{ R_{t,H} = 1 \mid \cF_{t, H-1} } }\\
     &\le \frac{HT\cdot  2^H}{2^{2H}-2T} + \frac{HT^2}{2^H} + \En^{\nu_0, \alg} \brk*{\ind{\eventfresh \wedge \eventnew} \Pr^{\nu_0, \alg} \brk*{ R_{t,H} = 1 \mid \cF_{t, H-1} } }. \\
     &\le \frac{HT^2}{2^{H-2}} + \En^{\nu_0, \alg} \brk*{\ind{\eventfresh \wedge \eventnew} \Pr^{\nu_0, \alg} \brk*{ R_{t,H} = 1 \mid \cF_{t, H-1} } }, \numberthis \label{eq:peeled-reward-eq}
\end{align*}
where the second line uses \pref{lem:fresh-states} and \pref{lem:repeated-transitions}, and the last line uses the fact that $T = 2^{O(H)}$. 

We will show inductively that under the distribution $\Pr^{\nu_0, \alg}$, rewards are nonzero with exponentially small (in $H$) probability. Then we use this bound to prove the final guarantee.

\paragraph{Inductive Claim.} Let $\cE_{R,t}$ be the event that after the $t$-th episode, all of the observed rewards are zero, i.e., $\cE_{R,t} \coloneqq \crl{R_{t',H} = 0 \text{ for all } t' \le t}$. We claim that 
\begin{align*}
    \Pr^{\nu_0, \alg} \brk*{\cE_{R,t}^c \wedge \eventfresh \wedge \eventnew} \le  t \cdot \frac{HT^2}{2^{H-7}}. \numberthis \label{claim:claim-reward}
\end{align*}
We show this via induction. The base case of $t=0$ trivially holds. Now suppose that \pref{claim:claim-reward} holds for at episode $t-1$. We show that it holds at episode $t$. We calculate that
\begin{align*}
    \Pr^{\nu_0, \alg} \brk*{ \cE_{R,t}^c \wedge \eventfresh \wedge \eventnew} 
    &\overset{(i)}{\le} \Pr^{\nu_0, \alg} \brk*{ \cE_{R,t-1}^c \wedge \eventfresh \wedge \eventnew} + \En^{\nu_0, \alg} \brk*{\ind{\cE_{R,t-1} \wedge \eventfresh \wedge \eventnew} \Pr^{\nu_0, \alg} \brk*{R_{t,H} = 1 \mid \cF_{t, H-1}} } \\
    &\overset{(ii)}{\le} (t-1) \cdot \frac{HT^2}{2^{H-7}} + \En^{\nu_0, \alg} \brk*{\ind{\cE_{R,t-1} \wedge \eventfresh \wedge \eventnew} \Pr^{\nu_0, \alg} \brk*{R_{t,H} = 1 \mid \cF_{t, H-1}} } \numberthis \label{eq:induction-one-step}
\end{align*}
Here, inequality $(i)$ uses the fact that if we see zero reward in the first $t-1$ episodes, $\cE_{R,t}^c$ can only happen if $R_{t,H} = 1$; inequality $(ii)$ uses the inductive hypothesis.

Now we will provide a bound on the reward distribution. We can calculate that
\begin{align*}
    \Pr^{\nu_0, \alg} \brk*{ R_{t,H} = 1 \mid \cF_{t, H-1} } 
    &= \sum_{\phi \in \Phi} \Pr^{\nu_0, \alg} \brk*{ R_{t,H} = 1 \mid \phi, \cF_{t, H-1} } \Pr^{\nu_0, \alg} \brk*{ \phi \mid\cF_{t, H-1} } \\
    &\overset{(i)}{=} \sum_{\phi \in \Phi} \ind{ \phi(X_{t,H}) = \sgood \text{ and } A_{t,H} = 0} \Pr^{\nu_0, \alg} \brk*{ \phi \mid \cF_{t, H-1} } \\
    &\le \sum_{\phi \in \Phi} \ind{ \phi(X_{t,H}) = \sgood } \Pr^{\nu_0, \alg} \brk*{ \phi \mid \cF_{t, H-1} } =  \Pr^{\nu_0, \alg} \brk*{ \phi(X_{t,H}) = \sgood \mid \cF_{t, H-1} }. \numberthis \label{eq:reward-upper-bound-induction}
\end{align*}
For $(i)$ we use the fact that the event $R_{t,H}=1$ is measurable with respect to $\phi$ and $\cF_{t, H-1}$.

\paragraph{Dataset as a DAG.}
To further bound Eq.~\eqref{eq:reward-upper-bound-induction}, we take the following viewpoint: for any $t\in[T], h \in [H]$, the collected dataset $\cF_{t,h}$ can be viewed as directed acyclic graph (DAG) with set of vertices given by the observations in $\cF_{t,h}$. In this DAG, the edges are labeled with $A \in \crl{0,1}$, and we draw an edge $X \to X'$ with label $a$ if the sample $(X,A,X')$ exists in the dataset $\cF_{t,h}$. For any observation $x \in \statesp$ and filtration $\cF$, we define the root-layer operation $\rootlayer(x \mid \cF)$ to be minimum layer $h$ for which there exists some path in the DAG representation of $\cF$ from some $X_h \to x$ with $X_h \in \cF$. If $x \notin \cF$, we have the convention that $\rootlayer(x \mid \cF) = h(x)$. We also denote $\mathsf{Root}(x \mid \cF)$ to be any observation $X_h \in \cF \cup \crl{x}$ which witnesses $\rootlayer(x \mid \cF) = h$.

We can further calculate that
\begin{align*}
    \hspace{2em}&\hspace{-2em} \En^{\nu_0, \alg} \brk*{ \ind{\cE_{R,t-1} \wedge \eventfresh \wedge \eventnew} \Pr^{\nu_0, \alg} \brk*{ \phi(X_{t,H}) = \sgood \mid \cF_{t, H-1} } } \\ 
    &\le  \En^{\nu_0, \alg} \brk*{ \ind{\cE_{R,t-1} \wedge \eventfresh \wedge \eventnew} \Pr^{\nu_0, \alg} \brk*{ \substack{\text{exists a path $X_1 \to X_{t,H}$ in $\cF_{t,H-1}$} \\ \text{labeled by $\optpi_{1:H-1}$ }} \mid \cF_{t, H-1}} }. \numberthis \label{eq:exists-path}
\end{align*}
The inequality is shown as follows: if $\rootlayer(X_{t,H} \mid \cF_{t,H-1}) \ge 2$, then event $\eventfresh$ guarantees that any observation $X_h \in \cF_{t,H-1}$ which witnesses the value of $\rootlayer$ has a bad label $\phi(X_h) = \sbad$, so therefore we must also have $\phi(X_{t,H}) = \sbad$. Otherwise, if $\rootlayer(X_{t,H} \mid \cF_{t,H-1}) = 1$, then $\phi(X_{t,H}) = \sgood$ implies that the path $X_1 \to X_{t,H}$ which witnesses $\rootlayer = 1$ must be labeled by $\optpi_{1:H-1}$.

\paragraph{Analyzing the Posterior of $\optpi$.} To bound Eq.~\eqref{eq:exists-path}, we apply chain rule and a change of measure argument. 
\begin{align*}
    \hspace{2em}&\hspace{-2em} \ind{\cE_{R,t-1} \wedge \eventfresh \wedge \eventnew}\Pr^{\nu_0, \alg} \brk*{ \substack{\text{exists a path $X_1 \to X_{t,H}$ in $\cF_{t,H-1}$} \\ \text{labeled by $\optpi_{1:H-1}$ }} \mid \cF_{t, H-1}} \\
    &= \ind{\cE_{R,t-1} \wedge \eventfresh \wedge \eventnew} \sum_{\pi \in \Pi_{1:H-1}} \Pr^{\nu_0, \alg} \brk*{ \substack{\text{exists a path $X_1 \to X_{t,H}$ in $\cF_{t,H-1}$} \\ \text{labeled by $\pi$ }} \mid \cF_{t, H-1}}  \cdot \Pr^{\nu_0, \alg} \brk*{ \optpi = \pi \mid \cF_{t, H-1}} \\
    &\le \frac{HT^2}{2^{H-6}} + \frac{1}{\abs{\Pi_{1:H-1} } } \sum_{ \pi \in \Pi_{1:H-1} } \Pr^{\nu_0, \alg} \brk*{ \substack{\text{exists a path $X_1 \to X_{t,H}$ in $\cF_{t,H-1}$} \\ \text{labeled by $\pi$ }} \mid \cF_{t, H-1}} \le \frac{HT^2}{2^{H-7}}. \numberthis \label{eq:final}
\end{align*}
The first inequality follows by the calculation in \pref{lem:posterior-of-optimal-generative}, and the second inequality follows because there are at most $T$ paths in the DAG representation of $\cF_{t,H-1}$.

\paragraph{Completing Induction for \pref{claim:claim-reward}.} By combining Eqs.~\eqref{eq:induction-one-step}--\eqref{eq:final} we see that as long as $T \le 2^{O(H)}$, then
\begin{align*}
    \Pr^{\nu_0, \alg} \brk*{\cE_{R,t}^c} \le (t-1) \cdot \frac{HT^2}{2^{H-7}} + \frac{HT^2}{2^{H-7}} = t \cdot \frac{HT^2}{2^{H-7}}.
\end{align*}
This proves the claim.

\paragraph{Final Bounds for \pref{lem:reward-bound-lb1}.} To prove the first inequality, we have directly by \pref{claim:claim-reward}
\begin{align*}
\En^{\nu_0, \alg} \brk*{\Pr^{\nu_0, \alg} \brk*{ R_{t,H} = 1 \mid \cF_{t, H-1} } } \le \Pr^{\nu_0, \alg} \brk*{\cE_{R, t}^c} \le t \cdot \frac{HT^2}{2^{H-7}}.
\end{align*}
To prove the second inequality in the lemma statement, we can get a similar bound as Eq.~\eqref{eq:peeled-reward-eq}:
\begin{align*}
    \hspace{2em}&\hspace{-2em} \En^{\nu_0, \alg} \brk*{\Pr^{\nu_1, \alg} \brk*{ R_{t,H} = 1 \mid \cF_{t, H-1} } } \\
    &\le \frac{HT^2}{2^{H-2}} + \En^{\nu_0, \alg} \brk*{\ind{\eventfresh \wedge \eventnew }\cdot
    \Pr^{\nu_1, \alg} \brk*{ R_{t,H} = 1 \mid \cF_{t, H-1} } } \\
    &\le \frac{HT^2}{2^{H-2}} + \Pr^{\nu_0, \alg} \brk*{ \cE_{R, t-1}^c \wedge \eventfresh \wedge \eventnew} + \En^{\nu_0, \alg} \brk*{\ind{\cE_{R,t-1}^c \wedge \eventfresh \wedge \eventnew }\cdot
    \Pr^{\nu_1, \alg} \brk*{ R_{t,H} = 1 \mid \cF_{t, H-1} } } \\
    &\le \frac{HT^2}{2^{H-2}} + (t-1) \frac{HT^2}{2^{H-5}} + \En^{\nu_0, \alg} \brk*{\ind{\cE_{R,t-1}^c \wedge \eventfresh \wedge \eventnew }\cdot
    \Pr^{\nu_1, \alg} \brk*{ R_{t,H} = 1 \mid \cF_{t, H-1} } },
\end{align*}
and from here one can replicate the above argument to get a bound on this quantity. The details are omitted. This concludes the proof of \pref{lem:reward-bound-lb1}.
\end{proof}

\begin{lemma}\label{lem:fresh-states}
    $\Pr^{\nu_0, \alg} \brk{\eventfresh^c} \le \frac{HT\cdot  2^H}{2^{2H}-2T}$.
\end{lemma}
\begin{proof}
Let us consider the set $\cI$ (which is a random variable that depends on the interaction of $\alg$ with $\nu_0$):
\begin{align*}
    \cI = \crl{(t,h): X_{t,h} \notin \cF_{t,h-1} }.
\end{align*}
We have
\begin{align*}
    \Pr^{\nu_0, \alg} \brk*{ \eventfresh^c } &\le \En^{\nu_0, \alg} \brk*{ \sum_{t=1}^T \sum_{h=2}^H \ind{(t,h) \in \cI \text{ and } \optdec(X_{t,h}) = \sgood}} \\
    &=  \sum_{t=1}^T \sum_{h=2}^H \En^{\nu_0, \alg} \brk*{\Pr\brk*{ (t,h) \in \cI \text{ and } \optdec(X_{t,h}) = \sgood \mid \cF_{t,h-1}} }\\
    &\le \sum_{t=1}^T \sum_{h=2}^H \En^{\nu_0, \alg} \brk*{\Pr\brk*{ \optdec(X_{t,h}) = \sgood \mid \cF_{t,h-1}, (t,h) \in \cI} }. \numberthis\label{eq:ess-upper-bound}
\end{align*}
Now we will bound the quantity $\Pr\brk*{ \optdec(X_{t,h}) = \sgood \mid \cF_{t,h-1}, (t,h) \in \cI}$ for any $t \in [T]$, $h \ge 2$. Consider the annotated filtration
\begin{align*}
    \cF_{t,h-1}' \coloneqq \sigma\prn*{ \cF_{t, h-1}, \crl*{\phi(X): X \in \cF_{t,h-1}} }
\end{align*}
which includes the decoder label for all observations seen thus far. We compute that for any $t\in[T], h \ge 2$:
\begin{align*}
    \Pr\brk*{ \optdec(X_{t,h}) = \sgood \mid \cF_{t,h-1}', (t,h) \in \cI} = \frac{2^H - \abs{\crl{X \in \cF_{t,h-1} : \optdec(X) = \sgood}}}{2^{2H} - 2t+1}, \numberthis\label{eq:new-state-upper-bound}
\end{align*}
since once we have fixed the value of the decoder on the $2t-1$ seen examples at layer $h$, the label of a new state is uniform over all remaining possibilities.

Continuing the calcuation from Eq.~\eqref{eq:ess-upper-bound}:
\begin{align*}
    \Pr^{\nu_0, \alg} \brk*{ \eventfresh^c }
    &\le \sum_{t=1}^T \sum_{h=2}^H \En^{\nu_0, \alg} \brk*{\Pr\brk*{ \optdec(X_{t,h}) = \sgood \mid \cF_{t,h-1}, (t,h) \in \cI} } \\
    &=  \sum_{t=1}^T \sum_{h=2}^H \En^{\nu_0, \alg} \brk*{ \En \brk*{ \Pr\brk*{ \optdec(X_{t,h}) = \sgood \mid \cF_{t,h-1}', (t,h) \in \cI} \mid \cF_{t,h-1}, (t,h) \in \cI } } \\
    &\le \sum_{t=1}^T \sum_{h=2}^H \frac{2^H}{2^{2H}-2T} \le \frac{HT \cdot 2^H}{2^{2H}-2T}.
\end{align*}
The second inequality uses Eq.~\pref{eq:new-state-upper-bound}. This completes the proof of \pref{lem:fresh-states}.
\end{proof}

\begin{lemma}\label{lem:repeated-transitions}
    $\Pr^{\nu_0, \alg} \brk{\eventnew^c} \le \frac{HT^2}{2^H}$.
\end{lemma}
\begin{proof}
Any sampled transition $X_{t,h}'$ has probability at most $T/2^H$ of being a repeated state (which is maximized if $X_{t,h}'$ has a good label and we have already sampled $T$ such observations from that given latent). Applying union bound over $T(H-1)$ transition samples gives us the final bound.
\end{proof}

\begin{lemma}[Posterior of $\optpi$]\label{lem:posterior-of-optimal-generative} 
    Fix any $t \in [T]$. Then 
    \begin{align*}
        \ind{\cE_{R,t-1} \wedge \eventfresh \wedge \eventnew} \cdot \nrm*{\Pr^{\nu_0, \alg} \brk*{\optpi = \cdot \mid \cF_{t, H-1}} - \unif(\Pi_{1:H-1})}_1 \le \frac{HT^2}{2^{H-6}}.
    \end{align*} 
\end{lemma}
\begin{proof}
In what follows all of the probabilities are taken with respect to $\Pr^{\nu_0, \alg}$. We can compute that
\begin{align*}
    \hspace{2em}&\hspace{-2em} \ind{\cE_{R,t-1} \wedge \eventfresh \wedge \eventnew} \cdot \nrm*{\Pr\brk*{\pi^\star = \cdot \mid \cF_{t, H-1} } - \unif(\Pi_{1:H-1})}_1 \\
    &= \ind{\cE_{R,t-1} \wedge \eventfresh \wedge \eventnew} \cdot \sum_{\pi \in \Pi_{1:H-1}} \abs*{ \Pr\brk*{\pi^\star = \pi \mid \cF_{t, H-1} } - \frac{1}{2^{H-1}} }\\
    &= \ind{\cE_{R,t-1} \wedge \eventfresh \wedge \eventnew} \cdot 2\sum_{\pi \in \Pi_{1:H-1}} \brk*{ \Pr\brk*{\pi^\star = \pi \mid \cF_{t, H-1} } - \frac{1}{2^{H-1}} }_{+} \\
    &= \ind{\cE_{R,t-1} \wedge \eventfresh \wedge \eventnew} \cdot \frac{2}{2^{H-1}} \sum_{\pi \in \Pi_{1:H-1}} \brk*{ \frac{ \Pr\brk*{\cF_{t, H-1} \mid \pi^\star = \pi } }{  \Pr\brk*{  \cF_{t, H-1} } } - 1 }_{+}\\
    &\le 2 \max_{\pi \in \Pi_{1:H-1}} \brk*{ \frac{\ind{\cE_{R,t-1} \wedge \eventfresh \wedge \eventnew} \cdot \Pr\brk*{\cF_{t, H-1} \mid \pi^\star = \pi } }{  \Pr\brk*{  \cF_{t, H-1} } } - 1 }_{+}. \numberthis \label{eq:posterior-ub-for-generative-model} 
\end{align*}
Now we will provide explicit calculations for the conditional distribution of $\cF_{t, H-1}$ for every choice of optimal policy $\pi \in \Pi_{1:H-1}$. Fix any $\cF_{t, H-1}$ such that $R_{i,H} = 0$ for all $i \in [t-1]$ and no repeated transitions (otherwise we can trivially upper bound Eq.~\eqref{eq:posterior-ub-for-generative-model} by 0). By chain rule we have
\begin{align*}
    \Pr\brk*{\cF_{t, H-1} \mid \pi^\star = \pi} = \prn*{ \prod_{i=1}^{t} \prod_{h=1}^{H-1} \Pr\brk*{X'_{i,h} \mid \optpi = \pi, \cF_{i, h-1} } } \times \prod_{i=1}^{t-1} \Pr\brk*{ R_{i,H} \mid \optpi = \pi, \cF_{i, H-1} }.
\end{align*}
We bound the transition and reward probabilities separately using \pref{claim:transition} and \pref{claim:reward}.
\begin{claim}\label{claim:transition}
Fix any $i \in [t]$ and $h \in [H-1]$. We have for every $\pi \in \Pi$:
\begin{align*}
    \Pr\brk*{X'_{i,h} \mid \optpi = \pi, \cF_{i, h-1} } \in \frac{1}{2^{2H}} \cdot \brk*{ \prn*{1 -  \frac{T}{2^H}}, \prn*{1 +  \frac{T}{2^H}} }.
\end{align*}
\end{claim}

To prove this claim, we can compute that
\begin{align*}
    \hspace{2em}&\hspace{-2em}\Pr\brk*{ X'_{i,h} \mid \optpi = \pi, \cF_{i, h-1} } = \sum_{\ell \in \crl{\sgood, \sbad}} 
    \Pr\brk*{ X'_{i,h} \mid \optpi = \pi, \cF_{i, h-1}, \phi_h(X_{i,h}) = \ell} 
    \Pr\brk*{\phi_h(X_{i,h}) = \ell \mid \optpi = \pi, \cF_{i, h-1}}
\end{align*}
\emph{Case 1:  if $A_{i,h} = \optpi_h$.} If we started in a good state then we would transition to the good state, so
\begin{align*}
\Pr\brk*{ X'_{i,h} \mid \optpi = \pi, \cF_{i, h-1} }
    &= \Pr \brk*{ \phi_h(X_{i,h}) = \mathsf{g} \mid \optpi = \pi, \cF_{i, h-1}} 
    \cdot \frac{ \Pr \brk*{ \phi_{h+1}(X_{i,h}') = \mathsf{g} \mid \optpi = \pi, \cF_{i, h-1}, \phi_h(X_{i,h}) = \sgood }}{2^{2H} - 2^H} \\ 
    &\quad + \Pr \brk*{ \phi_h(X_{i,h}) = \sbad \mid \optpi = \pi, \cF_{i, h-1}} \cdot \frac{ \Pr \brk*{ \phi_{h+1}(X_{i,h}') = \sbad \mid \optpi = \pi, \cF_{i, h-1}, \phi_h(X_{i,h}) = \sbad} }{2^{H}} 
\end{align*}
\emph{Case 2:  if $A_{i,h} \ne \optpi_h$.} In this case we know that regardless of the label of $X_{i,h}$ we transition to a bad state, so
    \begin{align*}
        \Pr\brk*{ X'_{i,h} \mid \optpi = \pi, \cF_{i, h-1} }
        &= \Pr \brk*{ \phi_h(X_{i,h}) = \mathsf{g} \mid \optpi = \pi, \cF_{i, h-1}} 
        \cdot \frac{ \Pr \brk*{ \phi_{h+1}(X_{i,h}') = \sbad \mid \optpi = \pi, \cF_{i, h-1}, \phi_h(X_{i,h}) = \sgood  }}{2^{2H} - 2^H} \\
        &\quad + \Pr \brk*{ \phi_h(X_{i,h}) = \mathsf{b} \mid \optpi = \pi, \cF_{i, h-1}} \cdot \frac{ \Pr \brk*{ \phi_{h+1}(X_{i,h}') = \sbad \mid \optpi = \pi, \cF_{i, h-1}, \phi_h(X_{i,h}) = \sbad} }{2^{2H} - 2^H}
    \end{align*}
Either way, applying \pref{lem:latent-calculation-generative} concludes the proof of \pref{claim:transition}.

\begin{claim}\label{claim:reward}
    Fix any $i \in [t-1]$. We have for every $\pi \in \Pi$:
    \begin{align*}
        \ind{\eventfresh} \Pr\brk*{ R_{i,H} = 1 \mid \optpi = \pi, \cF_{i, H-1} } \le \ind{  \substack{\text{exists a path $X_1 \to X_H$ in $\cF_{t,H-1}$} \\ \text{labeled by $\pi$ }} }.
\end{align*}
\end{claim}

To prove this claim, we use casework.

\emph{Case 1: if $\rootlayer(X_{i,H} \mid \cF_{i, H-1}) \ge 2$.} Then we must have
\begin{align*}
    \ind{\eventfresh} \cdot \Pr\brk*{ R_{i,H} = 1 \mid \optpi = \pi, \cF_{i, H-1} } &\le \ind{\eventfresh} \Pr\brk*{ \phi(\mathsf{Root}(X_{i,H})) = \mathsf{g} \mid \optpi = \pi, \cF_{i, H-1} } = 0.
\end{align*}
The equality holds because $\eventfresh \Rightarrow \crl{\phi(\mathsf{Root}(X_{i,H})) = \sbad}$. This proves \pref{claim:reward} in this case.

\emph{Case 2: if $\rootlayer(X_{i,h} \mid \cF_{i, H-1}) = 1$.} In this case we can compare the path witnessing $\rootlayer = 1$ with the labeling $\optpi$, and we get
\begin{align*}
    \Pr\brk*{ R_{i,H} = 1 \mid \optpi = \pi, \cF_{i, H-1} } \le \ind{  \substack{\text{exists a path $X_1 \to X_H$ in $\cF_{t,H-1}$} \\ \text{labeled by $\pi$ }} }.
\end{align*}
This concludes the proof of \pref{claim:reward}.

With \pref{claim:transition} and \pref{claim:reward} in hand, we return to the analysis of the posterior $\Pr\brk*{\cF_{t, H-1}\mid \pi^\star = \pi}$. Letting $O \coloneqq t(H-1)$ be the number of transitions we observe in $\cF_{t, H-1}$, we get that
\begin{align*}
    \Pr\brk*{\cF_{t, H-1}\mid \pi^\star = \pi} \le \frac{1}{2^{2H \cdot O}} \prn*{1 +  \frac{T}{2^H}}^{HT}. \numberthis \label{eq:ft-upper}
\end{align*}
We also have the lower bound that
\begin{align*}
    \Pr\brk*{  \cF_{t, H-1}} \ge \frac{1}{2^{H-1}} \sum_{\pi \in \Pi_{1:H-1}} \Pr\brk*{\cF_{t, H-1}\mid \pi^\star = \pi} \ge \frac{2^{H-1} - T}{2^{H-1}}  \cdot \frac{1}{2^{2H \cdot O}} \prn*{1 -  \frac{T}{2^H}}^{HT}, \numberthis \label{eq:ft-lower}
\end{align*}
where the last inequality follows because for any filtration $\cF_{t, H-1}$ we must have $\ind{  \substack{\text{no path $X_1 \to X_H$ in $\cF_{t,H-1}$} \\ \text{labeled by $\pi$ }} } = 1$ for at least $2^{H-1} - T$ such policies in $\Pi_{1:H-1}$.

Putting Eq.~\eqref{eq:ft-upper} and \eqref{eq:ft-lower} together we get that
\begin{align*}
    \ind{\cE_{R,t-1} \wedge \eventfresh \wedge \eventnew} \cdot \frac{\Pr\brk*{\cF_{t, H-1}\mid \pi^\star = \pi } }{  \Pr\brk*{  \cF_{t, H-1}} } \le  \prn*{1 +  \frac{T}{2^{H-2}}}^{2HT+1},
\end{align*}
which in turn using Eq.~\eqref{eq:posterior-ub-for-generative-model} implies that
\begin{align*}
    \hspace{2em}&\hspace{-2em} \ind{\cE_{R,t-1} \wedge \eventfresh \wedge \eventnew} \cdot \nrm*{\Pr\brk*{\pi^\star = \cdot \mid \cF_{t, H-1} } - \unif(\Pi_{1:H-1})}_1 \\
    &\le 2 \max_{\pi \in \Pi_{1:H-1}} \brk*{ \frac{\ind{\cE_{R,t-1} \wedge \eventfresh \wedge \eventnew} \cdot \Pr\brk*{\cF_{t, H-1} \mid \pi^\star = \pi } }{  \Pr\brk*{  \cF_{t, H-1} } } - 1 }_{+} \le  2  \prn*{\prn*{1 +  \frac{T}{2^{H-2}}}^{2HT+1} -1 }\\
    &\le \frac{2HT^2 + T}{2^{H-3}} \exp \prn*{\frac{2HT^2 + T}{2^{H-2}}} \le \frac{HT^2}{2^{H-6}}.
\end{align*}
We use the numerical inequalities $1+y \le \exp(y)$ and $\exp(y) - 1 \le y \exp y$. This concludes the proof of \pref{lem:posterior-of-optimal-generative}.
\end{proof}

\begin{lemma}\label{lem:latent-calculation-generative}
    Let $\cF$ be any filtration of $HT$ generative model samples as well as annotations $\phi(x)$ for a subset of observations $x \in \cF$. Let $\pi \in \Pi_{1:H-1}$ be any policy. Fix any $h \ge 2$, and let $x_\mathrm{new} \in \cX_h - \cF$. Then
    \begin{align*}
        \abs*{ \Pr^{\nu_0, \alg} \brk*{  \phi(x_\mathrm{new}) = \sgood    \mid \cF, \optpi = \pi }  - \prn*{1 - \frac{1}{2^H}}} &\le \frac{T}{2^{H}}, \quad \text{and} \quad
        \abs*{ \Pr^{\nu_0, \alg} \brk*{ \phi(x_\mathrm{new}) = \sbad  \mid \cF, \optpi = \pi } - \frac{1}{2^H}} \le \frac{T}{2^{H}}.
    \end{align*}
\end{lemma}

\begin{proof}[Proof of \pref{lem:latent-calculation-generative}]

Let us denote $\cF'$ to be the completely annotated $\cF$ which includes all labels $\crl{\phi(X): X \in \cF}$. We will show that the conclusion of the lemma applies to every completion $\cF'$, and since 
\begin{align*}
    \Pr^{\nu_0, \alg} \brk*{  \phi(x_\mathrm{new}) = \cdot   \mid \cF, \optpi = \pi } = \En^{\nu_0, \alg} \brk*{ \Pr^{\nu_0, \alg} \brk*{  \phi(x_\mathrm{new}) = \cdot   \mid \cF', \optpi = \pi } \mid \cF, \optpi = \pi},
\end{align*}
this will imply the result by Jensen's inequality and convexity of $\abs{\cdot}$.

We calculate the good label probability:
\begin{align*}
    \Pr^{\nu_0, \alg} \brk*{  \phi(x_\mathrm{new}) = \sgood    \mid \cF', \optpi = \pi } = \frac{2^H - \abs{\crl*{X \in \cF: \phi(X) = \sgood}}}{2^{2H} - \abs{\cF}}.
\end{align*}
For the lower bound we have
\begin{align*}
    \frac{2^H - \abs{ \crl*{X \in \cF: \phi(X) = \sgood}}}{2^{2H} - \abs{\cF}} \ge \frac{2^H - T}{2^{2H} } =  \frac{1}{2^H}\cdot \prn*{ 1 - \frac{T}{{2^{H}}} }.
\end{align*}
For the upper bound we have
\begin{align*}
    \frac{2^H - \abs{\crl*{X \in \cF: \phi(X) = \sgood}}}{2^{2H} - \abs{\cF}} \le \frac{2^H}{2^{2H} - T} =  \frac{1}{2^H}\cdot \prn*{ 1- \frac{T}{2^{2H}} }^{-1} \le \frac{1}{2^H}\cdot  \prn*{ 1 + \frac{T}{{2^{H}}} },
\end{align*}
which holds as long as $T \le 2^{H}$. Combining both upper and lower bounds proves the lemma for the good label. The calculation for $ \phi(x_\mathrm{new}) = \sbad$ is similar, so we omit it. This concludes the proof of \pref{lem:latent-calculation-generative}.
\end{proof}

\subsection{Proof of \pref{thm:lower-bound-policy-completeness}}\label{app:lower-bound-policy-completeness}

The lower bound constructions for the proof of \pref{thm:lower-bound-policy-completeness} have a similar flavor to the lower bound construction in \pref{thm:lower-bound-coverability}, but with a twist. In every layer, we include an additional distractor state $s^\sdis$ which is not reachable from $d_1$ but still sampled by $\mu$. The optimal policy at $s^\sdis_H$ is the opposite of the optimal policy from the states on the good chain $s^\sgood_H$, and given only online access, the learner cannot distinguish between the good states and the distractor states.

\paragraph{Lower Bound Construction.} Again, the policy class $\Pi$ is taken to be open loop policies:
\begin{align*}
    \Pi \coloneqq \crl{\pi: \forall x \in \statesp_h, \pi_h(x) \equiv a_h, (a_1, \cdots, a_H) \in \actionsp^H}.
\end{align*}
We define a family of Block MDPs $\cM = \crl{M_{\optpi, \optdec}}_{\optpi \in \Pi, \optdec \in \Phi}$ which are parameterized by an optimal policy $\optpi \in \Pi$ and a decoding function $\optdec \in \Phi$ (to be described). An example is illustrated in \pref{fig:lb2}.
\begin{itemize}
    \item \textbf{Latent MDP:} The latent state space $\latentsp$ is layered where each $\latentsp_h \coloneqq \crl{ s_h^\sgood, s_h^\sbad, s_h^\sdis}$ is comprised of a good, a bad, and a distractor state. We abbreviate the state as $\crl{\sgood, \sbad, \sdis}$ if the layer $h$ is clear from context. The starting state is always $\sgood$. The action space $\actionsp = \crl{0,1}$. Let $\optpi \in \Pi$ be any policy, which can be represented by a vector in $(\pi^\star_1, \cdots, \pi^\star_H) \in \crl{0,1}^H$. The latent transitions/rewards of an MDP parameterized by $\optpi \in \Pi$ are as follows for every $h \in [H]$: 
    \begin{align*}
        \optlatp(\cdot  \mid  s, a) = \begin{cases}
            \delta_{s_{h+1}^\sgood} & \text{if}~s = s_h^\sgood, a = \pi^\star_h\\
            \delta_{s_{h+1}^\sdis} &\text{if}~s = s_h^\sdis, a = \pi^\star_h\\
            \delta_{s_{h+1}^\sbad}& \text{otherwise}.
        \end{cases} \quad \text{and} \quad 
        \optlatr(s,a) = \begin{cases}
            1 &\text{if}~s = s_H^\sgood, a = \pi^\star_H\\
            1 &\text{if}~s = s_H^\sdis, a \ne \pi^\star_H\\
            \ber\prn*{\frac12} &\text{if}~s = s_H^\sbad\\
            0 &\text{otherwise.}
        \end{cases}
    \end{align*} 
    \item \textbf{Rich Observations:} The observation state space $\statesp$ is layered where each $\cX_h \coloneqq \crl{x_h^{(1)}, \cdots, x_h^{(m)}}$ with $m = 2^{H+2}$. The decoding function class $\Phi$ is the collection of all decoders which for every $h \ge 2$ assigns $s_h^\sgood, s_h^\sdis$ to disjoint subsets of $\statesp_h$ of size $2^H$ and $s_h^\sbad$ to the rest:
    \begin{align*}
        \Phi \coloneqq \Big\{\optdec:~\statesp \mapsto \latentsp :~&\forall~x_1 \in \statesp_1,~\optdec(x_1) = \sgood, \\
        &\forall~h \ge 2,~ \abs*{ \crl*{x_h \in \statesp_h: \optdec(x_h) = \sgood }} = 2^H \text{ and } \abs*{ \crl*{x_h \in \statesp_h: \optdec(x_h) = \sdis }} = 2^H \Big\} ,\\
        &\hspace{-8em}\text{so that}~\abs{\Phi} = \prn*{{2^{H+2} \choose 2^H} \cdot {2^{H+2} - 2^H \choose 2^H} }^{H-1} = 2^{2^{\wt{O}(H)}}.
    \end{align*}
    In the MDP parameterized by $\optdec \in \Phi$, the emission for every $s \in \latentsp$ is $\emission(s) = \unif\prn*{\crl*{x \in \statesp_h: \optdec(x) = s}}$.
    \item \textbf{Exploratory Distribution:} The exploratory distribution $\mu = \crl{\mu_h}_{h \in [H]}$ is set to be $\mu_h = \unif\prn{\statesp_h}$.
\end{itemize}
We establish several facts about any $M_{\optpi, \optdec} \in \cM$ defined by the construction.
\begin{itemize}
    \item The distribution $\mu$ has bounded concentrability: $\cconc(\mu; \Pi, M) \le 4$.
    \item The policy class $\Pi$ does not satisfy realizability, since the optimal policy at layer $H$ requires one to take different actions depending on whether the latent state is $\sgood$ or $\sdis$.
\end{itemize}

\paragraph{Sample Complexity Lower Bound.} We will use \pref{thm:interactive-lecam-cvx} to prove our lower bound. First we need to instantiate the parameter space. We will let $\Theta \coloneqq \crl{(\optpi, \optdec):~\optpi \in \Pi, \optdec \in \Phi}$ so that $\cM = \crl{M_\theta}_{\theta \in \Theta} = \crl{M_{\optpi, \optdec}}_{\optpi \in \Pi, \optdec \in \Phi}$. We further denote the subsets 
\begin{align*}
    \Theta_0 &\coloneqq \crl{(\optpi, \optdec):~\optpi \in \Pi~\text{s.t.}~\optpi_H = 0, \optdec \in \Phi} \\
    \Theta_1 &\coloneqq \crl{(\optpi, \optdec):~\optpi \in \Pi~\text{s.t.}~\optpi_H = 1, \optdec \in \Phi}
\end{align*}
The observation space $\cY$ is defined as the set of observations over $T$ rounds as well as returned policy for an algorithm interacting with the MDP, i.e.,
\begin{align*}
    \cY \coloneqq (\statesp \times \actionsp \times [0,1])^{HT} \times \Pi.
\end{align*}
(As a convention, we can assume that each sample collected by $\alg$ in the MDP is of length $H$; if $\alg$ decides to rollout from $\mu_h$ at an intermediate layer $h \ge 2$ then we can simply append ``dummy states'' to the prefix of the trajectory, which does not change the analysis.)

For an observation $y \in \cY$ we define the final returned policy as $y^\pi$. The loss function is given by
\begin{align*}
    L((\optpi, \phi), y) \coloneqq \ind{\optpi \ne y^\pi}. 
\end{align*}
Then we have for any $y \in \cY$, $(\optpi_0, \phi_0) \in \Theta_0$, and $(\optpi_1, \phi_1) \in \Theta_1$ that
\begin{align*}
    L((\optpi_0, \phi_0), y) +  L((\optpi_1, \phi_1), y) \ge 1 \coloneqq 2\Delta,
\end{align*}
since the last bit of $y^\pi$ can be either 0 or 1, thus only matching exactly one of $\optpi_0$ and $\optpi_1$.

Now we are ready to apply \pref{thm:interactive-lecam-cvx}. We get that for any $\alg$, we must have
\begin{align*}
    \sup_{(\optpi, \optdec) \in \Pi \times \Phi} \En_{Y \sim \Pr^{M_{\optpi, \optdec}, \alg}} \brk*{V^\star - V^{\estpi}} &= \sup_{(\optpi, \optdec) \in \Pi \times \Phi} \En_{Y \sim \Pr^{M_{\optpi, \optdec}, \alg}} \brk*{\frac{1}{2} - \frac{1}{2} \ind{\optpi = Y^\pi} } \\
    &= \frac{1}{2} \cdot \sup_{(\optpi, \optdec) \in \Pi \times \Phi} \En_{Y \sim \Pr^{M_{\optpi, \optdec}, \alg}} \brk*{ L((\optpi, \phi), Y) } \\
    &\ge \frac{1}{8} \cdot \max_{\nu_0 \in \Delta(\Theta_0), \nu_1 \in \Delta(\Theta_1)} \prn*{1 - \dtv\prn*{\Pr^{\nu_0, \alg}, \Pr^{\nu_1, \alg}}} \\
    &\ge \frac{1}{8} \cdot \prn*{1 - \dtv\prn*{\Pr^{\unif(\Theta_0), \alg}, \Pr^{\unif(\Theta_1), \alg}}}. 
\end{align*}
It remains to compute an upper bound $\dtv\prn*{\Pr^{\unif(\Theta_0), \alg}, \Pr^{\unif(\Theta_1), \alg}}$ which holds for any $\alg$. This is accomplished by the following lemma.

\begin{lemma}\label{lem:tv-bound-mu-reset}
For any deterministic $\alg$ that adaptively collects $T = 2^{O(H)}$ samples via $\mu$-reset access, we have
\begin{align*}
    \dtv\prn*{\Pr^{\unif(\Theta_0), \alg}, \Pr^{\unif(\Theta_1), \alg}} \le \frac{T^4H}{2^{H-10}}.
\end{align*}
\end{lemma}
Plugging in \pref{lem:tv-bound-mu-reset}, we conclude that for any $\alg$ that collects $2^{cH}$ samples for sufficiently small constant $c > 0$ must be $1/16$-suboptimal in expectation. This concludes the proof of \pref{thm:lower-bound-policy-completeness}.\qed

\subsection{Proof of \pref{lem:tv-bound-mu-reset} (TV Distance Calculation for \pref{thm:lower-bound-policy-completeness})}

Since the proof is similar to that of \pref{lem:tv-bound-generative} we omit some intermediate calculations. In the rest of the proof we denote $\nu_0 \coloneqq \unif(\Theta_0)$ and $\nu_1 \coloneqq \unif(\Theta_1)$. We have
\begin{align*}
     \hspace{2em}&\hspace{-2em} \dtv\prn*{\Pr^{\nu_0, \alg}, \Pr^{\nu_1, \alg}} \\
     &= \underbrace{ \sum_{t=1}^T \sum_{h=1}^{H-1} \En^{\nu_0, \alg} \brk*{\dtv \prn*{\Pr^{\nu_0, \alg} \brk*{ X_{t,h+1} \mid X_{t,h}, A_{t,h}, \cF_{t,h-1} }, \Pr^{\nu_1, \alg} \brk*{ X_{t,h+1} \mid X_{t,h}, A_{t,h}, \cF_{t,h-1} }}} }_{\text{transition TV distance}} \\
     &\qquad + 
     \underbrace{ \sum_{t=1}^T \En^{\nu_0, \alg} \brk*{\dtv \prn*{\Pr^{\nu_0, \alg} \brk*{ R_{t,H} \mid X_{t,H}, A_{t,H}, \cF_{t,H-1} }, \Pr^{\nu_1, \alg} \brk*{ R_{t,H} \mid X_{t,H}, A_{t,H}, \cF_{t,H-1} }}} }_{\text{reward TV distance}}.
\end{align*}

We bound each term separately.

\paragraph{Transition TV Distance.} Using triangle inequality and \pref{lem:tv-distance-from-uniform-resets} we get that for any $t \in [T], h \in [H-1]$:
\begin{align*}
 \En^{\nu_0, \alg} \brk*{\dtv \prn*{ \Pr^{\nu_0, \alg} \brk*{ X_{t,h+1} \mid X_{t,h}, A_{t,h}, \cF_{t, h-1} }, \Pr^{\nu_1, \alg} \brk*{X_{t,h+1} \mid X_{t,h}, A_{t,h}, \cF_{t, h-1} } } } \le \frac{t}{2^{H-3}}.  \numberthis \label{eq:ub-transition-calculation-reset}  
\end{align*}

\paragraph{Reward TV Distance.} Using triangle inequality, the fact that rewards are in $\crl{0,1}$, and \pref{lem:reward-tv-distance-policy-completeness} we get
\begin{align*}
    & \En^{\nu_0, \alg} \brk*{\dtv \prn*{\Pr^{\nu_0, \alg} \brk*{ R_{t,H} \mid X_{t,H}, A_{t,H}, \cF_{t,H-1} }, \Pr^{\nu_1, \alg} \brk*{ R_{t,H} \mid X_{t,H}, A_{t,H}, \cF_{t,H-1} }}} \\
    &\le \En^{\nu_0, \alg} \brk*{\abs*{\Pr^{\nu_0, \alg} \brk*{ R_{t,H} = 1\mid X_{t,H}, A_{t,H}, \cF_{t, H-1}} - \frac{1}{2} } } + \En^{\nu_0, \alg} \brk*{\abs*{\Pr^{\nu_1, \alg} \brk*{ R_{t,H} = 1 \mid X_{t,H}, A_{t,H}, \cF_{t,H-1}} - \frac{1}{2} } } \\
    &\le  t \cdot \frac{T^2H}{2^{H-9}}. \numberthis\label{eq:reward-triangle-inequality}
\end{align*}

\paragraph{Final Bound.} Thus, combining Eqs.~\eqref{eq:ub-transition-calculation-reset} and \eqref{eq:reward-triangle-inequality} we can conclude that:
\begin{align*}
    \dtv\prn*{\Pr^{\nu_0, \alg}, \Pr^{\nu_1, \alg}} \le \frac{T^2H}{2^{H-3}} + \frac{T^4H}{2^{H-9}} \le \frac{T^4H}{2^{H-10}}.
\end{align*}
This concludes the proof of \pref{lem:tv-bound-mu-reset}.\qed

\begin{lemma}[Transition TV Distance for the Construction in \pref{thm:lower-bound-policy-completeness}]\label{lem:tv-distance-from-uniform-resets}
    For any $t \in [T], h \in [H]$, we have
    \begin{align*}
        \nrm*{\Pr^{\nu_0, \alg}\brk*{X_{t,h}  \mid  \cF_{t, h-1} } - \unif(\statesp_{h}) }_1 &\le \frac{t}{2^{H-2}}, \\
        \nrm*{\Pr^{\nu_1, \alg}\brk*{X_{t,h}  \mid  \cF_{t, h-1}} - \unif(\statesp_{h})}_1 &\le \frac{t}{2^{H-2}}.
    \end{align*}
    \end{lemma}
    
    \begin{proof}[Proof of \pref{lem:tv-distance-from-uniform-resets}]
    We prove the bound for $\nu_0$, as the proof for $\nu_1$ is identical. If we sample $X_{t,h}$ directly from the $\mu$-reset distribution, then the result immediately follows since the distribution of $X_{t,h} = \unif(\cX_h)$. Otherwise, denote
    \begin{align*}
        \cF_{t, h-1}' = \sigma( \cF_{t, h-1}, \crl{\optdec(X): X \in \cF_{t, h-1}}, \crl{\ind{A = \optpi(X)}: (X,A) \in \cF_{t, h-1}} )
    \end{align*}
    to be the annotated sigma-field which also includes the latent state labels for all of the previous observions as well as whether the action taken followed $\optpi$ or not. Let us denote $\ell = \phi(X_{t,h}) \in \crl{\sgood, \sbad, \sdis}$ to be the latent state label of the next observation. Observe that the label $\ell$ is measurable with respect to $\cF_{t,h-1}'$ since the filtration $\cF_{t-1}'$ includes $\optdec(X_{t,h-1})$ as well as $\ind{A_{t,h-1} = \optpi(X_{t,h-1})}$. Furthermore denote $\cX_\mathsf{obs}$ to denote the total number of observations that we have encountered already in layer $h$ and $\cX_\mathsf{obs}^\ell$ to denote the observations we have encountered whose label is $\ell$.
    
    Under the uniform distribution over decoders, the assignment of the remaining observations is equally likely. Therefore we can write the distribution of $X_{t,h}$ as: 
    \begin{align*}
        \text{if}~\ell = \sgood:&\quad \Pr^{\unif(\Theta_0), \alg} \brk*{X_{t,h} = x  \mid  \cF_{t,h-1}'} = \begin{cases}
            \frac{1}{2^H} &\text{if}~x \in \cX_\mathsf{obs}^\ell \\
            0 &\text{if}~x \in \cX_\mathsf{obs} - \cX_\mathsf{obs}^\ell \\
            \frac{1}{2^H} \cdot \frac{2^H - \abs{\cX_\mathsf{obs}^\ell}}{2^{H+2} - \abs{\cX_\mathsf{obs}}} &\text{if}~x \in \cX_h - \cX_\mathsf{obs}
        \end{cases}\\
        \text{if}~\ell = \sbad:&\quad \Pr^{\unif(\Theta_0), \alg} \brk*{X_{t,h} = x  \mid  \cF_{t,h-1}'} = \begin{cases}
            \frac{1}{2^{H+1}} &\text{if}~x \in \cX_\mathsf{obs}^\ell \\
            0 &\text{if}~x \in \cX_\mathsf{obs} - \cX_\mathsf{obs}^\ell \\
            \frac{1}{2^{H+1}} \cdot \frac{2^{H+1} - \abs{\cX_\mathsf{obs}^\ell}}{2^{H+2} - \abs{\cX_\mathsf{obs}}} &\text{if}~x \in \cX_h - \cX_\mathsf{obs}
        \end{cases}\\
        \text{if}~\ell = \sdis:&\quad \Pr^{\unif(\Theta_0), \alg} \brk*{X_{t,h} = x  \mid  \cF_{t,h-1}'} = \begin{cases}
            \frac{1}{2^H} &\text{if}~x \in \cX_\mathsf{obs}^\ell \\
            0 &\text{if}~x \in \cX_\mathsf{obs} - \cX_\mathsf{obs}^\ell \\
            \frac{1}{2^H} \cdot \frac{2^H - \abs{\cX_\mathsf{obs}^\ell}}{2^{H+2} - \abs{\cX_\mathsf{obs}}} &\text{if}~x \in \cX_h - \cX_\mathsf{obs}
        \end{cases}
    \end{align*}
    We elaborate on the calculation for the last probability in each case. Suppose $\ell = \sgood$. Then for any $x \in \cX_h - \cX_\mathsf{obs}$ which has not been observed yet we assign $\optdec(x) = \ell$ in
    \begin{align*}
        &{2^{H+2} - \abs{\cX_\mathsf{obs}} -1 \choose 2^{H} -\abs{\cX_\mathsf{obs}^\ell} - 1 } \text{ ways out of } {2^{H+2} - \abs{\cX_\mathsf{obs}} \choose 2^{H} -\abs{\cX_\mathsf{obs}^\ell } } \text{ assignments.} \\
        &\quad \Longrightarrow \optdec(x) = \sgood \text{ with probability } \frac{2^{H} -\abs{\cX_\mathsf{obs}^\ell}}{2^{H+2} - \abs{\cX_\mathsf{obs}}}.
    \end{align*}
    For each assignment where $\optdec(x) = \sgood$ we will select it with probability $1/2^H$ since the emission is uniform, giving us the final probability as claimed. A similar calculation can be done for the cases where $\ell = \sbad, \sdis$.
    
    Therefore we can calculate the final bound that
    \begin{align*}
        \nrm*{\Pr^{\nu_0, \alg}\brk*{X_{t,h}  \mid  \cF_{t, h-1}'} - \unif(\statesp_{h}) }_1 &= \sum_{x \in \cX_h} \abs*{\Pr^{\nu_0, \alg}\brk*{X_{t,h} = x \mid  \cF_{t, h-1}} - \frac{1}{2^{H+2}}} \\
        &\le \begin{cases}
            \frac{\abs{\cX_\mathsf{obs}^\sgood}}{2^H} + \frac{\abs{\cX_\mathsf{obs}^\sbad} + \abs{\cX_\mathsf{obs}^\sdis}}{2^{H+2}} + \abs*{\frac{2^H - \abs{\cX_\mathsf{obs}^\sgood}}{2^{H}} - \frac{2^{H+2} - \abs{\cX_\mathsf{obs}}}{2^{H+2}}} &\text{if}~\ell = \sgood, \\[0.5em]
            \frac{\abs{\cX_\mathsf{obs}^\sbad}}{2^{H+1}} + \frac{\abs{\cX_\mathsf{obs}^\sgood} + \abs{\cX_\mathsf{obs}^\sdis}}{2^{H+2}} + \abs*{\frac{2^{H+1} - \abs{\cX_\mathsf{obs}^\sbad}}{2^{H+1}} - \frac{2^{H+2} - \abs{\cX_\mathsf{obs}}}{2^{H+2}}} &\text{if}~\ell = \sbad, \\[0.5em]
           \frac{\abs{\cX_\mathsf{obs}^\sdis}}{2^H} + \frac{\abs{\cX_\mathsf{obs}^\sbad} + \abs{\cX_\mathsf{obs}^\sgood}}{2^{H+2}} + \abs*{\frac{2^H - \abs{\cX_\mathsf{obs}^\sgood}}{2^{H}} - \frac{2^{H+2} - \abs{\cX_\mathsf{obs}}}{2^{H+2}}} &\text{if}~\ell = \sdis.
        \end{cases}\\
        &\le \frac{4 \cdot \abs{\cX_\mathsf{obs}}}{2^H} \le \frac{4 t}{2^H}.
    \end{align*}
    Since $\Pr^{\nu_0, \alg}\brk*{X_{t,h}  \mid  \cF_{t, h-1}} = \En^{\nu_0, \alg} \Pr^{\nu_0, \alg}\brk{X_{t,h}  \mid  \cF_{t, h-1}'}$, we have by convexity of $\ell_1$ norm and Jensen's inequality,
    \begin{align*}
        \nrm*{\Pr^{\nu_0, \alg}\brk*{X_{t,h}  \mid  \cF_{t, h-1}} - \unif(\statesp_{h}) }_1 \le  \En^{\nu_0, \alg} \brk*{ \nrm*{\Pr^{\nu_0, \alg}\brk*{X_{t,h}  \mid  \cF_{t, h-1}'} - \unif(\statesp_{h}) }_1 } \le \frac{4 t}{2^H},
    \end{align*}
    which concludes the proof of \pref{lem:tv-distance-from-uniform-resets}.
\end{proof}

\begin{lemma}[Reward TV Distance for the Construction in \pref{thm:lower-bound-policy-completeness}]\label{lem:reward-tv-distance-policy-completeness}
For any $t \in [T]$ we have
    \begin{align*}
        \En^{\nu_0, \alg} \brk*{\abs*{\Pr^{\nu_0, \alg} \brk*{ R_{t,H} = 1\mid X_{t,H}, A_{t,H}, \cF_{t, H-1}} - \frac{1}{2} } } &\le t \cdot \frac{T^2H}{2^{H-8}}, \\
        \En^{\nu_0, \alg} \brk*{\abs*{\Pr^{\nu_1, \alg} \brk*{ R_{t,H} = 1\mid X_{t,H}, A_{t,H}, \cF_{t, H-1}} - \frac{1}{2} } } &\le  t \cdot \frac{T^2H}{2^{H-8}}.
    \end{align*}
\end{lemma}

\begin{proof}[Proof of \pref{lem:reward-tv-distance-policy-completeness}] Let us denote $\cfnor \coloneqq \sigma(X_{t,H}, A_{t,H}, \cF_{t, H-1})$. 

We will prove the first inequality of \pref{lem:reward-tv-distance-policy-completeness}; the second inequality is obtained using similar arguments.

\paragraph{Peeling Off Bad Event.} First, let us peel off the event that $\cfnor$ has repeated observations: denoting $\eventnew \coloneqq \crl{X_{t,h} \notin \cF_{t, h-1} ~\forall t \in [T], h \in [H]}$, we have
\begin{align*}
    \En^{\nu_0, \alg} \brk*{\abs*{\Pr^{\nu_0, \alg} \brk*{ R_{t,H} = 1\mid \cfnor } - \frac{1}{2} } } &\le \Pr^{\nu_0, \alg}[\eventnew^c] + \En^{\nu_0, \alg} \brk*{\ind{ \eventnew } \abs*{\Pr^{\nu_0, \alg} \brk*{ R_{t,H} = 1\mid \cfnor } - \frac{1}{2} } } \\
    &\le \frac{T^2 H}{2^H} + \En^{\nu_0, \alg} \brk*{ \ind{ \eventnew } \abs*{\Pr^{\nu_0, \alg} \brk*{ R_{t,H} = 1\mid \cfnor} - \frac{1}{2} } }, \numberthis\label{eq:peeled-eventnew}
\end{align*}
where the last inequality follows by an identical argument as \pref{lem:repeated-transitions}. Therefore, it suffices to bound the expectation only for the $\cfnor$ which have no repeated states. 

\paragraph{Inductive Claim.} Now we define the event $\cE_{R, t}$ to be the event that among the first $t$ episodes, $\alg$ never performs an online rollout (meaning it starts from layer 1) which follows $\optpi$, i.e.,
\begin{align*}
    \cE_{R,t} \coloneqq \crl*{ \forall t' \le t: A_{t, 1:H-1} \ne \optpi}.
\end{align*}
A subtle point is that unlike the reward TV distance calculation for \pref{thm:lower-bound-coverability}, the event $\cE_{R,t-1}$ is \emph{not measurable} with respect to $\cF_{t, H-1}$ (since there is still uncertainty as to what $\optpi$ is). This causes some technical complications in the proof. To remedy this, we can consider working with an augmented filtration which appends a special token $\top$ at the end of every online trajectory that $\alg$ takes if the sequence of actions $A_{t,1:H-1}$ matches $\optpi$; \emph{now} $\cE_{R, t-1}$ is measurable with respect to the augmented filtration (namely the event $\cE_{R,t-1}$ holds if the augmented contains no special tokens $\top$). This augmentation does not affect the overall argument, and for the rest of the proof we assume that $\cfnor$ has been augmented in this way. 

Central to our proof is the following claim that $\cE_{R,t} \cup \eventnew$ happens with high probability:
\begin{align*}
    \text{for all}~t \in [T]: \quad \Pr^{\nu_0, \alg} \brk*{ \cE_{R,t}^c \cup \eventnew} \le t \cdot \frac{T^2H}{2^{H-7}}. \numberthis \label{claim:inductive-online}
\end{align*}
Now we will establish \pref{claim:inductive-online} using an inductive argument. The base case of $t=0$ trivially holds. Now suppose that \pref{claim:inductive-online} holds at time $t-1$. Then
\begin{align*}
    \Pr^{\nu_0, \alg} \brk*{ \cE_{R,t}^c \cup \eventnew} &\le \Pr^{\nu_0, \alg} \brk*{ \cE_{R,t-1}^c \cup \eventnew} + \En^{\nu_0, \alg} \brk*{ \ind{\cE_{R,t-1}\cup \eventnew} \Pr^{\nu_0, \alg} \brk*{A_{t, 1:H-1} = \optpi \mid \cfnor} } \\
    &\le (t-1) \cdot \frac{T^2H}{2^{H-7}}+ \En^{\nu_0, \alg} \brk*{\ind{\cE_{R,t-1}\cup \eventnew} \Pr^{\nu_0, \alg} \brk*{A_{t, 1:H-1} = \optpi \mid \cfnor} } \\
    &\le (t-1) \cdot \frac{T^2H}{2^{H-7}} + \En^{\nu_0, \alg} \brk*{ \ind{\cE_{R,t-1}\cup \eventnew} \sum_{\pi \in \Pi_{1:H-1}} \ind{A_{t, 1:H-1} = \pi} \Pr^{\nu_0, \alg} \brk*{ \optpi = \pi \mid \cfnor} } \\
    &\le (t-1) \cdot \frac{T^2H}{2^{H-7}} + \frac{1}{2^{H-1}} + \frac{T^2H}{2^{H-6}} \le t \cdot \frac{T^2H}{2^{H-7}},
\end{align*}
Here, the second-to-last inequality uses \pref{lem:posterior-of-optimal} and the fact that $A_{t, H-1}$ can only match a single policy $\pi \in \Pi_{1:H-1}$.

\paragraph{Casework on Reward TV Distance.} Armed with \pref{claim:inductive-online}, we now return to the proof of the reward TV distance calculation. We consider two cases. In the first case, the $t$-th trajectory is generated by an online rollout from $h=1$ with the sequence of actions $A_{1:H}$. In the second case, the $t$-th trajectory is generated by first querying the $\mu$-reset model starting from $h_\bot \ge 2$, then rolling out with the sequence of actions $A_{h_\bot:H}$.

\emph{Case 1: Online Rollout from Layer 1.} First, we peel off the probability of $\cE_{R,t-1}$ occuring:
\begin{align*}
    \hspace{2em}&\hspace{-2em} \En^{\nu_0, \alg} \brk*{\ind{\eventnew} \abs*{\Pr^{\nu_0, \alg} \brk*{ R_{t,H} = 1\mid \cfnor} - \frac{1}{2} } } \\
    &\le \Pr^{\nu_0, \alg} \brk*{ \cE_{R,t-1}^c \cup \eventnew} + \En^{\nu_0, \alg} \brk*{\ind{\cE_{R,t-1}\cup \eventnew} \cdot \abs*{\Pr^{\nu_0, \alg} \brk*{ R_{t,H} = 1\mid \cfnor} - \frac{1}{2} } } \\
    &\le (t-1) \cdot \frac{T^2H}{2^{H-7}} + \En^{\nu_0, \alg} \brk*{\ind{\cE_{R,t-1}\cup \eventnew} \cdot \abs*{ \Pr^{\nu_0, \alg} \brk*{ R_{t,H} = 1\mid \cfnor} - \frac{1}{2} } }. 
    \numberthis \label{eq:reward-tv-induction-peel}
\end{align*}

Now we compute
\begin{align*}
    \hspace{2em}&\hspace{-2em} \ind{\cE_{R,t-1}\cup \eventnew} \Pr^{\nu_0, \alg} \brk*{ R_{t,H} = 1 \mid \cfnor } \\
    &= \ind{\cE_{R,t-1}\cup \eventnew} \En^{\nu_0, \alg} \brk*{ \ind{ \phi(X_{t,H}) = \sgood \wedge A_{t,H} = 0 } + \frac12 \ind{ \phi(X_{t,H}) = \sbad } \mid \cfnor  } \\
    &\overset{(i)}{=} \ind{\cE_{R,t-1}\cup \eventnew}  \prn*{ \Pr^{\nu_0, \alg} \brk*{ A_{t, 1:H} = \optpi \circ 0  \mid \cfnor  } + \frac12 \Pr^{\nu_0, \alg} \brk*{  A_{t,1:H-1} \ne \optpi  \mid \cfnor  } }\\
    &\overset{(ii)}{=} \sum_{\pi \in \Pi_{1:H-1}} \prn*{ \ind{A_{t, 1:H} = \pi \circ 0} + \frac{1}{2} \ind{A_{t, 1:H-1} \ne \pi} } \ind{\cE_{R,t-1}\cup \eventnew} \Pr^{\nu_0, \alg}\brk*{ \optpi = \pi \mid \cfnor,  \cE_{R,t-1} } \\
    &\overset{(iii)}{\le} \frac{T^2H}{2^{H-6}}  + \frac{\ind{\cE_{R,t-1}\cup \eventnew}}{2^{H-1}} \sum_{\pi \in \Pi_{1:H-1}} \prn*{ \ind{A_{t, 1:H} = \pi \circ 0} + \frac{1}{2} \ind{A_{t, 1:H-1} \ne \pi} } \\
    &\overset{(iv)}{\le} \frac{T^2H}{2^{H-7}} + \frac{\ind{\cE_{R,t-1}\cup \eventnew}}{2}. \numberthis \label{eq:case1-bound}
\end{align*}
For equality $(i)$ we use the fact that if the $X_{t,H}$ has a good label then we must have taken $\optpi$ for the first $H-1$ layers. For equality $(ii)$ we use the fact that the indicators are measurable with respect to $\cF_{t, H-1}$ and $\crl{\optpi = \pi}$. For $(iii)$ we apply a change-of-measure argument using \pref{lem:posterior-of-optimal}. For $(iv)$ we use the fact that the sequence of actions $A_{t,1:H-1}$ can match at exactly one of the policies in $\Pi_{1:H-1}$. 

Note that the other side of the inequality can be shown analogously. Therefore by plugging in Eq.~\eqref{eq:case1-bound} into \eqref{eq:reward-tv-induction-peel} we get the bound that
\begin{align*}
\En^{\nu_0, \alg} \brk*{\ind{\cE_{R,t-1}\cup \eventnew} \abs*{\Pr^{\nu_0, \alg} \brk*{ R_{t,H} = 1\mid \cfnor} - \frac{1}{2} } } \le t \cdot \frac{T^2H}{2^{H-7}}. \numberthis\label{eq:case1-bound-onreward}
\end{align*}

\emph{Case 2: $\mu$-Reset Rollout from Layer $h_\bot \ge 2$.} Let us analyze the second case. Using the construction details,
\begin{align*}
    &\Pr^{\nu_0, \alg} \brk*{ R_{t,H} = 1 \mid \cfnor} \\
    &= \En^{\nu_0, \alg} \brk*{ \ind{ \phi(X_{t,H}) = \sgood \wedge A_{t,H} = 0 } + \frac12 \ind{ \phi(X_{t,H}) = \sbad } + \ind{ \phi(X_{t,H}) = \sdis \wedge A_{t,H} = 1 } \mid \cfnor } \\
    &= \En^{\nu_0, \alg} \brk*{ \ind{ \phi(X_{t,h_\bot}) = \sgood \wedge A_{t, h_\bot:H} = \optpi \circ 0 } \mid \cfnor } \\
    &\qquad + \frac12 \En^{\nu_0, \alg} \brk*{ \ind{ A_{t,h_\bot:H-1} \ne \optpi } + \ind{\phi(X_{t,h_\bot}) = \sbad \wedge A_{t,h_\bot:H-1} = \optpi} \mid \cfnor }\\
    &\qquad + \En^{\nu_0, \alg} \brk*{ \ind{ \phi(X_{t,h_\bot}) = \sdis \wedge A_{t,h_\bot:H} = \optpi\circ 1 } \mid \cfnor } \\
    &= \sum_{\pi \in \Pi_{h_\bot:H-1}} \Pr^{\nu_0, \alg}\brk*{ \optpi = \pi \mid \cfnor } \bigg(  \ind{ A_{t, h_\bot:H} = \pi \circ 0 }
    \Pr^{\nu_0, \alg} \brk*{ \phi(X_{t,h_\bot}) = \sgood  \mid \cfnor, \optpi = \pi }  \\
    &\qquad +  \frac12 \cdot  \ind{ A_{t,h_\bot:H-1} \ne \pi } + \frac12 \cdot \ind{A_{t,h_\bot:H-1} = \pi} \Pr^{\nu_0, \alg} \brk*{  \phi(X_{t,h_\bot}) = \sbad  \mid \cfnor, \optpi = \pi } \\
    &\qquad +  \ind{A_{t,h_\bot:H} = \pi \circ 1} \Pr^{\nu_0, \alg} \brk*{ \phi(X_{t,h_\bot}) = \sdis \mid \cfnor, \optpi = \pi } \bigg).
\end{align*}
We apply \pref{lem:posterior-of-state-label-v2} separately to the terms inside the parentheses for every $\pi$. Then using a casework argument on the value of $A_{t, h_\bot:H}$ and then averaging over the posterior of $\optpi$ gives
\begin{align*}
    \ind{\eventnew} \abs*{ \Pr^{\nu_0, \alg} \brk*{ R_{t,H} = 1 \mid \cfnor} - \frac12 } \le \frac{TH}{2^{H-5}}. \numberthis\label{eq:case2-bound-onreward}
\end{align*}

\paragraph{Putting It Together.} To conclude, the worst-case TV distance is the maximum of the two bounds we have shown in Eqs.~\eqref{eq:case1-bound-onreward} and \eqref{eq:case2-bound-onreward}, so therefore plugging into Eq.~\eqref{eq:peeled-eventnew} we have
\begin{align*}
    \hspace{2em}&\hspace{-2em} \En^{\nu_0, \alg} \brk*{\abs*{\Pr^{\nu_0, \alg} \brk*{ R_{t,H} = 1\mid \cfnor} - \frac{1}{2} } } \le \frac{T^2 H}{2^H} + \max \crl*{ t \cdot \frac{T^2H}{2^{H-7}},  \frac{TH}{2^{H-5}} } \le t \cdot \frac{T^2H}{2^{H-8}}.
\end{align*}
The proof of second inequality is obtained similarly, as one just needs to change the law to be under $\Pr^{\nu_1, \alg}$ in the above argument. This concludes the proof of \pref{lem:reward-tv-distance-policy-completeness}.
\end{proof}

\begin{lemma}[Posterior of $\optpi$]\label{lem:posterior-of-optimal}
Fix any $t \in [T]$. Assume that $\cfnor$ contains no repeated states. Then 
\begin{align*}
    \ind{\cE_{R,t-1}} \cdot \nrm*{\Pr^{\nu_0, \alg} \brk*{\optpi = \cdot \mid \cfnor} - \unif(\Pi_{1:H-1})}_1 \le \frac{T^2H}{2^{H-6}}.
\end{align*}
\end{lemma}

\begin{proof}
In what follows, all of the probabilities $\Pr[\cdot] \coloneqq \Pr^{\nu_0, \alg}[\cdot]$. Let $[x]_{+} \coloneqq \max\crl{x, 0}$. We can compute that
\begin{align*}
    \ind{\cE_{R,t-1}} \cdot \nrm*{\Pr\brk*{\pi^\star = \cdot \mid \cfnor } - \unif(\Pi_{1:H-1})}_1 &= \ind{\cE_{R,t-1}} \cdot \sum_{\pi \in \Pi_{1:H-1}} \abs*{ \Pr\brk*{\pi^\star = \pi \mid \cfnor } - \frac{1}{2^{H-1}} }\\
    &= \ind{\cE_{R,t-1}} \cdot 2\sum_{\pi \in \Pi_{1:H-1}} \brk*{ \Pr\brk*{\pi^\star = \pi \mid \cfnor } - \frac{1}{2^{H-1}} }_{+} \\
    &= \ind{\cE_{R,t-1}} \cdot \frac{2}{2^{H-1}} \sum_{\pi \in \Pi_{1:H-1}} \brk*{ \frac{ \Pr\brk*{\cfnor \mid \pi^\star = \pi } }{  \Pr\brk*{  \cfnor } } - 1 }_{+}\\
    &\le 2 \max_{\pi \in \Pi_{1:H-1}} \brk*{ \frac{\ind{\cE_{R,t-1}} \cdot \Pr\brk*{\cfnor \mid \pi^\star = \pi } }{  \Pr\brk*{  \cfnor } } - 1 }_{+}. \numberthis \label{eq:upper-bound-on-posterior-tv}
\end{align*}
We now proceed by explicitly calculating the conditional distribution of $\cfnor$ for every choice of optimal policy $\pi \in \Pi_{1:H-1}$. We will show that regardless of the choice $\pi \in \Pi_{1:H-1}$, the conditional distribution looks roughly like the uniform distribution over observations with a $\ber(1/2)$ reward at the end of every trajectory.

First, we will break up the distribution into trajectories:
\begin{align*}
    \Pr\brk*{\cfnor \mid \pi^\star = \pi } = \prn*{ \prod_{i < t} \Pr \brk*{ \tau_i \mid \pi^\star = \pi , \cF_{i-1} } } \cdot \Pr \brk*{ (X_{t,h_\bot:H}, A_{t,h_\bot:H}) \mid \pi^\star = \pi , \cF_{t-1}}. \numberthis\label{eq:prob-filtration-given-pi}
\end{align*}

\begin{claim}\label{claim:claim1}
Fix any $i \in [t]$. If $\tau_i$ is generated by sampling the $\mu$-reset distribution at some layer $h_\bot \ge 2$ and then rolling out, we have for every $\pi \in \Pi_{1:H-1}$,
\begin{align*}
    \Pr \brk*{ \tau_i \mid \pi^\star = \pi , \cF_{i-1} }  \in \frac{1}{2} \cdot \frac{1}{2^{(H+2) \cdot (H - h_\bot + 1)}} \cdot \brk*{ \prn*{1 - \frac{T}{2^H}}^{H-1}, \prn*{1 + \frac{T}{2^H}}^{H-1} }
\end{align*} 
\end{claim}

We start by showing \pref{claim:claim1}. In our proof, we assume that $h_\bot = 2$ (the proof is easy to adapt to any $h_\bot \ge 2$ with minor modification). Fix any index $i \in [t]$ and let $\tau_i = (X_{2:H}, A_{2:H}, R)$ be the $i$-th trajectory. Also fix any $\pi \in \Pi_{1:H-1}$. Then we can calculate
\begin{align*}
    \hspace{2em}&\hspace{-2em} \Pr\brk*{ (X_{2:H}, A_{2:H}, R) \mid \pi^\star = \pi , \cF_{i-1}} \\
    &= \sum_{\phi_2} \Pr \brk*{ (X_{2:H}, A_{2:H}, R) \mid \pi^\star = \pi , \cF_{i-1}, \phi_2} \Pr\brk*{\phi_2 \mid \pi^\star= \pi , \cF_{i-1}} \\
    &= \sum_{\ell \in \crl{\sgood, \sbad, \sdis}}\sum_{\phi_2: \phi_2(X_2) = \ell} \Pr \brk*{ (X_{2:H}, A_{2:H}, R) \mid \pi^\star = \pi , \cF_{i-1}, \phi_2} \Pr\brk*{\phi_2 \mid \pi^\star= \pi , \cF_{i-1}}.
\end{align*}
We can separately analyze the sum for the different choices of the label of the initial state $X_2$. First, we do the case where $\phi_2(X_2) = \sbad$:
\begin{align*}
    \hspace{2em}&\hspace{-2em} \sum_{\phi_2: \phi_2(X_2) = \sbad} \Pr  \brk*{ (X_{2:H}, A_{2:H}, R) \mid \pi^\star = \pi , \cF_{i-1}, \phi_2} \Pr\brk*{\phi_2 \mid \pi^\star= \pi , \cF_{i-1}} \\
    &= \sum_{\phi_2: \phi_2(X_2) = \sbad} \Pr \brk*{ (X_2, A_2) \mid \pi^\star = \pi , \cF_{i-1}, \phi_2} \Pr\brk*{ (X_{3:H}, A_{3:H}, R) \mid \pi^\star = \pi , \cF_{i-1}, \phi_2, X_2, A_2} \Pr\brk*{\phi_2 \mid \pi^\star= \pi , \cF_{i-1}} \\
    &\overset{(i)}{=} \frac{1}{2^{H+2}} \sum_{\phi_2: \phi_2(X_2) = \sbad} \Pr\brk*{ (X_{3:H}, A_{3:H}, R) \mid \pi^\star = \pi , \cF_{i-1}, \phi_2, X_2, A_2} \Pr\brk*{\phi_2 \mid \pi^\star= \pi , \cF_{i-1}}\\
    &\overset{(ii)}{=} \frac{1}{2^{H+2}} \sum_{\phi_2: \phi_2(X_2) = \sbad} \Pr\brk*{ (X_{3:H}, A_{3:H}, R) \mid \pi^\star = \pi , \cF_{i-1}, \phi_2(X_2) = \sbad} \Pr\brk*{\phi_2 \mid \pi^\star= \pi , \cF_{i-1}} \\
    &= \frac{ \Pr\brk*{\phi_2(X_2) = \sbad \mid \pi^\star= \pi , \cF_{i-1}} }{2^{H+2}} \Pr\brk*{ (X_{3:H}, A_{3:H}, R) \mid \pi^\star = \pi , \cF_{i-1}, \phi_2(X_2) = \sbad} \\
    &= \cdots \\[0.5em]
    &\overset{(iii)}{=} \frac{ \Pr\brk*{\phi_2(X_2) = \sbad \mid \pi^\star= \pi , \cF_{i-1}} }{2^{H+2}} \times \frac{\Pr \brk*{  \phi_3(X_3) = \sbad \mid \pi^\star = \pi , \cF_{i-1}, \phi_2(X_2) = \sbad}}{2^{H+1}} \\
    &\qquad \dots \times  \frac{\Pr\brk*{\phi_H(X_H) = \sbad \mid \pi^\star= \pi , \cF_{i-1}, \crl*{\phi_h(X_h)= \sbad, 2 \le h \le H-1} } }{2^{H+1}} \\
    &\qquad\qquad \times \Pr\brk*{ R \mid \pi^\star = \pi , \cF_{i-1}, \crl*{\phi_h(X_h) = \sbad, 2 \le h \le H}} \\[0.5em]
    &= \frac{ \Pr\brk*{\phi_2(X_2) = \sbad \mid \pi^\star= \pi , \cF_{i-1}} }{2^{H+2}} \times \frac{\Pr \brk*{  \phi_3(X_3) = \sbad \mid \pi^\star = \pi , \cF_{i-1}, \phi_2(X_2) = \sbad}}{2^{H+1}} \\
    &\qquad \dots \times  \frac{\Pr\brk*{\phi_H(X_H) = \sbad \mid \pi^\star= \pi , \cF_{i-1}, \crl*{\phi_h(X_h)= \sbad, 2 \le h \le H-1} } }{2^{H+1}}\times \frac12.
\end{align*}
The equality $(i)$ follows because the first state is chosen $\unif(\cX_2)$ and the action $A_2$ is the one selected by $\alg$ via the policy $\pi^{(i)}$ which is measurable with respect to $\cF_{i-1}$. The equality $(ii)$ follows because the distribution over the next state is only determined by $\cF_{i-1}$ (which includes some information about the decoder $\phi_3$) and the labeling $\phi_2(X_2) = \sbad$. Equality $(iii)$ follows by applying chain rule over and over, noting that since $\phi_2(X_2) = \sbad$ it must be the case that $\phi_h(X_h) = \sbad$ for all $h > 2$, and therefore the probability of observing any given observation with a bad label is $1/2^{H+1}$.

Now we apply the posterior state label calculation of \pref{lem:posterior-of-state-label} (using the fact that $\cfnor$ contains no repeated states) to each term in the previous display to get that:
\begin{align*}
    \hspace{2em}&\hspace{-2em} \sum_{\phi_2: \phi_2(X_2) = \sbad} \Pr  \brk*{ (X_{2:H}, A_{2:H}, R) \mid \pi^\star = \pi , \cF_{i-1}, \phi_2} \Pr\brk*{\phi_2 \mid \pi^\star= \pi , \cF_{i-1}} \\
    &\in \frac14 \cdot \prn*{ \frac{1}{2^{H+2}} }^{H-1} \cdot \brk*{ \prn*{1 - \frac{T}{2^H}}^{H-1}, \prn*{1 + \frac{T}{2^H}}^{H-1} }. \numberthis\label{eq:case-bad}
\end{align*}
Next, we consider other terms in the sum. To bound the quantity
\begin{align*}
    \sum_{\phi_2: \phi_2(X_2) = \sgood} \Pr  \brk*{ (X_{2:H}, A_{2:H}, R) \mid \pi^\star = \pi ,\cF_{i-1}, \phi_2} \Pr\brk*{\phi_2 \mid \pi^\star= \pi , \cF_{i-1}},
\end{align*}
a bit more care is required. In this case, we start off in the good latent state, and depending on whether the sequence of actions $A_{2:H}$ is equal to $\optpi$ we transit to the bad latent state. Let us denote $\overline{h}\ge 2$ denote the first layer at which $a_{\overline{h}}$ deviates from $\pi_{\overline{h}}$. Then using similar reasoning we have
\begin{align*}
    \hspace{2em}&\hspace{-2em}  \sum_{\phi_2: \phi_2(X_2) = \sgood} \Pr  \brk*{ (X_{2:H}, A_{2:H}, R) \mid \pi^\star = \pi , \cF_{i-1}, \phi_2} \Pr\brk*{\phi_2 \mid \pi^\star= \pi , \cF_{i-1}} \\
    &= \frac{ \Pr\brk*{\phi_2(X_2) = \sgood \mid \pi^\star= \pi , \cF_{i-1}} }{2^{H+2}} \times \cdots \times \frac{\Pr \brk*{  \phi_{\overline{h}}(X_{\overline{h}}) = \sgood \mid \pi^\star = \pi , \cF_{i-1}, \crl*{\phi_{h}(X_{h}) = \sgood, ~ \forall h < \overline{h}}  } }{2^{H}} \\
    &\qquad \times \frac{\Pr \brk*{  \phi_{\overline{h}+1}(X_{\overline{h}+1}) = \sbad \mid \pi^\star = \pi , \cF_{i-1}, \crl*{\phi_{h}(X_{h}) = \sgood, ~ \forall h \le \overline{h}}  } }{2^{H+1}} \\
    &\qquad \cdots \times \frac{\Pr\brk*{\phi_H(X_H) = \sbad \mid \pi^\star= \pi , \cF_{i-1},  \crl*{\phi_{h}(X_{h}) = \sgood, ~ \forall h \le \overline{h}}, \crl*{\phi_{h}(X_{h}) = \sbad, ~ \forall \overline{h}< h < H } } }{2^{H+1}} \\
    &\qquad \times \Pr\brk*{ R \mid  \pi^\star= \pi , \cF_{i-1},  \crl*{\phi_{h}(X_{h}) = \sgood, ~ \forall h \le \overline{h}}, \crl*{\phi_{h}(X_{h}) = \sbad, ~ \forall \overline{h}< h < H } }.
\end{align*}
Note that the conditional reward distribution given the latent state labels is
\begin{align*}
    \Pr\brk*{ R = 1 \mid \cdots } = \ind{A_{2:H} = \pi \circ 0} + \frac{1}{2} \ind{A_{2:H-1} \ne \pi}.
\end{align*}
Again by applying the posterior calculation of \pref{lem:posterior-of-state-label} we get
\begin{align*}
    \hspace{2em}&\hspace{-2em} \sum_{\phi_2: \phi_2(X_2) = \sgood} \Pr  \brk*{ (X_{2:H}, A_{2:H}, R) \mid \pi^\star = \pi , \cF_{i-1}, \phi_2} \Pr\brk*{\phi_2 \mid \pi^\star= \pi , \cF_{i-1}} \\
    &\in \begin{cases}
        \prn*{ \ind{A_{2:H} = \pi \circ 0} + \frac{1}{2} \ind{A_{2:H-1} \ne \pi } } \cdot \frac14 \prn*{ \frac{1}{2^{H+2}} }^{H-1} \cdot \brk*{ \prn*{1 - \frac{T}{2^H}}^{H-1}, \prn*{1 + \frac{T}{2^H}}^{H-1} } &\text{if}~ R = 1\\[0.5em]
        \prn*{ \ind{A_{2:H} = \pi \circ 1} + \frac{1}{2} \ind{A_{2:H-1} \ne \pi } } \cdot \frac14 \prn*{ \frac{1}{2^{H+2}} }^{H-1} \cdot \brk*{ \prn*{1 - \frac{T}{2^H}}^{H-1}, \prn*{1 + \frac{T}{2^H}}^{H-1} } &\text{if}~ R= 0
    \end{cases} \numberthis\label{eq:case-good}
\end{align*}
The last term in the sum, with $\phi_2(X_2) = \sdis$ is similar, so we get that
\begin{align*}
    \hspace{2em}&\hspace{-2em} \sum_{\phi_2: \phi_2(X_2) = \sdis} \Pr  \brk*{ (X_{2:H}, A_{2:H}, R) \mid \pi^\star = \pi , \cF_{i-1}, \phi_2} \Pr\brk*{\phi_2 \mid \pi^\star= \pi , \cF_{i-1}} \\
    &\in \begin{cases}
        \prn*{ \ind{A_{2:H} = \pi \circ 1} + \frac{1}{2} \ind{A_{2:H-1} \ne \pi } } \cdot \frac14  \prn*{ \frac{1}{2^{H+2}} }^{H-1} \cdot \brk*{ \prn*{1 - \frac{T}{2^H}}^{H-1}, \prn*{1 + \frac{T}{2^H}}^{H-1} } &\text{if}~ R = 1\\[0.5em]
        \prn*{ \ind{A_{2:H} = \pi \circ 0} + \frac{1}{2} \ind{A_{2:H-1} \ne \pi } } \cdot \frac14 \prn*{ \frac{1}{2^{H+2}} }^{H-1} \cdot \brk*{ \prn*{1 - \frac{T}{2^H}}^{H-1}, \prn*{1 + \frac{T}{2^H}}^{H-1} } &\text{if}~ R= 0
    \end{cases} \numberthis\label{eq:case-dis}
\end{align*}
(Note that the first indicators have been swapped in the previous display compared to Eq.~\eqref{eq:case-good}.)

Summing Eqs.~\eqref{eq:case-bad}, \eqref{eq:case-good}, and \eqref{eq:case-dis}, and applying casework on the different choices of $A_{2:H}$ we get that
\begin{align*}
    \Pr\brk*{ (X_{2:H}, A_{2:H}, R) \mid \pi^\star = \pi , \cF_{i-1}} \in \frac12 \cdot \prn*{ \frac{1}{2^{H+2}} }^{H-1} \cdot \brk*{ \prn*{1 - \frac{T}{2^H}}^{H-1}, \prn*{1 + \frac{T}{2^H}}^{H-1} },
\end{align*}
thus concluding the proof of \pref{claim:claim1}.

\begin{claim}\label{claim:claim2}
    If $\tau_i$ is generated by an online rollout, we have for every $\pi \in \Pi_{1:H-1}$,
    \begin{align*}
        \ind{\cE_{R, t-1}} \Pr \brk*{ \tau_i \mid \pi^\star = \pi, \cF_{i-1} }  \in \ind{\cE_{R, t-1}} \frac{1}{2} \cdot \frac{1}{2^{(H+2) \cdot H}} \cdot \brk*{ \prn*{1 - \frac{T}{2^H}}^{H}, \prn*{1 + \frac{T}{2^H}}^{H} }.
    \end{align*}
\end{claim}

Now we prove \pref{claim:claim2}. Most of the hard work has already been done in the proof of \pref{claim:claim1}.  Note that by construction $\phi_1(X_1) = \sgood$. Using a similar calculation we have
\begin{align*}
    \hspace{2em}&\hspace{-2em} \Pr  \brk*{ (X_{1:H}, A_{1:H}, r) \mid \pi^\star = \pi , \cF_{i-1}} \\
    &\in \begin{cases}
        \prn*{ \ind{A_{1:H} = \pi \circ 0} + \frac{1}{2} \ind{A_{1:H-1} \ne \pi } } \cdot \prn*{ \frac{1}{2^{H+2}} }^{H} \cdot \brk*{ \prn*{1 - \frac{T}{2^H}}^{H}, \prn*{1 + \frac{T}{2^H}}^{H} } &\text{if}~ R = 1\\[0.5em]
        \prn*{ \ind{A_{1:H} = \pi \circ 1} + \frac{1}{2} \ind{A_{1:H-1} \ne \pi } } \cdot \prn*{ \frac{1}{2^{H+2}} }^{H} \cdot \brk*{ \prn*{1 - \frac{T}{2^H}}^{H}, \prn*{1 + \frac{T}{2^H}}^{H} } &\text{if}~ R = 0
    \end{cases}
\end{align*}
However, observe that under the event $\cE_{R, t-1}$ we know that $A_{1:H-1} \ne \optpi$, so the first indicator cannot be $=1$ in either case; so multiplying both sides of the previous display by $\ind{\cE_{R, t-1}}$ gives us the result of \pref{claim:claim2}.

To tidy up, we also state the calculation on the last trajectory, which does not include the prefactor of $\frac12$ because there are no observed rewards at the end:
\begin{claim}\label{claim:claim3}
    \begin{align*}
        \Pr \brk*{ (X_{t, h_\bot:H}, A_{t,h_\bot:H}) \mid \pi^\star = \pi , \cF_{t-1}} \in \frac{1}{2^{(H+2) \cdot (H- h_\bot + 1)}} \brk*{ \prn*{1 - \frac{T}{2^H}}^{H}, \prn*{1 + \frac{T}{2^H}}^{H} }.
    \end{align*}
\end{claim}

Now with \pref{claim:claim1}, \ref{claim:claim2}, and \ref{claim:claim3} in hand, we can finally return to computing a bound on Eq.~\eqref{eq:upper-bound-on-posterior-tv}. Letting $O$ denote the total number of observations in $\cfnor$ (which can be at most $TH$), we have for any $\pi \in \Pi_{1:H-1}$,
\begin{align*}
    \hspace{2em}&\hspace{-2em}\ind{\cE_{R, t-1}} \Pr\brk*{\cfnor \mid \pi^\star = \pi} \\
    &\in \ind{\cE_{R, t-1}} \cdot \prn*{ \frac12 }^{t-1} \cdot \prn*{ \frac{1}{2^{H+2}} }^{O} \cdot \brk*{ \prn*{1 - \frac{T}{2^H}}^{TH}, \prn*{1 + \frac{T}{2^H}}^{TH} } \eqqcolon \ind{\cE_{R, t-1}} \cdot [\underline{B}, \overline{B}].
\end{align*}
Moreover, for any $\cfnor$ we have
\begin{align*}
    \Pr \brk*{ \cfnor} = \frac{1}{2^{H-1}} \sum_{\pi \in \Pi_{1:H-1}} \Pr\brk*{\cfnor \mid \pi^\star = \pi} \ge \frac{2^{H-1} - T}{2^{H-1}} \cdot \underline{B}.
\end{align*}
The last inequality follows because there are at most $T$ different action sequences which have been executed by online trajectories in $\cfnor$, so therefore for all but at most $T$ policies we have $\ind{\cE_{R, t-1}} \Pr\brk*{\cfnor \mid \pi^\star = \pi} = \Pr\brk*{\cfnor \mid \pi^\star = \pi}$. Thus we arrive at the bound
\begin{align*}
    \hspace{2em}&\hspace{-2em} \ind{\cE_{R, t-1}} \nrm*{\Pr\brk*{\pi^\star = \cdot \mid \cfnor} - \unif(\Pi_{1:H-1})}_1 \\
    &\le 2 \max_{\pi \in \Pi_{1:H-1}} \brk*{ \frac{\ind{\cE_{R, t-1}} \Pr\brk*{\cfnor \mid \pi^\star = \pi} }{ \Pr\brk*{ \cfnor } } - 1 }_{+} \le 2  \brk*{ \frac{ \overline{B} }{\prn*{1 - T/{2^{H-1}} } \cdot \underline{B} } - 1 }_{+} \\
    &\le 2 \cdot \prn*{ \prn*{1 + \frac{T}{2^{H-2}}}^{2TH+1} - 1} \le 2 \cdot \frac{2T^2H+T}{2^{H-2}} \exp\prn*{ \frac{2T^2H+T}{2^{H-2}}} \le \frac{T^2H}{2^{H-6}}.
\end{align*}
The second to last inequality uses the fact that $1+y \le e^y$ and $e^y - 1 \le y e^y$, and the last inequality uses the fact that $T = 2^{O(H)}$. This concludes the proof of \pref{lem:posterior-of-optimal}.
\end{proof}

\begin{lemma}[Posterior of New State Label]\label{lem:posterior-of-state-label}
Let $\cF$ be any filtration of $T$ trajectories as well as annotations $\phi(x)$ for a subset of observations $x \in \cF$. Let $\pi \in \Pi_{1:H-1}$ be any policy. Fix any $h \ge 2$, and let $x_\mathrm{new} \in \cX_h - \cF$. Then
\begin{align*}
    \abs*{ \Pr^{\nu_0, \alg} \brk*{  \phi(x_\mathrm{new}) = \sgood    \mid \cF, \optpi = \pi }  - \frac14} &\le \frac{T}{2^{H}},\\
    \abs*{ \Pr^{\nu_0, \alg} \brk*{ \phi(x_\mathrm{new}) = \sdis \mid \cF, \optpi = \pi } - \frac14} &\le \frac{T}{2^{H}},\\
    \abs*{ \Pr^{\nu_0, \alg} \brk*{ \phi(x_\mathrm{new}) = \sbad  \mid \cF, \optpi = \pi } - \frac12} &\le \frac{T}{2^{H}}.
\end{align*}
\end{lemma}

\begin{proof}
Let us denote $\cF'$ to be the completely annotated $\cF$ which includes all labels $\crl{\phi(X): X \in \cF}$. We will show that the conclusion of the lemma applies to every completion $\cF'$, and since 
\begin{align*}
    \Pr^{\nu_0, \alg} \brk*{  \phi(x_\mathrm{new}) = \cdot  \mid \cF, \optpi = \pi } = \En^{\nu_0, \alg} \brk*{ \Pr^{\nu_0, \alg} \brk*{  \phi(x_\mathrm{new}) = \cdot   \mid \cF', \optpi = \pi } \mid \cF, \optpi = \pi},
\end{align*}
this will imply the result by Jensen's inequality and convexity of $\abs{\cdot}$.

We calculate the good label probability:
\begin{align*}
    \Pr^{\nu_0, \alg} \brk*{  \phi(x_\mathrm{new}) = \sgood    \mid \cF', \optpi = \pi } = \frac{2^H - \abs{\crl*{X \in \cF: \phi(X) = \sgood}}}{2^{H+2} - \abs{\cF}}.
\end{align*}
For the lower bound we have
\begin{align*}
    \frac{2^H - \abs{ \crl*{X \in \cF: \phi(X) = \sgood}}}{2^{H+2} - \abs{\cF}} \ge \frac{2^H - T}{2^{H+2} } =  \frac{1}{4}\cdot \prn*{ 1 - \frac{T}{{2^{H}}} }.
\end{align*}
For the upper bound we have
\begin{align*}
    \frac{2^H - \abs{\crl*{x \in \cF: \phi(x) = \sgood}}}{2^{H+2} - \abs{\cF}} \le \frac{2^H}{2^{H+2} - T} =  \frac{1}{4}\cdot \prn*{ 1- \frac{T}{2^{H+2}} }^{-1} \le \frac{1}{4}\cdot  \prn*{ 1 + \frac{T}{{2^{H}}} },
\end{align*}
which holds as long as $T \le 2^{H}$. Combining both upper and lower bounds proves the lemma for the good label. The rest of the calculations are similar, so we omit them. This concludes the proof of \pref{lem:posterior-of-state-label}.
\end{proof}

\begin{lemma}[Posterior of State Label with Rollout]\label{lem:posterior-of-state-label-v2}
Fix any $t \in [T]$. Suppose that episode $t$ is sampled using the $\mu$-reset at layer $h_\bot \ge 2$, and that $\cfnor$ contains no repeated states. Then for any $\pi \in \Pi_{h_\bot:H-1}$,
    \begin{align*}
        \abs*{ \Pr^{\nu_0, \alg} \brk*{  \phi(X_{t, h_\bot}) = \sgood    \mid \cfnor, \optpi = \pi }  - \frac14} &\le \frac{TH}{2^{H-3}},\\
        \abs*{ \Pr^{\nu_0, \alg} \brk*{ \phi(X_{t, h_\bot}) = \sdis \mid \cfnor, \optpi = \pi } - \frac14} &\le \frac{TH}{2^{H-3}},\\
        \abs*{ \Pr^{\nu_0, \alg} \brk*{ \phi(X_{t, h_\bot}) = \sbad  \mid \cfnor, \optpi = \pi } - \frac12} &\le \frac{TH}{2^{H-3}}.
    \end{align*}
\end{lemma}

\begin{proof}
We will prove the result with $h_\bot = 2$, and it is easy to adapt it to the general case (in fact the setting where $h_\bot > 2$ only results in tighter bounds). Using repeated application of chain rule and \pref{lem:posterior-of-state-label} we get
\begin{align*}
    \Pr^{\nu_0, \alg} \brk*{  \phi(X_{t,  2}) = \sgood \wedge (X_{t,  2:H}, A_{t,  2:H}) \mid \cF_{t-1}, \pi^\star = \pi } &\in \frac14 \prn*{ \frac{1}{2^{H+2}} }^{H-1} \cdot \brk*{ \prn*{1 - \frac{T}{2^H}}^{H-1}, \prn*{1 + \frac{T}{2^H}}^{H-1} } \\
    \Pr^{\nu_0, \alg} \brk*{  \phi(X_{t,  2}) = \sbad \wedge (X_{t,  2:H}, A_{t,  2:H}) \mid \cF_{t-1}, \pi^\star = \pi } &\in \frac12 \prn*{ \frac{1}{2^{H+2}} }^{H-1} \cdot \brk*{ \prn*{1 - \frac{T}{2^H}}^{H-1}, \prn*{1 + \frac{T}{2^H}}^{H-1} } \\
    \Pr^{\nu_0, \alg} \brk*{  \phi(X_{t,  2}) = \sdis \wedge (X_{t,  2:H}, A_{t,  2:H}) \mid \cF_{t-1}, \pi^\star = \pi } &\in \frac14 \prn*{ \frac{1}{2^{H+2}} }^{H-1} \cdot \brk*{ \prn*{1 - \frac{T}{2^H}}^{H-1}, \prn*{1 + \frac{T}{2^H}}^{H-1} }.
\end{align*}
Let's prove the first inequality in the lemma statement. By Bayes Rule we have
\begin{align*}
    \Pr^{\nu_0, \alg} \brk*{  \phi(X_{t,  2}) = \sgood \mid \cfnor, \pi^\star = \pi } &= \frac{ \Pr^{\nu_0, \alg} \brk*{  \phi(X_{t,  2}) = \sgood \wedge (X_{t,  2:H}, A_{t,  2:H}) \mid \cF_{t-1}, \pi^\star = \pi } }{ \Pr^{\nu_0, \alg} \brk*{ (X_{t,  2:H}, A_{t,  2:H}) \mid \cF_{t-1}, \pi^\star = \pi} } \\
    &= \frac{ \Pr^{\nu_0, \alg} \brk*{  \phi(X_{t,  2}) = \sgood \wedge (X_{t,  2:H}, A_{t,  2:H}) \mid \cF_{t-1}, \pi^\star = \pi } }{\sum_{\ell \in \crl{\sgood, \sbad, \sdis}} \Pr^{\nu_0, \alg} \brk*{\phi(X_{t,  2}) = \ell \wedge (X_{t,  2:H}, A_{t,  2:H}) \mid \cF_{t-1}, \pi^\star = \pi} }.
\end{align*}
From here it is easy to compute the upper bound
\begin{align*}
    \Pr^{\nu_0, \alg} \brk*{  \phi(X_{t,  2}) = \sgood \mid \cfnor, \pi^\star = \pi } \le \frac{1}{4} \cdot \prn*{1+\frac{T}{2^{H-1}}}^{2H} \le \frac{1}{4} + \frac{TH}{2^{H-3}}.
\end{align*}
as well as the lower bound
\begin{align*}
    \Pr^{\nu_0, \alg} \brk*{  \phi(X_{t,  2}) = \sgood \mid \cfnor, \pi^\star = \pi } \ge \frac{1}{4} \cdot \prn*{1-\frac{T}{2^H}}^{2H} \ge \frac{1}{4} - \frac{TH}{2^{H-3}}.
\end{align*}
The other two inequalities are similarly shown, and this concludes the proof of \pref{lem:posterior-of-state-label-v2}.
\end{proof}

\clearpage
\section{Proof for the Warmup Algorithm \detalg{}}\label{app:proof-warmup}

In this section, we prove the following sample complexity guarantee for \detalg{}:
\begin{reptheorem}{thm:det-bmdp-solver-guarantee}
    Let $\eps, \delta \in (0,1)$ be given and suppose that \pref{ass:det-transitions} holds. Then with probability at least $1-\delta$, \detalg{} (\pref{alg:det-bmdp-solver}) finds an $\eps$-optimal policy using 
    \begin{align*}
        \wt{O}\prn*{\frac{S^5A^2H^5}{\eps^2} \cdot \log\frac{1}{\delta}} \quad\text{samples.}
    \end{align*}
\end{reptheorem}

\subsection{Proof of \pref{thm:det-bmdp-solver-guarantee}}
Our high-level strategy is to apply the inductive argument outlined in \pref{sec:warmup-algorithm-analysis-sketch} to control the growth of the Bellman error for all $(s,a) \in \latentsp_h \times \actionsp$ as we construct $\estlatentmdp$ from layer $H$ backwards. Recall our Bellman error decomposition:
\begin{align*}
    \abs*{Q^\pi(s,a) - \wh{Q}^\pi(s,a)} \le \underbrace{\abs*{\optlatr - \estlatr}}_{\text{reward error}} + \underbrace{\abs*{\wh{V}^\pi(\optlatp) - \wh{V}^\pi(\estlatp)}}_{\text{transition error}} + \underbrace{\abs*{V^\pi(\optlatp) - \wh{V}^\pi(\optlatp)}}_{\text{error at next layer}}. \tag{\ref{eq:bellman-error-main}}
\end{align*}

To control the transition error of Eq.~\eqref{eq:bellman-error-main}, we introduce a notion of test policy validity and give a lemma which shows that if \detdecoder{} is equipped with valid test policies, the transition estimation error can be bounded.

\begin{definition}[Test Policy Validity, Deterministic Version]\label{def:valid-test-policy-dd} Let $\eta > 0$ be a parameter. At layer $h \in [H]$, we say a collection of partial policies $\Pitest_{h} = \crl{\pi_{s,s'} \in \Pi_{h:H}: s, s' \in \latentsp_h}$ is an $\eta$-\emph{valid test policy set} for the estimated latent MDP $\estlatentmdp$ if for every $s, s' \in \latentsp_h$:
\begin{itemize}
    \item (Maximally distinguishing): $\pi_{s,s'} = \argmax_{\pi \in \Pi_{h:H}} \abs{\wh{V}^\pi(s) - \wh{V}^\pi(s')}$.
    \item (Accurate): $\abs{V^{\pi_{s,s'}}(s) - \wh{V}^{\pi_{s,s'}}(s)} \le \eta$ and $\abs{V^{\pi_{s,s'}}(s') - \wh{V}^{\pi_{s,s'}}(s')} \le \eta$.
\end{itemize}
\end{definition}

\begin{lemma}[Decoding]\label{lem:controlling-transition-error}
Fix any layer $h \in [H-1]$. Suppose that \detdecoder{} (\pref{alg:decoder}) is equipped with a $\tauref$-valid test policy $\Pitest_{h+1}$. Fix any tuple $(s_{h}, a_{h})$ and assume that $\optlatp(s_{h}, a_{h}) \in \cP(s_{h}, a_{h})$. With high probability, \detdecoder{} returns an updated $\cP$ such that: 
\begin{enumerate}[(1)]
        \item $\optlatp(s_{h}, a_{h}) \in \cP$;
        \item For every $\bar{s} \in \cP$ we have $\max_{\pi \in \Pi}~\abs{\wh{V}^\pi(\optlatp(s_{h}, a_{h})) - \wh{V}^\pi(\bar{s})} \le 7\tauref/2$. 
\end{enumerate}
\end{lemma}
The proof of \pref{lem:controlling-transition-error} is deferred to \pref{app:induction-lemmas}. 

In light of \pref{lem:controlling-transition-error}, as long as we have valid test policy sets $\crl{\Pitest_h}_{h\in[H]}$, \pref{lem:controlling-transition-error} provides control on the transition estimation error, and we can iteratively apply Eq.~\eqref{eq:bellman-error-main} to get the final bound on estimation error at layer 1.

\paragraph{Computing Test Policies via \detrefit{}.} Now we will analyze \detrefit{}. By standard concentration arguments, if \pref{line:great-success} is triggered, the test policies must be $\tauref$-accurate; furthermore, they are maximally distinguishing by construction. Unfortunately, since we require the test policies to satisfy a higher level of accuracy $\tauref$, due to estimation errors in $\estlatentmdp$, it may not be possible to find any valid test policies. To address this, we observe that inaccurate test policies act as a ``certificate'' and allow us to search for some transition $\estlatp \ne \optlatp$.

\begin{lemma}[Refitting]\label{lem:refitting}
Let $\eps > 0$ be given. Suppose that at layer $h \in [H]$, \detrefit{} (\pref{alg:refit}) is supplied confidence sets $\cP$ such that for all $(s,a) \in \latentsp_{h:H} \times \actionsp$ we have $\optlatp(s,a) \in \cP(s,a)$. If \detrefit{} terminates at \pref{line:great-success-2}, then with high probability:
    \begin{enumerate}[(1)]
        \item At least one $\estlatp$ was removed from its confidence set $\cP$.
        \item No ground truth transitions $\optlatp$ are removed from their confidence set $\cP$.
    \end{enumerate}
\end{lemma}
The proof of \pref{lem:refitting} is deferred to \pref{app:induction-lemmas}. 

Our analysis will track the invariant that the confidence sets $\cP$ always contain the ground truth transition $\optlatp$. Therefore, \pref{lem:refitting} allows us to use the size of the confidence sets as a potential function: if \detrefit{} fails to compute valid test policies at some layer $h$, we must delete some incorrect transition $\estlatp$ from its set $\cP$; this process cannot continue indefinitely, since we can delete at most $S(S-1)A$ states.

\paragraph{Proof by Induction.}
With \pref{lem:controlling-transition-error} and \pref{lem:refitting} in hand, we can show the final bound in \pref{thm:det-bmdp-solver-guarantee}. For technical convenience, we will show that the policy returned by \detalg{} is $O(\eps)$ suboptimal; rescaling the parameter $\eps$ does not change the final sample complexity apart from constant factors. Also, we omit the standard arguments (via concentration and union bound) which show that the conclusions of \pref{lem:controlling-transition-error} and \pref{lem:refitting} hold with probability at least $1-\delta$ over the randomness of sampling episodes from the MDP.

Take $\Gamma_h \coloneqq C(H-h+1)/H \cdot \eps$ for some suitably large constant $C > 0$. We will inductively show that these properties hold for all layers $h\in [H]$:
\begin{enumerate}[(A)]
    \item \emph{Policy Evaluation Accuracy.} For all pairs $(s,a) \in \latentsp_h \times \actionsp$ and $\pi \in \Pi_\mathsf{open}$: $\abs{Q^\pi(s,a) - \wh{Q}^\pi(s,a)} \le \Gamma_h$.
    \item \emph{Confidence Set Validity.} For all pairs $(s,a)\in \latentsp_h \times \actionsp$, we have $\optlatp(s,a) \in \cP(s,a)$.
    \item \emph{Test Policy Validity.} $\Pitest_{h}$ are $\tauref$-valid for $\estlatentmdp$ at layer $h$.
\end{enumerate}

To analyze \detalg{}, we will show that these properties always hold throughout at the end of every while loop for all layers $h > \ell_\mathsf{next}$.

\underline{Base Case.} We analyze the first loop with $\ell = H$. Note that (A) holds by concentration of the reward estimates, and (B) trivially holds because there are no transitions to be constructed at layer $H$. Now we investigate what happens when \detrefit{} is called. The computed test policies take the form $\pi_{s,s'} \equiv a$ for some $a \in \cA$; again by concentration of the reward estimates, \pref{line:great-success} of \detrefit{} is triggered. Therefore (A)--(C) hold after refitting, and we jump to $\ell_\mathsf{next} = H-1$.

\underline{Inductive Step.} Suppose the current layer index is $\ell$, and that properties (A)--(C) hold for all $h > \ell$. By \pref{lem:controlling-transition-error}, the updated transition confidence sets returned by \detdecoder{} at layer $\ell$ will satisfy (B). Furthermore, at the end of \pref{line:est-transition},  the error decomposition \eqref{eq:bellman-error-main} implies that for every $(s,a) \in \cS_\ell \times \actionsp$:
    \begin{align*}
        &\max_{\pi \in \Pi}~\abs*{Q^\pi(s, a) - \wh{Q}^\pi(s, a)} \le \Gamma_{\ell+1} + \frac{\eps}{H^2} + \frac{7\tauref}{2} \le \Gamma_{\ell}, \quad \Longrightarrow \quad  \text{Property (A) holds at layer $\ell$.}
    \end{align*}
Now we do casework on the outcome of \detrefit{}.
\begin{itemize}
    \item \textbf{Case 1: Return in \pref{line:great-success}.}  By construction, property (C) is satisfied for layer $\ell$. In this case, since \pref{alg:refit} made no updates to $\estlatentmdp$ or $\cP$, properties (A) and (B) continue to hold at layer $\ell$ onwards.
    \item \textbf{Case 2: Return in \pref{line:great-success-2}.}  By \pref{lem:refitting}, any updates to $\estlatentmdp$ maintain property (B). Let $\ell_\mathsf{next}$ denote the layer at which we jump to. By definition of $\ell_\mathsf{next}$, we made no updates to $\estlatentmdp$ at layers $\ell_\mathsf{next} + 1$ onwards, and therefore the previously computed test policies $\Pitest_{\ell_\mathsf{next}+1:H}$ must still be valid, so therefore properties (A) and (C) continue to hold at layer $\ell_\mathsf{next}$ onwards.
\end{itemize}

Continuing the induction, once $\ell \gets 0$ is reached in \detalg{} (which we know will eventually happen because Case 2 can only occur for $S^2A$ times), the estimated latent MDP $\estlatentmdp$ must satisfy the bound
\begin{align*}
    \max_{\pi \in \Pi}~\abs*{V^\pi(s_1) - \wh{V}^\pi(s_1)} \le \Gamma_1 = O(\eps).
\end{align*}

\paragraph{Sample Complexity Bound.} We now compute the final sample complexity required by \detalg{}:
\begin{itemize}
    \item Estimating rewards in the main algorithm uses $\wt{O}(H^4SA/\eps^2)$ samples.
    \item \detdecoder{} is called at most $SA \times S^2 A$ times, since we (re-)decode every transition $(s,a)$ at most $S^2A$ times. Every call to \detdecoder{} uses $\wt{O}(S^2/\tauref^2) = \wt{O}(S^2H^2/\eps^2)$ samples since we take $\tauref = 2^5 \cdot \eps/H$ in \pref{lem:refitting}. Therefore the total number of samples used by \detdecoder{} is at most $\wt{O}(S^5A^2H^2/\eps^2)$.
    \item \detrefit{} is called at most $S^2 AH$ times, since associated to every layer revisiting is an additional $H$ calls in the main while loop. In every call to \detrefit{}, we  use $\wt{O}(S^2H^2/\eps^2)$ calls to compute and verify the test policy set in \pref{line:eval-policy}. In addition, every time \pref{line:violation-statement} is triggered corresponds to at least one deletion in \pref{line:bad-state-2}, so the number of additional samples used by \pref{line:mc-additional} (across all calls to \detrefit{}) can be bounded by $\wt{O}(S^2 A H^3/\tauref^2) = \wt{O}(S^2 A H^5/\eps^2)$.
\end{itemize}  
Thus the final sample complexity is at most $\wt{O}\prn{S^5A^2H^5/\eps^2}$ samples. \qed

\subsection{Proof of Induction Lemmas }\label{app:induction-lemmas}

\begin{proof}[Proof of \pref{lem:controlling-transition-error}]
First we prove implication (1). Let us denote $s^\star = \optlatp(s_h, a_h)$. If $s^\star \notin \cP$ (the returned set), then there exists some $s'$ for which
\begin{align*}
    \abs*{\vestarg{x_{h+1}}{\pi_{s^\star, s'}} - \wh{V}^{\pi_{s^\star, s'}}(s^\star)} \ge 2\tauref.
\end{align*}
However, by assumption of test policy accuracy we know that
\begin{align*}
    \abs*{V^{\pi_{s^\star, s'}}(s^\star) - \wh{V}^{\pi_{s^\star, s'}}(s^\star)} \le \tauref.
\end{align*}
Since the quantity $\vestarg{x_{h+1}}{\pi_{s^\star, s'}}$ is an unbiased estimate of $V^{\pi_{s^\star, s'}}(s^\star)$ which is estimated to accuracy $\tauref/2$ we have a contradiction, so $s^\star \in \cP$.

Now we prove implication (2). If $\bar{s} \in \cP$, then we must have
\begin{align*}
   \abs*{\vestarg{x_{h+1}}{\pi_{s^\star, \bar{s}}} - \wh{V}^{\pi_{s^\star, \bar{s}}}(\bar{s})} = \abs*{\vestarg{s^\star}{\pi_{s^\star, \bar{s}}} - \wh{V}^{\pi_{s^\star, \bar{s}}}(\bar{s})} \le 2\tauref. 
\end{align*}
 Since we estimated $\vestarg{s^\star}{ \pi_{s^\star, \bar{s}}}$ up to $\tauref/2$ accuracy we know that
 \begin{align*}
     \abs*{V^{\pi_{s^\star, \bar{s}}}(s^\star) - \wh{V}^{\pi_{s^\star, \bar{s}}}(\bar{s})} \le 5\tauref/2, \quad \Longrightarrow \quad  \abs*{\wh{V}^{\pi_{s^\star, \bar{s}}}(s^\star) - \wh{V}^{\pi_{s^\star, \bar{s}}}(\bar{s})} \le 7\tauref/2,
 \end{align*}
 where the implication follows by the accuracy of $\Pitest_{h+1}$.
By the maximal distinguishing property of $\Pitest_{h+1}$, observe that the LHS of the above implication is equal to $\max_{\pi \in \Pi}~\abs{\wh{V}^{\pi}(s^\star) - \wh{V}^{\pi}(\bar{s})}$. This proves the second implication, and concludes the proof of \pref{lem:controlling-transition-error}.
\end{proof}

\begin{proof}[Proof of \pref{lem:refitting}]
We show the first implication. Let $(s_h,\pi)$ be any policy which satisfies $\abs{\vestarg{s_h}{\pi} - \wh{V}^\pi(s_h)} \ge \tauref - \eps/H$. Since we estimated $V^\pi(s_h)$ up to $\eps/H$ error, we have $\abs{V^\pi(s_h) - \wh{V}^\pi(s_h)} \ge \tauref - 2\eps/H$. Let $\bar{s}_h = s_h, \bar{s}_{h+1}, \cdots, \bar{s}_H$ be the sequence of states which are obtained by running $\pi$ on $\estlatentmdp$ starting at $s_h$. 

For sake of contradiction suppose that
\begin{align*}
    \abs*{\vestarg{\bar{s}}{\pi} - \wh{R}(\bar{s}, \pi) - \vestarg{\estlatp(\bar{s}, \pi)}{\pi} } \le \frac{4\eps}{H^2}, \quad \text{for all}~\bar{s} \in \crl{\bar{s}_h,\cdots, \bar{s}_H}.
\end{align*}
Since we estimated every $\vestarg{\cdot}{\pi}$ up to accuracy $\eps/H^2$ we see that
\begin{align*}
    \abs*{V^\pi(\bar{s}) - \wh{R}(\bar{s}, \pi) - V^\pi(\estlatp(\bar{s}, \pi)) } \le \frac{6\eps}{H^2}, \quad \text{for all}~\bar{s} \in \crl{\bar{s}_h,\cdots, \bar{s}_H}.
\end{align*}
By Performance Difference Lemma and applying the previous display recursively,
\begin{align*}
     \abs*{V^\pi(s_h) - \wh{V}^\pi(s_h)} &\le \abs*{V^\pi(\bar{s}_h) - \wh{R}(\bar{s}_h, \pi) - V^\pi(\bar{s}_{h+1})} + \abs*{ V^\pi(\bar{s}_{h+1}) - \wh{V}(\bar{s}_{h+1}) } \\
     &\le \frac{6\eps}{H^2} + \abs*{ V^\pi(\bar{s}_{h+1}) - \wh{V}(\bar{s}_{h+1}) } \le \cdots \le \frac{6\eps}{H}.
\end{align*}
This contradicts the statement that $\abs*{V^\pi(s_h) - \wh{V}^\pi(s_h)} \ge \tauref - 2\eps/H$ by the choice of $\tauref$. So we can conclude that there exists a state $\bar{s} \in \crl{\bar{s}_h,\cdots, \bar{s}_H}$ such that
\begin{align*}
    \abs*{\vestarg{\bar{s}}{\pi} - \wh{R}(\bar{s}, \pi) - \vestarg{\estlatp(\bar{s}, \pi)}{\pi} } \ge \frac{4\eps}{H^2},
\end{align*}
so therefore \pref{line:bad-state-2} is executed at least once, proving the first implication.

To prove the second implication, consider any $(\bar{s}, \pi)$ for which \pref{line:bad-state-2} is executed. We know that
\begin{align*}
    \abs*{ V^\pi(\bar{s}) - \wh{R}(\bar{s}, \pi) - V^\pi(\estlatp(\bar{s}, \pi)) } \ge \frac{2\eps}{H^2}.
\end{align*}
Recall that for all $(s,a)$, the estimation error on the rewards was $\abs{R(s,a) - \wh{R}(s,a)} \le \eps/H^2$. Therefore
\begin{align*}
    \abs*{ V^\pi(\bar{s}) - R(\bar{s}, \pi) - V^\pi(\estlatp(\bar{s}, \pi)) } \ge \frac{\eps}{H^2}, \quad \Longrightarrow \quad  \estlatp(\bar{s},\pi) \ne \optlatp(\bar{s}, \pi).
\end{align*}
Therefore as claimed we always delete $\estlatp(\bar{s},\pi) \ne \optlatp(\bar{s}, \pi)$ in \pref{line:bad-state-2} of \detrefit{}.
\end{proof}

\clearpage
\section{Proof of Main Upper Bound}\label{app:main-upper-bound-proofs}
In this section, we prove \pref{thm:block-mdp-result}.

\subsection{Preliminaries}\label{app:upper-bound-preliminaries}
We will define some additional concepts and notation which will be used in the analysis.

\begin{itemize}
    \item For any set $\cX' \subseteq \cX$ we denote the represented states as $\cS[\cX'] \coloneqq \crl{\optdec(x): x \in \cX'}$. For any latent state $s \in \cS$ and subset $\cX' \subseteq \cX$ we let $n_s[\cX'] \coloneqq \abs{\crl{x \in \cX': \optdec(x) = s}}$ count the total number of observations there are emitted from $s$. 
    \item We define the set of $\eps$-pushforward-reachable latent states
    \begin{align*}
        \epsPRS{h} \coloneqq \crl*{s_h: \max_{s_{h-1}, a_{h-1}} \optlatp(s_h \mid  s_{h-1}, a_{h-1}) \ge \frac{\eps}{S} },
    \end{align*}
    and furthermore let $\epsPRS{} \coloneqq \cup_{h=1}^H \epsPRS{h}$.
    \item For any $\cX' \subseteq \cX$, we let $\nrch[\cX'] \coloneqq \abs{\crl{x \in \cX': \optdec(x) = \epsPRS{}}}$ and $\nurch[\cX'] = \abs{\cX'} - \nrch[\cX']$.
\end{itemize}

\paragraph{Estimated Transitions and Projected Measures.} Recall that the ground truth latent transition is denoted $\optlatp: \latentsp \times \actionsp \to \Delta(\latentsp)$. We will use $\samplelatp$ to denote the empirical version of the latent transition which is sampled in \pref{line:sample-decode-dataset} of \pref{alg:stochastic-decoder-v2}:
\begin{align*}
    \samplelatp(\cdot  \mid  \optdec(x_{h}), a_{h}) = \frac{1}{\ndec} \sum_{x \in \cD} \delta_{\optdec(x)}.
\end{align*}
In addition, we introduce a notion of projected measures which will be used to relate the ground truth transition $P = \emission \circ \optlatp$ with the estimated transition $\estp$ of the policy emulator. While our algorithm never directly uses the projected measure, we track it in the analysis.

\begin{definition}[Projected Measure]\label{def:projected-measure}
For a distribution $p \in \Delta(\cS)$, define the \emph{projected measure} onto the observation set $\bar{\cX} \subseteq \statesp$ as
\begin{align*}
    \projm{\bar{\cX}}(p) \coloneqq \sum_{s \in \epsPRS{}} p (s) \cdot \unif (\crl{\empobs \in \bar{\cX}: \optdec(\empobs) = s}).
\end{align*}
Specifically, for any $\empobs \in \bar{\cX}$ we have:
\begin{align*}
    \projm{\bar{\cX}}(p)(\empobs) = p (\optdec(\empobs)) \cdot \frac{\ind{\optdec(\empobs) \in \epsPRS{}}}{n_{\optdec(\empobs)}[\bar{\cX}]}.
\end{align*}
Furthermore, for any subset $\bar{\cX}' \subseteq \bar{\cX}$ we denote $\projm{\bar{\cX}}(p)(\bar{\cX}') = \sum_{\empobs \in \bar{\cX}'} \projm{\bar{\cX}}(p)(\empobs)$.
\end{definition}
Formally, $\projm{\bar{\cX}}$ is not a true probability distribution, as the total measure might not sum up to 1. This would happen if $p(s) > 0$ for $s \in (\epsPRS{})^c$.

\textit{Remark.} In \pref{thm:block-mdp-result}, we assume that the distribution $\mu$ is factorizable. This can be removed with some extra work. One can modify the definition of the projected measure to replace the uniform distribution over observations with some other suitable importance-reweighted distribution; the existence of such distribution with desirable properties that allow concentration of the pushforward policies can be shown using pushforward concentrability (i.e., in \pref{lem:pushforward-concentration}).

\paragraph{Test Policy Validity.}
In our analysis, we will modify \pref{def:valid-test-policy-dd} as below.
\begin{definition}[Valid Test Policy]\label{def:valid-test-policy}
For a layer $h \in [H]$, we say a collection of policies $\Pitest_{h} = \crl*{\pi_{\empobs, \empobs'}}_{\empobs, \empobs' \in \estmdpobsspace{h}}$
    is a $\eta$-\emph{valid test policy set} for policy emulator $\wh{M}$ if the following hold.
    \begin{itemize}
        \item (Maximally distinguishing): $\pi_{\empobs, \empobs'} = \argmax_{\pi \in \cA \circ \Pi_{h+1:H}} \abs*{\wh{V}^\pi(\empobs) - \wh{V}^\pi(\empobs')}$.
        \item (Accurate): For all $\empobs, \empobs' \in \estmdpobsspace{h}$:
        \begin{align*}
            \abs*{V^{\pi_{\empobs, \empobs'}}(\empobs) - \wh{V}^{\pi_{\empobs, \empobs'}}(\empobs)} &\le \eta \quad \text{and} \quad    \abs*{V^{\pi_{\empobs, \empobs'}}(\empobs') - \wh{V}^{\pi_{\empobs, \empobs'}}(\empobs')} \le \eta.
        \end{align*}
    \end{itemize}
\end{definition}

\subsection{Supporting Technical Lemmas for Sampling}

In this section, we establish several technical lemmas which show that various conditions that we need in the analysis hold with high probability under samples from $M$.

\paragraph{Properties of Policy Emulator Initialization.} We prove several properties that hold with high probability when the policy emulator is initialized in \pref{line:init-start}-\ref{line:init-end} of \pref{alg:stochastic-bmdp-solver-v2}.

\begin{lemma}[Sampling of Pushforward-Reachable States]\label{lem:sampling-states}
With probability at least $1-\delta$:
\begin{align*}
    \forall~ h \in [H], \forall~ s \in 
\epsPRS{h}: \quad n_{s}[\estmdpobsspace{h}] \ge \frac{\eps}{2 \cpush S}\cdot \nreset.
\end{align*}
\end{lemma} 

\begin{proof}
Fix any $s \in \epsPRS{h}$. For any $i \in [\nreset]$, let $Z^{(i)}$ be the indicator variable of whether observation $x^{(i)}_h \sim \mu_h$ satisfies $\optdec(x^{(i)}_h) = s$. We know that $\En[Z^{(i)}] \ge \eps/(\cpush S)$. By Chernoff bounds we have
\begin{align*}
    \Pr \brk*{ \frac{1}{\nreset} \sum_{i=1}^{\nreset} Z^{(i)} \le \frac{1}{2} \cdot \frac{\eps}{\cpush S}} \le \exp \prn*{-\frac{\nreset \cdot \eps}{8\cpush S}},
\end{align*}
so as long as
\begin{align*}
    \nreset \ge \frac{8\cpush S}{\eps} \log \frac{SH}{\delta},
\end{align*}
by union bound, the conclusion of the lemma holds.
\end{proof}

\begin{lemma}[Pushforward Policy Concentration over $\mu$]\label{lem:pushforward-concentration} Suppose that the conclusion of \pref{lem:sampling-states} holds. Then with probability at least $1-\delta$:
\begin{align*}
    \forall~ h \in [H], \forall~ (s, a) \in \epsPRS{h} \times \cA: \quad 
    \max_{\pi \in \Pi}~\abs*{ \brk*{\pi \push \emission(s)}(a) - \brk*{\pi \push \unif(\crl{\empobs \in \estmdpobsspace{h}: \optdec(\empobs) = s})}(a) } \le \frac{\eps}{A}.
\end{align*}
\end{lemma}

\begin{proof}
Fix any $(s, a) \in \epsPRS{h} \times \cA$. Also fix any policy $\pi \in \Pi$. Denote the set $\cX_{s} = \crl{\empobs \in \estmdpobsspace{h}: \optdec(\empobs) = s}$, and observe that $\cX_{s}$ is drawn i.i.d.~from the emission distribution $\emission(s)$. By Hoeffding bounds we have
\begin{align*}
    \hspace{2em}&\hspace{-2em} \Pr\brk*{ \abs*{\brk*{\pi \push \emission(s)}(a) - \brk*{\pi \push \unif(\crl{\empobs \in \estmdpobsspace{h}: \optdec(\empobs) = s})}(a)} \ge \frac{\eps}{A}} \\
    &\le 2\exp \prn*{ -\frac{2 n_{s}[\estmdpobsspace{h}] \eps^2}{A^2} } \\
    &\le  2\exp \prn*{ -\frac{\nreset \eps^3}{\cpush S A^2} }, &\text{(\pref{lem:sampling-states})}
\end{align*}
Applying union bound we see that as long as 
\begin{align*}
    \nreset \ge \frac{\cpush SA^2}{\eps^3} \cdot \log \frac{2SAH\abs{\Pi}}{\delta}
\end{align*}
the conclusion of the lemma holds.
\end{proof}

\begin{lemma}[Sampling Rewards]\label{lem:sampling-rewards}
With probability at least $1-\delta$, every reward estimate $\wh{R}(\empobs,a)$ computed in  \pref{line:reward-estimation} of \pref{alg:stochastic-bmdp-solver-v2} satisfies $\abs{\wh{R}(\empobs, a) - R(\empobs,a)} \le \eps/H$.
\end{lemma} 
\begin{proof}
This follows by Hoeffding inequality and union bound over all $\nreset\cdot AH$ pairs $(\empobs, a) \in \estmdpobsspace{} \times \cA$.
\end{proof}

\paragraph{Properties of \stochdecoder{}.} Now we turn to analyzing a single call to \stochdecoder{}.

\begin{lemma}[Sampling Transitions]\label{lem:sampling-transitions} Fix any $(x_{h}, a_{h})$ for which we call $\stochdecoder{}$.  With probability at least $1-\delta$, the dataset $\cD$ sampled in \pref{line:sample-decode-dataset} of \pref{alg:stochastic-decoder-v2} satisfies
\begin{align*}
    \nrm*{\optlatp(\cdot \mid x_{h}, a_{h}) - \samplelatp(\cdot \mid x_{h}, a_{h})}_1 \le \eps.
\end{align*}
\end{lemma}
\begin{proof}
Every time a dataset $\cD$ is sampled, by concentration of discrete distributions we have for any $t > 0$:
\begin{align*}
    \Pr\brk*{\nrm*{\optlatp(\cdot \mid x_h, a_h) - \samplelatp(\cdot \mid x_h, a_h)}_1 \ge \sqrt{S} \cdot \prn*{\frac{1}{\sqrt{\ndec}} + t}} \le \exp(-\ndec t^2).
\end{align*}
Setting the RHS to $\delta$ we have that with probability at least $1-\delta$, %
\begin{align*}
    \nrm*{\optlatp(\cdot \mid x_h, a_h) - \samplelatp(\cdot \mid x_h, a_h)}_1 \le \sqrt{\frac{S \log (1/\delta)}{\ndec}}. 
\end{align*}
Therefore as long as 
\begin{align*}
    \ndec \ge \frac{S}{\eps^2} \cdot \log \frac{1}{\delta},
\end{align*}
the conclusion of the lemma holds.
\end{proof}

\begin{corollary}\label{corr:sampling-transitions} If the conclusion of \pref{lem:sampling-transitions} holds, then the proportion of observations from $(\epsPRS{h})^c$ in $\cD$ is at most $2\eps$.
\end{corollary}

\begin{lemma}[Pushforward Policy Concentration over Transitions]\label{lem:pushforward-concentration-transition}
Fix any $(x_h, a_h)$ for which we call $\stochdecoder{}$. With probability at least $1-\delta$, the dataset $\cD$ sampled in \pref{line:sample-decode-dataset} of \pref{alg:stochastic-decoder-v2} satisfies
\begin{align*}
    \forall~ s \in \cS[\cD], \forall~a \in \cA, \forall~ \pi \in \Pi:\quad 
    \Big| \brk*{\pi \push \emission(s)}(a) - \brk*{\pi \push \unif \prn*{\crl*{ \empobs \in \cD: \optdec(\empobs) = s } } }(a) \Big| \le \sqrt{\frac{2\log (2 SA \abs{\Pi}/\delta)}{n_s[\cD]}}.
\end{align*}
\end{lemma}

\begin{proof}
Fix a particular $s \in \latentsp[\cD]$, $a \in \cA$, and $\pi \in \Pi$. The set $\crl{\empobs \in \estmdpobsspace{h+1}: \optdec(\empobs) = s}$ is drawn i.i.d.~from the emission distribution $\emission(s)$. By Hoeffding bounds we have for any $t > 0$:
\begin{align*}
    \Pr\brk*{ \Big| \brk*{\pi \push \emission(s)}(a) - \brk*{\pi \push \unif\prn*{\crl*{ \empobs \in \cD: \optdec(\empobs) = s } } }(a) \Big| \ge t } &\le 2\exp \prn*{ -2 n_{s}[\cD] t^2}. 
\end{align*}
By union bound over all $(s,a)$ and $\pi$, with probability at least $1-\delta$:
\begin{align*}
    \Big| \brk*{\pi \push \emission(s)}(a) - \brk*{\pi \push \unif\prn*{\crl*{ \empobs \in \cD: \optdec(\empobs) = s } } }(a) \Big| &\le \sqrt{\frac{2\log (2 SA \abs{\Pi}/\delta)}{n_s[\cD]}}
\end{align*}
This concludes the proof of the lemma.
\end{proof}

\begin{lemma}[Monte Carlo Estimates for \stochdecoder]\label{lem:monte-carlo}
Fix any $(x_h, a_h)$ for which we call $\stochdecoder{}$. With probability at least $1-\delta$, every Monte Carlo estimate $\vestarg{x}{\pi}$ computed in \pref{line:monte-carlo} of \pref{alg:stochastic-decoder-v2} satisfies
$\abs*{ \vestarg{x}{\pi} - V^{\pi}(x) } \le \eps$.
\end{lemma}

\begin{proof}
By Hoeffding's inequality we know that for a fixed $(x, \pi)$ pair:
\begin{align*}
    \Pr\brk*{\abs*{\vestarg{x}{\pi} - V^{\pi}(x)} \ge \eps} \le 2 \exp(-2\nmc \eps^2).
\end{align*}
In total, we call \pref{line:monte-carlo} at most $\abs{\estmdpobsspace{h+1}}^2 \le \nreset^2$ times. Therefore, by union bound, as long as 
\begin{align*}
    \nmc \ge K \cdot \frac{1}{\eps^2} \cdot \log \frac{\cpush SAH \abs{\Pi}}{\eps \delta }
\end{align*}
where $K>0$ is an absolute constant determined by the value of $\nreset$, then the result holds.
\end{proof}

\paragraph{Properties of \stochrefit{}.} Now we establish the accuracy of estimates in a single call to \stochrefit{}.

\begin{lemma}[Monte Carlo Estimates for \stochrefit]\label{lem:refit-monte-carlo-accuracy}
    With probability at least $1-\delta$, every Monte Carlo estimate computed by \stochrefit{} (\pref{line:mc1} and \ref{line:mc2} of \pref{alg:stochastic-refit-v2}) is accurate up to error $\eps$.
\end{lemma}

\begin{proof}
In \stochrefit{} we compute Monte Carlo estimates for $2\nreset^2 + 2\nreset^3 \cdot AH$ times, since there are $2\nreset^2$ possible certificates $(x, \pi)$ and for each one we perform Monte Carlo estimates over all of the $(\bar{x},\bar{a})$ pairs in our policy emulator $\estmdp$. By Hoeffding bound and union bound we see that as long as 
\begin{align*}
    \nmc \ge K \cdot \frac{1}{\eps^2} \cdot \log \frac{\cpush SAH \abs{\Pi}}{\eps \delta},
\end{align*}
for some absolute constant $K>0$, the conclusion of the lemma holds.
\end{proof}

\paragraph{Additional Notation.}
Henceforth, let us define several events:
\begin{itemize}
    \item $\eventemulator \coloneqq \crl*{\text{the conclusions of \pref{lem:sampling-states} --- \ref{lem:sampling-rewards} hold}}$. We have $\Pr[\eventemulator] \ge 1-3\delta$.
    \item $\eventdec_t \coloneqq \crl*{\text{the conclusions of \pref{lem:sampling-transitions} --- \ref{lem:monte-carlo} hold on the $t$-th call to \stochdecoder{}}}$. We have $\Pr[\eventdec_t ] \ge 1-3\delta$. Furthermore, define the random variable $T_\mathsf{D}$ to be the total number of times that \stochdecoder{} is called.
    \item $\eventref_t \coloneqq \crl*{\text{the conclusion of \pref{lem:refit-monte-carlo-accuracy} holds on the $t$-th call to \stochrefit{}}}$. We have $\Pr[\eventref_t ] \ge 1-\delta$. Furthermore, define the random variable $T_\mathsf{R}$ to be the total number of times that \stochrefit{} is called.
\end{itemize}
In the analysis, we will drop the subscript $t$ when referring to $\eventdec_t$ and $\eventref_t$ if clear from the context.

\subsection{Analysis of \stochdecoder}

This section is dedicated to establishing \pref{lem:main-induction}, which is the main inductive lemma.

\begin{lemma}[Induction for \stochdecoder]\label{lem:main-induction} Fix any layer $h \in [H]$ and tuple $(x_h, a_h)$ on which \stochdecoder{} is called. Assume that:
\begin{itemize}
    \item $\eventemulator$ and $\eventdec$ hold.
    \item For all $\empobs \in \estmdpobsspace{h+1}$: $\max_{\pi \in \Pi_{h+1:H}}~ \abs{V^\pi(\empobs) - \wh{V}^\pi(\empobs
    )} \le \Gamma_{h+1}$.
    \item Input confidence set $\cP(x_h, a_h)$ satisfies $\projm{\estmdpobsspace{h+1}}(\optlatp(\cdot  \mid  x_h, a_h)) \in \cP(x_h, a_h)$.
    \item $\Pitest_{h+1}$ are $\taudec$-valid test policies for the policy emulator $\estmdp$.
\end{itemize}

Then \stochdecoder{} returns confidence set $\cP$ via Eq.~\eqref{eq:confidence-set-construction-v2} such that: 
    \begin{enumerate}[(1)]
        \item $\projm{\estmdpobsspace{h+1}}(\optlatp(\cdot  \mid  x_h, a_h)) \in \cP$;
        \item $\max_{\bar{p} \in \cP} \max_{\pi \in \Pi_{h+1:H}} \abs{Q^\pi(x_h, a_h) - \wh{R}(x_h, a_h) - \En_{\empobs \sim \bar{p}} \wh{V}^\pi(\empobs)} \le \Gamma_{h+1} + K \cdot \prn*{\beta + S \taudec}$.
    \end{enumerate} 
Here, $K>0$ is an absolute numerical constant.
\end{lemma}

\subsubsection{Structural Properties of the Decoder Graph}\label{app:decoder-graph}
For the lemmas in this section, we will assume the preconditions of \pref{lem:main-induction} and analyze properties of the decoder graph $\gobs$ constructed in a single call to \stochdecoder{}.

\begin{lemma}[Validity of Decoding Function]\label{lem:strong-validity-obs-decoding} Under the preconditions of \pref{lem:main-induction}, for every $x_l \in \cXL$, we have
\begin{align*}
    \crl{\empobs_r \in \cXR: \optdec(\empobs_r) = \optdec(x_l)} \subseteq \cT[x_l].
\end{align*}
\end{lemma}

\begin{proof}
The proof is a reprise of the argument used in Part (1) of \pref{lem:controlling-transition-error}. We prove this by contradiction. Suppose that there existed some $x_l \in \cXL$ and $\empobs_r \in \cXR$ such that $\optdec(x_l) = \optdec(\empobs_r)$ but $\empobs_r \notin \cT[x_l]$. Then $\empobs_r$ must have lost a test to some other $\empobs_r'$, i.e. there exists some $\empobs_r' \in \estmdpobsspace{h+1}$ such that
\begin{align*}
    \abs*{\vestarg{x_l}{\pi_{\empobs_r, \empobs_r'}} - \wh{V}^{\pi_{\empobs_r, \empobs_r'}}(\empobs_r) } \ge \taudec + 2\eps. \numberthis\label{eq:invalid}
\end{align*}
By accuracy of $\Pitest_{h+1}$ and the fact that $\pi_{\empobs_r, \empobs_r'}$ is open-loop at layer $h+1$, we have
\begin{align*}
    \abs*{ V^{\pi_{\empobs_r, \empobs_r'}}(x_l) - \wh{V}^{\pi_{\empobs_r, \empobs_r'}}(\empobs_r)} = \abs*{ V^{\pi_{\empobs_r, \empobs_r'}}(\empobs_r) - \wh{V}^{\pi_{\empobs_r, \empobs_r'}}(\empobs_r)} \le \taudec. \numberthis \label{eq:valid1}
\end{align*}
Furthermore, by \pref{lem:monte-carlo} we have 
\begin{align*}
    \abs*{ \vestarg{x_l}{\pi_{\empobs_r, \empobs_r'}} - V^{\pi_{\empobs_r, \empobs_r'}}(x_r) } = \abs*{ \vestarg{x_l}{\pi_{\empobs_r, \empobs_r'}} - V^{\pi_{\empobs_r, \empobs_r'}}(x_l) } \le \eps. \numberthis \label{eq:valid2}
\end{align*}
Combining \eqref{eq:valid1} and \eqref{eq:valid2} we get that
\begin{align*}
    \abs*{ \vestarg{x_l}{\pi_{\empobs_r, \empobs_r'}} - \wh{V}^{\pi_{\empobs_r, \empobs_r'}}(\empobs_r) } \le \taudec + \eps, 
\end{align*}
which contradicts \eqref{eq:invalid}. This proves the lemma.
\end{proof}

\begin{lemma}[Biclique Property]\label{lem:biclique}
Under the preconditions of \pref{lem:main-induction}, for any $s \in \cS[\cXL] \cap \cS[\cXR]$ the subgraph of $\gobs$ over vertices $\crl{x \in \cX_\mathsf{L} \cup \cX_\mathsf{R}: \optdec(x) = s}$ is a biclique.
\end{lemma}

\begin{proof}
Fix any $s \in \cS[\cXL] \cap \cS[\cXR]$. By \pref{lem:strong-validity-obs-decoding}, any $x_l \in \cXL$ such that $\optdec(x_l) = s$ has an edge to every observation $\crl{\empobs \in \cXR: \optdec(\empobs) = s}$ in $\gobs$. Therefore, the subgraph over $\crl{x \in \cX_\mathsf{L} \cup \cX_\mathsf{R}: \optdec(x) = s}$ forms a biclique in $\gobs$.
\end{proof}

\begin{lemma}\label{lem:eps-reachable-set-inclusion}
Under the preconditions of \pref{lem:main-induction}, for any connected component $\cc$, $\epsPRS{} \cap \cS[\ccL] \subseteq \epsPRS{} \cap \cS[\ccR]$.
\end{lemma}

\begin{proof}
Fix any $s \in \epsPRS{} \cap \cS[\ccL]$, and let $x_l \in \ccL$ be any arbitrary observation such that $\optdec(x_l) = s$. By \pref{lem:sampling-states}, since $s \in \epsPRS{}$, there exist some $\empobs_r \in \cXR$ such that $\optdec(\empobs_r) = s$; in other words, $s \in \cS[\cXR]$. Moreover by \pref{lem:strong-validity-obs-decoding}, there must be an edge from $x_l$ to $x_r$ in $\gobs$. Therefore $x_r \in \ccR$, so $s \in \cS[\ccR]$. %
\end{proof}

\begin{lemma}\label{lem:bound1}
Let $\empobs, \empobs' \in \cXR$ such that $\optdec(\empobs) = \optdec(\empobs')$. We have $\max_{a \in \cA, \pi \in \Pi}~\abs{ \wh{Q}^\pi(\empobs, a) - \wh{Q}^\pi(\empobs', a) } \le 2\taudec$.
\end{lemma}

\begin{proof}
Denote $\pi_{\empobs, \empobs'} = \argmax_{\pi \in \cA \circ \Pi}~\abs{\wh{V}^\pi(\empobs) - \wh{V}^\pi(\empobs')}$ to be the test policy for the pair $\empobs, \empobs' \in \cXR$. By accuracy of the test policy we know that 
\begin{align*}
    \abs*{V^{\pi_{\empobs, \empobs'}}(\empobs) - \wh{V}^{\pi_{\empobs, \empobs'}}(\empobs)} &\le \taudec \quad \text{and} \quad    \abs*{V^{\pi_{\empobs, \empobs'}}(\empobs') - \wh{V}^{\pi_{\empobs, \empobs'}}(\empobs')} \le \taudec.
\end{align*}
Furthermore since $\empobs, \empobs'$ are observations emitted from the same latent state and $\pi_{\empobs, \empobs'}$ is open loop at layer $h+1$, we have $V^{\pi_{\empobs, \empobs'}}(\empobs) = V^{\pi_{\empobs, \empobs'}}(\empobs')$. Therefore
\begin{align*}
    \max_{a \in \cA, \pi \in \Pi}~\abs*{\wh{Q}^\pi(\empobs,a) - \wh{Q}^\pi(\empobs',a)} &= \abs*{\wh{Q}^{\pi_{\empobs, \empobs'}}(\empobs, \pi_{\empobs, \empobs'}) - \wh{Q}^{\pi_{\empobs, \empobs'}}(\empobs', \pi_{\empobs, \empobs'})} \le 2\taudec.
\end{align*}
This concludes the proof of the lemma.
\end{proof}

\begin{lemma}\label{lem:bound2}
Fix any $x_l \in \cXL$. If $\empobs_r, \empobs_r' \in \cT[x_l]$, then $\max_{a \in \cA, \pi \in \Pi}~\abs{\wh{Q}^{\pi}(\empobs_r, a) - \wh{Q}^{\pi}(\empobs_r', a)} \le 2\taudec + 4 \eps$.
\end{lemma}

\begin{proof}
By definition of $\cT[x_l]$ we have
\begin{align*}
    \abs*{\vestarg{x_l}{\pi_{\empobs_r, \empobs_r'}} - \wh{V}^{\pi_{\empobs_r, \empobs_r'}}(\empobs_r) } \le \taudec+2\eps \quad\text{and}\quad \abs*{\vestarg{x_l}{\pi_{\empobs_r, \empobs_r'}} - \wh{V}^{\pi_{\empobs_r, \empobs_r'}}(\empobs_r') } \le \taudec+2\eps.
\end{align*}
Using the fact that test policies are maximally distinguishing we have
    \begin{align*}
        \max_{a \in \cA, \pi \in \Pi}~\abs*{\wh{Q}^{\pi}(\empobs_r, a) - \wh{Q}^{\pi}(\empobs_r', a)} = \abs*{\wh{V}^{\pi_{\empobs_r, \empobs_r'}}(\empobs_r) - \wh{V}^{\pi_{\empobs_r, \empobs_r'}}(\empobs_r')} \le 2\taudec + 4 \eps.
    \end{align*}
This proves the lemma.
\end{proof}

\begin{lemma}[Bounded Width of $\cc$]\label{lem:width-of-cc}
For any connected component $\cc \in \crl{\cc_j}_{j \ge 1}$ in $\gobs$ we have 
\begin{align*}
    \max_{\empobs, \empobs' \in \ccR} \max_{a \in \cA, \pi \in \Pi}~\abs{ \wh{Q}^\pi(\empobs, a) - \wh{Q}^\pi(\empobs', a) } \le 4S\taudec + 8S\eps.
\end{align*}
\end{lemma}

\begin{proof}
Let us take any $\empobs, \empobs' \in \ccR$. Since $\empobs, \empobs'$ belong to the same connected component, there exists a sequence of observations $\mathsf{seq} = (\empobs_1 = \empobs, \dots, \empobs_n = \empobs') \in (\ccR)^n$ such that for every consecutive pair $\empobs_i, \empobs_{i+1}$ there exists some $x_l \in \ccL$ such that $\empobs_i, \empobs_{i+1} \in \cT[x_l]$. 

Fix any $a \in \cA, \pi \in \Pi$. Now we will bound $\abs{ \wh{Q}^\pi(\empobs, a) - \wh{Q}^\pi(\empobs', a) }$. We construct an auxiliary sequence  $\wt{\mathsf{seq}} = (\wt{\empobs}_1, \dots, \wt{\empobs}_k)$ for some $k \le n$ as follows:
\begin{itemize}
    \item Initialize $\wt{\mathsf{seq}} = \emptyset$.
    \item For $i = 1, \cdots, n$:
    \begin{itemize}
        \item Add $\empobs_i$ to the end of $\wt{\mathsf{seq}}$.
        \item If there exists $\empobs_j$ with $j > i$ such that $\optdec(\empobs_i) = \optdec(\empobs_j)$ then set $i \gets j$.
    \end{itemize}
\end{itemize}
Observe that $\wt{\mathsf{seq}}$ satisfies the following conditions:
\begin{itemize}
    \item $\wt{\empobs}_1 = \empobs$ and $\wt{\empobs}_k = \empobs'$.
    \item  For every $s \in \cS$, at most two observations $\wt{\empobs}, \wt{\empobs}' \in \mathrm{supp}(\emission(s)) \cap \ccR$ are found in $\wt{\mathsf{seq}}$, and these observations must appear sequentially.
    \item For any $i \in [k-1]$, if $\optdec(\wt{\empobs}_i) \ne \optdec(\wt{\empobs}_{i+1})$ then there exists some $x_l \in \ccL$ such that $\wt{\empobs}_i, \wt{\empobs}_{i+1} \in \cT[x_l]$.
\end{itemize}
Now we can apply triangle inequality to $\wt{\mathsf{seq}}$:
\begin{align*}
    \abs*{ \wh{Q}^\pi(\empobs, a) - \wh{Q}^\pi(\empobs', a) } \le \sum_{i=1}^k \abs*{ \wh{Q}^\pi(\wt{\empobs}_i, a) - \wh{Q}^\pi(\wt{\empobs}_{i+1}, a) } \le 4S\taudec + 8S\eps.
\end{align*}
The final bound uses the aforementioned properties of $\wt{\mathsf{seq}}$, as well as  \pref{lem:bound1} and \pref{lem:bound2} to handle the individual terms in the summation. This completes the proof of \pref{lem:width-of-cc}.
\end{proof}

\subsubsection{Structural Properties of Projected Measures}\label{app:projected-measures}
Now we will prove several lemmas regarding the projected measure of the empirical latent distribution
\begin{align*}
    \samplelatp = \frac{1}{\abs{\cXL}} \sum_{x \in \cXL} \delta_{\optdec(x)}.
\end{align*}
which is sampled in \pref{line:sample-decode-dataset} of a single call to \stochdecoder{}.

\begin{lemma}\label{lem:proj-cc-expression}
    Under the preconditions of \pref{lem:main-induction}, for any connected component $\cc$ of $\gobs$:
    \begin{align*}
        \projR(\samplelatp)(\ccR) =  \sum_{s \in \epsPRS{} \cap \cS[\ccL] \cap \cS[\ccR]} \frac{n_{s}[\cXL]}{\abs{\cXL}}  = \sum_{s \in \epsPRS{} \cap \cS[\ccL] \cap \cS[\ccR]} \samplelatp(s).
    \end{align*}
\end{lemma}

\begin{proof}
We compute that
\begin{align*}
    \projR(\samplelatp)(\ccR) &= \sum_{\empobs \in \ccR} \projR(\samplelatp)(\empobs) \\
    &= \sum_{\empobs \in \ccR} \frac{n_{\optdec(\empobs)}[\cXL]}{\abs{\cXL}} \cdot \frac{\ind{\optdec(\empobs) \in \epsPRS{}}}{n_{\optdec(\empobs)}[\cXR]}\\
    &= \sum_{s \in \epsPRS{}} \frac{n_{s}[\cXL]}{\abs{\cXL}} \sum_{\empobs \in \ccR}  \frac{\ind{\optdec(\empobs) = s}}{n_{s}[\cXR]} \\
    &\overset{(i)}{=} \sum_{s \in \epsPRS{} \cap \cS[\ccR]} \frac{n_{s}[\cXL]}{\abs{\cXL}} \sum_{\empobs \in \ccR}  \frac{\ind{\optdec(\empobs) = s}}{n_{s}[\cXR]} \\
    &\overset{(ii)}{=}  \sum_{s \in \epsPRS{} \cap \cS[\cXL] \cap \cS[\ccR]} \frac{n_{s}[\cXL]}{\abs{\cXL}} \sum_{\empobs \in \ccR}  \frac{\ind{\optdec(\empobs) = s}}{n_{s}[\cXR]}.
\end{align*}
For $(i)$, observe that if $s \notin \cS[\ccR]$, then the sum $\sum_{\empobs \in \ccR}  \frac{\ind{\optdec(\empobs) = s}}{n_{s}[\cXR]} = 0$. For $(ii)$, we use the fact that $n_s[\cXL] = 0$ if $s \notin \cS[\cXL]$. From here, we apply the biclique lemma (\pref{lem:biclique}). The biclique lemma implies that if $s \in \cS[\ccR]$, then $\crl{x \in \cXL: \optdec(x) = s} \subseteq \ccL$, and therefore $\cS[\cXL] \cap \cS[\ccR] = \cS[\ccL] \cap \cS[\ccR]$. Furthermore for any $s \in \cS[\ccL] \cap \cS[\ccR]$, all of the observations $\crl{\empobs \in \cXR: \optdec(\empobs) = s} \subseteq \ccR$, so $n_s[\cXR] = n_s[\ccR]$. Thus we can continue the calculation as
\begin{align*}
    \projR(\samplelatp)(\ccR) &=  \sum_{s \in \epsPRS{} \cap \cS[\ccL] \cap \cS[\ccR]} \frac{n_{s}[\cXL]}{\abs{\cXL}} \sum_{\empobs \in \ccR}  \frac{\ind{\optdec(\empobs) = s}}{n_{s}[\cXR]}\\
    &= \sum_{s \in \epsPRS{} \cap \cS[\ccL] \cap \cS[\ccR]} \frac{n_{s}[\cXL]}{\abs{\cXL}} \sum_{\empobs \in \ccR}  \frac{\ind{\optdec(\empobs) = s}}{n_{s}[\ccR]} \\
    &= \sum_{s \in \epsPRS{} \cap \cS[\ccL] \cap \cS[\ccR]} \frac{n_{s}[\cXL]}{\abs{\cXL}} = \sum_{s \in \epsPRS{} \cap \cS[\ccL] \cap \cS[\ccR]} \samplelatp(s).
\end{align*} 
This concludes the proof of \pref{lem:proj-cc-expression}.
\end{proof}

\begin{corollary}\label{corr:estimating-projR} Under the preconditions of \pref{lem:main-induction}, then
$\sum_{\cc \in \crl{\cc_j}} \prn{ \frac{\abs{\ccL}}{\abs{\cXL}} 
 - \projR(\samplelatp)(\ccR) } \in [0,2\eps]$.
\end{corollary}
\begin{proof}
For any $\cc$ we have
\begin{align*}
    \hspace{2em}&\hspace{-2em} \frac{\abs{\ccL}}{\abs{\cXL}} - \projR(\samplelatp)(\ccR) \\
    &= \frac{\abs{\ccL}}{\abs{\cXL}} - \sum_{s \in \epsPRS{} \cap \cS[\ccL] \cap \cS[\ccR]} \frac{n_{s}[\cXL]}{\abs{\cXL}} &\text{(\pref{lem:proj-cc-expression})} \\
    &= \frac{\abs{\ccL}}{\abs{\cXL}} - \sum_{s \in \epsPRS{} \cap \cS[\ccL]} \frac{n_{s}[\cXL]}{\abs{\cXL}} &\text{(\pref{lem:eps-reachable-set-inclusion})} \\
    &= \sum_{s \in \epsPRS{} \cap \cS[\ccL]} \frac{n_s[\ccL]}{\abs{\cXL}} + \sum_{s \in (\epsPRS{})^c \cap \cS[\ccL]} \frac{n_s[\ccL]}{\abs{\cXL}} - \sum_{s \in \epsPRS{} \cap \cS[\ccL]} \frac{n_{s}[\cXL]}{\abs{\cXL}}  \\
    &= \sum_{s \in (\epsPRS{})^c \cap \cS[\ccL]} \frac{n_s[\ccL]}{\abs{\cXL}}.
\end{align*}
The last equality uses the fact that by \pref{lem:eps-reachable-set-inclusion}, $s \in \epsPRS{} \cap \cS[\ccL] \Rightarrow s \in \cS[\ccR]$, so in particular by \pref{lem:biclique} we have $\crl{x\in \cXL: \optdec(x) = s} \subseteq \ccL$, so therefore $n_s[\cXL] = n_s[\ccL]$. 

Summing over all $\cc$ and applying \pref{corr:sampling-transitions} we get that
\begin{align*}
    \sum_{\cc \in \crl{\cc_j}} \frac{\abs{\ccL}}{\abs{\cXL}} - \projR(\samplelatp)(\ccR) = \sum_{\bbC \in \crl{\bbC_j}} \sum_{s \in (\epsPRS{})^c \cap \cS[\ccL]} \frac{n_s[\ccL]}{\abs{\cXL}} = \sum_{s \in (\epsPRS{})^c} \frac{n_s[\cXL]}{\abs{\cXL}} \in [0, 2\eps].
\end{align*}
This proves \pref{corr:estimating-projR}.
\end{proof}

\begin{lemma}\label{lem:relating-proj-to-empirical}
Under the preconditions of \pref{lem:main-induction}, for every $\pi \in \Pi$ and $\cc \in \crl{\cc_j}$: 
\begin{align*}
    \max_{a \in \cA}~\abs*{ \brk*{ \pi\push \projR(\samplelatp)(\cdot \mid \ccR)}(a) - \brk*{\pi \push \unif(\ccL)}(a) } \le \frac{\eps}{A} +  K \cdot \sqrt{\frac{S \log \tfrac{SA \abs{\Pi}}{\delta}}{\nrch[\ccL]}}  + \frac{\nurch[\ccL]}{\nrch[\ccL]},
\end{align*}
where $K>0$ is an absolute constant.
\end{lemma}
\begin{proof}
Fix any $\pi \in \Pi$, $\cc \in \crl{\cc_j}$, and $a \in \cA$. We can calculate that: 
\begin{align*}
    \hspace{2em}&\hspace{-2em} \brk*{ \pi\push \projR(\samplelatp)(\cdot \mid \ccR)}(a) \\
    &= \sum_{\empobs \in \ccR} \frac{\projR(\samplelatp)(\empobs)}{\projR(\samplelatp)(\ccR)} \indd{\pi(\empobs) = a} \\
    &= \frac{1}{\projR(\samplelatp)(\ccR)} \sum_{\empobs \in \ccR} \frac{n_{\optdec(\empobs)}[\cXL]}{\abs{\cXL}} \cdot \frac{\ind{\optdec(\empobs) \in \epsPRS{}} \indd{\pi(\empobs) = a}}{n_{\optdec(\empobs)}[\cXR]} \\
    &= \frac{1}{\projR(\samplelatp)(\ccR)} \sum_{s \in \epsPRS{}} \prn*{ \frac{n_{s}[\cXL]}{\abs{\cXL}} \cdot \frac{1}{n_{s}[\cXR]} \cdot \sum_{\empobs \in \ccR}  \ind{\optdec(\empobs) =s} \indd{\pi(\empobs) = a} }\\
    &= \frac{1}{\projR(\samplelatp)(\ccR)} \sum_{s \in \epsPRS{} \cap \cS[\ccR] } \prn*{\frac{n_{s}[\cXL]}{\abs{\cXL}} \cdot \frac{1}{n_{s}[\cXR]} \cdot \sum_{\empobs \in \ccR}  \ind{\optdec(\empobs) =s} \indd{\pi(\empobs) = a} } \\
    &= \frac{1}{\projR(\samplelatp)(\ccR)} \sum_{s \in \epsPRS{} \cap \cS[\cXL] \cap \cS[\ccR] } \prn*{\frac{n_{s}[\cXL]}{\abs{\cXL}} \cdot \frac{1}{n_{s}[\cXR]} \cdot \sum_{\empobs \in \ccR}  \ind{\optdec(\empobs) =s} \indd{\pi(\empobs) = a} } \\
    &= \frac{1}{\projR(\samplelatp)(\ccR)} \sum_{s \in \epsPRS{} \cap \cS[\ccL] \cap \cS[\ccR] } \prn*{\frac{n_{s}[\ccL]}{\abs{\cXL}} \cdot \frac{1}{n_{s}[\ccR]} \cdot \sum_{\empobs \in \ccR}  \ind{\optdec(\empobs) =s} \indd{\pi(\empobs) = a} }.
\end{align*}
The last line uses the biclique lemma (\pref{lem:biclique}) in the same fashion as the proof of \pref{lem:proj-cc-expression}. Now we apply the conclusions of \pref{lem:pushforward-concentration} and \pref{lem:pushforward-concentration-transition}, along with the fact that for every $s \in \epsPRS{} \cap \cS[\ccL] \cap \cS[\ccR]$ we have $\crl{\empobs \in \cXR : \optdec(\empobs) = s } \subseteq \ccR$ as well as $\crl{\empobs \in \cXL : \optdec(\empobs) = s } \subseteq \ccL$ (which is again implied by the biclique lemma):
\begin{align*}
     \hspace{4em}&\hspace{-4em} \brk*{ \pi\push \projR(\samplelatp)(\cdot \mid \ccR)}(a) \\
    \le\quad& \frac{1}{\projR(\samplelatp)(\ccR) } \sum_{s \in \epsPRS{} \cap \cS[\ccL] \cap \cS[\ccR] } \frac{n_{s}[\ccL]}{\abs{\cXL}} \cdot \prn*{ \brk*{\pi \push \emission(s)}(a) + \frac{\eps}{A} } \\
    \le\quad& \frac{1}{\projR(\samplelatp)(\ccR) } \sum_{s \in \epsPRS{} \cap \cS[\ccL] \cap \cS[\ccR] } \frac{n_{s}[\ccL]}{\abs{\cXL}} \cdot \Bigg( \frac{\eps}{A} + \sqrt{\frac{2 \log \tfrac{2SA \abs{\Pi}}{\delta}}{n_s[\ccL]}} \\
    &\hspace{20em}+ \frac{1}{n_{s}[\ccL]} \cdot \sum_{\empobs \in \ccL}  \ind{\optdec(\empobs) =s} \indd{\pi(\empobs) = a}\Bigg) \\
    =\quad& \frac{1}{\projR(\samplelatp)(\ccR) \abs{\cXL}} \sum_{ s \in \epsPRS{} \cap \cS[\ccL] \cap \cS[\ccR] } \Bigg( \frac{\eps}{A} \cdot n_{s}[\ccL] + \sqrt{2n_s[\ccL] \log \tfrac{2 SA \abs{\Pi}}{\delta}} \\
    &\hspace{20em}+ \sum_{\empobs \in \ccL}  \ind{\optdec(\empobs) =s} \indd{\pi(\empobs) = a} \Bigg)
\end{align*}
By \pref{lem:proj-cc-expression}, we have $\abs{\cXL} \projR(\samplelatp)(\ccR) = \sum_{s \in \epsPRS{} \cap \cS[\ccL] \cap \cS[\ccR]} n_{s}[\cXL] = \nrch[\ccL] $. Using Cauchy-Schwarz we get that
\begin{flalign*}
\hspace{2em}&\hspace{-2em} \brk*{ \pi\push \projR(\samplelatp)(\cdot \mid \ccR)}(a) \\
    &\le \frac{\eps}{A} + K \cdot \sqrt{\frac{S \log \tfrac{SA \abs{\Pi}}{\delta}}{\nrch[\ccL]}} + \frac{1}{\nrch[\ccL]} \sum_{s \in \epsPRS{} \cap \cS[\ccL] \cap \cS[\ccR]  } \prn*{\sum_{\empobs \in \ccL}  \ind{\optdec(\empobs) =s} \indd{\pi(\empobs) = a} } 
    \\&= \frac{\eps}{A} + K \cdot \sqrt{\frac{S \log \tfrac{SA \abs{\Pi}}{\delta}}{\nrch[\ccL]}} + \frac{\abs{\ccL}}{\nrch[\ccL]} \cdot \frac{1}{\abs{\ccL}}  \sum_{s \in \epsPRS{} \cap \cS[\ccL] \cap \cS[\ccR]} \prn*{\sum_{\empobs \in \ccL}  \ind{\optdec(\empobs) =s} \indd{\pi(\empobs) = a} } \numberthis \label{eq:partial}
\end{flalign*}
Let us investigate the last term. We have
\begin{align*}
\hspace{4em}&\hspace{-4em}
\frac{1}{\abs{\ccL}}  \sum_{s \in \epsPRS{} \cap \cS[\ccL] \cap \cS[\ccR]} \prn*{\sum_{\empobs \in \ccL}  \ind{\optdec(\empobs) =s} \indd{\pi(\empobs) = a} } \\
    =\quad& \frac{1}{\abs{\ccL}}  \sum_{s \in \epsPRS{} \cap \cS[\ccL]} \prn*{\sum_{\empobs \in \ccL}  \ind{\optdec(\empobs) =s} \indd{\pi(\empobs) = a} } &\text{(\pref{lem:eps-reachable-set-inclusion})}\\
    \le\quad& \frac{1}{\abs{\ccL}}  \sum_{s \in \epsPRS{} \cap \cS[\ccL]} \prn*{\sum_{\empobs \in \ccL}  \ind{\optdec(\empobs) =s} \indd{\pi(\empobs) = a} } \\
    &\hspace{5em} + \frac{1}{\abs{\ccL}}  \sum_{s \in (\epsPRS{})^c \cap \cS[\ccL]} \prn*{\sum_{\empobs \in \ccL}  \ind{\optdec(\empobs) =s} \indd{\pi(\empobs) = a} } \\
    =\quad& \brk*{\pi \push \unif(\ccL)}(a)
\end{align*}
Plugging this back into Eq.~\eqref{eq:partial} we get that
\begin{align*}
\brk*{ \pi\push \projR(\samplelatp)(\cdot \mid \ccR)}(a) &\le \frac{\eps}{A} +  K \cdot \sqrt{\frac{S \log \tfrac{SA \abs{\Pi}}{\delta}}{\nrch[\ccL]}} + \frac{\abs{\ccL}}{\nrch[\ccL]} \cdot \brk*{\pi \push \unif(\ccL)}(a),
\end{align*}
and rearranging and using the fact that $\brk*{\pi \push \unif(\ccL)}(a) \in [0,1]$ we get that
\begin{align*}
    \brk*{ \pi\push \projR(\samplelatp)(\cdot \mid \ccR)}(a) - \brk*{\pi \push \unif(\ccL)}(a) \le \frac{\eps}{A} +   K \cdot \sqrt{\frac{S \log \tfrac{SA \abs{\Pi}}{\delta}}{\nrch[\ccL]}}  + \frac{\nurch[\ccL]}{\nrch[\ccL]}.
\end{align*}
One can repeat the same steps to get the lower bound. Therefore,
\begin{align*}
    \abs*{ \brk*{ \pi\push \projR(\samplelatp)(\cdot \mid \ccR)}(a) - \brk*{\pi \push \unif(\ccL)}(a) } \le \frac{\eps}{A} +   K \cdot \sqrt{\frac{S \log \tfrac{SA \abs{\Pi}}{\delta}}{\nrch[\ccL]}} + \frac{\nurch[\ccL]}{\nrch[\ccL]}.
\end{align*}
This proves \pref{lem:relating-proj-to-empirical}.
\end{proof}

\subsubsection{Proof of \pref{lem:main-induction}}

Fix the $(x_h, a_h)$ pair on which we call \stochdecoder.

For notational convenience we will denote $\cXL \coloneqq \cD$ and $\cXR \coloneqq \estmdpobsspace{h+1}$, as well as use $\optlatp = \optlatp(\cdot \mid \optdec(x_{h}), a_{h})$ to denote the ground truth latent transition function. Throughout the proof, we use $K>0$ to denote absolute constants whose values may change line-by-line. 

\paragraph{Part (1).} Since the input confidence set $\cP$ satisfies the third bullet, it suffices to show that that $\projm{\cXR}(\optlatp)$ satisfies both of the constraints in the confidence set construction of Eq.~\eqref{eq:confidence-set-construction-v2}.

For the first constraint, observe that by \pref{corr:estimating-projR},
\begin{align*}
    \sum_{\cc \in \crl{\cc_j}} \abs*{ \frac{\abs{\ccL}}{\abs{\cXL}} 
 - \projR(\samplelatp)(\ccR) } \le 2\eps.
\end{align*}
Therefore it suffices to show that 
\begin{align*}
    \sum_{\cc \in \crl{\cc_j}} \abs*{\projR(\optlatp)(\ccR) - \projR(\samplelatp)(\ccR) } \le \eps.
\end{align*}
We calculate that
\begin{align*}
    \hspace{2em}&\hspace{-2em} \sum_{\cc \in \crl{\cc_j}} \abs*{\projR(\optlatp)(\ccR) - \projR(\samplelatp)(\ccR)} \\
    &\le \sum_{\empobs \in \cXR}  \abs*{\projR(\optlatp)(\empobs) - \projR(\samplelatp)(\empobs)} \\
    &= \sum_{\empobs \in \cXR}  \abs*{\prn*{\optlatp(\optdec(\empobs)) - \samplelatp(\optdec(\empobs))} \cdot \frac{\ind{\optdec(\empobs) \in \epsPRS{}}}{n_{\optdec(\empobs)}[\cXR]}} \\
    &= \sum_{s \in \epsPRS{}} \sum_{\empobs \in \cXR}  \abs*{\prn*{\optlatp(s) - \samplelatp(s)} \cdot \frac{\ind{\optdec(\empobs) = s}}{n_{s}[\cXR]}} \\
    &= \sum_{s \in \epsPRS{}} \abs*{\optlatp(s) - \samplelatp(s) } \le \eps. &\text{(\pref{lem:sampling-transitions})} \numberthis \label{eq:bound-on-empirical-proj}
\end{align*}
Now we prove that $\projR(\optlatp)$ also satisfies the second constraint, i.e.,
\begin{align*}
    \sum_{\cc \in \crl{\cc_j}}  \frac{\abs{\ccL}}{\abs{\cXL}} \cdot \nrm*{\pi \push \unif(\ccL) - \pi \push \projR(\optlatp)(\cdot \mid \ccR) }_1 \le \eps.
\end{align*}

Observe that we can break up the bound as follows:
\begin{align*}
    \hspace{2em}&\hspace{-2em} \sum_{\cc \in \crl{\cc_j}}  \frac{\abs{\ccL}}{\abs{\cXL}} \cdot \nrm*{\pi \push \unif(\ccL) - \pi \push \projR(\optlatp)(\cdot \mid \ccR) }_1 \\
    &\le \underbrace{ \sum_{\cc \in \crl{\cc_j}}  \frac{\abs{\ccL}}{\abs{\cXL}} \cdot \nrm*{\pi \push \unif(\ccL) - \pi \push \projR(\samplelatp)(\cdot \mid \ccR) }_1 }_{\eqqcolon \mathtt{Term}_1}  \\
    &\hspace{5em} + \underbrace{\sum_{\cc \in \crl{\cc_j}}  \frac{\abs{\ccL}}{\abs{\cXL}} \cdot \nrm*{\pi \push \projR(\samplelatp)(\cdot \mid \ccR) - \pi \push \projR(\optlatp)(\cdot \mid \ccR) }_1 }_{\eqqcolon \mathtt{Term}_2}.
\end{align*}

\emph{Bounding $\mathtt{Term}_1$.} To bound $\mathtt{Term}_1$, we compute:
\begin{align*}
    \hspace{2em}&\hspace{-2em} \sum_{\cc \in \crl{\cc_j}}  \frac{\abs{\ccL}}{\abs{\cXL}} \cdot \nrm*{\pi \push \unif(\ccL) - \pi \push \projR(\samplelatp)(\cdot \mid \ccR) }_1 \\
    &\le \sum_{\cc \in \crl{\cc_j}:  \frac{\abs{\ccL}}{\abs{\cXL}} \ge 4\eps}  \frac{\abs{\ccL}}{\abs{\cXL}} \cdot \nrm*{\pi \push \unif(\ccL) - \pi \push \projR(\samplelatp)(\cdot \mid \ccR) }_1 + \sum_{\cc \in \crl{\cc_j} :  \frac{\abs{\ccL}}{\abs{\cXL}} < 4\eps }  \frac{2\abs{\ccL}}{\abs{\cXL}} \\
    &\overset{(i)}{\le} \sum_{\cc \in \crl{\cc_j}:  \frac{\abs{\ccL}}{\abs{\cXL}} \ge 4\eps}  \frac{\abs{\ccL}}{\abs{\cXL}} \cdot \nrm*{\pi \push \unif(\ccL) - \pi \push \projR(\samplelatp)(\cdot \mid \ccR) }_1 + (8S+4)\eps \\
    &\overset{(ii)}{\le}  \sum_{\cc \in \crl{\cc_j} :  \frac{\abs{\ccL}}{\abs{\cXL}} \ge 4\eps }  \frac{\abs{\ccL}}{\abs{\cXL}} \cdot \prn*{ K \cdot \sqrt{\frac{S A^2 \log \tfrac{SA \abs{\Pi}}{\delta}}{\nrch[\ccL]}} + A \cdot \frac{\nurch[\ccL]}{\nrch[\ccL]} } + (8S+5)\eps\\
    &=  \sum_{\cc \in \crl{\cc_j} :  \frac{\abs{\ccL}}{\abs{\cXL}} \ge 4\eps }  \frac{\nrch[\ccL] + \nurch[\ccL]}{\abs{\cXL}} \cdot \prn*{ K \cdot \sqrt{\frac{S A^2 \log \tfrac{SA \abs{\Pi}}{\delta}}{\nrch[\ccL]}}  + A \cdot \frac{\nurch[\ccL]}{\nrch[\ccL]} } + (8S+5)\eps. \numberthis\label{eq:term1-intermediate}
\end{align*}
The inequality $(i)$ follows by casework on $\cc \in \crl{\cc_j}$:
\begin{itemize}
    \item If $\epsPRS{} \cap \cS[\ccL] \ne \emptyset$ then by the biclique lemma (\pref{lem:biclique}) we have $\crl{x \in \cXL: \optdec(x) \in \epsPRS{} \cap \cS[\ccL]} \subseteq \ccL$. In other words, all of the observations from states in $\epsPRS{} \cap \cS[\ccL]$ are contained in this $\ccL$. Therefore, there can be at most $S$ such components $\cc$, and their contribution to the sum is $8\eps \cdot S$.
    \item If $\epsPRS{} \cap \cS[\ccL] = \emptyset$, then $\ccL$ only contains observations from $(\epsPRS{})^c$, and therefore the total size of such $\ccL$ can be bounded by $2\eps \cdot \abs{\cXL}$ using \pref{corr:sampling-transitions}. Their contribution to the sum is $4\eps$.
\end{itemize}
Furthermore, $(ii)$ uses \pref{lem:relating-proj-to-empirical}.

We now proceed to separately bound the terms in Eq.~\eqref{eq:term1-intermediate}. First, observe that
\begin{align*}
    K \sqrt{SA^2 \log \tfrac{SA \abs{\Pi}}{\delta}}\cdot \sum_{\cc \in \crl{\cc_j}: \frac{\abs{\ccL}}{\abs{\cXL}} \ge 4\eps}  \frac{\sqrt{\nrch[\ccL]} }{\abs{\cXL}}  &\le K \sqrt{\frac{S^2A^2 \log \tfrac{SA \abs{\Pi}}{\delta}}{\nrch[\cXL]}} \\
    &\le K \sqrt{\frac{S^2A^2 \log \tfrac{ SA \abs{\Pi}}{\delta}}{\ndec}} \\
    &\le \eps. \numberthis \label{eq:term1-intermediate2}
\end{align*}
The first inequality follows because by the biclique lemma (\pref{lem:biclique}) we know that the summation must have at most $S$ terms, since each of the $\cc$ contains some $s \in \epsPRS{}$, so we can apply Cauchy-Schwarz for $S$-dimensional vectors. The second inequality is a consequence of \pref{corr:sampling-transitions}, and the last inequality follows by our choice of $\ndec$. 

In addition by \pref{corr:sampling-transitions},
\begin{align*}
    \sum_{\cc \in \crl{\cc_j}: \frac{\abs{\ccL}}{\abs{\cXL}} \ge 4\eps }  \frac{\nrch[\ccL] }{\abs{\cXL}} \frac{\nurch[\ccL]}{\nrch[\ccL]} \le 2\eps. \numberthis \label{eq:term1-intermediate3}
\end{align*}
For the other two terms, observe that by \pref{lem:relating-proj-to-empirical}, when $ \frac{\abs{\ccL}}{\abs{\cXL}} \ge 4\eps$ we must have $\frac{\nurch[\ccL]}{\nrch[\ccL]} \le 1$ so therefore
\begin{align*}
    \hspace{2em}&\hspace{-2em} \sum_{\cc \in \crl{\cc_j} : \frac{\abs{\ccL}}{\abs{\cXL}} \ge 4\eps}  \frac{\nurch[\ccL]}{\abs{\cXL}} 
  \prn*{ K \cdot \sqrt{\frac{SA^2 \log \tfrac{SA \abs{\Pi}}{\delta}}{\nrch[\ccL]}}+ \frac{\nurch[\ccL]}{\nrch[\ccL]} } \\
  &\le \sum_{\cc \in \crl{\cc_j}: \frac{\abs{\ccL}}{\abs{\cXL}} \ge 4\eps} \frac{\nurch[\ccL]}{\abs{\cXL}} 
  \prn*{ K \cdot \sqrt{SA^2 \log \frac{ SA \abs{\Pi}}{\delta} } +1} \\
  &\le K \sqrt{SA^2 \log \frac{ SA \abs{\Pi}}{\delta} } \cdot \eps. \numberthis \label{eq:term1-intermediate4}
\end{align*}
Combining Eqns.~\eqref{eq:term1-intermediate}, \eqref{eq:term1-intermediate2}, \eqref{eq:term1-intermediate3}, and \eqref{eq:term1-intermediate4} we get that
\begin{align*}
    \sum_{\cc \in \crl{\cc_j}}  \frac{\abs{\ccL}}{\abs{\cXL}} \cdot \nrm*{\pi \push \unif(\ccL) - \pi \push \projR(\samplelatp)(\cdot \mid \ccR) }_1 \le K \prn*{  \sqrt{SA^2 \log \frac{ SA \abs{\Pi}}{ \delta} } + S}\eps. \numberthis \label{eq:term1-final}
\end{align*}
\emph{Bounding $\mathtt{Term}_2$.} To bound $\mathtt{Term}_2$, fix any $\cc \in \crl{\cc_j}$. Note that
\begin{align*}
     \hspace{2em}&\hspace{-2em} \nrm*{\pi \push \projR(\optlatp)(\cdot \mid \ccR) - \pi\push \projR(\samplelatp)(\cdot \mid \ccR)}_1 \\
     &= \sum_{a \in \cA} \abs*{\sum_{\empobs \in \ccR} \prn*{ \frac{\projR(\samplelatp)(\empobs)}{\projR(\samplelatp)(\ccR)}  - \frac{\projR(\optlatp)(\empobs)}{\projR(\optlatp)(\ccR)} }\indd{\pi(\empobs) = a}}\\
     &\le \sum_{\empobs \in \ccR} \abs*{ \frac{\projR(\samplelatp)(\empobs)}{\projR(\samplelatp)(\ccR)}  - \frac{\projR(\optlatp)(\empobs)}{\projR(\optlatp)(\ccR)} } \\
     &=  \sum_{\empobs \in \ccR} \abs*{ \frac{ \samplelatp(\optdec(\empobs))}{\projR(\samplelatp)(\ccR)}  - \frac{\optlatp(\optdec(\empobs))}{\projR(\optlatp)(\ccR)} } \cdot \frac{\ind{\optdec(\empobs) \in \epsPRS{}}}{n_{\optdec(\empobs)}[\cXR]} \\
     &= \sum_{s \in \epsPRS{} \cap \cS[\ccL] \cap \cS[\ccR]} \abs*{ \frac{ \samplelatp(s)}{\projR(\samplelatp)(\ccR)}  - \frac{ \optlatp(s)}{\projR(\optlatp)(\ccR)} } \\
     &=  \frac{ 1 }{\projR(\samplelatp)(\ccR)}\sum_{s \in \epsPRS{} \cap \cS[\ccL] \cap \cS[\ccR]} \abs*{ \samplelatp(s) - \optlatp(s) \cdot \frac{ \projR(\samplelatp)(\ccR)}{\projR(\optlatp)(\ccR)} } \\
     &\le \frac{ \eps }{\projR(\samplelatp)(\ccR)} \\
     &\hspace{3em} + \frac{ 1 }{\projR(\samplelatp)(\ccR)}\sum_{s \in \epsPRS{} \cap \cS[\ccL] \cap \cS[\ccR]}\optlatp(s) \abs*{ 1 - \frac{\projR(\samplelatp)(\ccR)}{\projR(\optlatp)(\ccR)} } &\text{(\pref{lem:sampling-transitions})}\\
     &= \frac{ \eps }{\projR(\samplelatp)(\ccR)} +  \frac{ 1 }{\projR(\samplelatp)(\ccR)} \abs*{\projR(\optlatp)(\ccR)  -  \projR(\samplelatp)(\ccR) } \\
     &\le \frac{ 2\eps }{\projR(\samplelatp)(\ccR)} =  2\eps \frac{\abs{\cXL} }{\nrch[\ccL]}. &\text{(using Eq.~\pref{eq:bound-on-empirical-proj})} 
\end{align*}
Also, we have the trivial bound that $ \nrm*{\pi \push \projR(\optlatp)(\cdot \mid \ccR) - \pi\push \projR(\samplelatp)(\cdot \mid \ccR)}_1 \le 2$, because it is a difference of two probability measures, so we can write the bound
\begin{align*}
    \nrm*{\pi \push \projR(\optlatp)(\cdot \mid \ccR) - \pi\push \projR(\samplelatp)(\cdot \mid \ccR)}_1 \le 2\eps \frac{\abs{\cXL} }{\nrch[\ccL]} \wedge 2. \numberthis\label{eq:term2}
\end{align*}
Using Eq.~\pref{eq:term2} we get that
\begin{align*}
    \hspace{2em}&\hspace{-2em} \sum_{\cc \in \crl{\cc_j}}  \frac{\abs{\ccL}}{\abs{\cXL}} \cdot \nrm*{\pi \push \projR(\samplelatp)(\cdot \mid \ccR) - \pi \push \projR(\optlatp)(\cdot \mid \ccR) }_1 \\
    &\le 2\sum_{\cc \in \crl{\cc_j}}  \frac{\abs{\ccL}}{\abs{\cXL}} \cdot \prn*{ \frac{\eps  \abs{\cXL} }{\nrch[\ccL]} \wedge 1} = 2\sum_{\cc \in \crl{\cc_j}} \prn*{ \frac{\eps \abs{\ccL} }{\nrch[\ccL]} \wedge \frac{\abs{\ccL}}{\abs{\cXL}} } \\
    &\le 2\eps \sum_{\cc \in \crl{\cc_j}: \frac{\abs{\ccL}}{\abs{\cXL}} \ge 4\eps}  \frac{\nrch[\ccL] + \nurch[\ccL]}{\nrch[\ccL]}  + 2 \sum_{\cc \in \crl{\cc_j}: \frac{\abs{\ccL}}{\abs{\cXL}} < 4\eps} \frac{\abs{\ccL}}{\abs{\cXL}} \le (8S + 8)\eps. \numberthis \label{eq:final-bound-term2}
\end{align*}
The last inequality uses the facts that (1) \pref{corr:sampling-transitions} implies that for any $\cc \in \crl{\cc_j}$ such that $\frac{\abs{\ccL}}{\abs{\cXL}} \ge 4\eps$ we have $\frac{\nrch[\ccL] + \nurch[\ccL]}{\nrch[\ccL]} \le 2$ and (2) the same casework we showed above to handle the summation for $\cc \in \crl{\cc_j}$ such that $\frac{\abs{\ccL}}{\abs{\cXL}} < 4\eps$.

Putting together Eqns.~\eqref{eq:term1-final} and \eqref{eq:final-bound-term2}:
\begin{align*}
    \sum_{\cc \in \crl{\cc_j}}  \frac{\abs{\ccL}}{\abs{\cXL}} \cdot \nrm*{\pi \push \unif(\ccL) - \pi \push \projR(\optlatp)(\cdot \mid \ccR) }_1 &\le K \prn*{  \sqrt{SA^2 \log \frac{SA  \abs{\Pi}}{\delta} }+ S}\eps \eqqcolon \beta.
\end{align*}
Thus, we can conclude that $\projR(\optlatp) \in \cP$, thus concluding the proof of Part (1).

\paragraph{Part (2).} Observe that in light of Part (1), the set $\cP$ is nonempty so therefore the maximum is well defined.

We want to show a bound on
\begin{align*}
    \max_{\bar{p} \in \cP} \max_{\pi \in \Pi_{h+1:H}}~\abs*{Q^\pi(x_h, a_h) - \wh{R}(x_h, a_h) - \En_{\empobs \sim \bar{p}} \wh{V}^\pi(\empobs)}.
\end{align*}

Fix any $\bar{p} \in \cP$ and $\pi \in \Pi_{h+1:H}$. We compute
\begin{align*}
    \hspace{2em}&\hspace{-2em} \abs*{Q^\pi(x_h, a_h) - \wh{R}(x_h, a_h) - \En_{\empobs \sim \bar{p}} \wh{V}^\pi(\empobs)} \\
    &\le \frac{\eps}{H} + \abs*{ \En_{s \sim \optlatp} V^\pi(s) - \En_{s \sim \samplelatp} V^\pi(s) } + \abs*{ \En_{s \sim \samplelatp} V^\pi(s) - \En_{\empobs \sim \bar{p}} \wh{V}^\pi(\empobs) } &\text{(\pref{lem:sampling-rewards})}\\
    &\le 2\eps + \abs*{ \En_{s \sim \samplelatp} V^\pi(s) - \En_{\empobs \sim \bar{p}} \wh{V}^\pi(\empobs) } &\text{(\pref{lem:sampling-transitions})} \\
    &\le 2\eps + \underbrace{ \abs*{ \En_{s \sim \samplelatp} V^\pi(s) - \En_{\empobs \sim \projR(\samplelatp)} V^\pi(\empobs) } }_{ \eqqcolon \mathtt{Term}_1} \\
    &\qquad + \underbrace{ \abs*{ \En_{\empobs \sim \projR(\samplelatp)} V^\pi(\empobs) - \En_{\empobs \sim \projR(\samplelatp)} \wh{V}^\pi(\empobs)} }_{\eqqcolon \mathtt{Term}_2} + \underbrace{ \abs*{\En_{\empobs \sim \projR(\samplelatp)} \wh{V}^\pi(\empobs) - \En_{\empobs \sim \bar{p}} \wh{V}^\pi(\empobs) } }_{\eqqcolon \mathtt{Term}_3}.
\end{align*}
\emph{Bounding $\mathtt{Term}_1$.} For the first term, we can calculate that
\begin{align*}
    \mathtt{Term}_1 &= \abs*{ \En_{s \sim \samplelatp} V^\pi(s) - \En_{\empobs \sim \projR(\samplelatp)} V^\pi(\empobs) } \\
    &= \abs*{ \En_{s \sim \samplelatp} \En_{x \sim \emission(s) } V^\pi(x) - \En_{\empobs \sim \projR(\samplelatp)} V^\pi(\empobs) } \\
    &= \abs*{ \En_{s \sim \samplelatp} \brk*{ \En_{x \sim \emission(s) } V^\pi(x) - \En_{\empobs \sim \mathrm{Unif}(\crl{\empobs \in \cXR: \optdec(\empobs) = s})} V^\pi(\empobs) } } \\
    &\le 2\eps + \abs*{ \En_{s \sim \samplelatp} \brk*{ \ind{s \in \epsPRS{h}} \prn*{ \En_{x \sim \emission(s) } V^\pi(x) - \En_{\empobs \sim \mathrm{Unif}(\crl{\empobs\in \cXR: \optdec(\empobs) = s})} V^\pi(\empobs) } } } \\
    &\le 3\eps. \numberthis\label{eq:t1}
\end{align*}
The first inequality follows by \pref{corr:sampling-transitions}, and the second inequality follows by \pref{lem:pushforward-concentration}.

\emph{Bounding $\mathtt{Term}_2$.} For the second term, we have by assumption that:
\begin{align*}
    \mathtt{Term}_2 = \abs*{ \En_{\empobs \sim \projR(\samplelatp)} \brk*{V^\pi(\empobs) - \wh{V}^\pi(\empobs)}  } \le \Gamma_{h+1}. \numberthis\label{eq:t2}
\end{align*}

\emph{Bounding $\mathtt{Term}_3$.} Now we calculate a bound on $\mathtt{Term}_3$. In what follows for any connected component $\cc$ we let $\empobs_\cc$ denote an arbitrary fixed observation from $\ccR$ (for example, the lowest indexed one). Observe that for any $p \in \Delta(\cXR)$ we have
\begin{align*}
    \En_{\empobs \sim p} \wh{V}^\pi(\empobs) 
    &= \sum_{\cc\in \crl{\cc_j}} \sum_{\empobs \in \ccR} p(\empobs) \cdot \wh{Q}^\pi(\empobs, \pi(\empobs)) &\text{($\crl{\cc_j}$ form a partition of $\cXR$)}\\
    &\le 4S\taudec + 8S\eps +  \sum_{\cc \in \crl{\cc_j}} \sum_{\empobs \in \ccR} p(\empobs) \cdot \wh{Q}^\pi(\empobs_\cc, \pi(\empobs_\cc)) &\text{(\pref{lem:width-of-cc})}\\
    &= 4S\taudec + 8S\eps +  \sum_{\cc \in \crl{\cc_j}} p(\ccR) \sum_{\empobs \in \ccR} \frac{p(\empobs)}{p(\ccR)} \wh{Q}^\pi(\empobs_\cc, \pi(\empobs_\cc)).
\end{align*}
Similarly, one can show the lower bound on $\En_{\empobs \sim p} \wh{V}^\pi(\empobs)$. Therefore we apply the bound to get:
\begin{align*}
    \hspace{2em}&\hspace{-2em} \abs*{\En_{\empobs \sim \projR(\samplelatp)} \wh{V}^\pi(\empobs) - \En_{\empobs \sim \bar{p}} \wh{V}^\pi(\empobs) } \\
    &\le 8S\taudec + 16S\eps \\
    &\qquad + \sum_{\cc \in \crl{\cc_j}} \abs*{\projR(\samplelatp)(\ccR) \sum_{\empobs \in \ccR} \frac{\projR(\samplelatp)(\empobs)}{\projR(\samplelatp)(\ccR) } \wh{Q}^\pi(\empobs_\cc, \pi(\empobs)) - \bar{p}(\cc) \sum_{\empobs \in \ccR} \frac{\bar{p}(\empobs)}{\bar{p}(\cc)} \wh{Q}^\pi(\empobs_\cc, \pi(\empobs)) } \\
    &\overset{(i)}{\le} 8S\taudec + 16S\eps + 2\eps \\
    &\qquad + \sum_{\cc \in \crl{\cc_j}} \abs*{\frac{\abs{\ccL}}{\abs{\cXL}} \sum_{\empobs \in \cc} \frac{\projR(\samplelatp)(\empobs)}{\projR(\samplelatp)(\ccR) } \wh{Q}^\pi(\empobs_\cc, \pi(\empobs)) - \bar{p}(\cc) \sum_{\empobs \in \cc} \frac{\bar{p}(\empobs)}{\bar{p}(\cc)} \wh{Q}^\pi(\empobs_\cc, \pi(\empobs)) } \\
    &\overset{(ii)}{\le}  8S\taudec + 16S\eps + 5\eps + \sum_{\cc \in \crl{\cc_j}}  \frac{\abs{\ccL}}{\abs{\cXL}} \cdot \abs*{\sum_{\empobs \in \cc} \frac{\projR(\samplelatp)(\empobs)}{\projR(\samplelatp)(\ccR) } \wh{Q}^\pi(\empobs_\cc, \pi(\empobs)) - \frac{\bar{p}(\empobs)}{\bar{p}(\cc)} \wh{Q}^\pi(\empobs_\cc, \pi(\empobs)) },\\
    &\le 8S\taudec + 16S\eps + 5\eps + \sum_{\cc \in \crl{\cc_j}}  \frac{\abs{\ccL}}{\abs{\cXL}} \cdot \nrm*{\pi\push \projR(\samplelatp)(\cdot \mid \ccR) - \pi \push \bar{p}(\cdot \mid \cc) }_1 ,
\end{align*}
where $(i)$ follows by \pref{corr:estimating-projR} and the bound $\frac{\projR(\samplelatp)(\empobs)}{\projR(\samplelatp)(\ccR) } \wh{Q}^\pi(\empobs_\cc, \pi(\empobs)) \in [0,1]$, and $(ii)$ follows by the first constraint on $\bar{p} \in \cP$ and the bound $\frac{p(\empobs)}{p(\cc)} \wh{Q}^\pi(\empobs_\cc, \pi(\empobs)) \in [0,1]$. 

From here, we will use the second constraint on $\bar{p} \in \cP$:
\begin{align*}
    \hspace{2em}&\hspace{-2em}  \abs*{\En_{\empobs \sim \projR(\samplelatp)} \wh{V}^\pi(\empobs) - \En_{\empobs \sim \bar{p}} \wh{V}^\pi(\empobs) } \\
    &\le  8S\taudec + 16S\eps + 5\eps + \sum_{\cc \in \crl{\cc_j}}  \frac{\abs{\ccL}}{\abs{\cXL}} \cdot \nrm*{\pi\push \projR(\samplelatp)(\cdot \mid \ccR) - \pi \push \bar{p}(\cdot \mid \cc) }_1 \\
    &\le  8S\taudec + 16S\eps + 5\eps + \beta + \sum_{\cc \in \crl{\cc_j}}  \frac{\abs{\ccL}}{\abs{\cXL}} \cdot \nrm*{\pi\push \projR(\samplelatp)(\cdot \mid \ccR) -\pi \push \unif(\ccL) }_1 \\
    &\le  8S\taudec + 16S\eps + 5\eps + 2\beta. \numberthis\label{eq:t3}
\end{align*}
The last inequality follows because our proof for Part (1) of the lemma actually showed that $\projR(\samplelatp) \in \cP$. Combining Eqns.~\eqref{eq:t1}, \eqref{eq:t2}, and \eqref{eq:t3} we get the final bound
\begin{align*}
    \abs*{Q^\pi(x_h, a_h) - \wh{R}(x_h, a_h) - \En_{\empobs \sim \bar{p}} \wh{V}^\pi(\empobs)} \le \Gamma_{h+1} + K \cdot \prn*{\beta + S \taudec}.
\end{align*}
This completes the proof of \pref{lem:main-induction}.

\subsection{Analysis of \stochrefit}
\begin{lemma}[Certificate Implies Transition Inaccuracy]\label{lem:certificate-obs}
Assume that $\eventemulator$ hold. Let $\wh{M}$ be a policy emulator. Suppose there exists a certificate $(\empobs, \pi) \in \estmdpobsspace{h} \times (\cA \circ \Pi_{h+1:H})$ such that
\begin{align*}
    \abs*{\wh{V}^{\pi}(\empobs) - V^{\pi}(\empobs)} \ge \tauref.
\end{align*}
Then there exists some tuple $(\bar{\empobs}, \bar{a}) \in \estmdpobsspace{} \times \actionsp$ such that
\begin{align*}
    \abs*{\En_{\empobs' \sim \wh{P}(\cdot \mid  \bar{\empobs}, \bar{a})} V^{\pi}(\empobs') - \En_{\empobs' \sim P(\cdot \mid  \bar{\empobs}, \bar{a})} V^{\pi}(\empobs')} \ge \frac{\tauref}{2H}. \numberthis\label{eq:bad-state}
\end{align*}
\end{lemma}

\begin{proof}Suppose that Eq.~\eqref{eq:bad-state} did not hold for any $(\bar{\empobs}, \bar{a})$. Then by the Performance Difference Lemma we have
\begin{align*}
    \hspace{2em}&\hspace{-2em} \abs*{ V^{\pi}(\empobs) - \wh{V}^{\pi}(\empobs) } \\
    &\le \abs*{\wh{R}(x, \pi) - R(x,\pi)} + \abs*{ \En_{x' \sim P(\cdot \mid \empobs, \pi)} V^{\pi}(x') - \En_{\empobs' \sim \wh{P}(\cdot \mid \empobs, \pi)}V^{\pi}(\empobs') } \\
    &\hspace{5em} + \abs*{\En_{\empobs' \sim \wh{P}(\cdot \mid \empobs, \pi)}V^{\pi}(\empobs') - \En_{\empobs' \sim \wh{P}(\cdot \mid \empobs, \pi)} \wh{V}^{\pi}(\empobs') }\\
    &\overset{(i)}{\le} \frac{\tauref}{2H} + \frac{\eps}{H} + \abs*{\En_{\empobs' \sim \wh{P}(\cdot \mid \empobs, \pi)}V^{\pi}(\empobs') - \En_{\empobs' \sim \wh{P}(\cdot \mid \empobs, \pi)} \wh{V}^{\pi}(\empobs') } \\
    &\le \frac{\tauref}{2H} + \frac{\eps}{H} + \max_{\empobs' \in \estmdpobsspace{h+1}}\abs*{  V^{\pi}(\empobs') - \wh{V}^{\pi}(\empobs') } \\
    &\le \cdots \overset{(ii)}{\le} \frac{\tauref}{2} + \eps,
\end{align*}
where $(i)$ uses \pref{lem:sampling-rewards} and the negation of Eq.~\eqref{eq:bad-state}, and $(ii)$ applies the bound recursively. Since $\tauref > 2\eps$, we have reached a contradiction. This proves \pref{lem:certificate-obs}.
\end{proof}

\begin{lemma}[Refitting Correctness]\label{lem:refitting-correctness} Assume that $\eventemulator, \eventref$ hold. The following are true about \pref{alg:stochastic-refit-v2} in the search for incorrect transitions (\pref{line:else-triggered}-\ref{line:loss-vector-obs} are executed):
\begin{enumerate}
    \item[(1)] For every $(\empobs, \pi)$ from in \pref{line:for-every}, at least one such $(\bar{\empobs}, \bar{a})$ pair is identified by \pref{line:bad-obs-stochastic}.
    \item[(2)] Every $(\bar{\empobs}, \bar{a})$ pair identified by \pref{line:bad-obs-stochastic} satisfies
    \begin{align*}
    \abs*{\En_{\empobs' \sim \wh{P}(\cdot \mid \bar{\empobs}, \bar{a})} V^{\pi}(\empobs') - \En_{\empobs' \sim \projm{\estmdpobsspace{h(\bar{\empobs})+1}}(\optlatp) } V^\pi(\empobs') } \ge \frac{\tauref}{16H}.
    \end{align*}
    \item[(3)] For every $(\bar{\empobs}, \bar{a})$ identified by \pref{line:bad-obs-stochastic}, the corresponding loss vector $\ell_\mathrm{loss}$ from \pref{line:loss-vector-obs} satisfies
    \begin{align*}
        \tri*{\wh{P}(\cdot \mid \bar{\empobs}, \bar{a}) - \projm{\estmdpobsspace{h(\bar{\empobs})+1}}(\optlatp(\cdot \mid \bar{\empobs}, \bar{a})), \ell_\mathrm{loss} } \ge \frac{\tauref}{16H}.
    \end{align*}
\end{enumerate}
\end{lemma}

\begin{proof}
To prove Part (1) we use \pref{lem:certificate-obs}, which shows that there exists at least one such $(\bar{\empobs}, \bar{a})$ such that
\begin{align*}
    \abs*{\En_{\empobs' \sim \wh{P}(\cdot \mid  \bar{\empobs}, \bar{a})} V^{\pi}(\empobs') - \En_{\empobs' \sim P(\cdot \mid  \bar{\empobs}, \bar{a})} V^{\pi}(\empobs')} \ge \frac{\tauref}{2H}. \numberthis\label{eq:lower-bound}
\end{align*}
Therefore we know that for such $(\bar{\empobs}, \bar{a})$:
\begin{align*}
    &\abs*{\En_{\empobs' \sim \wh{P}(\cdot \mid  \bar{\empobs}, \bar{a})} V^{\pi}(\empobs') - \En_{\empobs' \sim P(\cdot \mid  \bar{\empobs}, \bar{a})} V^{\pi}(\empobs')} \\
    &\le \abs*{\En_{\empobs' \sim \wh{P}(\cdot \mid  \bar{\empobs}, \bar{a})} \vestarg{\empobs'}{\pi} + \wh{R}(\bar{\empobs}, \bar{a}) -\qestarg{ \bar{\empobs}, \bar{a} }{\pi}} + 3\eps &\text{(\pref{lem:refit-monte-carlo-accuracy} and \pref{lem:sampling-rewards})}\\
    &= \abs*{ \Delta(\bar{\empobs}, \bar{a}) } + 3\eps \\
    &\Longrightarrow  \quad \abs*{ \Delta(\bar{\empobs}, \bar{a}) } \ge \frac{\tauref}{2H} - 3\eps \ge \frac{\tauref}{8H}, &\text{(Using Eq.~\eqref{eq:lower-bound})}
\end{align*}
so therefore this $(\bar{\empobs}, \bar{a})$ is identified by \pref{line:bad-obs-stochastic} of \stochrefit{}.

Now we prove Part (2). Fix any $(\bar{\empobs}, \bar{a})$ pair identified by \pref{line:bad-obs-stochastic} of \stochrefit{}. Let $h=h(\bar{\empobs})$ denote the layer that a given $\bar{\empobs}$ is found. First we observe that
\begin{align*}
    \hspace{2em}&\hspace{-2em} \En_{\empobs' \sim P(\cdot \mid  \bar{\empobs}, \bar{a})} V^{\pi}(\empobs') - \En_{\empobs' \sim \projm{\estmdpobsspace{h+1}}(\optlatp) } V^\pi(\empobs') = \En_{\empobs' \sim \emission \circ \optlatp } V^\pi(\empobs') - \En_{\empobs' \sim \projm{\estmdpobsspace{h+1}}(\optlatp) } V^\pi(\empobs') \\
    &= \En_{s \sim \optlatp} \brk*{\En_{\empobs' \sim \emission(s) } \brk*{ V^\pi(\empobs') } - \ind{s \in \epsPRS{}} \En_{\empobs' \sim \mathrm{Unif}(\crl{\empobs \in \estmdpobsspace{h+1}: \optdec(\empobs) = s})} \brk*{V^\pi(\empobs')} }\\
    &\le \eps + \En_{s \sim \optlatp} \brk*{\ind{s \in \epsPRS{}} \prn*{ \En_{\empobs' \sim \emission(s) } \brk*{ V^\pi(\empobs') } -  \En_{\empobs' \sim \mathrm{Unif}(\crl{\empobs\in \estmdpobsspace{h+1}: \optdec(\empobs) = s})} \brk*{V^\pi(\empobs')} } } \\
    &\le 2\eps,
\end{align*}
where the last inequality uses \pref{lem:pushforward-concentration}. The other side of the inequality can be similarly shown, so
\begin{align*}
    \abs*{ \En_{\empobs' \sim P(\cdot \mid  \bar{\empobs}, \bar{a})} V^{\pi}(\empobs') - \En_{\empobs' \sim \projm{\estmdpobsspace{h+1}}(\optlatp) } V^\pi(\empobs') }
    \le 2\eps. \numberthis \label{eq:relating-q-to-proj}
\end{align*}
We can compute that
\begin{align*}
    \frac{\tauref}{8H} &\le \abs{ \Delta(\bar{\empobs}, \bar{a}) } \\
    &=  \abs*{\En_{\empobs' \sim \wh{P}(\cdot \mid  \bar{\empobs}, \bar{a})} \vestarg{\empobs'}{\pi} + \wh{R}(\bar{\empobs}, \bar{a}) -\qestarg{ \bar{\empobs}, \bar{a} }{\pi}} \\
    &\le \abs*{\En_{\empobs' \sim \wh{P}(\cdot  \mid  \bar{\empobs},\bar{a})} V^\pi(\empobs') + \wh{R}(\bar{\empobs}, \bar{a}) - Q^\pi(\bar{\empobs}, \bar{a})} + 2\eps &\text{(\pref{lem:refit-monte-carlo-accuracy})} \\
    &\le \abs*{\En_{\empobs' \sim \wh{P}(\cdot  \mid  \bar{\empobs},\bar{a})} V^\pi(\empobs') - \En_{\empobs' \sim P(\cdot \mid  \bar{\empobs}, \bar{a})} V^{\pi}(\empobs') } + 3\eps &\text{(\pref{lem:sampling-rewards})} \\
    &\le \abs*{ \En_{\empobs' \sim \wh{P}(\cdot  \mid  \bar{\empobs},\bar{a})} V^\pi(\empobs') -\En_{\empobs' \sim \projm{\estmdpobsspace{h+1}}(\optlatp) } V^\pi(\empobs') } + 5\eps. &\text{(Eq.~\eqref{eq:relating-q-to-proj})}
\end{align*}
Rearranging we see that
\begin{align*}
    \abs*{ \En_{\empobs' \sim \wh{P}(\cdot  \mid  \bar{\empobs},\bar{a})} V^\pi(\empobs') -\En_{\empobs' \sim \projm{\estmdpobsspace{h+1}}(\optlatp) } V^\pi(\empobs') } \ge \frac{\tauref}{8H} - 5\eps \ge \frac{\tauref}{16H},
\end{align*}
and this proves Part (2).

For Part (3), suppose that $\Delta(\bar{\empobs}, \bar{a}) \ge \tauref/8H$ (the case where $\Delta(\bar{\empobs}, \bar{a}) \le -\tauref/8H$ can be tackled similarly). Then we have $\ell_\mathrm{loss} \coloneqq \qestarg{\cdot, \pi(\cdot)}{\pi} \in [0,1]^{\estmdpobsspace{h+1}}$. We can compute that
\begin{align*}
\frac{\tauref}{8H} &\le \En_{\empobs' \sim \wh{P}(\cdot \mid  \bar{\empobs}, \bar{a})} \qestarg{x', \pi(x')}{\pi} + \wh{R}(\bar{\empobs}, \bar{a}) -\qestarg{ \bar{\empobs}, \bar{a} }{\pi} \\
&= \tri*{ \wh{P}(\cdot  \mid  \bar{\empobs},\bar{a}), \ell_\mathrm{loss}} + \wh{R}(\bar{\empobs}, \bar{a}) -\qestarg{ \bar{\empobs}, \bar{a} }{\pi}  \\
&\le \eps + \tri*{ \wh{P}(\cdot  \mid  \bar{\empobs},\bar{a}), \ell_\mathrm{loss}} + \wh{R}(\bar{\empobs}, \bar{a}) - Q^\pi(\bar{\empobs}, \bar{a}) &\text{(\pref{lem:refit-monte-carlo-accuracy})}\\
&\le 4\eps + \tri*{ \wh{P}(\cdot  \mid  \bar{\empobs},\bar{a}), \ell_\mathrm{loss}} - \En_{\empobs' \sim \projm{\estmdpobsspace{h+1}}(\optlatp) } V^\pi(\empobs') &\text{(\pref{lem:sampling-rewards} and Eq.~\eqref{eq:relating-q-to-proj})}\\
&\le 5\eps + \tri*{ \wh{P}(\cdot  \mid  \bar{\empobs},\bar{a}) -  \projm{\estmdpobsspace{h+1}}(\optlatp(\cdot  \mid  \bar{\empobs},\bar{a})), \ell_\mathrm{loss}} &\text{(\pref{lem:refit-monte-carlo-accuracy})}
\end{align*}
Rearranging we get $ \tri*{ \wh{P}(\cdot  \mid  \bar{\empobs},\bar{a}) -  \projm{\estmdpobsspace{h+1}}(\optlatp(\cdot  \mid  \bar{\empobs},\bar{a})), \ell_\mathrm{loss}}  \ge \tauref/(16H)$, thus proving part (3).
\end{proof}

\begin{lemma}[Bound on Number of Refits]\label{lem:bound-on-number-refits}
Assume that $\eventemulator, \eventref$ hold, and that every time \pref{alg:stochastic-refit-v2} is called, the confidence sets $\cP$ satisfy 
\begin{align*}
    \forall~(x,a) \in \estmdpobsspace{}\times \cA: \quad \projm{\estmdpobsspace{h(x)+1}}(\optlatp(\cdot  \mid  x, a)) \in \cP(x, a).
\end{align*}
Then across all calls to \pref{alg:stochastic-refit-v2}, \pref{line:loss-vector-obs} is executed at most $(\nreset AH/ \eps^2) \cdot \log \nreset$ times.
\end{lemma}  
    
\begin{proof}
Fix $h \in [H]$ and a pair $(\empobs, a) \in \estmdpobsspace{h} \times \cA$. Suppose we execute \pref{line:loss-vector-obs} for $T_\mathrm{refit}$ times on $(\empobs, a)$. Denote the sequence of transition estimates as $\crl{\wh{P}^{(t)}(\cdot \mid x, a)}_{t \in [T_\mathrm{refit}]}$ and the sequence of loss vectors as $\crl{\ell_\mathrm{loss}^{(t)}}_{t \in [T_\mathrm{refit}]}$.
    
    By Part (3) of \pref{lem:refitting-correctness}, for all times $t \in [T_\mathrm{refit}]$,
    \begin{align*}
        \tri*{\wh{P}^{(t)}(\cdot \mid  x, a) - \projm{\estmdpobsspace{h+1}}(\optlatp(\cdot \mid x,a)), \ell_\mathrm{loss}^{(t)} } \ge \frac{\tauref}{16H}. \numberthis \label{eq:refit-lb-obs}
    \end{align*}
    On the other hand, we have a bound on the total regret of OMD with step size $\eps$ \cite[see, e.g., Thm.~5.2 of][]{bubeck2011introduction}:
    \begin{align*}
        \hspace{2em}&\hspace{-2em} \sum_{t=1}^{T_\mathrm{refit}} \tri*{\wh{P}^{(t)}(\cdot \mid x,a) - \projm{\estmdpobsspace{h+1}}(\optlatp(\cdot \mid x,a)), \ell_\mathrm{loss}^{(t)} } \\
        &\le \frac{1}{\eps} D_\mathsf{ne} \prn*{ \projm{\estmdpobsspace{h+1}}(\optlatp(\cdot \mid x,a)) ~\Vert~ \wh{P}^{(1)}(\cdot \mid x,a) } + \frac{\eps}{2} \sum_{t=1}^{T_\mathrm{refit}} \nrm*{\ell_\mathrm{loss}^{(t)}}_\infty\\
        &\le \frac{\log \nreset}{\eps} + \frac{\eps T_\mathrm{refit}}{2}. \numberthis \label{eq:refit-ub-obs}
    \end{align*}
    Therefore, combining Eqs.~\eqref{eq:refit-lb-obs} and \eqref{eq:refit-ub-obs} along with the value $\tauref = 80H\eps$ we have the bound
    \begin{align*}
        T_\mathrm{refit}
        \le \frac{\log \nreset}{\eps^2}.
    \end{align*}
    Using the fact that there are $\nreset AH$ such $(\empobs, a)$ pairs proves the result.
\end{proof}

\subsection{Proof of \pref{thm:block-mdp-result}}

In the proof, we will assume that $\eventemulator$ holds, that $\eventdec_t$ holds for all times $t \in [T_\mathsf{D}]$, that $\eventref_t$ holds for all times $t \in [T_\mathsf{R}]$. By union bound, this holds with probability at least $1-(3 T_\mathsf{D} + T_\mathsf{R} + 3) \delta$. 

We will show that under the choice of parameters $\nreset$, $\ndec$, and $\nmc$ in the algorithm pseudocode, \stochalg{} returns a $\wt{O}(SAH^2 \eps)$-optimal policy, and that $T_\mathsf{D}, T_\mathsf{R} \le \poly(\cpush, S,A, H, \eps^{-1}, \log \abs{\Pi}, \log \delta^{-1})$. Therefore, rescaling $\eps$ and $\delta$ will imply the final result.

\paragraph{Proof by Induction.} Take $\Gamma_h \coloneqq K (H-h+1) \prn{\beta +SH}\eps$ for some suitably large constant $K > 0$. Furthermore set $\taudec = 81 H \eps$. We will show that the following properties holds at every layer $h \in [H]$:
\begin{enumerate}
    \item [(1)] \textit{Transition set includes ground truth:} $\forall~ (x,a) \in \estmdpobsspace{h} \times \cA$, $\projm{\estmdpobsspace{h}}(\optlatp(\cdot  \mid  x,a)) \in \cP(x,a)$.
    \item [(2)] \textit{Accurate value estimates:} $\forall~ (x,a) \in \estmdpobsspace{h} \times \cA, \pi \in \Pi_{h+1:H}$, $\abs{Q^\pi(x,a) - \wh{Q}^\pi(x,a)} \le \Gamma_h$.
    \item [(3)] \textit{Valid test policies:} $\Pitest_h$ are $\taudec$-valid for $\estmdp$ at layer $h$.
\end{enumerate}

To analyze \stochalg{} we will show that at the end of every while loop, these properties always hold for all layers $h > \ell_\mathsf{next}$.

\underline{Base Case.} For the first loop with $\ell = H$, property (1) holds because there are no transitions to be constructed at layer $H$. By \pref{lem:sampling-rewards}, property (2) holds after the initialization of the policy emulator in \pref{line:init-end}. Furthermore, in the first call to \stochrefit{}, the computed test policies are open loop, so again using \pref{lem:sampling-rewards}, we see that \pref{line:great-success-stochastic} is triggered. Therefore, properties (2) and (3) hold at the end of the while loop. The current layer index is set to $\ell \gets H-1$.

\underline{Inductive Step.} Suppose the current layer index is $\ell$, and that properties (1)--(3) hold for all $h > \ell$. By \pref{lem:main-induction}, for every $(x_\ell, a_\ell)$ that we call \stochdecoder{} on the updated confidence sets $\cP$ returned by  satisfy property (1), and the choice $\wh{P} \in \cP$ satisfies property (2). Now we do casework on the outcome of \stochrefit{} called at layer $\ell$.

\begin{itemize}
    \item \textbf{Case 1: Return in \pref{line:great-success-stochastic}.}  By construction, property (3) is satisfied for layer $\ell$. In this case, since \stochrefit{} made no updates to $\estlatentmdp$ or $\cP$, properties (1) and (2) continue to hold at layer $\ell$ onwards.
    \item \textbf{Case 2: Return in \pref{line:great-success-stochastic-else}.} Property (1) holds because \stochrefit{} does not modify $\cP$. Let $\ell_\mathsf{next}$ denote the layer at which we jump to. \stochrefit{} made no updates to $\estlatentmdp$ at layers $\ell_\mathsf{next} + 1$ onwards, and therefore the previously computed test policies $\Pitest_{\ell_\mathsf{next}+1:H}$ must still be valid, so therefore properties (2) and (3) continue to hold at layer $\ell_\mathsf{next}$ onwards.
\end{itemize}

Continuing the induction, once $\ell \gets 0$ is reached in \stochalg{}, the policy emulator $\estmdp$ satisfies the bound
\begin{align*}
    \max_{\pi \in \Pi}~\abs*{V^\pi(s_1) - \En_{x_1 \sim \unif(\estmdpobsspace{1}) }\brk{\wh{V}^\pi(x_1)}} \le \Gamma_1 \le \wt{O}(SAH^2 \eps).\numberthis\label{eq:final-bound}
\end{align*}

\paragraph{Bounding the Number of Calls to \stochdecoder{} and \stochrefit{}.} By \pref{lem:bound-on-number-refits}, the total number of executions of \pref{line:loss-vector-obs} in \pref{alg:stochastic-refit-v2} is at most $(\nreset AH/ \eps^2) \cdot \log \nreset$. In the worst case, every revisit to an already computed layer (i.e., jumping back to $\ell_\mathsf{next} \ge \ell$) requires us to restart \stochdecoder{} at layer $H$ and therefore decode at most $\nreset AH$ additional times, so therefore
\begin{align*}
    T_\mathsf{D} \le \frac{\nreset^2 A^2H^2}{\eps^2} \log \nreset.
\end{align*}
Similarly, every revisit in the worst case requires $H$ additional calls to \stochrefit{} so therefore
\begin{align*}
    T_\mathsf{R} \le \frac{\nreset A H^2}{\eps^2} \log \nreset.
\end{align*}
As claimed, both $T_\mathsf{D}$ and $T_\mathsf{R}$ are $\poly(\cpush, S,A, H, \eps^{-1}, \log \abs{\Pi}, \log \delta^{-1})$.

\paragraph{Final Sample Complexity Bound.} Now we compute the total number of samples.
\begin{itemize}
    \item \pref{alg:stochastic-bmdp-solver-v2} uses $\nreset \cdot AH$ samples to $\mu_h$ to form the state space of the policy emulator, and for each state-action pair we sample $\wt{O}(H^2\eps^{-2})$ times to estimate the reward.
    \item \pref{alg:stochastic-decoder-v2} is called $T_\mathsf{D} \le \wt{O}(\nreset^2 A^2 H^2 \eps^{-2})$  times, and each time we use $\ndec \cdot \nreset^2 \nmc$ rollouts.
    \item \pref{alg:stochastic-refit-v2} is called $T_\mathsf{R} \le \wt{O}(\nreset A H^2 \eps^{-2})$ times, and each time we use $2\nreset^2 \nmc$ to evaluate the test policies. Furthermore, by \pref{lem:bound-on-number-refits}, \pref{line:mc2} is triggered at most $\wt{O}(\nreset A H \eps^{-2})$ times, with every time requiring an additional $\nmc \cdot \nreset AH$ rollouts.
\end{itemize}

Therefore in total we use
\begin{align*}
     \hspace{2em}&\hspace{-2em} \nreset \frac{AH^3}{\eps^2} + \nreset^4 \ndec \nmc \frac{A^2H^2}{\eps^2} + \nreset^3 \nmc \frac{AH^2}{\eps^2} + \nreset^2 \nmc \frac{A^2H^2}{\eps^2} \\
     &= \frac{\cpush^4 S^6 A^{12}H^3 }{\eps^{18}} \cdot \mathrm{polylog} \prn*{\cpush, S, A, H, \abs{\Pi}, \eps^{-1}, \delta^{-1}}\quad \text{samples.}
\end{align*}
Afterwards, we can rescale $\eps \gets \eps/\wt{O}(SAH^2)$ so that the bound Eq.~\eqref{eq:final-bound} is at most $\eps$, and rescale $\delta \gets \delta / (3 T_\mathsf{D} + T_\mathsf{R} + 1)$ so that the guarantee occurs with probability at least $1-\delta$. The final sample complexity is
\begin{align*}
    \frac{\cpush^4 S^{24} A^{30} H^{39} }{\eps^{18}} \cdot \mathrm{polylog} \prn*{\cpush, S, A, H, \abs{\Pi}, \eps^{-1}, \delta^{-1}}\quad \text{samples.}
\end{align*}\qed

\end{document}